\documentclass[11pt]{article}
\usepackage{latexsym}
\usepackage{amsmath}
\usepackage{amssymb}
\usepackage{amsthm}
\usepackage{epsfig}
\usepackage{xcolor}
\usepackage{graphicx}
\usepackage{bm}
\usepackage{enumitem}
\usepackage{mathtools}
\usepackage{graphicx}
\usepackage{ mathrsfs }
\usepackage[toc,page]{appendix}
\usepackage{graphicx}
\usepackage{bbm}

\usepackage[top=1in, bottom=1in, left=1in, right=1in]{geometry}

\usepackage{hyperref}
\hypersetup{
    colorlinks,
    linkcolor={blue!80!black},
    citecolor={green!50!black},
}
\colorlet{linkequation}{blue}

\renewcommand{\P}{\mathbb{P}}
\newcommand{\E}{\mathbb{E}}
\newcommand{\Var}{\text{Var}}
\newcommand{\Cov}{\text{Cov}}

\newcommand{\Z}{\mathbb{Z}}
\newcommand{\R}{\mathbb{R}}

\newcommand{\N}{\mathbb{N}}
\renewcommand{\S}{\mathbb{S}}

\newcommand{\eps}{\varepsilon} 

\def\id{{\mathbf I}}

\newcommand{\<}{\langle}
\renewcommand{\>}{\rangle}

\newcommand{\diag}{\text{diag}}

\newcommand{\op}{{\rm op}}

\def\sT{{\mathsf T}}
\def\bzero{{\boldsymbol 0}}

\DeclareMathOperator*{\argmin}{arg\,min}

\newtheorem{theorem}{Theorem}
\newtheorem*{theorem*}{Theorem}
\newtheorem{lemma}{Lemma}

\newtheorem{assumption}{Assumption}
\newtheorem{definition}{Definition}
\newtheorem{proposition}{Proposition}

\newtheorem{claim}{Claim}

\newtheorem{corollary}{Corollary}
\newtheorem{example}{Example}

\theoremstyle{definition}

\newtheoremstyle{myremark} 
    {\topsep}                    
    {\topsep}                    
    {\rm}                        
    {}                           
    {\bf}                        
    {.}                          
    {.5em}                       
    {}  

\theoremstyle{myremark}
\newtheorem{remark}{Remark}[section]

\DeclareSymbolFont{rsfs}{U}{rsfs}{m}{n}
\DeclareSymbolFontAlphabet{\mathscrsfs}{rsfs}

\def\bA{{\boldsymbol A}}
\def\bB{{\boldsymbol B}}
\def\bC{{\boldsymbol C}}
\def\bD{{\boldsymbol D}}

\def\bG{{\boldsymbol G}}

\def\bH{{\boldsymbol H}}

\def\bK{{\boldsymbol K}}

\def\bM{{\boldsymbol M}}

\def\bQ{{\boldsymbol Q}}
\def\bR{{\boldsymbol R}}
\def\bS{{\boldsymbol S}}

\def\bU{{\boldsymbol U}}
\def\bV{{\boldsymbol V}}
\def\bW{{\boldsymbol W}}
\def\bX{{\boldsymbol X}}

\def\bZ{{\boldsymbol Z}}

\def\ba{{\boldsymbol a}}

\def\be{{\boldsymbol e}}
\def\boldf{{\boldsymbol f}}
\def\bg{{\boldsymbol g}}
\def\bh{{\boldsymbol h}}
\def\bi{{\boldsymbol i}}

\def\bu{{\boldsymbol u}}
\def\bv{{\boldsymbol v}}
\def\bw{{\boldsymbol w}}
\def\bx{{\boldsymbol x}}
\def\by{{\boldsymbol y}}
\def\bz{{\boldsymbol z}}

\def\bbeta{{\boldsymbol \beta}}
\def\bdelta{{\boldsymbol\delta}}
\def\beps{{\boldsymbol \eps}}

\def\btheta{{\boldsymbol \theta}}

\def\bDelta{{\boldsymbol \Delta}}

\def\bSigma{{\boldsymbol \Sigma}}

\def\hf{{\hat f}}

\def\cR{\mathcal{R}}
\def\test{{\rm test}}
\def\train{{\rm train}}
\def\CV{\text{CV}}
\def\GCV{{\rm GCV}}
\def\sfs{{\sf s}}

\def\spn{{\rm span}}

\def\de{{\rm d}}
\def\Tr{{\rm Tr}}

\def\de{{\rm d}}
\def\Unif{{\rm Unif}}

\def\ddiag{{\rm ddiag}}

\def\spn{{\rm span}}

\def\cV{{\mathcal V}}

\def\cP{{\mathcal P}}

\def\cT{{\mathcal T}}
\def\cC{{\mathcal C}}

\def\cL{{\mathcal L}}
\def\cF{{\mathcal F}}
\def\cE{{\mathcal E}}
\def\cS{{\mathcal S}}

\def\cV{{\mathcal V}}

\def\cP{{\mathcal P}}

\def\cT{{\mathcal T}}
\def\cH{{\mathcal H}}
\def\cA{{\mathcal A}}

\def\sM{{\mathsf M}}

\def\Unif{{\sf Unif}}
\def\normal{{\sf N}}
\def\proj{{\mathsf P}}

\def\sM{{\sf M}}

\def\naturals{{\mathbb N}}

\def\Kop{{\mathbb K}}

\def\normal{{\sf N}}
\def\proj{{\mathsf P}}

\def\sM{{\mathsf M}}

\def\Unif{{\sf Unif}}
\def\normal{{\sf N}}

\def\proj{{\mathsf P}}

\def\sM{{\sf M}}

\def\naturals{{\mathbb N}}
\def\proj{{\mathsf P}}

\def\sU{{\sf U}}
\def\sV{{\sf V}}

\def\tcE{\widetilde{\cE}}
\def\tmu{\widetilde  \mu}
\def\tbD{\widetilde{\bD}}

\def\seff{\mbox{\tiny\rm eff}}

\def\Ker{K}

\def\hf{{\hat f}}

\def\tbDelta{\widetilde{\bDelta}}

\def\cE{{\mathcal E}}
\def\cD{{\mathcal D}}
\def\cX{{\mathcal X}}
\def\cF{{\mathcal F}}
\def\cS{{\mathcal S}}

\def\de{{\rm d}}
\def\Unif{{\rm Unif}}

\def\cE{{\mathcal E}}

\def\normal{{\sf N}}

\def\bDelta{{\boldsymbol \Delta}}

\def\cX{{\mathcal X}}

\def\bA{{\boldsymbol A}}
\def\btheta{{\boldsymbol \theta}}

\def\cT{{\mathcal T}}
\def\cV{{\mathcal V}}

\def\diag{{\rm diag}}
\def\bS{{\boldsymbol S}}

\def\bD{{\boldsymbol D}}

\def\hf{\hat f}

\def\bR{{\boldsymbol R}}

\def\evn{{\mathsf m}}

\def\bC{{\boldsymbol C}}

\def\ind{\mathbbm{1}}

\def\balpha{\boldsymbol{\alpha}}
\def\bgamma{\boldsymbol{\gamma}}
\def\cU{\mathcal{U}}

\def\boldf{\boldsymbol{f}}

\def\cB{\mathcal{B}}

\def\Im{{\rm Im}}

\def\sR{\mathsf R}
\def\sV{\mathsf V}
\def\sB{\mathsf B}

\def\obR{\overline{\bR}}
\def\obM{\overline{\bM}}
\def\wbM{\widetilde{\bM}}

\def\hbtheta{\hat \btheta}

\def\br{{\boldsymbol r}}

\def\oxi{\overline{\xi}}

\def\sfG{\textsf{G}}

\def\balpha{\boldsymbol{\alpha}}
\def\bgamma{\boldsymbol{\gamma}}
\def\cU{\mathcal{U}}

\def\boldf{\boldsymbol{f}}

\def\cB{\mathcal{B}}

\def\Im{{\rm Im}}

\def\sR{\mathsf R}
\def\sV{\mathsf V}
\def\sB{\mathsf B}

\def\hbtheta{\hat \btheta}

\def\sfC{{\sf C}}
\def\sfc{{\sf c}}
\def\sfD{{\sf D}}
\def\sfM{{\sf M}}
\def\rmI{{\rm I}}
\def\rmII{{\rm II}}
\def\obQ{\overline{\bQ}}
\def\tS{\widetilde{S}}

\def\onu{\overline{\nu}}
\def\oT{\overline{T}}
\def\sL{\mathsf{L}}
\def\bq{\boldsymbol{q}}
\def\og{\overline{g}}
\def\oq{\overline{q}}

\def\bs{{\boldsymbol s}}
\def\obD{\overline{\bD}}

\def\sflf{{\sf leaf}}
\def\sfT{{\sf T}}
\def\sfG{{\sf G}}
\def\bsfT{{\boldsymbol \sfT}}
\def\bsfG{{\boldsymbol \sfG}}
\def\obi{\overline{\bi}}
\def\obsfT{\overline{\bsfT}}
\def\obsfG{\overline{\bsfG}}
\def\oi{\overline{i}}
\def\osfT{\overline{\sfT}}
\def\osfG{\overline{\sfG}}
\def\sfH{{\sf H}}
\def\tbD{\widetilde{\bD}}
\def\polylog{\text{polylog}}

\def\seff{{\sf eff}}
\def\sG{\mathsf{G}}
\def\sKL{\mathsf{KL}}
\def\oevn{\overline{\evn}}
\def\obeta{\overline{\beta}}
\def\oC{\overline{C}}
\def\opt{{\rm opt}}
\def\tbSigma{\widetilde{\bSigma}}

\title{A non-asymptotic theory of Kernel Ridge Regression: \\
deterministic equivalents, test error, and GCV estimator}
\author{Theodor Misiakiewicz\thanks{Toyota Technological Institute at Chicago. Email: \texttt{theodor.misiakiewicz@ttic.edu}}, \;\;
  \and
  Basil Saeed\thanks{Department of Electrical Engineering, Stanford University. Email:
  \texttt{bsaeed@stanford.edu}}
  }

\begin{document}

\maketitle

\begin{abstract}
We consider learning an unknown target function $f_*$ using kernel ridge regression (KRR) given i.i.d.~data $(\bu_i,y_i)$, $i\leq n$, where $\bu_i \in \cU$ is a covariate vector and $y_i = f_* (\bu_i) +\eps_i \in \R$. A recent string of work has empirically shown that the test error of KRR can be well approximated by a closed-form estimate derived from an `equivalent' sequence model that only depends on the spectrum of the kernel operator. However, a theoretical justification for this equivalence has so far relied either on restrictive assumptions---such as subgaussian independent eigenfunctions---, or asymptotic derivations for specific kernels in high dimensions. 

In this paper, we prove that this equivalence holds for a general class of problems satisfying some spectral and concentration properties on the kernel eigendecomposition.  Specifically, we establish in this setting a non-asymptotic deterministic approximation for the test error of KRR---with explicit non-asymptotic bounds---that only depends on the eigenvalues and the target function alignment to the eigenvectors of the kernel. Our proofs rely on a careful derivation of deterministic equivalents for random matrix functionals in the dimension free regime pioneered by Cheng and Montanari \cite{cheng2022dimension}. 

We apply this setting to several classical examples and show an excellent agreement between theoretical predictions and numerical simulations. These results rely on having access to the eigendecomposition of the kernel operator. Alternatively, we prove that, under this same setting, the generalized cross-validation (GCV) estimator concentrates on the test error uniformly over a range of ridge regularization parameter that includes zero (the interpolating solution). As a consequence, the GCV estimator can be used to estimate from data the test error and optimal regularization parameter for KRR.

\end{abstract}

\clearpage

\tableofcontents

\clearpage
\section{Introduction}

\subsection{Background}

Consider the classical regression problem: we are given i.i.d.~samples $(\bu_i, y_i)_{i \leq n}$ from a common probability distribution $\P$ on $\cU \times \R$, where $\bu_i \in \cU$ is a covariate vector and $y_i \in \R$ is the response variable. The goal is to learn a model $\hat f : \cU \to \R$ which given a new data point $\bu_{\test}$ predicts the response via $\hf ( \bu_{\test}) $ with small test error
\[
\cR_{\test} ( \hf , \P) := \E_{(\bu_{\test},y_{\test}) \sim \P} \Big[ \big( y_{\test} - \hf ( \bu_{\test}) \big)^2 \Big] .
\]
We assume a model whereby the responses are $y_i = f_* (\bu_i) + \eps_i$, with noise $\eps_i$ independent of $\bu_i$, with $\E [ \eps_i ] = 0$ and $\E [ \eps_i^2 ] =\sigma_\eps^2$. We denote $\P_\bu$ the distribution of $\bu$ and assume that the target function $f_* \in L^2 (\cU) := L^2 (\cU , \P_{\bu} )$ is square integrable.

A standard approach for solving this problem is kernel ridge regression (KRR), which is now considered to be among the most fundamental tools in machine learning. Conceptually, KRR proceeds in two steps: (1) it embeds the covariates $\bu_i$ into a rich `feature space' $\phi (\bu_i) \in \cF$, a Hilbert space with inner-product $\<\cdot, \cdot \>_{\cF}$, via a featurization map $\phi: \cU \to \cF$; and (2) it fits a linear predictor $\hf_{\lambda} ( \bu ) = \< \hbtheta_{\lambda} , \phi (\bu) \>_{\cF} $ with respect to this embedding, where
\begin{equation}\label{eq:KRR_problem}
\hbtheta_{\lambda} := \argmin_{\btheta \in \cF} \Big\{ \sum_{i =1}^n \big( y_i - \< \btheta , \phi (\bu_i) \>_{\cF} \big)^2 + \lambda \| \btheta \|_{\cF}^2 \Big\} .
\end{equation}
Here  $\| \btheta \|_{\cF} = \< \btheta , \btheta \>_{\cF}^{1/2}$ denotes the norm in the feature space and $\lambda$ is the ridge regularization parameter. In practice, the featurization map $\phi$ and feature space $\cF$ are often defined\footnote{These two constructions are equivalent. Indeed, to any positive definite kernel $K : \cU \times \cU \to \R$, we can associate a feature space and feature map $\phi :\cU \to (\cF, \< \cdot, \cdot \>_{\cF})$ such that $K(\bu , \bu') = \< \phi (\bu) , \phi (\bu') \>_{\cF}$. In that case, we have $\cH= \{ \bu \mapsto \< \btheta , \phi (\bu) \>_{\cF} : \btheta \in \cF \}$ with RKHS norm given by $\| f \|_{\cH} = \inf \{ \| \btheta \|_{\cF}: f = \< \btheta , \phi \>_{\cF}\}$.} implicitly through a positive semi-definite kernel $K : \cU \times \cU \to \R$, via $K (\bu , \bu' ) := \< \phi (\bu) , \phi (\bu') \>_{\cF}$, which defines $(\cH, \< \cdot,\cdot \>_{\cH})$ a reproducing kernel Hilbert space (RKHS) \cite{berlinet2011reproducing}. The KRR problem \eqref{eq:KRR_problem} is thus often written compactly as
\begin{equation}\label{eq:KRR_problem_RKHS}
\hat f_{\lambda} = \argmin_{f \in \cH} \Big\{ \sum_{i =1}^n \big( y_i - f (\bu_i) \big)^2 + \lambda \| f \|_{\cH}^2 \Big\}  ,
\end{equation}
with solution that only depends on the kernel evaluations $K(\bu_i , \bu_j)$. This is known as the `representer theorem' or `kernel trick'.

The test error of KRR crucially depends on the eigenvalue decomposition of the kernel\footnote{We assume in this paper that $L^2(\cU, \P_{\bu})$ is separable and $K$ is trace-class, i.e., $\E_{\bu} [ K(\bu,\bu)]<\infty$.}:

\[
K (\bu , \bu ' ) = \sum_{j = 1}^\infty \xi_j \psi_j (\bu) \psi_j (\bu') ,
\]
where $\xi_1 \geq \xi_2 \geq \xi_3 \geq \cdots > 0$ are the positive eigenvalues in nonincreasing order, and $\{ \psi_j  \}_{j \geq 1}$ are the orthonormal eigenvectors. Without loss of generality, we can choose the featurization  map to be $\phi (\bu) := \left( \sqrt{\xi_j} \psi_j ( \bu) \right)_{j \geq 1}$ and $\cF = (\ell_2 , \< \cdot , \cdot \>)$ the space of $\ell_2$ sequences. We introduce $\bSigma = \E_{\bu} [ \phi (\bu) \phi (\bu)^\sT ] = \diag (\xi_1 , \xi_2 , \xi_3, \ldots )$ the covariance matrix of the featurization map (understood as a trace-class self-adjoint operator). For simplicity, we assume the target function satisfy $f_* \in \text{span} \{ \psi_j: j\geq 1\}$, so that we can decompose
\begin{align*}
    f_* (\bu) = \sum_{j = 1}^\infty \beta_{*,j} \psi_j (\bu) = \< \btheta_* , \phi (\bu) \>  , \qquad \| f_* \|_{L^2} = \| \bbeta_*\|_2 = \| \bSigma^{1/2} \btheta_* \|_2 < \infty  .
\end{align*}
This is automatically verified when $K$ is universal and $\{ \psi_j  \}_{j \geq 1}$ forms a complete basis of $L^2 (\cU)$. We can therefore forget about the covariate $\bu$ and directly work in the feature space. We will set hereafter $\bx := \phi (\bu)$, with distribution induced by the distribution over $\bu$, and $f_* ( \bx) := f_*(\bu) = \< \btheta_* , \bx \>$ with a slight abuse of notations.

Using these notations, the KRR solution and test error admit the following explicit formulas:
\begin{equation}\label{eq:KRR_explicit}
\begin{aligned}
    &\hbtheta_{\lambda} = \bX^\sT ( \bX \bX^\sT + \lambda \id_n )^{-1} \by  , \qquad \qquad \bX := [ \bx_1 , \ldots , \bx_n]^\sT \in \R^{n \times \infty} ,\\
       & \cR_{\test} (\hat f_\lambda , \P ) = \E \left[ \big( y - \hat f_\lambda (\bu) \big)^2\right]  = \big\| \btheta_* - \hbtheta_\lambda  \big\|_{\bSigma}^2  + \sigma_\eps^2,
\end{aligned}
\end{equation}
where we denoted $\| \bv \|_{\bSigma} = \| \bSigma^{1/2} \bv \|_2$.
The test error of KRR has been studied by many authors over the years \cite{bartlett2005local,caponnetto2007optimal,hsu2012random,wainwright2019high}. Among this extensive literature, we can distinguish two general analysis approaches:
\begin{itemize}
    \item An approximation-generalization decomposition using uniform convergence bounds \cite{bartlett2005local,koehler2021uniform} or source and capacity conditions \cite{caponnetto2007optimal,fischer2020sobolev,lin2020optimal}. This line of work characterizes the decay rate $n^{-\alpha}$ of the test error for a fixed model as $n \to \infty$. In particular, these rates are expected to be minimax optimal over classes of functions.

    \item High-dimensional asymptotics \cite{dicker2016ridge,richards2021asymptotics,mei2022generalization,xiao2022precise}. This approach considers a sequence of models $n:=n(d)$  and derives sharp asymptotics for the test error (up to an additive vanishing constant) as $n,d \to \infty$ while $n \asymp d^\gamma$.
\end{itemize}

\begin{figure}[t]
\centering
\includegraphics[width=0.75\textwidth]{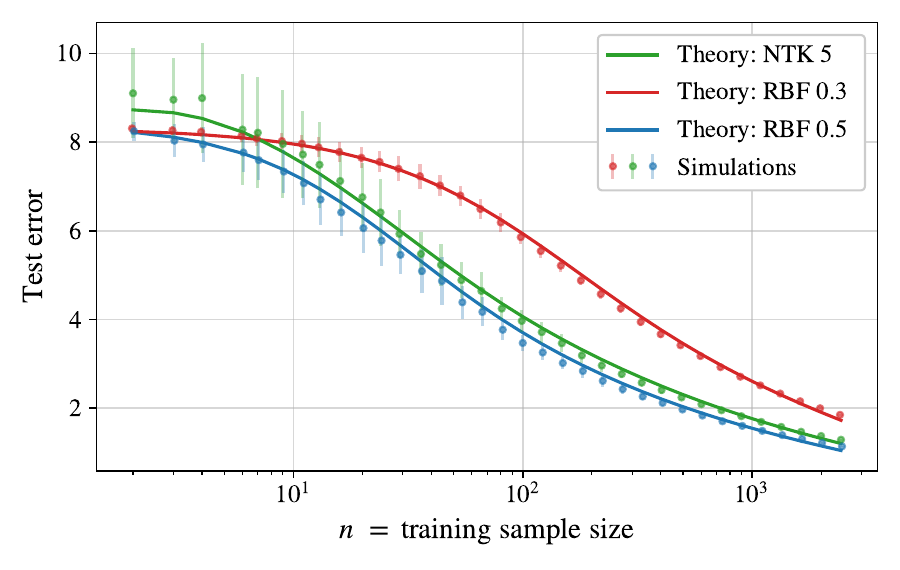}
\vspace{-10pt}
\caption{Test error of KRR plotted against the training sample size $n \in \{2, \ldots, 20000\}$. We consider data $(\bu_i,y_i)$ from MNIST where $u_i$ is a $d=28^2$ dimensional image and $y_i \in \{0,9\}$ is its label. We fit this data using KRR \eqref{eq:KRR_problem_RKHS} with three standard kernels $K_j (\bu,\bu')$; one simulation uses the ReLU NTK of depth 5 and the other two correspond to the RBF kernel with the bandwidths specified in the figure. We take $\lambda =0$ (interpolating solution). The continuous lines correspond to the theoretical predictions from the deterministic equivalent \eqref{eq:DetEquiv_Risk_Intro}, where the eigenvalues of $\bSigma$ and $\bbeta^*$ are estimated from a sample of the data of size $25000$. For the empirical test errors (markers), we solve KRR on $n$ images sampled uniformly from the training set, and report the average and the standard deviation of the test error over $50$ independent realizations.
\label{fig:real_test}
}
\end{figure}

These two approaches typically impose restrictive assumptions on the target function and the eigenvalue decay. Most importantly, both are expected to provide accurate insights to the statistician only \textit{asymptotically}, i.e., for a large dimension and/or number of samples. In parallel, a recent line of work \cite{spigler2020asymptotic,bordelon2020spectrum,canatar2021spectral,loureiro2021learning,cui2021generalization,simon2023eigenlearning} has shown empirically that the test error can be well approximated by a closed-form estimate that only depends on the eigenvalues $\bSigma$ of the kernel operator and the coefficients $\bbeta_*$ of the target function. In particular, this estimate is deterministic and non-asymptotic---depending on finite $n$ and a fixed covariance matrix---, and shows excellent agreement with numerical simulations for small $n$ and moderate covariate dimension. For concreteness, we illustrate these theoretical predictions in Figure \ref{fig:real_test} for KRR on MNIST data with RBF and NTK kernels, and Figure \ref{fig:sphere_test} for KRR with spherical data and inner-product kernels. We refer to \cite{bordelon2020spectrum,canatar2021spectral,loureiro2021learning,cui2021generalization,simon2023eigenlearning} for further numerical examples.
However a theoretical justification for this closed-form estimate has so far relied on either heuristic derivations \cite{canatar2021spectral} or strong assumptions on the data distribution \cite{loureiro2021learning,hastie2022surprises,cheng2022dimension}. 

The present paper is concerned with providing a mathematical foundation for such a non-asymptotic theory of kernel ridge regression. We present a set of abstract assumptions under which this deterministic approximation is accurate, with explicit non-asymptotic bounds, and verify these conditions in several classical settings. Our analysis builds upon recent advances in the study of random matrix functionals \cite{couillet2022random,cheng2022dimension} and non-linear kernels \cite{mei2022generalization,xiao2022precise}.

\subsection{A deterministic equivalent for the test error}
\label{sec:intro_det_equiv}

Let us denote from now on the test error $\cR_{\test} (\bbeta_*; \bX,\beps,\lambda) := \cR_{\test} (\hat f_\lambda ; \P)$ to emphasize the dependency on the feature matrix $\bX$, the label noise $\beps = (\eps_i )_{i\in [n]}$, the regularization parameter $\lambda$, and the coefficients of the target function $\bbeta_*$.

 It will be instructive to first consider the bias-variance decomposition of the test error over the label noise $\beps$ in the training data
\[
\E_{\beps} [ \cR_{\test} (\bbeta_*; \bX,\beps,\lambda)] = \cB ( \bbeta_* ; \bX, \lambda) +  \cV (\bX, \lambda) + \sigma_\eps^2,
\]
where the bias and variance are given explicitly by
\begin{equation}\label{eq:bias_variance_intro}
\begin{aligned}
     \cB ( \bbeta_* ; \bX, \lambda) =&~ \big\| \btheta_* - \E_{\beps} \big[ \hbtheta_\lambda \big]  \big\|_{\bSigma}^2 = \lambda^2 \< \bbeta_* , \bSigma^{-1/2} ( \bX^\sT \bX +\lambda)^{-1} \bSigma ( \bX^\sT \bX +\lambda)^{-1} \bSigma^{-1/2} \bbeta_* \> , \\
     \cV (\bX, \lambda) =&~ \sigma_\eps^2 \Tr \big( \bSigma \Cov_\beps (  \hbtheta_\lambda) \big) = \sigma_\eps^2 \cdot \Tr \big( \bSigma \bX^\sT\bX (\bX^\sT \bX + \lambda )^{-2}  \big)  .
\end{aligned}
\end{equation}

\paragraph*{Deterministic equivalents.}  While the bias and variance are random variables, it was observed in \cite{hastie2022surprises,cheng2022dimension} that they are both well-concentrated around non-random values. Following the terminology in random matrix theory, we will simply refer to these deterministic estimates as `deterministic equivalents' \cite{couillet2011random}. 

Define the effective regularization $\lambda_*$ to be the unique non-negative fixed point to the equation
\begin{equation}\label{eq:eff_reg_intro}
n - \frac{\lambda}{\lambda_*} = \Tr \big( \bSigma ( \bSigma + \lambda_* )^{-1} \big)  .
\end{equation}
The deterministic equivalents for the bias and variance terms are given by
\begin{align}\label{eq:det_equiv_bias_intro}
    \sB_n (\bbeta_*, \lambda) := &~ \frac{\lambda_*^2 \< \bbeta_* , ( \bSigma + \lambda_* )^{-2} \bbeta_* \> }{ 1 - n^{-1} \Tr ( \bSigma^2 ( \bSigma + \lambda_* )^{-2} )}  , \\
    \sV_n  (\lambda) :=&~ \frac{\sigma_\eps^2 \Tr ( \bSigma^2 ( \bSigma + \lambda_* )^{-2} ) }{n - \Tr ( \bSigma^2 ( \bSigma + \lambda_* )^{-2} )}, \label{eq:det_equiv_variance_intro}
\end{align}
and for the test error
\[
\sR_n (\bbeta_* , \lambda) := \sB_n + \sV_n + \sigma_\eps^2.
\]
Cheng and Montanari \cite{cheng2022dimension} established the two approximation bounds
\begin{equation}\label{eq:Cheng_rates}
\cB (\bbeta_* ; \bX,\lambda) = (1 + \widetilde{O} (n^{-1/2}) ) \cdot \sB_n(\bbeta_*,\lambda), \qquad\quad \cV (\bX,\lambda) = (1 + \widetilde{O} (n^{-1})) \cdot \sV_n (\lambda),
\end{equation}
 with probability $1 - o_n(1)$. These guarantees display two remarkable features. First, the bounds are \textit{dimension-free} and apply to infinite-dimensional features $\bx_i \in \R^{\infty}$. In particular, compared to previous work, they do not require a feature dimension $p$ with $1/C \leq p/n \leq C$ or a bounded condition number on the covariance matrix $\bSigma$. Second, the bounds are \textit{multiplicative} and the deterministic equivalents remain accurate even for vanishing bias and variance.

However, their analysis crucially relies on two important assumptions: (i) The features are `concentrated' in the sense that they satisfy the Hanson-Wright inequality. This includes features with independent sub-Gaussian coordinates or that satisfy the convex Lipschitz concentration property, but excludes more realistic kernels such as most non-linear kernels on the sphere. (ii) The target function must be very `smooth' with $\| \bSigma^{-1} \bbeta_* \|_2 < \infty$, i.e., a small subset of the RKHS\footnote{Recall that $f \in L^2(\cU)$ belongs to the RKHS $\cH$ if and only if $\| \bSigma^{-1/2} \bbeta_* \|_2 < \infty$.}. In particular, this implies that the bias $\cB_n = O(n^{-2})$ in their setting.

\begin{figure}[t]
\centering
\includegraphics[width=0.75\textwidth]{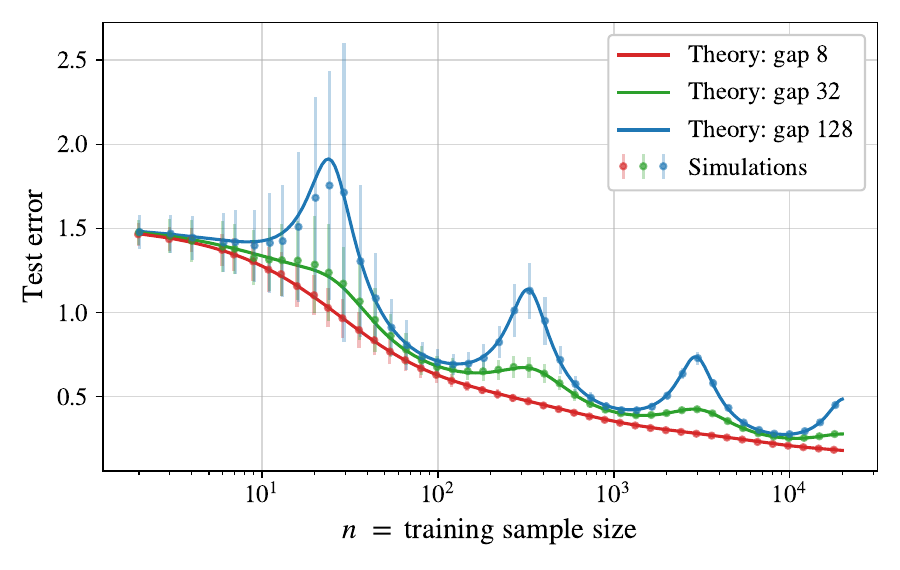}
\vspace{-10pt}
\caption{Test error of KRR plotted against the training sample size $n \in \{2, \ldots, 20000\}$. We consider data $(\bu_i,y_i)$ with $\bu_i \sim_{iid} \Unif (\S^{d-1} (\sqrt{d}))$, $d=24$, and $y_i = f_* (\bu_i) + \eps_i$ with independent label noise $\eps_i \sim \normal (0, \sigma_\eps^2)$, $\sigma_\eps^2 = 0.1$. We fit this data using KRR \eqref{eq:KRR_problem_RKHS} with three different inner-product kernels $K_j (\bu,\bu') = h_j (\< \bu , \bu'\>/d)$, $j \in [3]$, corresponding to spectral gaps $\in \{8,32,128\}$, and regularization parameter $\lambda =0$. The continuous lines correspond to the theoretical predictions from the deterministic equivalent \eqref{eq:DetEquiv_Risk_Intro}. For the empirical test errors (markers), we report the average and the standard deviation of the test error over $50$ independent realizations.
See Section \ref{sec:inner-product_main} for details. \label{fig:sphere_test}}
\end{figure}

To extend these approximation guarantees to (i) non-concentrated features and (ii) general target functions $f_* \in L^2(\cU)$, with $\| \bbeta_* \|_2 < \infty$, requires significant changes to the proof approach. In particular, it necessitates to establish dimension-free deterministic equivalents to higher-order functionals of the feature matrix.


\subsection{Summary of main results}
\label{sec:summary}

The main result of this paper is a guarantee for the test error $\cR_{\test} (\bbeta_* ; \bX, \beps,\lambda)$ to be approximated by its associated deterministic equivalent
\begin{equation}\label{eq:DetEquiv_Risk_Intro}
    \sR_n (\bbeta_*,\lambda) =  \frac{\lambda_*^2 \< \bbeta_*, (\bSigma + \lambda_* )^{-2} \bbeta_* \> + \sigma_\eps^2}{1 - \frac{1}{n} \Tr( \bSigma^2 ( \bSigma + \lambda_* )^{-2} )} ,
\end{equation}
under general `abstract' assumptions on the feature vector $\bx$. While the test error is the most relevant quantity for practitioners, this approximation guarantee will hold for more general functionals of the feature matrix $\bX$, including the training error and the Stieltjes transform of the empirical kernel matrix (in particular, the GCV estimator stated below).
It is useful to first describe informally these general assumptions.

\paragraph*{General assumptions.} Non-linear kernels typically contain high-frequency eigenfunctions with heavy-tailed distribution (e.g., high-degree polynomials) in their diagonalization. However, it was observed in a recent string of papers \cite{ghorbani2020neural,mei2022generalization,misiakiewicz2022learning,xiao2022precise} that for a number of classical kernels, the top eigenspaces are associated to low-frequency eigenfunctions whose tails can be controlled. Furthermore, the high-frequency part of the kernel matrix concentrates on a non-random diagonal matrix. Following these insights, our analysis will treat separately the top eigenspaces from the rest of the features.

For any integer $\evn \in \naturals$, define the \textit{regularized tail rank} of $\bSigma$ with regularization $\lambda \geq 0$:
\[
r_\lambda (\evn) = \frac{\lambda + \sum_{j = \evn+1}^\infty \xi_j}{\xi_{\evn+1}}.
\]
This quantity first appeared in the benign-overfitting literature \cite{bartlett2020benign,tsigler2023benign,koehler2021uniform} that studied the properties of the (near-)interpolating solutions of ridge regression. We denote $\bx_{\leq \evn}$ the projection of the feature vector $\bx$ onto its top $\evn$ eigenspaces, with covariance $\bSigma_{\leq \evn} = \diag (\xi_1,\xi_2 , \ldots, \xi_\evn)$, and $\bx_{>\evn}$ the projection onto the rest of the eigenspaces, with covariance $\bSigma = \diag (\xi_{\evn+1} , \xi_{\evn+2}, \ldots)$.

For a given $n$, we assume that there exists an integer $\evn := \evn (n)$ such that, splitting the feature vector into $\bx = (\bx_{\leq \evn}, \bx_{>\evn})$, the following hold.
\begin{enumerate}
    \item \textit{Low-degree features.} We assume that for any deterministic vector $\bv \in \R^{\evn}$ and p.s.d.~matrix $\bA \in \R^{\evn \times \evn}$,
    \begin{align}
        \P \left( \big\vert \< \bv, \bx_{\leq \evn} \>  \big\vert \geq t \cdot   \bv^\sT \bSigma_{\leq \evn} \bv \right) \leq&~ \sfC_x \exp \left\{ - \sfc_x t^{2/\beta } \right\}, \label{eq:ass_quad_concentration_intro} \tag{a1}\\
    \P \left( \big\vert \bx_{\leq \evn}^\sT \bA \bx_{\leq \evn} - \Tr( \bSigma_{\leq \evn} \bA) \big\vert \geq t \cdot \varphi_{1} (\evn) \cdot \big\| \bSigma_{\leq \evn}^{1/2} \bA \bSigma_{\leq \evn}^{1/2} \big\|_F \right) \leq&~ \sfC_x \exp \left\{ - \sfc_x t^{1/\beta } \right\}  . \label{eq:ass_quad_concentration_intro_2} \tag{a2} 
    \end{align}
    Equation \eqref{eq:ass_quad_concentration_intro} amounts to an hypercontractivity condition on the subspace spanned by the top $\evn$ eigenfunctions. Equation \eqref{eq:ass_quad_concentration_intro_2} relaxes previous Hanson-Wright-type inequalities that were used to prove deterministic equivalents in \cite{louart2018concentration,cheng2022dimension}.

    \item \textit{High-degree features.} Denote $\bX_{>\evn}$ the high-degree part of the feature matrix. We assume that there exists $\varphi_{2,n} (\evn) \geq 1$ such that with high probability
    \begin{equation}\label{eq:high-degree_intro}
    \| \bX_{>\evn} \bX_{>\evn}^\sT - \Tr (\bSigma_{>\evn} ) \cdot  \id_n \|_\op \leq \varphi_{2,n}(\evn)  \sqrt{\frac{n}{r_{\lambda} (\evn)}} \cdot \left\{ \lambda + \Tr( \bSigma_{>\evn}) \right\}. \tag{b1}
    \end{equation}
    We expect $\varphi_{2,n} (\evn) = O(\polylog (n))$ as long as $r_\lambda (\evn) \gg n$ in many settings of interest. We present a number of such sufficient conditions throughout the paper, but choose to keep Assumption \eqref{eq:high-degree_intro} in this form for the sake of generality.
\end{enumerate}

The first assumption is inspired by polynomial kernels on Gaussian or spherical data $\bu \in \R^d$. In these cases, if $\bx_{\leq \evn}$ corresponds to an orthonormal basis of polynomials up to degree $\ell$, then Equations \eqref{eq:ass_quad_concentration_intro} and \eqref{eq:ass_quad_concentration_intro_2} are satisfied with $\beta = \ell$ and $\varphi_1 (\evn) = \Theta (d^{(\ell-1)/2})$. We show in a companion paper how these conditions can be relaxed in some settings. Note that taking $\bA = (\bX_{\leq \evn}^\sT \bX_{\leq \evn} + \lambda)^{-1}$ and $\varphi_1 (\evn) = o(\sqrt{n})$, Equation \eqref{eq:ass_quad_concentration_intro_2} is a quantitative version of a necessary condition for the Marchenko-Pastur theorem to hold \cite{yaskov2016necessary,misiakiewicz2022spectrum}.

Assumption \eqref{eq:high-degree_intro} allows to replace the high-frequency part of the kernel matrix $\bX_{>\evn} \bX_{>\evn}^\sT + \lambda \id$ by the deterministic matrix $(\lambda +\Tr(\bSigma_{>\evn}))\cdot \id$ whenever $\evn$ is chosen with $r_\lambda (\evn) \gg n$. This assumption covers two regimes of interest: the classical setting of uniform convergence when $\lambda \gg n \xi_{\evn+1}$, where $\lambda$ is chosen positive (e.g., by cross-validation), and the benign overfitting regime when $r_0 (\evn) \gg n$, where we can take $\lambda \to 0^+$ (the interpolating solution) while still achieving good generalization.

Besides these assumptions on the feature $\bx$, our formal statements will require a bound on the growth of the moments of the target function. However, we do not require $\| \bSigma^{-1/2} \bbeta_* \|_2 < \infty$ ($f_*$ in the RKHS) or $\| \bSigma^{-1} \bbeta_* \|_2 < \infty$, and our results apply more generally to $f_* \in L^2(\cU)$. 

We refer the reader to Section \ref{sec:assumptions_test_error} for a formal presentation of these assumptions.

\paragraph*{Approximation guarantees.}
We prove under the above assumptions that with high probability 
\begin{equation}\label{eq:test_approx_guarantee_intro}
\left| \cR_{\test} (\bbeta_*; \bX, \beps,\lambda) - \sR_n (\bbeta_*,\lambda) \right| = \frac{\widetilde{O} (1)}{\lambda_{>\evn}^6}\cdot \left\{ \frac{\varphi_1 (\evn)}{\sqrt{n}} + \varphi_{2,n}(\evn)  \sqrt{\frac{n}{r_\lambda (\evn)}} \right\} \sR_n (\bbeta_*,\lambda),
\end{equation}
where we denoted $\lambda_{>\evn} = \lambda + \Tr(\bSigma_{>\evn})$. Here $\Tilde O(1)$ hides $\polylog (n)$ factors and a mild dependence on an `effective rank' of the matrix $\bSigma_{\leq \evn}$. Equation \eqref{eq:test_approx_guarantee_intro} implies that the test error is well approximated by $\sR_n (\bbeta_*,\lambda)$ as soon as there exists $\evn := \evn (n)$ such that the features satisfy the above assumptions with
\[
\lambda_{>\evn} \gtrsim 1, \qquad \qquad \sqrt{n}\gtrsim \varphi_1 (\evn)  , \qquad \qquad r_\lambda (\evn) \gtrsim n \cdot \varphi_{2,n} (\evn)^2.
\]
The detailed statement of this approximation guarantee (Theorem \ref{thm:abstract_Test_error}) can be found in Section \ref{sec:main_theorem_test_error}.

\vspace{+14pt}

We next describe the two classical examples on which we illustrate our results.

\paragraph*{Example 1: concentrated features.} We assume that the feature vector $\bx$ satisfy Equation \eqref{eq:ass_quad_concentration_intro_2} with $\evn = \infty$ and $\varphi_1 (\evn) = 1$. This setting includes a number of popular theoretical models---such as feature vector $\bSigma^{-1/2} \bx$ with independent sub-Gaussian entries or that satisfies the Lipschitz convex concentration property---which were considered previously in \cite{bartlett2020benign,tsigler2023benign,cheng2022dimension}. Our general guarantee \eqref{eq:test_approx_guarantee_intro} simplifies in this case to
\[
\left| \cR_{\test} (\bbeta_*; \bX, \beps,\lambda) - \sR_n (\bbeta_*,\lambda) \right| =\widetilde{O} (1) \cdot \frac{1}{\lambda^6 \sqrt{n}} \cdot \sR_n (\bbeta_*,\lambda).
\]
This result improves on \cite{cheng2022dimension} in two significant ways. First, the relative rate $\Tilde{O} (n^{-1/2})$ holds \textit{uniformly} over all $\bbeta_*$. Second, our guarantees applies to \textit{any square-integrable function} $\| \bbeta_* \|_2 < \infty$ without requiring $\| \bSigma^{-1} \bbeta_* \|_2 < \infty$. 

We further note that in this setting, Assumption \eqref{eq:high-degree_intro} is verified with $\varphi_2(\evn) = O(\polylog(n))$ for any $\evn$ with $r_{\lambda} (\evn)\gtrsim n$. The flexibility of choosing $\evn < \infty$ allows to provide approximation guarantees for the interpolating solution $\lambda =0$ whenever $\Tr(\bSigma_{>\evn}) = \Omega (1)$. For example, if $\Tr(\bSigma_{>n^2}) = \Omega (\Tr(\bSigma))$, then the relative approximation rate is $\widetilde{O} (n^{-1/2})$ for any $\lambda \geq 0$.


\paragraph*{Example 2: inner-product kernels on the sphere.} We consider covariates $\bu_i \sim \Unif (\S^{d-1} (\sqrt{d}))$ uniformly distributed on the $d$-dimensional sphere of radius $\sqrt{d}$, and an arbitrary inner-product kernel $K(\bu,\bu') = h(\<\bu,\bu'\> /d)$. In this example, the kernel operator admits a simple eigendecomposition in terms of spherical harmonics which has been leveraged in a recent string of work to investigate properties of non-linear kernels \cite{ghorbani2021linearized,mei2022generalization,xiao2022precise}. In particular, KRR in this model displays a number of interesting phenomena such as benign overfitting and multiple descents in the risk curve.

Applying the general guarantee \eqref{eq:test_approx_guarantee_intro} to this setting, we obtain
\begin{equation}\label{eq:approx_sphere_intro}
\left| \cR_{\test} (\bbeta_*; \bX, \beps,\lambda) - \sR_n (\bbeta_*,\lambda) \right| =\widetilde{O} (1) \cdot \left\{ \sqrt{\frac{d^{\ell-1}}{n}} + \sqrt{\frac{n}{d^{\ell+1}}}\right\}\cdot \sR_n (\bbeta_*,\lambda),
\end{equation}
where $\ell \geq1$ denotes the closest integer to $\log(n)/\log(d)$. Note that the rate of approximation is always better than $d^{-1/4}$.
We illustrate this result in Figure \ref{fig:sphere_test} where we compare the test errors obtained from simulations and the theoretical predictions from the deterministic equivalent \eqref{eq:DetEquiv_Risk_Intro} for $d = 24$ and varying $n \in \{2, \ldots, 20000\}$. Figure \ref{fig:sphere_test} shows excellent agreement between the empirical test error and its deterministic equivalent even for small $n$ and complex learning curves with multiple descents.

\vspace{+14pt}

The deterministic equivalent for the test error described above relies on the exact eigendecomposition of the kernel operator and hard-to-verify assumptions on the featurization map. 
 Nevertheless, previous studies  \cite{bordelon2020spectrum,canatar2021spectral,loureiro2021learning} empirically evaluated the kernel eigendecomposition from data and showed that the estimated deterministic equivalent \eqref{eq:DetEquiv_Risk_Intro} accurately predicts the KRR learning curves across various scenarios with real data and complex kernels, including trained neural networks. For completeness, we provide an illustration in Figure \ref{fig:real_test} with MNIST dataset and three standard kernels: an NTK and the RBF with two different bandwidths (see figure caption for details). It is an interesting research direction to relax our assumptions\footnote{It is worth noting that we can easily construct kernels with $\cR_\test$ that does not concentrate for any $n\in\naturals$ \cite{misiakiewicz2022spectrum}.} and connect them to more interpretable properties.

\paragraph*{Uniform consistency of the GCV estimator.} 

As a concrete application of the deterministic equivalents developed in this paper, we study the \textit{generalized cross-validation} (GCV) estimator.
Under the same general assumptions as above, we provide novel non-asymptotic guarantees for the GCV estimator in KRR.

The GCV estimator was introduced in \cite{craven1978smoothing,golub1979generalized} as an approximation to the leave-one-out cross-validation estimator. It is given by
\[
\widehat{\GCV}_\lambda ( \bK , \by) = n \frac{\by^\sT ( \bK + \lambda \id_n )^{-2} \by}{\Tr ( (\bK + \lambda)^{-1} )^2} ,
\]
where we denoted $\bK := \bX \bX^\sT = (K(\bu_i,\bu_j) )_{ij \in [n]}$ the empirical kernel matrix and $\by = (y_1, \ldots, y_n)$ the vector of labels. Under the same general assumptions as above, we prove that with high probability
\begin{equation}\label{eq:GCV_guarantee_intro}
\sup_{\lambda \in [0,\lambda_{\max}]} \left| \frac{\widehat{\GCV}_\lambda (\bK,\by)}{\cR_\test (\bbeta_* ; \bX,\beps,\lambda)}  - 1\right| = \frac{\widetilde{O} (1)}{\Tr( \bSigma_{>\evn})^6}\cdot \left\{ \frac{\varphi_1 (\evn)}{\sqrt{n}} + \varphi_{2,n}(\evn)  \sqrt{\frac{n}{r_{0} (\evn)}} \right\}.
\end{equation}
In particular, if we denote $\lambda_{\opt} \in \R_{\geq 0}$ the regularization parameter that minimizes $\cR_\test (\bbeta_* ; \bX,\beps,\lambda)$, and $\widehat\lambda_{\GCV}\in \R_{\geq 0}$ the parameter that minimizes $\widehat{\GCV}_\lambda (\bK,\by)$, the above guarantee implies
\[
\begin{aligned}
&~\left| \cR_{\test} (\bbeta_* ; \bX, \beps, \widehat\lambda_{\GCV}) -  \cR_{\test} (\bbeta_* ; \bX, \beps, \lambda_{\opt} ) \right| \\
=&~ \frac{\widetilde{O} (1)}{\Tr( \bSigma_{>\evn})^6} \left\{ \frac{\varphi_1 (\evn)}{\sqrt{n}} + \varphi_{2,n}(\evn)  \sqrt{\frac{n}{r_{0} (\evn)}} \right\} \cdot  \cR_{\test} (\bbeta_* ; \bX, \beps, \lambda_{\opt} ).
\end{aligned}
\]
Hence, the GCV estimator can be used to efficiently estimate the test error and optimally tune the ridge regularization. This corroborates recent findings in \cite{jacot2018neural,wei2022more} that have empirically demonstrated the accuracy of GCV in KRR. 

Compared to previous works that have studied the uniform consistency of the GCV estimator \cite{xu2019consistent,hastie2022surprises,patil2021uniform}, the guarantee \eqref{eq:GCV_guarantee_intro} is \textit{non-asymptotic} and applies to settings with infinite-dimensional features and non-linear kernels. Furthermore, it is \textit{scale-free} and does not depend on the value of the test error, which can be vanishing.


\paragraph*{Technical contributions.} Our main technical contributions are novel characterizations for random matrix functionals in the dimension free regime introduced by \cite{cheng2022dimension} (see Section \ref{sec:related} for background). In this regime, we fix the distribution of the feature vector $\bx \in \R^p$---potentially infinite dimensional ($p=\infty$)---and vary $n \in \naturals$. We prove deterministic equivalents for different functionals of the feature matrix $\bX = [ \bx_1, \bx_2, \ldots, \bx_n]^\sT \in \R^{n \times p}$ with explicit multiplicative bounds. For example, in the case of concentrated features $\bx \in \R^\infty$ (Example 1 above), we show that for any deterministic p.s.d.~operator $\bA \in \R^{\infty \times \infty}$ with $\Tr(\bA) < \infty$, the following holds with probability at least $1 - n^{-D}$
\[
\left| \lambda \Tr( \bA (\bX^\sT \bX + \lambda)^{-1} ) - \lambda_*\Tr( \bA ( \bSigma + \lambda_*)^{-1} ) \right| \leq C_{x,D}  \frac{\nu_\lambda (n)^{5/2} \log^{\beta +\frac12} (n)}{\sqrt{n}} \cdot \lambda_* \Tr( \bA (  \bSigma + \lambda_*)^{-1} ),
\]
where $\bSigma = \E[ \bx \bx^\sT]$ is the feature covariance matrix, $\lambda_*$ is the effective regularization as defined in Eq.~\eqref{eq:eff_reg_intro}, $\nu_\lambda (n)$ is defined in Eq.~\eqref{eq:reduced_nu_lambda}, and $C_{x,D}$ is a constant that only depends on $D>0$ and $\sfc_x,\sfC_x,\beta$ appearing in Eq.~\eqref{eq:ass_quad_concentration_intro_2}.
The full statement of these deterministic equivalents can be found in Section \ref{sec_outline:det_equiv}.

\paragraph*{Organization.} Section \ref{sec_main:test_error} introduces notations and assumptions that will be used throughout the paper, and states our main result on the deterministic equivalent of the KRR test error (Theorem \ref{thm:abstract_Test_error} in Section \ref{sec:main_theorem_test_error}). We further discuss and connect our assumptions to previous work in Section \ref{sec:kernel_eigendecomposition}. In Section \ref{sec_main:examples}, we apply our general results to the examples of concentrated features (Section \ref{sec:example_concentrated}) and inner-product kernels on the sphere (Section \ref{sec:inner-product_main}). Section \ref{sec_main:GCV} presents the GCV estimator and states our main guarantee on its uniform consistency (Theorem \ref{thm:abstract_GCV}). Finally, we outline the strategy for proving our main results in Section \ref{sec:outline_proofs}. The complete proofs are deferred to the appendices.

\subsection{Related literature}
\label{sec:related}

The performance of KRR has been studied extensively over the years \cite{bartlett2005local,caponnetto2007optimal,hsu2012random,wainwright2019high,fischer2020sobolev,lin2020optimal,bartlett2020benign,tsigler2023benign}. For example, it was shown in \cite{caponnetto2007optimal} that KRR achieves minimax optimal rates over subclasses of functions under source and capacity conditions. Recently, a number of works have sought to precisely characterize the learning curves of KRR in high dimensions. \cite{dicker2016ridge} considered ridge regression with Gaussian features $\bx_i \sim \normal (0,\id_d)$ and computed the asymptotic risk in the proportional asymptotics regime with $d,n \to \infty$ with $d/n \to \gamma \in (0,\infty)$. This result was later extended to anisotropic feature distribution in \cite{dobriban2018high}, and to general linear target functions in \cite{richards2021asymptotics,wu2020optimal,hastie2022surprises}.  The asymptotic test error of KRR with an inner-product kernel was derived in \cite{liang2020just,bartlett2021deep} using the linearization of the kernel in this regime \cite{el2010spectrum}. The setting of our second example was investigated in \cite{ghorbani2021linearized,xiao2022precise} which considered inner-product kernels on the sphere and computed precise asymptotics for the test error in the polynomial high-dimensional regime, where $n,d \to \infty$, with $n /d^\kappa \to \gamma$, for any $\kappa,\gamma>0$. They showed that the learning curve displays a staircase decay as $n$ increases, where each time $\log(n)/\log(d)$ crosses an integer value, KRR fits one more degree polynomial approximation to the target function. At the critical regimes $n \asymp d^\ell, \ell \in \naturals$ , a peak can appear due to the degeneracy of the eigenspace spanned by the degree-$\ell$ spherical harmonics \cite{xiao2022precise}. We improve upon these works in two important ways. First, our bounds and deterministic equivalents are non-asymptotic and resolve the ambiguity of the polynomial scaling at finite $n,d$. Second, our bounds hold for a fixed target function without requiring to randomize its coefficients. 

The paper \cite{mei2022generalization} developed a framework for computing the asymptotic test error of KRR under abstract conditions on the kernel eigendecomposition. Namely, they assume the top eigenspaces satisfy an hypercontractivity condition (corresponding to our condition \eqref{eq:ass_quad_concentration_intro}) and the high-degree part of the empirical kernel matrix concentrates on a deterministic matrix (our condition \eqref{eq:high-degree_intro}). In this paper, we simplify these assumptions and make them quantitative to allow for explicit non-asymptotic bounds. Most importantly, \cite{mei2022generalization} assumes a `spectral gap' condition with the low-degree part being low-dimensional, i.e., $\evn \ll n$. In this regime, KRR behaves effectively as a shrinkage operator  
and their proof only uses matrix concentration bounds. Removing the spectral gap assumption requires random matrix tools and in particular, to prove deterministic equivalents for random matrix functionals. Further discussion on \cite{mei2022generalization} can be found in Section \ref{sec:kernel_eigendecomposition}.

In parallel to this high-dimensional work, a line of research \cite{sollich2001gaussian,spigler2020asymptotic,bordelon2020spectrum,canatar2021spectral,loureiro2021learning,cui2021generalization,simon2023eigenlearning} heuristically derived and empirically validated a non-asymptotic closed-form estimate for the KRR test error, i.e., the deterministic equivalent \eqref{eq:DetEquiv_Risk_Intro}. \cite{bordelon2020spectrum} presents two different approaches to obtain this analytical expression: a continuous approximation to the learning curves inspired by the Gaussian process literature \cite{sollich2001gaussian}, and replica method with a saddle-point approximation. \cite{loureiro2021learning} assumes the eigenfunctions to be independent Gaussians and derive Eq.~\eqref{eq:DetEquiv_Risk_Intro} using CGMT in the proportional regime. \cite{simon2023eigenlearning} provides a simple heuristic derivation based on a `conservation law'. \cite{spigler2020asymptotic,bordelon2020spectrum,canatar2021spectral,loureiro2021learning,simon2023eigenlearning} conducted comprehensive empirical evaluations and demonstrated the accuracy of these predictions in various synthetic and real-data settings. We refer to \cite{bordelon2020spectrum,simon2023eigenlearning} for further background and discussion on this literature. 
Our paper aims to bridge the gap between these heuristic closed-form estimates and random matrix theory, and provide further supporting evidence for this line of work.


Non-asymptotic deterministic equivalents for the test error were first derived for ridge regression in \cite{hastie2022surprises} and KRR with inner-product kernels in the proportional regime $n \asymp d$ in \cite{bartlett2021deep} using the `anisotropic local law' proved in \cite{knowles2017anisotropic}.  Most closely related to our paper, \cite{cheng2022dimension} considered the case of infinite-dimensional concentrated features (our Example 1) and derived non-asymptotic and multiplicative approximation guarantees (see Section \ref{sec:intro_det_equiv}). The regime they consider, which we call \textit{dimension-free regime} following the title of their work, is quite different from how RMT is usually set up. It is useful to spell out the difference here. The classical setting of RMT typically considers a sequence of problems indexed by $d \in \naturals$. For each $d$, we sample $n(d)$ independent features $(\bx_i^{(d)})_{i \in [n(d)]}$ with common covariance $\tbSigma^{(d)} \in \R^{p(d) \times p(d)}$. The goal is to study the behavior of the random matrix $\bX^{(d)} = [\bx_1^{(d)}, \ldots, \bx_{n(d)}^{(d)}]^\sT \in \R^{n(d) \times p(d)}$, e.g., the spectrum of $n(d)^{-1} (\bX^{(d)})^\sT \bX^{(d)}$,
as $d \to \infty$. The sequence of covariance matrices $\tbSigma^{(d)}$ is typically assumed to satisfy $\| \tbSigma^{(d)} \|_\op \leq C$ with empirical spectral distribution that converges to a density with mass bounded away from $0$. In contrast, \cite{cheng2022dimension} fixes the distribution of the feature vector $\bx \in \R^\infty$---taken infinite-dimensional--- and vary $n \in \naturals$. In particular, the covariance matrix $\bSigma$ is fixed with $\Tr(\bSigma ) < \infty$ and does not have bounded conditioning number. Intuitively, as $n$ increases, the empirical spectrum of $\bX^\sT \bX$ has spikes along the eigenspaces with eigenvalues $\xi_k \gg 1/n$ and a bulk that comes from the contribution of the eigenspaces with $\xi_k \asymp 1/n$. We extends the results of \cite{cheng2022dimension} in two directions. First, we provide an alternative proof of the deterministic equivalent for the linear functional considered in \cite[Theorem 5]{cheng2022dimension}. This proof is significantly simpler and allows for general p.s.d.~matrix (see Remark \ref{rmk:comparison_Cheng_theorems}). Second, we prove dimension-free deterministic equivalents for higher-order functionals of the resolvent which are necessary to study KRR with arbitrary square integrable target functions.


From a technical viewpoint, our paper belongs to the growing body of work that have derived deterministic equivalents to study problems in statistics and machine learning, including spectral clustering \cite{liao2021sparse}, random Fourier features \cite{liao2020random}, multi-layer random features \cite{schroder2023deterministic}, and the SGD dynamics on GLMs \cite{collins2023hitting}. We refer the reader to \cite{couillet2011random,couillet2022random} for background and references on deterministic equivalents.







\section{Test error of kernel ridge regression}
\label{sec_main:test_error}

In this section, we present our main results on the test error of kernel ridge regression (KRR). We start in Section \ref{sec:setting_definitions} by introducing our setting and some definitions. We state the general assumptions under which our results hold in Section \ref{sec:assumptions_test_error} and our master theorem (Theorem \ref{thm:abstract_Test_error}) in Section \ref{sec:main_theorem_test_error}. 
Finally, while Section \ref{sec:assumptions_test_error} states general assumptions on the feature map $\bx = \phi (\bu)$, Section \ref{sec:kernel_eigendecomposition} connects these conditions to properties of the kernel operator.

\subsection{Setting and definitions}\label{sec:setting_definitions}

We are given $n$ i.i.d.~samples $(\bx_i,y_i)_{i \in [n]}$ with feature vectors $\bx_i \in \R^p$ and responses $y_i = f_* (\bx_i) +\eps_i$. The target function is linear in the feature space $f_* (\bx) = \< \btheta_* , \bx\>$ and we assume the label noise $\eps_i$ to be independent with $\E[\eps_i] = 0$ and $\E[ \eps_i^2] = \sigma_\eps^2$. Here we denote $p$ the dimension of the feature vector and consider both the classical linear model setting with finite dimensional features $p<\infty$ and the RKHS case with infinite dimensional features $p = \infty$. 
 
We denote $\bSigma = \E[ \bx \bx^\sT]$ the covariance matrix of the features\footnote{Note that we will only assume that $\E[\bx_{>\evn}] = 0$, and we allow for $\E[\bx_{\leq \evn}] \neq 0$. In that case, $\bSigma$ is most commonly referred to as the `raw' covariance matrix or second moment matrix.}.   Without loss of generality, we will set $\| \bSigma \|_\op = 1$ and take $\bSigma$ to be diagonal with
\[
\bSigma = \diag (\xi_1, \xi_2, \xi_3, \ldots),
\]
where $1 = \xi_1 \geq \xi_2 \geq \xi_3 \geq \cdots $ are the positive eigenvalues in nonincreasing order. For infinite-dimensional features $p = \infty$, we assume the covariance to be trace class $\Tr(\bSigma) <\infty$. In that case, $\bx$ is a random element in the Hilbert space $\ell_2 := \{ \bx = (x_1 , x_2 , x_3,\ldots): \sum_{j = 1}^\infty x_j^2  < \infty\} $ with inner product $\< \bu, \bv \> := \< \bu, \bv \>_{\ell_2} = \sum_{j =1}^\infty u_j v_j$, and $\bSigma$ is understood to be a trace class self-adjoint operator. We introduce $\bz_i = \bSigma^{-1/2} \bx_i$ the whitened features and $\bbeta_* = \bSigma^{1/2} \btheta_*$ the whitened coefficients. Without loss of generality, we assume $\E [ f_*(\bx)^2] = \< \btheta_*, \bSigma \btheta_*\> <\infty$ and we can write
\[
f_* (\bx) = \< \bx , \btheta_*\> = \< \bz , \bbeta_*\>, \qquad \quad \| f_*\|_{L^2}^2 = \E[f_*(\bx)^2] = \| \bbeta_* \|_2^2 < \infty.
\]

Our approach consists in analyzing separately the top eigenspaces ---which are associated to `low-frequency', well-concentrated features--- from the rest of the eigenspaces ---associated to `high-frequency', heavy-tailed features. More precisely, for every $\evn \in \naturals \cup \{\infty\}$, $\evn \leq p$, we split the feature vector $\bx = (\bx_{\leq \evn} , \bx_{>\evn})$ into: 
\begin{itemize}
    \item A low-degree part $\bx_{\leq \evn}$ corresponding to the projection onto the top eigenspaces associated to the $\evn$ largest eigenvalues. We denote the covariance matrix of the low-degree part by
    \[
    \bSigma_{\leq \evn} = \E [ \bx_{\leq \evn} \bx_{\leq \evn}^\sT ] = \diag ( \xi_1 , \xi_2 , \ldots , \xi_m).
    \]

    \item A high-degree part $\bx_{>\evn}$ corresponding to the projection onto the bottom $p - \evn$ eigenspaces with smallest eigenvalues. We denote the covariance matrix of the high-degree part by
    \[
    \bSigma_{>\evn} = \E [ \bx_{>\evn} \bx_{>\evn}^\sT ] = \diag (\xi_{\evn+1} , \xi_{\evn+2}, \ldots ) .
    \]
\end{itemize}

In the case of $\evn = p$, we simply take $\bx = \bx_{\leq \evn}$. For simplicity of presentation, we will always assume that $\E[ \bx_{>\evn} ] = 0$ (e.g., the constant function is in the span of the top $\evn$ eigenfunctions). We write the feature matrix $\bX = [\bx_1 , \ldots, \bx_n]^\sT \in \R^{n \times p}$ in block form $\bX= [ \bX_{\leq \evn } , \bX_{>\evn} ]$ with $\bX_{\leq \evn} \in \R^{n\times \evn}$ and $\bX_{>\evn} \in \R^{n \times (p - \evn)}$. Similarly, we decompose the target function into a low-degree and a high-degree part 
\[
f_*(\bx) =  f_{*,\leq \evn}(\bx) + f_{*,>\evn} (\bx) ,
\]
where
 \[
 f_{*,\leq \evn}(\bx) = \< \btheta_{*,\leq \evn} , \bx_{\leq \evn}\> = \< \bbeta_{*,\leq \evn} , \bz_{\leq \evn} \>  , \qquad \quad f_{*,> \evn}(\bx) = \< \btheta_{*,> \evn} , \bx_{> \evn}\> = \< \bbeta_{*,> \evn} , \bz_{> \evn} \>  .
 \]

Note that the deterministic equivalents depend on an \emph{effective regularization} $\lambda_*$ which corresponds to the ridge regularization parameter in an associated `equivalent sequence model' \cite{cheng2022dimension}.

\begin{definition}[Effective regularization]\label{def:effective_regularization}
    For an integer $n$, covariance $\bSigma$, and regularization $\lambda \geq 0$, we define the \emph{effective regularization} $\lambda_*$ associated to $(n,\bSigma,\lambda)$ to be the unique non-negative solution to the equation
    \begin{equation}\label{eq:def_lambda_star}
        n - \frac{\lambda}{\lambda_*} = \Tr \big( \bSigma ( \bSigma + \lambda_* )^{-1} \big)  .
    \end{equation}
\end{definition}

Equation \eqref{eq:def_lambda_star} corresponds to the general Marchenko-Pastur equation \cite{pastur1967distribution}. In particular, we show that under the setting of Theorem \ref{thm:abstract_Test_error}, the inverse effective regularization $1/\lambda_*$ approximates the Stieltjes transform of the empirical kernel matrix (see Theorem \ref{thm:stieltjes_general} in Appendix \ref{app_test_error:Stieltjes_general} for a formal statement):
\[
 \frac{1}{n }\Tr\left( ( \bX \bX^\sT + \lambda)^{-1} \right) = (1 + o(1)) \cdot \frac{1}{n\lambda_*}.
\]

Our approximation guarantees will depend on the covariance $\bSigma$ through two measures of ranks: the \textit{regularized tail rank} of $\bSigma_{>\evn}$ (high-degree features) and the \textit{effective rank} of $\bSigma_{\leq \evn}$ (low-degree features) defined below.

\begin{definition}[Ranks]\label{def:effective_tail_rank}
    For a covariance matrix $\bSigma = \diag (\xi_1,\xi_2,\ldots) \in \R^{p \times p}$ with eigenvalues in nonincreasing order $\xi_1 \geq \xi_2 \geq \xi_3 \geq \cdots$, we define two notions of rank:

    \begin{enumerate}
        \item[(a)] \emph{(Regularized tail rank.)} For any integer $\evn \leq p$ and regularization parameter $\lambda \geq 0$, the regularized tail rank $r_{\lambda} (\evn)$ is defined by
        \begin{equation}
            r_{\lambda} (\evn) = \frac{\lambda + \sum_{j =\evn+1}^p \xi_j}{\xi_{\evn+1}}.
        \end{equation}
        We use the convention $r_\lambda (\evn) = \infty$ for $\evn = p$.

        \item[(b)] \emph{(Effective rank.)} For any integers $n$ and $\evn \leq p$, the effective rank $r_{\seff,\evn} (n)$ of $(n,\bSigma_{\leq \evn})$ is defined as the smallest scalar such that $r_{\seff,\evn} (n)\geq n$ and
        \[
        r_{\seff,\evn} (n) \geq  \frac{\sum_{j =k+1}^\evn \xi_j}{\xi_{k+1}},\qquad \quad \text{for all $\;\;\;0\leq k \leq \min(n,\evn) - 1$.}
        \]
    \end{enumerate}
\end{definition}

The effective rank $r_{\seff, \evn} (n)$ at $\evn = p$ is the same as the one that appeared in \cite{cheng2022dimension} to study deterministic equivalents in the dimension free regime. We can think of $r_{\seff,\evn} (n)$ as corresponding to a global version of the notion of \textit{intrinsic dimension} \cite[Chapter 7]{tropp2015introduction}. More precisely, denote $\bSigma_{\leq\evn, >k} = \diag ( \xi_{k+1}, \xi_{k+2}, \ldots , \xi_\evn)$ the covariance matrix $\bSigma_{\leq \evn}$ projected orthogonally to the top $k$ eigenspaces. The intrinsic dimension of $\bSigma_{\leq\evn, >k}$ is defined as
\[
\text{intDim} (\bSigma_{\leq\evn, >k}) := \frac{\Tr( \bSigma_{\leq\evn, >k})}{\| \bSigma_{\leq\evn, >k}\|_\op} = \frac{\sum_{j =k+1}^\evn \xi_j}{\xi_{k+1}},
\]
and captures the number of `relevant' dimensions of $\bSigma_{\leq\evn, > k,}$, i.e., eigenspaces that have significant spectral content. Thus, $r_{\seff,\evn} (n)$ upper bounds the 
 intrinsic dimension of $\bSigma_{\leq \evn}$ at all scales up to the $n$-th eigenvalue, i.e., for all $\bSigma_{\leq\evn, >k}$ with $k =0, \ldots, n-1$.

 The (regularized) tail rank was introduced in the study of ridge regression in the (near-)interpolating regime \cite{bartlett2020benign,tsigler2023benign,koehler2021uniform}. This rank appears naturally in our analysis in the following sense. For convenience, denote from now on
\[
\lambda_{>\evn} = \lambda + \Tr( \bSigma_{>\evn}) = \lambda + \sum_{j =\evn+1}^p \xi_j.
\]
Let $\lambda_{*,\evn}$ be the effective regularization associated to $(n,\bSigma_{\leq \evn},\lambda_{>\evn})$ and define
\[
\sR_{n,\leq \evn} (\bbeta_*, \lambda_{>\evn}) := \frac{\lambda_{*,\evn}^2 \< \bbeta_{\leq \evn}, ( \bSigma_{\leq \evn} + \lambda_{*,\evn} )^{-2} \bbeta_{\leq \evn} \> + \| \bbeta_{>\evn} \|_2^2 + \sigma_\eps^2 }{1 - \frac{1}{n} \Tr( \bSigma_{\leq \evn}^2 ( \bSigma_{\leq \evn} + \lambda_{*,\evn} )^{-2} ) }.
\]
This simply corresponds to the test error of a truncated model with $n$ features $\bx_{\leq \evn}$, regularization $\lambda_{>\evn}$, target function $\< \bbeta_{\leq \evn}, \bz_{\leq \evn}\>$, and independent label noise with variance $\| f_{*,>\evn} \|_{L^2}^2 + \sigma_\eps^2$. This truncated model approximates the original model with 
\[
\left| \sR_{n} (\bbeta_*, \lambda)  - \sR_{n,\leq \evn} (\bbeta_*, \lambda_{>\evn})  \right|  \lesssim \frac{n}{r_{\lambda} (\evn)} \cdot  \sR_{n} (\bbeta_*, \lambda),
\]
and this bound is tight.

\subsection{Assumptions}
\label{sec:assumptions_test_error}

For a given $n$, we will assume that there exists an integer $\evn := \evn (n)$ such that the low-degree features $\bx_{\leq \evn}$, the high-degree features $\bx_{>\evn}$, and the high degree part of the target function $f_{*,>\evn}$ satisfy the following assumptions.

\begin{assumption}[Concentration at $n \in \naturals$]\label{ass:main_assumptions}
There exist $\sfc_x,\sfC_x,\beta>0$ and $\evn \in \naturals \cup \{ \infty\}$, $\evn \leq p$, such that the regularized tail rank at $\evn$ satisfies
\[
r_{\lambda} (\evn) \geq 2n,
\]
and the following hold.
\begin{itemize}
    \item[\emph{(a)}] \emph{(Low-degree features.)} There exists $\varphi_1 (\evn)>0$ such that for any deterministic vector $\bv \in \R^{\evn}$ with $\| \bSigma_{\leq \evn}^{1/2} \bv \|_2 < \infty$ and p.s.d.~matrix $\bA \in \R^{\evn \times \evn}$ with $\Tr(\bSigma_{\leq \evn} \bA) < \infty$, we have
    \begin{align}
        \P \left( \big\vert \< \bv, \bx_{\leq \evn} \>  \big\vert \geq t \cdot   \bv^\sT \bSigma_{\leq \evn} \bv \right) \leq&~ \sfC_x \exp \left\{ - \sfc_x t^{2/\beta } \right\}, \label{eq:ass_quad_concentration_2} \tag{a1}\\
    \P \left( \big\vert \bx_{\leq \evn}^\sT \bA \bx_{\leq \evn} - \Tr( \bSigma_{\leq \evn} \bA) \big\vert \geq t \cdot \varphi_{1} (\evn) \cdot \big\| \bSigma_{\leq \evn}^{1/2} \bA \bSigma_{\leq \evn}^{1/2} \big\|_F \right) \leq&~ \sfC_x \exp \left\{ - \sfc_x t^{1/\beta } \right\}  . \label{eq:ass_quad_concentration_1} \tag{a2} 
    \end{align}

    \item[\emph{(b)}] \emph{(High-degree features.)} There exist $p_{2,n}(\evn) \in (0,1)$ and $\varphi_{2,n}(\evn)  \geq 1$ such that with probability at least $1 - p_{2,n} (\evn)$, we have
    \begin{equation}
    \| \bX_{>\evn} \bX_{>\evn}^\sT - \Tr (\bSigma_{>\evn} ) \cdot  \id_n \|_\op \leq \varphi_{2,n}(\evn)  \sqrt{\frac{n}{r_{\lambda} (\evn)}} \cdot \left\{ \lambda + \Tr( \bSigma_{>\evn}) \right\}. \tag{b1}
    \end{equation}

    \item[\emph{(c)}] \emph{(Target function.)} The high-degree part of the target function satisfies the tail bound
    \begin{equation}\label{eq:equation_c1_assumption}
\P \left( | f_{*,>\evn} (\bx ) | \geq t \cdot \| f_{*,>\evn} \|_{L^2} \right) \leq \sfC_x \exp \left\{ - \sfc_x t^{2/\beta} \right\}. \tag{c1}
\end{equation}
\end{itemize}
\end{assumption}

 Assumption \ref{ass:main_assumptions}.(b) implies that we can approximate the Gram matrix $\bX_{>\evn} \bX_{>\evn}^\sT$ of the high-frequency features by $\Tr(\bSigma_{>\evn} ) \cdot \id_n$ with probability $1 - p_{2,n} (\evn)$. In particular, we can replace on this event the resolvent matrix in the expression of the test error by
\[
( \bX \bX^\sT + \lambda )^{-1} \approx ( \bX_{\leq \evn} \bX_{\leq \evn}^\sT + \lambda_{>\evn} )^{-1},
\]
which does not depend on the high-degree features anymore. 
We can then use Assumption \ref{ass:main_assumptions}.(a) to derive deterministic equivalents for functionals of the low-degree feature matrix $\bX_{\leq \evn}$ (see Section \ref{sec_outline:det_equiv}). Finally, Assumption \ref{ass:main_assumptions}.(c) allows us to show that the high-degree part of the target function effectively behaves as independent additive label noise with mean $0$ and variance $\| f_{*,>\evn} \|_{L^2}^2$ (see Lemma \ref{lem:high-degree_part} in Appendix \ref{app_test_error:tech_high-degree_target}). We provide further discussion on these assumptions in Section \ref{sec:kernel_eigendecomposition}.

To show the concentration of the test error over the randomness of the label noise, we will further consider the following assumption.

\begin{assumption}[Label noise]\label{ass:noise_subGaussian}
    The label noises $\{\eps_i\}_{i\in[n]}$ are independent, mean-zero, and $\tau_\eps^2$-sub-Gaussian with variance denoted $\sigma_\eps^2 = \E[ \eps_i^2]$.
\end{assumption}

\subsection{A master theorem}
\label{sec:main_theorem_test_error}

Throughout the paper, our bounds will be explicit in terms of the model parameters $\{n, \bSigma, \lambda,\bbeta_*,\sigma_\eps^2\}$, including the tail ranks $\{r_\lambda (\evn),r_{\seff,\evn}\}$, as well as the constants $\{\varphi_1 (\evn), \varphi_{2,n} (\evn),p_{2,n}(\evn)\}$ appearing in Assumption \ref{ass:main_assumptions}. For the rest, we will denote $C_{a_1,\dots,a_k}$ constants that only depend on the values of $\{a_i\}_{i\in[k]}$. We use $a_i = `x$' to denote the dependency on $\sfc_x,\sfC_x,\beta$ from Assumption \ref{ass:main_assumptions}, and $a_i = `\eps$' to denote the dependency on $\tau^2_\eps$ from Assumption \ref{ass:noise_subGaussian}.

Our relative approximation bound will depend on the effective rank $r_{\seff,\evn} (n)$ through
\begin{equation}\label{eq:reduced_nu_lambda}
\nu_{\lambda,\evn} (n) = 1 + \frac{\xi_{\lfloor \eta n \rfloor,\evn}\cdot  r_{\seff,\evn} (n) \sqrt{\log (r_{\seff,\evn} (n))}  }{\lambda_{>\evn}},
\end{equation}
where $\eta = \eta_x \in (0,1/2)$ is a constant that will only depend on $\sfc_x,\sfC_x,\beta$, and we define $\xi_{\lfloor \eta n \rfloor,\evn} = \xi_{\lfloor \eta n \rfloor}$ if $\lfloor \eta n \rfloor \leq \evn$ and $0$ otherwise. When $ \evn = p$, we will simply denote $\nu_{\lambda} (n) := \nu_{\lambda,p} (n)$.

We are now in position to state our master theorem on the test error of kernel ridge regression.

\begin{theorem}[Deterministic equivalent for the KRR test error]\label{thm:abstract_Test_error}
    Consider $D,K>0$, integer $n$, regularization parameter $\lambda \geq 0$, and target function $f_* \in L^2 (\cU)$ with parameters $\| \bbeta_*\|_2 = \| f_* \|_{L^2} < \infty$. Assume that the features $\{\bx_i\}_{i\in[n]}$ and $f_*$ satisfy Assumption \ref{ass:main_assumptions} with some $\evn := \evn (n) \in \naturals \cup \{\infty\}$, and the $\{\eps_i\}_{i\in [n]}$ satisfy Assumption \ref{ass:noise_subGaussian}. There exist constants $\eta := \eta_x \in(0,1/2)$, $C_{D,K} >0$, and $C_{x,\eps,D,K}>0$ such that for all $n \geq C_{D,K}$ and $\lambda_{>\evn} >0$, if it holds that
\begin{equation}\label{eq:condition_test_abstract}
\lambda_{>\evn} \cdot \nu_{\lambda,\evn} (n) \geq n^{-K},\qquad
\varphi_{2,n} (\evn)  \sqrt{\frac{n}{r_\lambda (\evn)}} \leq \frac{1}{2},
\qquad
 \varphi_1(\evn) \nu_{\lambda,\evn} (n)^{8} \log^{3\beta + \frac{1}{2}} (n) \leq K \sqrt{n},
\end{equation}
then with probability at least $1 - n^{-D} - p_{2,n} (\evn)$, we have
\begin{equation}\label{eq:abstract_test_error}
\left|\cR_{\test} ( \bbeta_*;\bX,\beps, \lambda) -    \sR_{n} (\bbeta_*, \lambda) \right| \leq C_{x,\eps,D,K} \cdot \cE_{\sR,n} (\evn) \cdot  \sR_{n} (\bbeta_*, \lambda),
\end{equation}
where the relative approximation rate is given by
\begin{equation}\label{eq:asbstract_relative_error}
    \cE_{\sR,n}  (\evn) :=  \frac{\varphi_1(\evn) \nu_{\lambda,{\evn}}(n)^{6} \log^{3\beta +1/2}(n)}{\sqrt{n}}    +  \nu_{\lambda,\evn} (n)  \varphi_{2,n} (\evn)  \sqrt{\frac{n}{r_\lambda (\evn)}}  .
\end{equation}
\end{theorem}

 Section \ref{sec:outline_proofs} outlines the proof strategy of this theorem  and presents a self-contained proof in the special case of concentrated features. The full proof of Theorem \ref{thm:abstract_Test_error} can be found in Appendix \ref{app:main_proofs}.

Before illustrating this theorem in specific examples, it is useful to develop some intuition about the general expressions in Eqs.~\eqref{eq:condition_test_abstract} and \eqref{eq:asbstract_relative_error}. For most cases of interest, including regularly varying spectrum (see \cite{cheng2022dimension}), the effective rank satisfies $r_{\seff} (n) \lesssim n^C$ and $\xi_{\lfloor \eta n \rfloor,\evn} \cdot r_{\seff} (n) \lesssim \log^C (n)$ for some constant $C>0$. Thus for concreteness, we assume $\nu_{\lambda,\evn} (n) \lesssim \polylog (n) / \min (1,\lambda_{>\evn})$ and $\varphi_{2,n} (\evn) \lesssim \polylog(n)$ (see remark in Section \ref{sec:summary}) in the discussion below. 

We expect the technical conditions~\eqref{eq:condition_test_abstract} to be mild under Assumption \ref{ass:main_assumptions}: (i) $\lambda_{>\evn} \cdot \nu_{\lambda,\evn} (n) \geq n^{-K}$ is satisfied as soon as $\lambda$ or $\Tr(\bSigma_{>\evn}) \geq n^{-K}$; (ii) $\varphi_{2,n} (\evn)\sqrt{n/r_\lambda (\evn)} \leq 1/2$ implies that the high-frequency part is indeed concentrated around identity with $\lambda_{>\evn}/2 \leq \lambda_{\min} (\bX_{>\evn} \bX_{>\evn}^\sT ) \leq \lambda_{\max} (\bX_{>\evn} \bX_{>\evn}^\sT ) \leq 3 \lambda_{>\evn}/2$; and (iii) we are in a regime where $\varphi_1(\evn) \cdot \polylog (n) \lesssim \lambda_{>\evn}^8 \sqrt{n}$.

Hence, the rate of approximation \eqref{eq:asbstract_relative_error} scales as
\begin{equation}\label{eq:intuitive_rate}
\cE_{\sR,n}  (\evn) \lesssim \frac{\polylog (n)}{\min (\lambda_{>\evn}, 1)^6} \left\{ \frac{\varphi_1(\evn)}{ \sqrt{n}} + \sqrt{\frac{n\xi_{\evn+1}}{\lambda_{>\evn} }} \right\}.
\end{equation}
Therefore, we expect the deterministic estimate $\sR_n (\bbeta_*,\lambda)$ to be a good approximation of the test error as long as we can choose $\evn (n ) \in \naturals \cup \{\infty\}$ such that Assumption \ref{ass:main_assumptions} is verified and $\lambda_{>\evn} = \Omega (1)$, $\varphi_1(\evn) \ll \sqrt{n}$, and $n \xi_{\evn+1} \ll 1$. We will show such a setting in Section \ref{sec:inner-product_main} which studies the case of inner-product kernels on the sphere. Note that we did not optimize the proof over the dependency on $\lambda_{>\evn}$ and we do not expect the scaling $\lambda_{>\evn}^{-6}$ to be tight in Eq.~\eqref{eq:intuitive_rate}. 

Let us further comment on some important features of the approximation guarantee established in Theorem \ref{thm:abstract_Test_error}:
\begin{itemize}
    \item The deterministic equivalent $\sR_{n} (\bbeta_*, \lambda)$ and approximation guarantee are fully \textit{non-asymptotic}: they depend explicitly on finite $n$, fixed feature distribution $\bx$ with covariance $\bSigma$, and fixed target function $f_*$. In particular,  Theorem \ref{thm:abstract_Test_error} only depends on the feature dimension $p$ through $\varphi_1 (\evn)$, $\varphi_{2,n} (\evn)$, and the tail rank $r_\lambda (\evn)$, and applies to the important case of infinite-dimensional features.
    
    \item The bound in Eq.~\eqref{eq:abstract_test_error} is \textit{multiplicative} and applies to settings where the test error is small, e.g., the test error follows a power law $\cR_{\test} \approx n^{-\gamma}$ \cite{wei2022more,cui2022error}. Furthermore the relative approximation error \eqref{eq:asbstract_relative_error} does not depend on the target function $f_*$ as long as it satisfies Assumption \ref{ass:main_assumptions}.(c). In particular, compared to previous work, we do not impose $f_*$ to be in the RKHS, random, or worst-case over a class of functions.

    \item The approximation guarantee holds for the \textit{interpolating solution}---which has attracted a lot of attention in the past few years \cite{belkin2019reconciling,bartlett2020benign,hastie2022surprises,cheng2022dimension}---as long as the conditions of Theorem \ref{thm:abstract_Test_error} are satisfied for $\lambda = 0$. 
\end{itemize}

\subsection{Kernel operator and sufficient conditions}\label{sec:kernel_eigendecomposition}

In this section, we connect Assumption \ref{ass:main_assumptions} and Theorem \ref{thm:abstract_Test_error} to properties of the kernel operator. We start with some background on kernel methods.

Let $(\cU , \P_\bu)$ be a Polish probability space and denote $L^2 (\cU) := L^2 ( \cU, \P_\bu)$ the space of square-integrable functions on $(\cU , \P_{\bu})$. Note that since $(\cU,\P_{\bu})$ is a Polish probability space, $L^2(\cU)$ is separable. We will denote the scalar product and norm on $L^2(\cU)$ by $\< \cdot , \cdot \>_{L^2}$ and $\| \cdot \|_{L^2}$:
\[
\< f , g\>_{L^2} = \int_{\cU} f(\bu) g(\bu) \P_\bu (\de \bu).
\]
 We further introduce $\| f \|_{L^p} = \E_{\bu \sim \P_\bu} [ |f (\bu) |^p ]^{1/p}$ the $L^p$-norm of $f$ for $p\geq 1$.  

We consider $K : \cU \times \cU \to \R$ a positive semi-definite kernel, that is a symmetric function such that for all finite sets of points $\bu_1, \ldots , \bu_k \in \cU$ and coefficients $c_1, \ldots ,c_k \in \R$, $k \in \naturals$,
\[
\sum_{i,j \in [k]} c_i c_j K (\bu_i,\bu_j) \geq 0.
\]
In this paper, we assume that the function $\bu \mapsto K(\bu,\bu)$ is integrable with respect to $\P_\bu$, i.e., $\int_{\cU} K( \bu,\bu) \P_{\bu} (\de \bu) < \infty$. Note that by Cauchy-Schwarz inequality, this implies that the kernel is square integrable $K \in L^2 ( \cU \times \cU)$.

The kernel $K$ induces the integral operator $\Kop :L^2 (\cU) \to L^2 (\cU)$ defined via
\[
\Kop f (\bu) = \int_{\cU} K(\bu,\bu') f(\bu') \P_{\bu} (\de \bu'), \qquad \forall f \in L^2(\cU).
\]
It is immediate that the operator $\Kop$ is self-adjoint, positive semi-definite, and trace-class with $\Tr(\Kop) = \E_\bu [K(\bu,\bu)] <\infty$. We define $\cD \subseteq L^2 (\cU)$ the closed-linear subspace such that $\cD^\perp = \Ker^{-1}(0)$, i.e., $\Kop f=0 $ for all $f \in \cD^\perp$ and $\Kop f \neq 0$ for all $f \in \cD\setminus \{0\}$. If $\cD = L^2 (\cU)$, we say that the kernel $K$ is universal in $L^2(\cU)$. 

By the spectral theorem of compact operators, we can decompose the kernel operator $\Kop$ into
\[
\Kop = \sum_{j =1}^\infty \xi_j \psi_j \psi_j^*,
\]
where $( \psi_j )_{j \geq 1}$ is an orthonormal basis of $\cD = \spn \{ \psi_j : j \geq 1\} \subseteq L^2(\cU)$,  $\< \psi_j , \psi_{j'} \>_{L^2} = \delta_{jj'}$, and $\{ \xi_j \}_{j\geq 1} \subset \R_{>0}$ are the positive eigenvalues in nonincreasing order $\xi_1 \geq \xi_2 \geq \cdots$. In particular, $\sum_{j = 1}^\infty \xi_j < \infty$ by the trace-class assumption. The kernel can be decomposed similarly as
\[
K (\bu_1 , \bu_2 ) = \sum_{ j=1}^\infty \xi_j \psi_j (\bu_1) \psi_j (\bu_2).
\]

The kernel $K$ defines a reproducing kernel Hilbert space (RKHS) $\cH \subseteq \cD$ via
\[
\cH = \left\{ f \in \cD : \| f \|_{\cH}^2 = \sum_{j =1}^\infty \frac{\< f , \psi_j \>_{L^2}^2}{\xi_j} < \infty \right\},
\]
where $\| \cdot \|_{\cH}$ corresponds to the RKHS norm associated to $\cH$. In particular, $\cH$ is dense in $\cD$. In fact, we can construct an isometry between $\cD$ and $\cH$ as follows. Denote $\Kop^{1/2} :\cD \to \cD$ the unique positive self-adjoint square root of $\Kop$ defined by $\Kop^{1/2} = \sum_{j \geq 1} \sqrt{\xi_j} \psi_j \psi_j^*$. Then the image $\Im (\Kop^{1/2} (\cD)) = \cH$ and conversely, for all $h \in \cH$, there exists a unique $f \in \cD$ such that $h = \Kop^{1/2} f$. In particular, denoting $\Kop^{-1/2}: \cH \to \cD$ the inverse, we get that for all $f,g \in \cH$,
\[
\< f, g \>_{\cH} = \< \Kop^{-1/2} f , \Kop^{-1/2} g \>_{L^2},
\]
and $\Kop^{1/2}$ indeed forms an isometry between $\cD$ and $\cH$.

Using these notations, the KRR solution is given by
\begin{equation}\label{eq:KRR_solution_discussion}
\hat f_\lambda = \argmin_{f \in \cH} \left\{ \sum_{i \in [n]} \left(y_i - f (\bu_i) \right)^2 + \lambda \| f \|_{\cH}^2 \right\},
\end{equation}
and the deterministic equivalent of the test error can be rewritten as
\begin{equation}\label{eq:det_equiv_operator}
\sR_n (f_* , \lambda) =  \left( 1 - \frac{1}{n} \Tr \big( \Kop^2 ( \Kop + \lambda_*)^{-2}\big) \right)^{-1} \Big[ \< f_*,  ( \Kop + \lambda_* )^{-2} f_* \>_{L^2}  + \sigma_\eps^2 \Big] .
\end{equation}
Observe that this deterministic equivalent is proportional to the test error of an effective shrinkage estimator $f^{\seff}_{\lambda_*}$, i.e.,
\begin{equation}\label{eq:effective_shrinkage}
\sR_n (f_* , \lambda)  = C_\star (\lambda_*) \cdot \E\left[ \left( y - f^{\seff}_{\lambda_*} (\bu) \right)^2\right],
\end{equation}
where $f^{\seff}_{\lambda_*}$ is the solution of the ridge regression problem
\begin{equation}\label{eq:shrinkage_estimator}
f^{\seff}_{\lambda_*} := \argmin_{f \in \cH} \left\{ \| f_* - f \|_{L^2}^2 + \lambda_* \|  f \|_{\cH}^2 \right\} = \Kop (\Kop + \lambda_* )^{-1} f_*.
\end{equation}
This corresponds to replacing in Eq.~\eqref{eq:KRR_solution_discussion} the empirical risk by the population risk $\| f_* - f \|_{L^2}^2 = \E_{\bu} [ (f_*(\bu) - f(\bu))^2]$ and the regularization parameter $\lambda$ by the effective regularization $\lambda_*$. Equation \eqref{eq:effective_shrinkage} agrees with the standard intuition of KRR: the solution $\hat f_\lambda$ estimates accurately the projection of $f_*$ onto the top eigenspaces with $\xi_j \gg \lambda_*$, while shrinking to $0$ the projection of $f_*$ onto eigenspaces with $\xi_j \ll \lambda_*$. Note that Equation \eqref{eq:effective_shrinkage} was observed in \cite{mei2022generalization} for KRR under a simplified setting: they assume that there exists $\evn = o(n)$ such that $n\xi_{\evn+1} = o (  \Tr(\bSigma_{>\evn}))$ (a `spectral gap' assumption) which leads to  $\lambda_* = (1 +o(1)) \cdot \lambda_{>\evn}/n$ and  $C_\star (\lambda_*) = 1+ o(1)$.

Let us now turn to Assumption \ref{ass:main_assumptions} under which the deterministic estimate~\eqref{eq:det_equiv_operator} approximates the true test error. For any integer $\evn \in \naturals$, we introduce the low and high-degree part of the kernel
\[
\begin{aligned}
    K_{\leq \evn} (\bu_1 , \bu_2) = &~ \sum_{j = 1}^\evn \xi_j \psi_j (\bu_1 ) \psi_j(\bu_2), \qquad \quad K_{> \evn} (\bu_1 , \bu_2) = \sum_{j = \evn+1}^\infty \xi_j \psi_j (\bu_1 ) \psi_j(\bu_2),
\end{aligned}
\]
and their associated integral operators $\Kop_{\leq \evn}$ and $\Kop_{>\evn}$.

\paragraph*{Low-degree part of the kernel.} As discussed in Section \ref{sec:summary}, Equation \eqref{eq:ass_quad_concentration_2} amounts to an hypercontractivity condition on the subspace spanned by the top $\evn$ eigenfunctions:

\begin{assumption}[Hypercontractivity of the top eigenspaces]\label{ass:hypercontractivity} There exist constants $C,\beta>0$ such that for any $h \in \spn \{ \psi_j : j = 1, \ldots , \evn \} $ and integer $q \geq 2$,
\begin{equation}\label{eq:hypercont}
\| h \|_{L^{q}}^2 \leq (Cq)^{\beta} \cdot \| h \|_{L^2}^2.
\end{equation}
\end{assumption}

This assumption was already exploited in \cite{mei2022generalization,mei2021learning,misiakiewicz2022learning,xiao2022precise} to study various non-linear kernels on the sphere and hypercube. Examples of function spaces satisfying Assumption \ref{ass:hypercontractivity} include the subspace of degree-$\ell$ polynomials over either the Gaussian measure, or uniform measure on the sphere or the hypercube \cite{gross1975logarithmic,bonami1970etude,beckner1975inequalities,beckner1992sobolev}, where Equation \eqref{eq:hypercont} holds with $\beta = \ell$ and $C=1$.

For a p.s.d.~matrix $\bA = (a_{ij})_{ij\in[\evn]} \in \R^{\evn \times \evn}$, define the function $A(\bu) = \sum_{i,j \in [\evn]} a_{ij} \psi_i (\bu)\psi_j (\bu)$. Then we can show that under Assumption \ref{ass:hypercontractivity} (see for example \cite[Lemma 6]{mei2022generalization}):
\begin{equation}\label{eq:a2_from_hypercontractivity}
\E \left[ \left| A(\bu) - \Tr(\bA) \right|^q\right]^{2/q} \leq (Cq)^{2\beta} \cdot\Tr(\bA)^2 \leq (Cq)^{2\beta} \cdot \evn \| \bA \|_F^2, \qquad \forall q\geq2.
\end{equation}
Therefore Equation \eqref{eq:ass_quad_concentration_2} implies Equation \eqref{eq:ass_quad_concentration_1} with $\varphi_1 (\evn) = \sqrt{\evn}$. However this bound can be improved in many cases of interest. In particular, when $\{ \psi_j\}_{j\in[\evn]}$ is an orthonormal basis of degree-$\ell$ polynomials with respect to either the Gaussian, uniform on the sphere or hypercube measures, then $\varphi_1 (\evn) = \widetilde{O} (\sqrt{\evn/d})$. This bound is proved for spherical harmonics in Proposition \ref{prop_app:quad_form_quadratic} of Appendix \ref{app_ex:ass_a_spherical}. The case of multivariate Hermite polynomials and Fourier-Walsh basis are much easier and follow from a similar argument.

\paragraph*{High-degree part of the kernel.} We present below sufficient conditions for Assumption \ref{ass:main_assumptions}.(b) to hold.

\begin{assumption}[High-degree kernel concentration]\label{ass:sufficient_high_degree} We assume the following hold.

\begin{itemize}
    \item[(a)] For an integer $\oevn  > \evn$ with $n \cdot \xi_{\oevn+1} \leq \xi_{\evn +1}$, there exist constants $\obeta,\oC>0$ such that for any $h \in \spn \{ \psi_j : j = \evn+1, \ldots , \oevn \} $ and integer $q \geq 2$,
    \[
    \| h \|_{L^{q}}^2 \leq (\oC q)^{\obeta} \cdot \| h \|_{L^2}^2.
    \]

    \item[(b)] There exists $\alpha_1>0$ such that
    \[
    \E\left[ \max_{i \in [n]} \left| K_{>\evn} (\bu_i,\bu_i) - \E[K_{>\evn} (\bu_i,\bu_i)] \right| \right] \leq \alpha_1 \sqrt{\frac{n\xi_{\evn+1}}{\Tr(\Kop_{>\evn})}}  \cdot \E[K_{>\evn} (\bu_i,\bu_i)].
    \]
\end{itemize}
\end{assumption}

This assumption is a quantitative version of \cite[Assumption 4]{mei2022generalization}. We show next using a standard heavy-tailed matrix concentration result \cite[Theorem 5.48]{vershynin2010introduction} that Assumption \ref{ass:sufficient_high_degree} implies a bound in expectation on the operator norm.

\begin{proposition}\label{prop:sufficient_high_degree}
Define $\bK_{>\evn} = (K_{>\evn} (\bu_i,\bu_j))_{ij \in [n]} \in \R^{n \times n}$ the high-degree part of the empirical kernel matrix. Under Assumption \ref{ass:sufficient_high_degree}, there exists a universal constant $C>0$ such that
\begin{align}\label{eq:op_exp}
    \E \left[ \left\| \bK_{>\evn} - \Tr(\Kop_{>\evn})\cdot \id_n \right\|_\op \right] \leq C \left\{ \alpha_1 + (2\oC)^{\obeta} \log^{\obeta}(n) \right\} \sqrt{\frac{n\xi_{\evn+1}}{\Tr(\bSigma_{>\evn})}} \cdot \Tr(\bSigma_{>\evn}).
\end{align}
\end{proposition}

The proof of this proposition can be found in Appendix \ref{app:sufficient}. Using Markov's inequality and Equation~\eqref{eq:op_exp}, we deduce that for any $\delta \in (0,1)$, Assumption \ref{ass:sufficient_high_degree} implies Assumption \ref{ass:main_assumptions}.(b) with $p_{2,n} (\evn) = \delta$ and 
\[
\varphi_{2,n} (\evn) = \frac{C}{\delta} \left\{ \alpha_1 + (2\oC)^{\obeta} \log^{\obeta}(n)\right\}.
\]
Assumption \ref{ass:sufficient_high_degree} is readily verified for the various kernels on the sphere and hypercube considered in \cite{mei2022generalization,mei2021learning,misiakiewicz2022learning,misiakiewicz2022spectrum} with $\alpha_1 = o_n( 1)$. In Section \ref{sec:inner-product_main}, we will consider a tighter moment method analysis in the case of inner-product kernels on the sphere to show that the bound \eqref{eq:op_exp} holds with probability at least $1 - n^{-D}$ instead of in expectation.

\paragraph*{Target function.} The particular form of Assumption \ref{ass:main_assumptions}.(c) is motivated by the goal of showing a \textit{multiplicative} approximation bound $\cE_{\sR,n} (\evn) \cdot \sR_n (\bbeta_*,\lambda)$. If instead, we relax this approximation bound to $\cE_{\sR,n} (\evn) \cdot \E [y^2] = \cE_{\sR,n} (\evn) \cdot (\| f_* \|_{L^2}^2 + \sigma_\eps^2)$, we can replace Assumption \ref{ass:main_assumptions}.(c) by a simpler assumption to verify on the growth of the moments of the target function
\[
\E_\bu \left[ |f_* (\bu) |^q\right] \leq (\sfC_{\beta} q)^{q \beta /2} \| f_* \|_{L^2}^q, \qquad \quad \forall q \geq 2.
\]
Indeed using the hypercontractivity Assumption \ref{ass:hypercontractivity} on $f_{*,\leq \evn}$, we have by triangular inequality
\[
\| f_{*,>\evn} \|_{L^q} \leq \| f_{*} \|_{L^q} + \| f_{*,\leq \evn} \|_{L^q} \leq (Cq)^{\beta /2} \| f_* \|_{L^2} + (Cq)^{\beta /2} \| f_{*,<\evn} \|_{L^2} \leq 2(Cq)^{\beta /2}\| f_* \|_{L^2},
\]
which implies that
\[
\P \left( |f_{*,>\evn} (\bu) | \geq t \cdot \| f_* \|_{L^2} \right) \leq \sfC \exp \left\{ -\sfc t^{2/\beta} \right\},
\]
for some constant $\sfC,\sfc>0$ that only depends on $\beta,\sfC_\beta$. See the proof of Theorem \ref{thm:spherical_test} in Appendix \ref{app:proof_inner_test} for more details.

\begin{remark}[The case $\cD \subsetneq L^2 (\cU)$.] \label{rmk:uncorrelated_part_target}
    For simplicity of presentation, we assumed throughout this section that $f_* \in \cD = \spn \{ \psi_j : j \geq 1\}$. However, our proof readily extend to the case $f_* \not\in \cD$ as long as we assume that $\proj_{\perp,\cD} f_*$ satisfy Assumption \ref{ass:main_assumptions}.(c), where $\proj_{\perp,\cD}$ denotes the projection orthogonal to $\cD$. See the proof of Lemma \ref{lem:high-degree_part} in Appendix \ref{app_test_error:tech_high-degree_target} which controls the high-degree part of the target function  and only uses that $\E[f_{*,>\evn} (\bx_i) \bx_{i,\leq \evn}] = 0$.
\end{remark}

\section{Two classical examples}
\label{sec_main:examples}

We illustrate our master theorem in two classical settings:
\begin{itemize}
    \item[(1)] \textit{Concentrated features.} This setting includes a number of popular theoretical models that have been considered in the literature. While not the main goal of this paper, this first example provides a simple setting to test the predictions of Theorem \ref{thm:abstract_Test_error} and compare them to previous works.

    \item[(2)] \textit{Inner-product kernels on the sphere.} In this example, we consider covariates uniformly distributed on the $d$-dimensional sphere $\bu_i \sim \Unif(\S^{d-1} (\sqrt{d}))$ and an inner-product kernel $K(\bu_i,\bu_j) = h(\< \bu_i,\bu_j\>/d)$. This setting has been used as a simple model to explore properties of non-linear kernels \cite{bach2017breaking,ghorbani2021linearized,misiakiewicz2022spectrum,xiao2022precise}. In particular, compared to the first example, the feature vector is not well concentrated and includes polynomials of arbitrary large degree. This setting will require to fully exploit Theorem \ref{thm:abstract_Test_error} to obtain approximation guarantees. 
\end{itemize}

While we only consider the above two examples, we believe that the proof techniques developed in the case of inner-product kernels on the sphere extend to various settings that have been previously studied in the literature, including inner-product kernels on anisotropic data \cite{ghorbani2020neural}, invariant kernels \cite{mei2021learning,bietti2021sample}, and convolutional kernels \cite{misiakiewicz2022learning,xiao2022precise}. The main challenge when applying Theorem \ref{thm:abstract_Test_error} is to have access to the eigendecomposition of the kernel to verify Assumption \ref{ass:main_assumptions}. It is an interesting research direction to relax this requirement.

\subsection{The case of concentrated features}
\label{sec:example_concentrated}

In this first example, the feature vector $\bx \in \R^p$, $p \in \naturals \cup \{\infty\}$, satisfy the following assumption:

\begin{assumption}[Concentrated features.]\label{ass:concentrated}
There exist $\sfc_x,\sfC_x,\beta >0$ such that for any p.s.d.~matrix $\bA \in \R^{p\times p}$ with $\Tr(\bSigma \bA) <\infty$, we have
\begin{equation}\label{eq:ass_concentrated_features}
        \P \left( \left| \bx^\sT \bA \bx - \Tr(\bSigma \bA) \right| \geq t \cdot \| \bSigma^{1/2} \bA \bSigma^{1/2} \|_F \right) \leq \sfC_x \exp \left\{ -\sfc_x t^{1/\beta } \right\} .
    \end{equation}
\end{assumption}

This assumption is verified in a number of popular theoretical models:
\begin{itemize}
    \item The feature vector $\bz = \bSigma^{-1/2} \bx$ has independent sub-Gaussian coordinates with uniformly-bounded sub-Gaussian norm. This is for example the setting studied in \cite{bartlett2020benign,tsigler2023benign}.

    \item The feature vector $\bz = \bSigma^{-1/2} \bx$ satisfy the Lipschitz convex concentration property: for any $1$-Lipschitz convex function $\varphi:\R^p \to \R$ and $t > 0$,
    \begin{equation}\label{eq:convex_concentration}
    \P \left( \left| \varphi (\bz) - \E[\varphi (\bz)]\right| \geq t \right) \leq 2 \exp \left\{ - t^2 / \tau_x^2 \right\}.
    \end{equation}
    For example, the convex concentration property is verified by random vectors $\bz$ that satisfy a log-Sobolev inequality or that have strongly log-concave probability density. See discussion in \cite{cheng2022dimension}.
\end{itemize}

These two examples satisfy the Hanson-Wright inequality \cite{rudelson13hanson,adamczak2015note} which implies Assumption \ref{ass:concentrated} with $\beta = 1$. This is the  condition assumed in \cite{cheng2022dimension} to study ridge regression with infinite-dimensional features. The case $\beta >1$ was considered in \cite{louart2018concentration} where they prove that a similar inequality to Eq.~\eqref{eq:ass_concentrated_features} is implied by a relaxation of the convex concentration property \eqref{eq:convex_concentration} with $t^2$ replaced by $t^{2/\beta}$. Using this property, \cite{louart2018concentration} derive non-asymptotic bounds on deterministic equivalents for random matrix functionals with $p =O(n)$.

Observe that Equation \eqref{eq:ass_concentrated_features} applied to $\bA = \bv \bv^\sT$ implies Equation \eqref{eq:ass_quad_concentration_2} (with a slight reparametrization of the constants $\sfc_x,\sfC_x$). Thus under Assumption \ref{ass:concentrated}, the feature vector $\bx$ directly satisfy the general Assumption \ref{ass:main_assumptions} with $\evn = p$.
Applying Theorem \ref{thm:abstract_Test_error} to this setting, we obtain the following simple approximation guarantee, where we recall that $\nu_\lambda (n) := \nu_{\lambda,p} (n)$ is defined in Eq.~\eqref{eq:reduced_nu_lambda}.

\begin{corollary}[Test error with concentrated features]\label{cor:well_concentrated_Test_error}
Under Assumptions \ref{ass:noise_subGaussian} and \ref{ass:concentrated}, for any $D,K>0$, there exist constants $\eta := \eta_x \in(0,1/2)$, $C_{K,D}>0$, and $C_{x,\eps,D,K}>0$ such that the following holds. For any $n \geq C_{K,D}$, regularization parameter $\lambda > 0$, and target function $f_* \in L^2(\cU)$ with parameters $\| \bbeta _* \|_2 < \infty$, if it holds that
\[
\lambda \cdot \nu_{\lambda} (n) \geq n^{-K},
\qquad \qquad
 \nu_{\lambda} (n)^{8} \log^{3\beta + \frac{1}{2}} (n) \leq K \sqrt{n},
\]
then with probability at least $1 - n^{-D}$, we have
\begin{equation}\label{eq:approx_concentrated_test}
\left|\cR_{\test} ( \bbeta_*;\bX,\beps, \lambda) -    \sR_{n} (\bbeta_*, \lambda) \right| \leq C_{x,\eps,D,K} \frac{\nu_{\lambda}(n)^{6} \log^{3\beta +1/2}(n)}{\sqrt{n}}  \sR_{n} (\bbeta_*, \lambda).
\end{equation}
\end{corollary}
For the reader's convenience, we provide a self-contained proof of this corollary in Section \ref{sec_outline:proof_concentrated}.

From Equation~\eqref{eq:approx_concentrated_test}, we see that the rate of approximation scales as $\widetilde{O} (n^{-1/2})$ for concentrated feature vectors. This matches the optimal rate expected from the \textit{local law} fluctuations for the resolvent \cite{alex2014isotropic,knowles2017anisotropic} (see Remark \ref{rmk:local_average_law} below). Furthermore, in contrast to \cite{cheng2022dimension}, the rate $\widetilde{O} (n^{-1/2})$ is achieved for \textit{any} target function without any dependency on $\bbeta_*$. In particular, the approximation guarantee \eqref{eq:approx_concentrated_test} applies to all square-integrable functions $f_* \in L^2(\cU)$.

Theorem \ref{thm:abstract_Test_error} allows to prove a more general guarantee by setting $\evn < p$. Indeed, for any integer $\evn \in \naturals$ chosen with $r_{\lambda} (\evn) \geq 2 n$, Assumption \ref{ass:concentrated} implies Assumption \ref{ass:main_assumptions}.(b) with 
\[
p_{2,n} (\evn) = n^{-D}, \qquad \quad \varphi_{2,n} (\evn) = C_{x,D} \log^{2\beta +1} (n),
\]
for any constant $D>0$. This is proved in Appendix \ref{app:proof_concentrated_high_deg}. Moreover, Assumption \ref{ass:main_assumptions}.(c) directly follows by taking $\bA = \bSigma_{>\evn}^{-1/2} \bbeta_{*,>\evn} \bbeta_{*,>\evn}^\sT \bSigma_{>\evn}^{-1/2}$ in Equation \eqref{eq:ass_concentrated_features}. For the reader's convenience, we write a second corollary for the case $\evn < \infty$:

\begin{corollary}\label{cor:well_concentrated_interpolating}
 Under Assumptions \ref{ass:noise_subGaussian} and \ref{ass:concentrated}, for any $D,K>0$, there exist constants $\eta := \eta_x \in(0,1/2)$, $C_{K,D}>0$, and $C_{x,\eps,D,K}>0$ such that the following holds. For any $n \geq C_{K,D}$, regularization parameter $\lambda > 0$, target function $f_* \in L^2(\cU)$ with parameters $\| \bbeta _* \|_2 < \infty$, and integer $\evn \in \naturals$, if it holds that 
\[
\lambda_{>\evn} \cdot \nu_{\lambda,\evn} (n) \geq n^{-K}, \qquad \log^{3\beta + 1} (n) \sqrt{\frac{n}{r_\lambda (\evn)}} \leq \frac{1}{2},
\qquad
 \nu_{\lambda,\evn} (n)^{8} \log^{3\beta + 1} (n) \leq K \sqrt{n},
\]
then with probability at least $1 - n^{-D}$, we have
\[
\left|\cR_{\test} ( \bbeta_*;\bX,\beps, \lambda) -    \sR_{n} (\bbeta_*, \lambda) \right| \leq C_{x,\eps,D,K}  \cdot \nu_{\lambda,\evn}(n) \log^{3\beta +1}(n) \left\{ \frac{\nu_{\lambda,\evn}(n)^{5} }{\sqrt{n}}  + \sqrt{\frac{n}{r_\lambda (\evn)}} \right\} \cdot\sR_{n} (\bbeta_*, \lambda).
\]   
\end{corollary}

Corollary \ref{cor:well_concentrated_interpolating} allows to provide approximation guarantees for the interpolating solution $\lambda =0$ as soon as $\evn$ is chosen such that $\Tr( \bSigma_{\geq \evn}) = \Omega (1)$.
In particular, if $\Tr (\bSigma_{>n^2}) = \Omega ( \Tr(\bSigma))$, Corollary \ref{cor:well_concentrated_interpolating} shows that the rate of approximation is $\Tilde{O}(n^{-1/2})$ for any $\lambda \geq 0$ by taking $\evn = n^2$.

\begin{figure}[t]
\centering
\includegraphics[width=0.75\textwidth]{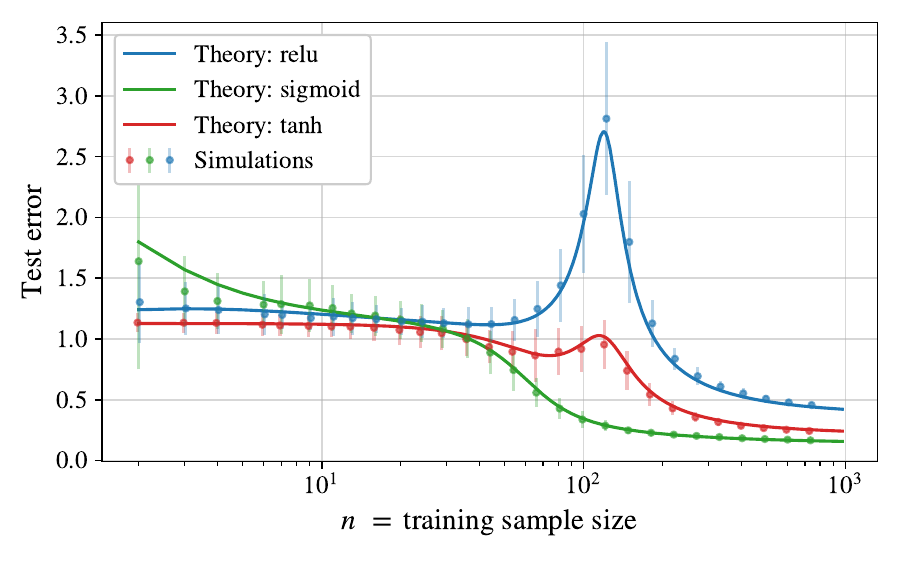}
\vspace{-10pt}
\caption{Test error of ridge regression with features $\bx = \sigma (\bW \bu) \in \R^p$ with covariates $\bu \sim \normal (0,\id_d)$, where the weight matrix $\bW = [\bw_1 , \ldots , \bw_p]^\sT \in \R^{p \times d}$ is fixed and the activation is chosen $\sigma \in \{ \text{ReLu}, \text{sigmoid}, \text{tanh}\}$. We take $\bw_j \sim_{i.i.d.} \normal(0,\id_d/d)$ and target function $f_* (\bu) = \frac{1}{\sqrt{2}}\< \be, \bu \> + \frac{1}{2} \left( \< \be , \bu \>^2 - 1 \right)$ with $\be \in \R^d$, $\| \be \|_2 = 1$ chosen arbitrarily. Here we set $d = 60$, $p = 120$, $\lambda = 0.1$, and $\sigma_\eps^2 =0$.
For the empirical test errors (markers), we report the average and the standard deviation of the test error over $50$ independent realizations. \label{fig:concentrated_test}}
\end{figure}

We illustrate Corollary \ref{cor:well_concentrated_Test_error} in Figure \ref{fig:concentrated_test}. We consider a setting where the data is given by $(\bu_i,y_i)_{i \in [n]}$ with $\bu_i \sim \normal (0,\id_d)$ and $y_i = f_*(\bu_i) + \eps_i$, for some target function $f_*$ verifying Eq.~\eqref{eq:equation_c1_assumption} (e.g., $f_*$ is Lipschitz or a polynomial).
Fix a weight matrix $\bW = [ \bw_1 , \ldots , \bw_p]^\sT \in \R^{p \times d}$. We consider fitting this data using a two-layer neural network 
\[
\hat f (\bu)= \sum_{j \in [p]} \theta_j \sigma ( \< \bw_j , \bu \>) = \btheta^\sT \sigma (\bW \bu),
\]
where we only train the second layer weights $\btheta \in \R^p$ using ridge regression with regularization parameter $\lambda$. Following our notations, this corresponds to KRR with a feature map $\bx = \sigma (\bW\bu) \in \R^p$ which has second moment matrix $\bSigma = \E[\bx \bx^\sT] = (H (\bw_i, \bw_j) )_{ij \in [p]}$ where
\begin{equation}\label{eq:def_H_example}
H (\bw_i, \bw_j) := \E_{\bu \sim \normal(0,\id_d)} [\sigma( \< \bw_i , \bu\>) \sigma (\<\bw_j , \bu\>)].
\end{equation}
We assume that there exists a constant $\sfC>0$ such that $\sigma : \R \to \R$ is $\sfC$-Lipschitz and $\| \bW \|_\op \leq \sfC$. Then, $\bx$ satisfies the Lipschitz convex concentration property \eqref{eq:convex_concentration} with a constant that only depends on $\sfC$ \cite{couillet2018random}. Therefore, we can apply Corollary \ref{cor:well_concentrated_Test_error} (note that the part of $f_*$ uncorrelated to $\bx$ can be taken care of by following Remark \ref{rmk:uncorrelated_part_target}), and the test error is well approximated by the deterministic equivalent
\[
\sR_n (f_*, \lambda) = \frac{\< \bbeta_*, (\bSigma + \lambda_* )^{-2} \bbeta_* \> + \sigma_\eps^2 + \| f_* \|_{L^2}^2 - \| \bbeta_*\|_2^2}{1 - \frac{1}{n}\Tr(\bSigma^2(\bSigma +\lambda_*)^{-2})},
\]
where $\bSigma \in \R^{p \times p}$ has entries given by Eq.~\eqref{eq:def_H_example} and $\bbeta_* = \bSigma^{-1/2} \E_\bu[\sigma(\bW \bu) f_*(\bu)]$.

\begin{remark}[Local and average laws]\label{rmk:local_average_law} In the proof of Corollary \ref{cor:well_concentrated_Test_error}, we show approximation rates $\widetilde{O}(n^{-1/2})$ for both the bias and variance terms \eqref{eq:bias_variance_intro}. While we expect this rate to be optimal for the bias term, \cite{cheng2022dimension} proves a faster rate $\widetilde{O} (n^{-1})$ for the variance term, which matches the scale of fluctuations expected from the \textit{average law} for the resolvent. This is a well known phenomenon in random matrix theory: linear statistics of random matrices have variance $\widetilde{O} (n^{-1})$, much smaller than $\widetilde{O}(n^{-1/2})$, due to eigenvalue correlations. 

We note that the rate $\widetilde{O}(n^{-1/2})$ is a fundamental limitation of our proof strategy when showing deterministic equivalents (see Remark \ref{rmk:comparison_Cheng_theorems} in Section \ref{sec_outline:det_equiv} and Remark \ref{rmk:interpolation_path} in Appendix \ref{app_det_equiv:preliminaries} for further discussion). On the other hand, our approach is elementary and significantly improves the approximation guarantee for the bias term, which was the primary motivation. Combining Corollary \ref{cor:well_concentrated_Test_error} and the results from \cite{cheng2022dimension}, we thus obtain
\[
\cB (\bbeta_* ; \bX,\lambda) = (1 + \widetilde{O} (n^{-1/2}) ) \cdot \sB_n(\bbeta_*,\lambda), \qquad\quad \cV (\bX,\lambda) = (1 + \widetilde{O} (n^{-1})) \cdot \sV_n (\lambda),
\]
for any target function $\| \bbeta_* \|_2 < \infty$.
\end{remark}

\subsection{Inner-product kernels on the sphere}
\label{sec:inner-product_main}

Denote $\S^{d-1} (\sqrt{d}) := \{ \bu \in \R^d: \| \bu \|_2 = \sqrt{d}\}$ the $d$-dimensional sphere of radius $\sqrt{d}$, and $\tau_d$ the uniform measure on $\S^{d-1} (\sqrt{d})$. In our second example, we assume that the covariates are distributed according to $\{ \bu_i \}_{i \in [n]} \sim_{i.i.d.} \tau_d$ and that the p.s.d.~kernel $K: \S^{d-1} (\sqrt{d}) \times \S^{d-1} (\sqrt{d}) \to \R$ is an inner-product kernel defined by $K(\bu,\bu') = h (\<\bu,\bu'\> / d)$ for some function $h : [-1,1] \to \R$.

We start by introducing some notations. Let $L^2 ( \S^{d-1} (\sqrt{d})):= L^2 (\S^{d-1} (\sqrt{d}),\tau_d)$ be the space of square integrable functions with respect to $\tau_d$. We denote $\< f,g\>_{L^2} = \E_{\bu \sim \tau_d}[f(\bu) g(\bu)]$ the scalar product and $\| f\|_{L^2}  = \<f , f \>^{1/2}_{L^2}$ the norm in $L^2 ( \S^{d-1} (\sqrt{d}))$. For each $k \in \naturals$, let $V_{d,k}$ be the linear subspace spanned by degree-$k$ polynomials that are orthogonal (with respect to $\< \cdot, \cdot \>_{L^2}$) to all polynomials of degree less or equal to $k-1$. Using these definitions, we have the following orthogonal decomposition
\[
L^2 ( \S^{d-1} (\sqrt{d})) = \bigoplus_{k = 0}^\infty V_{d,k}.
\]
 Denote $B_{d,k} := \dim (V_{d,k})$ and $\proj_k$ the orthogonal projection onto $V_{d,k}$ in $L^2 ( \S^{d-1} (\sqrt{d}))$. For each $k \in \naturals$, we further introduce a set $\{ Y_{ks} \}_{s \in [B_{d,k}]}$ of degree-$k$ spherical harmonics that forms an orthonormal basis of $V_{d,k}$. In particular, by construction
 \[
 \< Y_{ks}, Y_{lr} \>_{L^2} = \delta_{kl}\delta_{sr},
 \]
 and $\{ Y_{ks} \}_{k\geq 0, s\in [B_{d,k}]}$ forms an orthonormal basis of $L^2 ( \S^{d-1} (\sqrt{d}))$.

The inner-product kernel has the following simple eigendecomposition in this basis of spherical harmonics
\[
K ( \bu,\bu') = h(\< \bu , \bu'\> / d) = \sum_{k = 0}^\infty \oxi_k \sum_{s \in [B_{d,k}]} Y_{ks} (\bu) Y_{ks} (\bu'),
\]
where the eigenvalues $\oxi_k$ have multiplicity $B_{d,k}$ and are given explicitly by
\[
\oxi_k =  \frac{1}{\sqrt{B_{d,k}}}\E_{\bu} \left[ h(u_1/\sqrt{d}) Q_k ( u_1/\sqrt{d})\right].
\]
Here $\{ Q_k \}_{k \geq 0}$ denotes the orthonormal basis of Gegenbauer polynomials in $L^2 ( [-1,1], \tau_{d,1})$ where $\tau_{d,1}$ is the marginal distribution of $u_1/\sqrt{d}$ with $\bu \sim \tau_d$. See Appendix \ref{app_ex:background_spherical} for further background on spherical harmonics and Gegenbauer polynomials.

We will assume that the kernel function $h:[-1,1] \to \R$ satisfy the following conditions:

\begin{assumption}[Inner-product kernel at $L \in \naturals$]\label{ass:inner_product_sphere}
    We assume that there exists a constant $\sfC_L >0$ such that the kernel function $h :[-1,1] \to \R$ satisfy the following.
    \begin{itemize}
        \item[\emph{(a)}] We assume that $h(1) \leq \sfC_L$ and 
        \[
        \oxi_k B_{d,k} \geq \frac{1}{\sfC_{L}}, \qquad\quad \text{for $k=0,\ldots,L$.}
        \]
        \item[\emph{(b)}] We assume that $h_{>L}(1) \geq 1/\sfC_L$. In particular, $h$ is not a degree-$L$ polynomial.

        \item[\emph{(c)}] We assume the moments of the target function have bounded growth: for any integer $q \geq 2$,
        \begin{equation}\label{eq:bound_moment_f_*_spherical}
        \E_\bu \left[ | f_*( \bu)|^q\right] \leq  (\sfC_L q)^{qL/2} \| f_*\|_{L^2}^q.
        \end{equation}
        We further assume that $\| \proj_{>L} f_* \|_{L^2} \geq \| f_* \|_{L^2} /\sfC_L$.
    \end{itemize}
\end{assumption}

For example, Assumption \ref{ass:inner_product_sphere}.(c) is verified for $f_* ( \bx) = g(\<\be, \bx\>)$ with $|g(x)| \leq \sfC_{L} \| g \|_{L^2} (1 + |x|)^L$ where $g$ is not a degree-$L$ polynomial. Below, we will denote $C_L$ constants that only depends on the values of $L,\sfC_L$. Note that the conditions in Assumption \ref{ass:inner_product_sphere} are chosen to simplify the statement of our approximation guarantees and can be relaxed using a more careful analysis.

The effective regularization is given in this setting as the unique non-negative solution of
\begin{equation}\label{eq:eff_reg_Sphere}
\begin{aligned}
    n - \frac{\lambda}{\lambda_*} =&~ \sum_{k=0}^\infty \frac{B_{d,k} \oxi_k}{\oxi_k + \lambda_*}.
\end{aligned}
\end{equation}
In this example, we denote the test error $\cR_{\test} (f_*;\bU,\beps,\lambda)$ to emphasize the dependency on the covariates $\bU := [ \bu_1 , \ldots, \bu_n]^\sT \in \R^{n \times d}$.
Its associated deterministic equivalent is given by
\begin{equation}\label{eq:det_equiv_Sphere}
\sR_n (f_* ,\lambda) =  \frac{1}{1 - \frac{1}{n} \sum_{k=0}^\infty \frac{B_{d,k} \oxi_k^2}{(\oxi_k + \lambda_*)^2}}  \left( \sum_{k=0}^\infty \| \proj_k f_* \|_{L^2}^2 \frac{\lambda_*^2}{(\oxi_k + \lambda_*)^2} + \sigma_\eps^2 \right).
\end{equation}

Applying Theorem \ref{thm:abstract_Test_error} to this setting, we obtain the following approximation guarantee:

\begin{theorem}[Test error for inner-product kernels on the sphere]\label{thm:spherical_test} For any integer $L >0$ and constant $D>0$, and assuming that the target function $f_* \in L^2 (\S^{d-1} (\sqrt{d}))$ and the inner-product kernel $h:[-1,1] \to \R$ satisfy Assumption \ref{ass:inner_product_sphere} at $L$, there exist constants $C_{L,D}>0$ and $C_{L,\eps,D}$ such that the following holds. For any integers $d \geq C_{L,D}$ and $C_{L,D} \leq n \leq d^L$, and defining $\ell\geq 1$ to be the closest positive integer to $\log(n) / \log(d)$, we have with probability at least $1 - n^{-D}$, 
    \[
    \left| \cR_{\test} ( f_* ; \bU,\beps, \lambda) - \sR_n (f_*,\lambda) \right| \leq C_{L,\eps,D} \cdot \log^{3L +7/2} (n\wedge d) \left\{ \sqrt{\frac{d^{\ell -1}}{n}} + \sqrt{\frac{n}{d^{\ell+1}}} \right\}\cdot \sR_n (f_*,\lambda).
    \]
\end{theorem}

The proof of this theorem can be found in Appendix \ref{app:proof_inner_test}. Observe that the approximation rate is always better than $\widetilde{O} (d^{-1/4})$ (for $n \geq d^{1/4}$), with rate $\widetilde{O} (d^{-1/2})$ when $n = d^\ell$.

The KRR test error with inner-product kernel on the sphere was precisely characterized in the high-dimensional polynomial scaling \cite{ghorbani2021linearized,xiao2022precise} where $n,d \to \infty$ with $n/d^\kappa \to \gamma$ for $\kappa,\gamma \in \R_{>0}^2$. We recover these asymptotics by taking the polynomial scaling limit in Eqs.~\eqref{eq:eff_reg_Sphere} and \eqref{eq:det_equiv_Sphere}. However, Theorem \ref{thm:spherical_test} improves over these previous results in two significant ways:
\begin{itemize}
    \item The deterministic equivalent \eqref{eq:det_equiv_Sphere} depends explicitly on finite $n,d$ and resolves the ambiguity of the polynomial scaling. For instance, taking $d = 100$ and $n =1000$, it is unclear which polynomial scaling prediction to use, e.g., the one at $n = 10 d$, $n = d^2/10$, or $n = d^{3/2}$.

    \item The approximation guarantee is \textit{multiplicative} instead of \textit{additive}, with an explicit convergence rate. Furthermore, it holds for a fixed target function $f_* \in L^2(\S^{d-1} (\sqrt{d}))$ without randomizing its coefficients \cite{xiao2022precise} or assuming $f_*$ to be a ridge function \cite{hu2022sharp}.
\end{itemize}

We illustrate Theorem \ref{thm:spherical_test} in Figure \ref{fig:sphere_test}. We consider a similar experimental setting as the one considered in \cite[Figure 4]{xiao2022precise}. The data distribution $(\bu_i,y_i)$ is taken to be $\bu_i \sim \Unif (\S^{d-1} (\sqrt{d}))$ with $d= 24$ and $y_i = f_* (\bu_i) + \eps_i$ with $\eps_i \sim \normal (0,\sigma_\eps^2)$ and $\sigma_\eps^2 = 0.1$. We fix the target function to be
\[
f_* (\bu) = \sum_{k = 1}^7 C_{d,k} \sum_{j \in [d]} \prod_{s = j}^{j+k -1} u_s,
\]
where we use the cyclic convention $u_{d+i} = u_i$ and we choose the coefficients $C_{d,k}$ such that $\| \proj_k f_* \|_{L^2}^2 = C_{d,k}^2 d \cdot\E[ u_1^2 \cdots u_k^2 ] = k^{-2}$. To fit this data, we consider KRR with regularization parameter $\lambda =0$ (the `ridge-less' minimum-norm interpolating solution) and inner-product kernel
\[
h (\< \bu , \bu'\> /d) = \sum_{k=1}^7 \oxi_k \sqrt{B_{d,k}} \cdot Q_k ( \< \bu , \bu'\> /d),
\]
where $\oxi_k = \text{Gap}^{-(k-1)}$. We choose $\text{Gap} \in \{8,32,128\}$. Figure \ref{fig:sphere_test} reports the theoretical predictions from the deterministic equivalent \eqref{eq:det_equiv_Sphere} (continuous lines) and the empirical test errors (markers) obtained by solving KRR on $n$ i.i.d.~data points, for $n$ equally spaced (in log-scale) between $n=2$ and $n= 20000$. We plot the average and the standard deviation of the empirical test errors over $50$ independent realizations. 

Figure \ref{fig:sphere_test} shows that the empirical test errors agrees with the deterministic equivalent predictions \eqref{eq:det_equiv_Sphere} for a wide range of $n$, moderate $d$, and for learning curves displaying vastly different behavior (monotonically decreasing for $\text{Gap} = 8$, or with complex multiple descents for $\text{Gap} = 128$). In particular, the average empirical test error is well approximated by the deterministic equivalent predictions even for $n =2$, while the standard deviation decreases rapidly with $n$, becoming smaller than $15\%$ of the mean as soon as $n \geq n_0 \in \{2,54,66\}$ for $\text{Gap} \in \{8,64, 128\}$ respectively.

\section{Uniform consistency of the GCV estimator}
\label{sec_main:GCV}

The deterministic approximation to the test error studied in Section \ref{sec_main:test_error} crucially depends on having access to the eigendecomposition of the kernel, which is generically intractable. Furthermore, we would like in practice to estimate the relevant statistics of our predictive model \textit{from data}. As a concrete application of the deterministic approximations developed in Section \ref{sec_main:test_error}, we study below the generalized cross-validation (GCV) estimator. In short, we show that under the abstract Assumption \ref{ass:main_assumptions}, the GCV estimator can be used to efficiently estimate the test error and optimally tune the regularization parameter $\lambda$, with explicit non-asymptotic bounds.

First, we briefly recall the definition of the cross-validation estimator. Consider $\hat f_\lambda$ the KRR predictor trained on $n$ data points $(\bu_i,y_i)_{i \in [n]}$ as defined in problem~\eqref{eq:KRR_problem_RKHS}. For each $i\in [n]$, we introduce $\hf^{-i}_\lambda$ the KRR predictor trained on all but the $i$-th sample $(\bu_i , y_i)$. The leave-one-out cross-validation (CV) estimator is then given by
\[
\widehat{\CV}_n (\lambda) = \frac{1}{n } \sum_{i = 1}^n \big( y_i - \hf_\lambda^{-i} (\bu_i)  \big)^2 ,
\]
which can be used as an estimate of the out-of-sample prediction error $\cR_\test (\hat f_\lambda , \P)$. 

For linear models, the leave-one-out predictors $\hf^{-i}_\lambda$ do not need to be computed for each $i \in [n]$, which would be computationally expensive. Instead, $\widehat{\CV}_n (\lambda) $ can be rewritten as a weighted average of the training errors, which is known as the ``shortcut formula''. Denote $\hat\boldf_\lambda = [ \hf_{\lambda} (\bu_1) , \ldots , \hf_{\lambda} (\bu_n ) ]$ the KRR predictor evaluated on the $n$ training data points and observe that $\hat\boldf_\lambda$ can be written as a linear transformation of the $n$ labels $\by = (y_1, \ldots , y_n)$ with
\[
\hat\boldf_\lambda = \bS_{\lambda} \by \, , \qquad\qquad \bS_{\lambda } := \bX \bX^\sT  (\bX \bX^\sT + \lambda)^{-1}  .
\]
Then, a simple application of the Sherman-Morrison-Woodbury formula yields the identity
\[
\widehat{\CV}_n (\lambda) = \frac{1}{n} \sum_{i =1}^n \left( \frac{y_i - \hf_{\lambda} (\bx_i) }{1 - (\bS_{\lambda})_{ii}} \right)^2 \, .
\]

An alternative to the CV estimator was introduced in \cite{craven1978smoothing,golub1979generalized} using the following heuristic approximation: observing that the training samples are exchangeable, we replace $(\bS_{\lambda})_{ii}$ by their average $\Tr(\bS_\lambda)/n$ in the expression of $\widehat{\CV}_n (\lambda) $. The resulting estimator, known as the \textit{generalized cross-validation} (GCV) estimator, reads
\begin{equation}\label{eq:expression_GCV}
\widehat{\GCV}_\lambda (\bK,\by) = \frac{1}{n} \frac{\| \by - \hat\boldf_\lambda  \|_2^2}{(1 - \Tr ( \bS_{\lambda} )/n)^2} = n \frac{\by^\sT ( \bK + \lambda \id_n )^{-2} \by}{\Tr ( (\bK + \lambda)^{-1} )^2} ,
\end{equation}
where we denoted $\bK := \bX \bX^\sT = (K(\bu_i,\bu_j) )_{ij \in [n]}$ the empirical kernel matrix. In words, GCV estimates the test error using the train error normalized by the square of the Stieltjes transform of the empirical kernel matrix. As observed in \cite{hastie2022surprises,wei2022more}, the right most expression in Eq.~\eqref{eq:expression_GCV} is well defined when taking $\lambda \to 0$ (train error equal to $0$) as long as $\bK$ is full rank, and the GCV estimator apply to the min-norm interpolating solution.

\begin{figure}[t]
\centering
\includegraphics[width=0.9\textwidth]{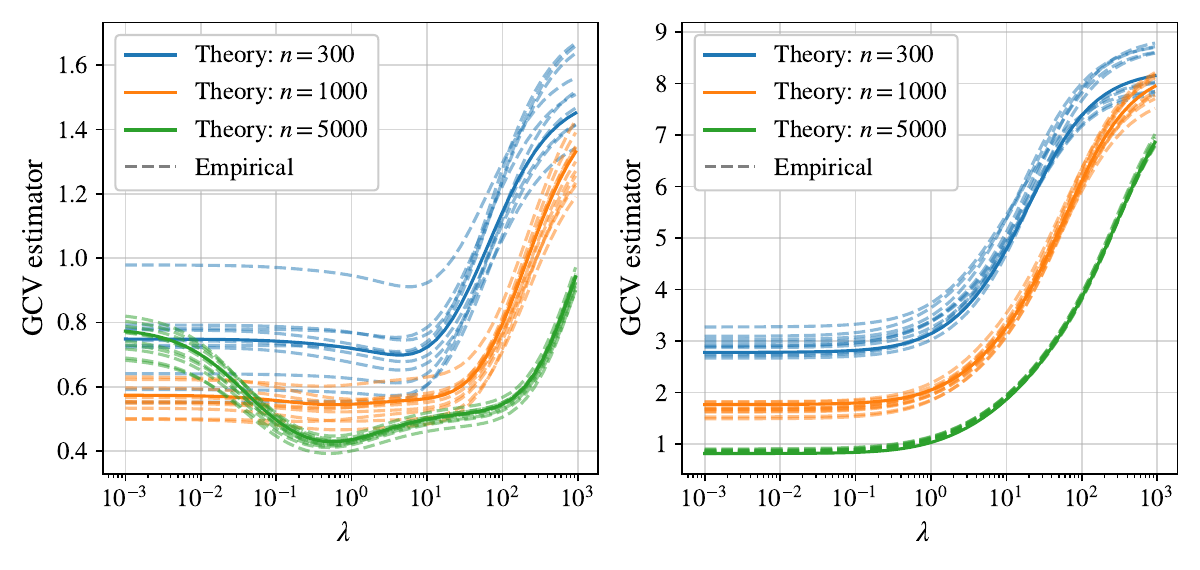}
\vspace{-10pt}
\caption{Predictions of the GCV estimator compared to its deterministic equivalent plotted against the regularization parameter $\lambda \in [10^{-3}, 10^3]$ and sample size $n \in \{300,1000,5000\}$. 
The setting for the synthetic data (left) is the same as Figure~\ref{fig:sphere_test} with gap equal to $128$, while the setting for the real data (right) is the same as in Figure~\ref{fig:real_test} with the NTK kernel of depth 5. The continuous lines correspond to the theoretical predictions as computed by the deterministic equivalents, while each of the 20 dashed lines corresponds to the GCV estimator computed from a sample of size $n$.
\label{fig:GCV}}
\end{figure}

As an application of the deterministic approximations developed in Section \ref{sec_main:test_error}, we show that under the abstract setting of Section \ref{sec:assumptions_test_error}, the GCV estimator uniformly approximates the test error over the choice of the ridge parameter $\lambda \in [0, n^K]$.

\begin{theorem}[Uniform consistency of the GCV estimator]\label{thm:abstract_GCV}
Consider $D,K>0$, integer $n$, a maximum regularization parameter $\lambda_{\max} \leq n^K$, and target function $f_* \in L^2 (\cU)$ with parameters $\| \bbeta_*\|_2 = \| f_* \|_{L^2} < \infty$. Assume that the features $\{\bx_i\}_{i\in[n]}$ and $f_*$ satisfy Assumption \ref{ass:main_assumptions} with $\lambda= 0$ and some $\evn := \evn (n) \in \naturals \cup \{\infty\}$, and the $\{\eps_i\}_{i\in [n]}$ satisfy Assumption \ref{ass:noise_subGaussian}. There exist constants $\eta := \eta_x \in(0,1/2)$, $C_{D,K} >0$, $C_{x,\eps,D,K}>0$, and $\widetilde{C}_{x,\eps,D,K} >0$ such that, if it holds that $n \geq C_{D,K}$ and
\begin{align}\label{eq:condition_GCV_abstract1}
 \Tr(\bSigma_{>\evn})  \geq n^{-K},\qquad \qquad C_{x,\eps,D,K}\cdot \cE_{\sR,n} (\evn) \leq 1.
\end{align}
where $\cE_{\sR,n} (\evn)$ is defined in Eq.~\eqref{eq:asbstract_relative_error}, then with probability at least $1 - n^{-D} - p_{2,n} (\evn)$, we have
    \begin{equation}\label{eq:abstract_GCV_error}
    \sup_{\lambda \in [0, \lambda_{\max}]}\left| \frac{\widehat{\GCV}_\lambda (\bK,\by)}{\cR_{\test} (\bbeta_*;\bX,\beps,\lambda)} - 1 \right| \leq \widetilde{C}_{x,\eps,K,D} \cdot \cE_{\sG,n} (\evn),
    \end{equation}
    where 
    \begin{equation}\label{eq:asbstract_relative_error_GCV}
    \cE_{\sG,n}  (\evn) :=  \frac{\varphi_1(\evn) \nu_{0,{\evn}}(n)^{8} \log^{3\beta +1/2}(n)}{\sqrt{n}}    +  \nu_{0,\evn} (n)  \varphi_{2,n} (\evn)  \sqrt{\frac{n}{r_0 (\evn)}}  .
\end{equation}
\end{theorem}

Compared to previous works that have studied the uniform consistency of the GCV estimator \cite{xu2019consistent,hastie2022surprises,patil2021uniform}, the guarantee \eqref{eq:abstract_GCV_error} displays the following favorable properties:
\begin{itemize}
    \item It is \textit{non-asymptotic} and provides guarantees for the GCV estimator for moderate number of samples. Furthermore, it applies beyond the standard linear regression setting, e.g., infinite-dimensional features (Section \ref{sec:example_concentrated}) or inner-product kernels on the sphere (Section \ref{sec:inner-product_main}, with $\cE_{\sG,n} (\evn)$ same rate as in Theorem \ref{thm:spherical_test}).

    \item The guarantee is \textit{scale-free} and does not depend on the magnitude of $\cR_\test$. Note that this is stronger than showing $\widehat{\GCV}_\lambda - \cR_\test \to 0$ as this does not suffice to argue that $\widehat{\GCV}_\lambda$ recovers the scale of $\cR_\test$ whenever $\cR_\test$ is itself vanishing, e.g., $\cR_\test \approx n^{-\gamma}$.
\end{itemize}

As a consequence of Theorem \ref{thm:abstract_GCV}, we can use the GCV estimator to estimate the optimally-tuned ridge regularization, i.e., the parameter $\lambda_{\opt} \geq 0$ that minimizes the test error
\[
\lambda_{\opt} =  \argmin_{\lambda \geq 0} \;\;\cR_{\test} (\bbeta_* ; \bX, \beps, \lambda).
\]
Consider $\hat\lambda_{\GCV}$ that minimizes the GCV estimator
\[
\begin{aligned}
\widehat\lambda_{\GCV} =&~ \argmin_{\lambda \geq 0} \;\;\widehat{\GCV}_\lambda (\bK,\by).
\end{aligned}
\]
Then, under mild assumptions on the model, we have with the same probability as Eq.~\eqref{eq:abstract_GCV_error} that
\[
\left| \cR_{\test} (\bbeta_* ; \bX, \beps, \widehat\lambda_{\GCV}) -  \cR_{\test} (\bbeta_* ; \bX, \beps, \lambda_{\opt} ) \right| \leq \widetilde{C}_{x,\eps,K,D} \cdot \cE_{\sG,n} (\evn) \cdot \cR_{\test} (\bbeta_* ; \bX, \beps, \lambda_{\opt} ).
\]

The proof of Theorem \ref{thm:abstract_GCV} proceeds by comparing the GCV estimator and the test error to their common deterministic equivalent and can be found in Appendix \ref{app_test:GCV}.

As long as $\Tr(\bSigma_{>\evn})$ is large enough, Theorem \ref{thm:abstract_GCV} provides guarantees for the interpolating solution $\lambda =0$ (or even negative regularization as remarked in \cite{patil2021uniform}). If it is not the case, e.g., overfitting is not benign, then we can replace the interval $\lambda \in [0,\lambda_{\max}]$ by $\lambda \in [ \lambda_{\min},\lambda_{\max}]$ with $\lambda_{\min} >0$. For concreteness, we write a separate corollary in this case, which we specialize to the example of concentrated features.

\begin{corollary}[GCV estimator with concentrated features]\label{thm:well_concentrated_GCV} Under Assumptions \ref{ass:noise_subGaussian} and \ref{ass:concentrated}, for any $D,K>0$, there exist constants $\eta := \eta_x \in(0,1/2)$, $C_{K,D}>0$, $C_{x,\eps,D,K}>0$, and $\widetilde{C}_{x,\eps,D,K} >0$ such that the following holds. For any $n \geq C_{K,D}$, regularization parameter range $0 < \lambda_{\min} \leq \lambda_{\max} \leq n^K$, and $f_* \in L^2(\cU)$, if it holds that
\[
\lambda_{\min} \cdot \nu_{\lambda_{\min}} (n) \geq n^{-K},
\qquad \qquad
 C_{x,\eps,K,D} \cdot \nu_{\lambda_{\min}} (n)^{8} \log^{3\beta + \frac{1}{2}} (n) \leq \sqrt{n},
\]
then with probability at least $1 - n^{-D}$, we have
\[
\sup_{\lambda \in [\lambda_{\min}, \lambda_{\max}]}\left| \frac{\widehat{\GCV}_\lambda (\bK,\by)}{\cR_{\test} (\bbeta_*;\bX,\beps,\lambda)} - 1 \right| \leq \widetilde{C}_{x,\eps,K,D} \frac{\nu_{\lambda_{\min}}(n)^{8} \log^{3\beta +1/2}(n)}{\sqrt{n}}.
\]
\end{corollary}

We illustrate Theorem \ref{thm:abstract_GCV} in Figure \ref{fig:GCV}. The left plot corresponds to the setting of Figure \ref{fig:sphere_test} and the right plot to the setting of Figure \ref{fig:real_test} (see caption of the figure for details).

\section{Outline of the proofs}
\label{sec:outline_proofs}

In this section, we outline the proof strategy for the main results in this paper. To prove our results on the GCV estimator, we will establish approximation guarantees on the training error and the Stieltjes transform of the empirical kernel matrix similar to Theorem \ref{thm:abstract_Test_error}. We denote the Stietljes transform and the training error by
\begin{align}
    \label{eq:def_Stieltjes_main}
 s_n (\bX, \lambda) :=&~ \frac{1}{n}  \Tr \big[ ( \bX \bX^\sT + \lambda)^{-1} \big], \\
 \cL_{\train} ( \bbeta_*;\bX,\beps, \lambda) =&~ \frac{1}{n} \sum_{i \in [n]} \big(y_i - \hat f_\lambda (\bx_i)\big)^2. \label{eq:def_training_main}
\end{align}
Their associated deterministic equivalents are given by
\begin{align}
  \sfs_n (\lambda) :=&~ \frac{1}{n \lambda_*}, \label{eq:def_equiv_Stieltjes_main}\\
    \sL_n (\bbeta_* , \lambda) := &~   \left( \frac{\lambda}{n \lambda_*}\right)^2 \cdot \frac{\lambda_*^2 \< \bbeta_*, (\bSigma + \lambda_*)^{-2} \bbeta_*\> +\sigma_\eps^2}{1 - \frac{1}{n} \Tr( \bSigma^2 (\bSigma + \lambda_*)^{-2})}.\label{eq:def_equiv_training_main}
\end{align}
From these expressions, we can directly observe that
\[
\frac{ \sL_n (\bbeta_* , \lambda)}{\lambda^2\sfs_n (\lambda)^2} = \sR_n (\bbeta_* , \lambda),
\]
and the test error and the GCV estimator have the same deterministic equivalents. Thus, Theorem \ref{thm:abstract_GCV} immediately follows from approximation guarantees between the training error and Stieltjes transform and their deterministic equivalents, using a standard union bound over an $\eps$-net of $\lambda \in [0 , \lambda_{\max}]$. The details of this proof can be found in Appendix \ref{app_test:GCV}.

The main technical results underlying our proofs are deterministic equivalents for functionals of the low-degree feature matrix. We describe and discuss these results in Section \ref{sec_outline:det_equiv}. We then outline the general strategy to establish our deterministic approximations in Section \ref{sec_outline:proof_strategy}. Finally, Section \ref{sec_outline:proof_concentrated} provides a self-contained proof in the case of concentrated features (Assumption \ref{ass:concentrated}). The complete proofs of Theorem \ref{thm:abstract_Test_error} and Theorem \ref{thm:abstract_GCV} can be found in Appendix \ref{app:main_proofs}.

\subsection{Deterministic equivalents for feature matrix functionals}
\label{sec_outline:det_equiv}

Consider a feature matrix $\bX = [ \bx_1 , \ldots , \bx_n]^\sT \in \R^{n \times p}$, where $p \in \naturals \cup \{ \infty\}$. For the sake of readability, we will assume that $\bX$ directly corresponds to the low-degree part of the feature matrix, i.e., the feature vectors $\bx_i$'s satisfy Assumption \ref{ass:main_assumptions}.(a) with $\evn = p$ and constants $\sfc_x,\sfC_x,\beta,\varphi_1 (p)>0$. Recall that $\lambda_*$ denote the effective regularization associated to $(n,\bSigma,\lambda)$ (Definition \ref{def:effective_regularization}) and $\nu_{\lambda} (n) := \nu_{\lambda,p} (n)$ is defined as per Eq.~\eqref{eq:reduced_nu_lambda}. 

The test error, training error and Stieltjes transform are functionals of the feature matrix $\bX$ and in particular of the resolvent matrix
\[
\bR := (\bX^\sT \bX + \lambda )^{-1} \in \R^{p \times p}.
\]
As we allow for infinite-dimensional features $p = \infty$ and covariance matrices $\bSigma$ without bounded conditioning number, a key role will be played in our analysis by the following rescaled resolvent
\[
\bM :=  \bSigma^{1/2} (\bX^\sT \bX + \lambda )^{-1} \bSigma^{1/2} \in \R^{p \times p}.
\]
Note that $\Tr( \bM) \leq \Tr(\bSigma)/\lambda < \infty$ by the trace-class assumption on the kernel. In our regime, we show that $\bM$ behaves effectively as a deterministic matrix
\[
\obM := \bSigma ( \mu_* \bSigma + \lambda)^{-1},
\]
where we defined $\mu_* := \lambda / \lambda_*$.

While our analysis can be applied more generally, we will focus on four functionals that are relevant to the purpose of the paper. For a general p.s.d.~matrix $\bA \in \R^{p\times p}$, define
\[
\begin{aligned}
    \Phi_1(\bX;\bA) := &~ \Tr \left(\bA \bSigma^{1/2} (\bX^\sT \bX + \lambda )^{-1}  \bSigma^{1/2}  \right), \\
    \Phi_2 (\bX) := &~  \Tr \left( \frac{\bX^\sT \bX}{n} ( \bX^\sT \bX + \lambda )^{-1} \right), \\
    \Phi_3(\bX ; \bA) :=&~  \Tr \left(  \bA \bSigma^{1/2} ( \bX^\sT \bX + \lambda)^{-1} \bSigma ( \bX^\sT \bX + \lambda)^{-1} \bSigma^{1/2} \right) ,\\
    \Phi_4 ( \bX ; \bA) := &~ \Tr\left( \bA \bSigma^{1/2} ( \bX^\sT \bX + \lambda)^{-1}\frac{\bX^\sT \bX}{n} ( \bX^\sT \bX + \lambda)^{-1}\bSigma^{1/2} \right)   . 
\end{aligned}
\]
Recalling $\bX =\bZ \bSigma^{1/2}$, note that these functionals only depend on $\bM$ and $\bZ^\sT \bZ$. We show below that we can approximate these functionals by the following deterministic functions
\[
\begin{aligned}
    \Psi_1 ( \mu_*; \bA) :=&~  \Tr \left( \bA \bSigma (\mu_* \bSigma + \lambda )^{-1}  \right) ,\\
     \Psi_2 ( \mu_*) :=&~\frac{1}{n} \Tr\left( \bSigma ( \bSigma + \lambda_* )^{-1} \right) ,\\
      \Psi_3 ( \mu_*; \bA) :=&~ \frac{\Tr( \bA \bSigma^2 ( \mu_* \bSigma + \lambda)^{-2})}{1 - \frac{1}{n} \Tr(\bSigma^2 (\bSigma + \lambda_*)^{-2})} , \\
       \Psi_4 ( \mu_*; \bA) :=&~ \frac{1}{n} \cdot \frac{\Tr(\bA \bSigma^2 ( \bSigma + \lambda_*)^{-2})}{n -  \Tr(\bSigma^2 ( \bSigma + \lambda_* )^{-2} ) }.
\end{aligned}
\]
Without loss of generality, we can assume that $\Tr(\bA \bSigma) <\infty$ for $\Phi_1$, since otherwise $\Phi_1 (\bX,\bA)  = \Psi_1 (\bX) = \infty$ almost surely, and $\Tr(\bA \bSigma^2) <\infty$ for $\Phi_3$ and $\Phi_4$, since otherwise $\Phi_j (\bX;\bA) = \Psi_j (\mu_*; \bA) = \infty$, $j = 3,4$, almost surely.

The following theorem gathers the approximation guarantees for these four functionals.

\begin{theorem}[Dimension-free deterministic equivalents]\label{thm:main_det_equiv_summary}
    Assume the features $(\bx_i)_{i\in[n]}$ satisfy Assumption \ref{ass:main_assumptions}.(a) with some constants $\sfc_x,\sfC_x,\beta,\varphi_1 (p)>0$. For any $D,K>0$, there exist constants $\eta := \eta_x \in (0,1/2)$ (only depending on $\sfc_x,\sfC_x,\beta$), $C_{D,K} >0$ (only depending on $K,D$), and $C_{x,D,K}>0$ (only depending on $\sfc_x,\sfC_x,\beta,D,K$), such that the following holds. For all $n \geq C_{D,K}$ and  $\lambda >0$, if it holds that 
    \begin{equation}\label{eq:conditions_det_equiv_main}
    \lambda \cdot \nu_\lambda (n) \geq n^{-K}, \qquad \varphi_1(p) \nu_\lambda (n)^{5/2} \log^{\beta + \frac{1}{2}} (n) \leq K \sqrt{n} ,
    \end{equation}
    then for any deterministic p.s.d.~matrix $\bA$, we have with probability at least $1 - n^{-D}$ that
    \begin{align}
         \big\vert \Phi_1(\bX;\bA) - \Psi_1 (\mu_* ; \bA) \big\vert \leq&~ C_{x,D,K} \frac{\varphi_1(p) \nu_\lambda (n)^{5/2} \log^{\beta +\frac12} (n) }{\sqrt{n}}     \Psi_1 (\mu_* ; \bA) , \label{eq:det_equiv_phi1_main} \\
         \big\vert \Phi_2(\bX) - \Psi_2 (\mu_* ) \big\vert \leq&~ C_{x,D,K} \frac{\varphi_1(p) \nu_\lambda (n)^{5/2} \log^{\beta +\frac12} (n) }{\sqrt{n}}     \Psi_2 (\mu_*) , \label{eq:det_equiv_phi2_main} \\
         \big\vert \Phi_3(\bX;\bA) - \Psi_3 (\mu_* ; \bA) \big\vert \leq&~ C_{x,D,K} \frac{\varphi_1(p) \nu_\lambda (n)^{6} \log^{2\beta +\frac12} (n) }{\sqrt{n}}     \Psi_3 (\mu_* ; \bA) , \label{eq:det_equiv_phi3_main} \\
         \big\vert \Phi_4(\bX;\bA) - \Psi_4 (\mu_* ; \bA) \big\vert \leq&~ C_{x,D,K} \frac{\varphi_1(p) \nu_\lambda (n)^{6} \log^{\beta +\frac12} (n) }{\sqrt{n}}     \Psi_4 (\mu_* ; \bA) . \label{eq:det_equiv_phi4_main} 
    \end{align}
\end{theorem}

Before describing the proof strategy for this theorem, let us make a few simple comments about the approximation bounds \eqref{eq:det_equiv_phi1_main}--\eqref{eq:det_equiv_phi4_main}:
\begin{enumerate}
    \item These guarantees are \textit{non-asymptotic} and depend explicitly on finite $n$, regularization $\lambda >0$ and fixed feature distribution with covariance $\bSigma$. Furthermore, they are \textit{dimension-free}: they only depend on $p$ through the assumption $\varphi_1(p)$ (e.g., $p = \infty$ and $\varphi_1(p)=1$ in the concentrated feature case) and $\bSigma$ through $\nu_\lambda (n)$ (recall that we expect $\nu_\lambda (n) \lesssim \polylog(n) /\lambda$ for most regularly varying spectrum).

    \item The bounds are \textit{multiplicative} and do not depend on $\bA$: the approximation guarantees hold regardless of the scale of $\Psi_j (\bX;\bA)$. In particular, this allows to study settings where these functionals decay as $n^{-\gamma}$.

    \item Compared to \cite{cheng2022dimension} which required $\| \bA \|_\op <\infty$ to prove a version of Eq.~\eqref{eq:det_equiv_phi1_main}, these approximation guarantees hold without restriction on the p.s.d.~matrix $\bA$. Crucially for our purpose, Equations \eqref{eq:det_equiv_phi3_main} and \eqref{eq:det_equiv_phi4_main} apply to any matrix $\bA = \bSigma^{-1} \bbeta_* \bbeta_*^\sT \bSigma^{-1}$ which satisfy $\Tr(\bA \bSigma^2) = \| \bbeta_* \|_2^2 <\infty$ without imposing $\Tr(\bA) = \| \bSigma^{-1} \bbeta_* \|_2^2 < \infty$.
\end{enumerate}

The proof of Theorem \ref{thm:main_det_equiv_summary} follows a general blueprint to show deterministic equivalents of random matrix functionals \cite{louart2018concentration,couillet2022random}. We split our bound into a martingale part and a deterministic part
\[
\big\vert \Phi(\bX) - \Psi (\mu_*) \big\vert \leq \underbrace{\big\vert \Phi(\bX) - \E[\Phi(\bX)] \big\vert}_{(\sfM)} + \underbrace{\big\vert\E[\Phi(\bX)] - \Psi (\mu_*)\big\vert}_{(\sfD)}.
\]
For the martingale part $(\sfM)$, we construct a martingale difference sequence that interpolates between $\Phi (\bX)$ and $\E[\Phi(\bX)]$ by integrating one more feature at a time. More precisely, denote $\E_{i}$ the expectation over features $\{ \bx_{i+1}, \ldots, \bx_n\}$, for $i = 0, \ldots , n$. Then we can write
\[
\Phi(\bX) - \E[\Phi(\bX)]  = \sum_{i = 1}^n \Big( \E_i - \E_{i-1} \Big)\Phi(\bX)  .
\]
We can then directly bound this term using Azuma-Hoeffding inequality with a standard truncation argument. For the deterministic part, we follow a leave-one-out type computation similar to the one used in \cite[Theorem 2.6]{couillet2022random} which considered $\Phi_1 (\bX;\bA)$ with $p/n$ bounded.

The proof of Theorem \ref{thm:main_det_equiv_summary} and further background on our setting can be found in Appendix \ref{app:det_equiv}.

\begin{remark}[Comparison with {\cite{cheng2022dimension}}]\label{rmk:comparison_Cheng_theorems} A dimension-free approximation guarantee on $\Phi_1(\bX;\bA)$ was previously proven in \cite[Corollary 6.5]{cheng2022dimension} in the case of $\varphi_1(p) =1$, $\beta = 1$, and $\| \bA \|_\op < \infty$, using a different interpolation path (see Remark \ref{rmk:interpolation_path} in Appendix \ref{app_det_equiv:preliminaries}). For all $D>0$, they obtained the following bound with probability $1 - O (n^{-D})$:
\begin{equation}\label{eq:Bound_Cheng}
| \Phi_1 (\bX ; \bA) - \Psi_1 ( \mu_* ; \bA) | \leq  C_{x,D} \frac{\nu_\lambda (n)^3 \log^2(n)}{ \kappa^{6.5} n \sqrt{\rho_1 (\bA)}} \cdot \Psi_1 ( \mu_* ; \bA) ,
\end{equation}
where 
\[
\kappa := \min( \lambda/(n\lambda_*), 1 - \lambda/(n\lambda_*) ) , \qquad \quad \rho_1 (\bA) := \frac{\Psi_1 (\mu_*; \bA/\| \bA \|_\op)}{\Psi_1 (\mu_* ; \id )} .
\]
Note that $\rho_1( \bA) \leq 1$ for all matrices $ \bA $. In contrast, Theorem \ref{thm:main_det_equiv_summary} provides an approximation rate $\widetilde{O} (n^{-1/2})$ uniformly over all $\bA$ and does not impose $\| \bA \|_\op < \infty$.

On one hand, Equation \eqref{eq:Bound_Cheng} achieves approximation rate $\widetilde{O} (n^{-1})$ when $\rho_1 (\bA) = \Omega(1)$. In particular, using that $\rho_1(\id) = 1$, the above bound allows \cite{cheng2022dimension} to establish the tight rate $\widetilde{O}(n^{-1})$ for the variance term \eqref{eq:Cheng_rates} expected from the `average law' fluctuations.  On the other hand, the bound can become much worse when $\rho_1 (\bA)$ is small. In particular, Theorem \ref{thm:main_det_equiv_summary} always improves on Eq.~\eqref{eq:Bound_Cheng} for rank 1 matrices $\bA$, since $\rho_1 (\bu \bu^\sT) \leq C/n$ for any $\bu \in \R^p$. Further, note for example that $\rho_1 (\be_j \be_j^\sT ) \lesssim \xi_j / \lambda \to 0$ as $j \to \infty$ and therefore the bound \eqref{eq:Bound_Cheng} becomes vacuous. 
\end{remark}

\subsection{Proof strategy}
\label{sec_outline:proof_strategy}

Given Theorem \ref{thm:main_det_equiv_summary}, the proofs for the deterministic equivalents of the Stieltjes transform, training error and test error follow from a simple outline. We first simplify the functionals by using concentration over the high-degree features and the label noise:
\begin{itemize}
    \item From Assumption \ref{ass:main_assumptions}.(b), with probability at least $1 - p_{2,n}(\evn)$, we can approximate the resolvent $( \bX \bX^\sT + \lambda)^{-1}$ by $(\bX_{\leq \evn} \bX_{\leq \evn}^\sT + \lambda_{>\evn})^{-1}$ and, recalling condition \eqref{eq:condition_test_abstract}, bound
    \[
   \| \bX_{>\evn} \bSigma_{>\evn} \bX_{>\evn} \|_\op \leq \frac{3}{2}\xi_{\evn+1} \lambda_{>\evn}.
    \]
    \item Using Assumption \ref{ass:main_assumptions}.(c), we show that the high-frequency part of the target function effectively behave as independent additive noise to the label. More precisely, denote $\boldf_{>\evn} = (f_{*,>\evn} (\bx_1),\ldots , f_{*,>\evn} (\bx_n))$ and consider a random matrix $\bB \in \R^{n \times n}$ that only depends on $\bX_{\leq \evn}$, and a fixed vector $\bv \in \R^n$. Then we prove under Assumption \ref{ass:main_assumptions}.(c) and using $\E[f_{*,>\evn} (\bx_i) \bx_{i,\leq \evn} ] = \bzero$, that with high probability
    \[
    \boldf_{>\evn}^\sT  \bB \boldf_{>\evn} = (1+o(1)) \cdot \| f_{*,>\evn} \|_{L^2}^2 \cdot \Tr( \bB) , \qquad \bv^\sT \bB     \boldf_{>\evn} = o(1) \cdot \| f_{*,>\evn} \|_{L^2} \| \bB \bv \|_2.
    \]

    \item Under Assumption \ref{ass:noise_subGaussian}, the random vector $\beps$ satisfy the Hanson-Wright inequality \cite{rudelson13hanson}. Therefore the functionals concentrate on their expectations with respect to $\beps$.
\end{itemize}

Using the above three points, we show that the three quantities of interest concentrates on functionals that only depend on the low-degree features $\bX_{\leq \evn}$, regularization $\lambda_{>\evn}$, parameter $\bbeta_{\leq \evn}$, and the variance $\| f_{*,>\evn} \|_{L^2}^2 + \sigma_\eps^2$. In fact, we can express them in terms of $\Phi_2(\bX_{\leq \evn})$, $\Phi_3(\bX_{\leq \evn};\bA)$, and $\Phi_4 (\bX_{\leq \evn};\bA)$ with regularization parameter $\lambda_{>\evn}$ and $\bA$ equal to either $\id_\evn$, $\bSigma_{\leq \evn}^{-1}$, or $\bSigma_{\leq \evn}^{-1} \bbeta_{\leq \evn}\bbeta_{\leq \evn}^\sT \bSigma_{\leq \evn}^{-1}$. Recalling that $\bX_{\leq \evn}$ satisfy Assumption \ref{ass:main_assumptions}.(a), we can apply Theorem \ref{thm:main_det_equiv_summary} and approximate these functionals by $\Psi_2 (\mu_{*,\evn})$, $\Psi_3 (\mu_{*,\evn};\bA)$, and $\Psi_4 (\mu_{*,\evn};\bA)$ where $\mu_{*,\evn} = \lambda_{>\evn}/\lambda_{*,\evn}$ with $\lambda_{*,\evn}$ the effective regularization associated to $(n,\bSigma_{\leq \evn}, \lambda_{>\evn})$. Finally, we show that these deterministic equivalents, that depend on $\lambda_{*,\evn}$, are close to the original deterministic equivalents in terms of the effective regularization $\lambda_{*}$ associated to $(n,\bSigma,\lambda)$ as soon as $r_{\lambda}(\evn) \gtrsim n$.


\subsection{The case of concentrated feature}
\label{sec_outline:proof_concentrated}

For concreteness, we provide below a simple self-contained proof in the case of concentrated features.

\begin{theorem}\label{thm:well_concentrated_test_train}
Under Assumptions \ref{ass:noise_subGaussian} and \ref{ass:concentrated}, for any $D,K>0$, there exist constants $\eta := \eta_x \in(0,1/2)$, $C_{K,D}>0$, $C_{x,K,D} >0$ and $C_{x,\eps,D,K}>0$ such that the following holds. For any $n \geq C_{K,D}$, regularization $\lambda > 0$, and target function $f_* \in L^2(\cU)$ with parameters $\| \bbeta _* \|_2 < \infty$, if it holds that
\begin{equation}\label{eq:condition1_test_error_app_concentrated}
\lambda \cdot \nu_\lambda (n) \geq n^{-K}, \qquad  \nu_\lambda (n)^{7} \log^{2\beta + \frac{1}{2}} (n) \leq K \sqrt{n} ,
\end{equation}
then with probability at least $1 - n^{-D}$, we have
\begin{align}
     | s_n ( \bX, \lambda)  - \sfs_n (\lambda) | \leq&~  C_{x,K,D} \frac{\nu_\lambda (n)^{7/2} \log^{\beta + \frac{1}{2}} (n)}{\sqrt{n}} \sfs_n (\lambda) ,\label{eq_thm:stielt_det_equiv_well_concentrated} \\
      | \cL_{\train} (\bbeta_* ; \bX, \beps, \lambda) - \sL_{n} ( \bbeta_* , \lambda) | \leq &~ C_{x,\eps,K,D} \frac{\nu_\lambda (n)^{7} \log^{2\beta + \frac{1}{2}} (n)}{\sqrt{n}}   \sL_{n} ( \bbeta_* , \lambda),  \label{eq_thm:train_det_equiv_well_concentrated} \\
        | \cR_{\test} (\bbeta_* ; \bX, \beps, \lambda) - \sR_{n} ( \bbeta_* , \lambda) | \leq &~ C_{x,\eps,K,D} \frac{\nu_\lambda (n)^{7} \log^{2\beta + \frac{1}{2}} (n)}{\sqrt{n}}    \sR_{n} ( \bbeta_* , \lambda).  \label{eq_thm:test_det_equiv_well_concentrated}
\end{align}

\end{theorem}

\begin{proof}[Proof of Theorem \ref{thm:well_concentrated_test_train}]
For convenience, we introduce the following notations:
\[
\begin{aligned}
    \bG :=&~ (\bX \bX^\sT + \lambda )^{-1} \in \R^{n \times n}, \qquad \quad &\bR:= &~ (\bX^\sT \bX + \lambda)^{-1} \in \R^{p \times p}, \\
    \bM:=&~ \bSigma^{1/2} \bR \bSigma^{1/2} \in \R^{p \times p}, \qquad \quad & \obM :=&~ \bSigma ( \mu_* \bSigma + \lambda)^{-1}\in \R^{p \times p}, \\
    \Upsilon_1 :=&~ \frac{1}{n} \Tr( \bSigma (\bSigma + \lambda_*)^{-1}), \qquad \quad & \Upsilon_2 :=&~ \frac{1}{n} \Tr( \bSigma^2 (\bSigma + \lambda_*)^{-2}).
\end{aligned}
\]
Recall that $\Upsilon_1 = 1 - \frac{\lambda}{n\lambda_*} $ by definition of $\lambda_*$. We will further denote $\cE_{j,n}$ the right-hand side rates in Eqs.~\eqref{eq:det_equiv_phi1_main}--\eqref{eq:det_equiv_phi4_main} so that with probability at least $1 -n^{-D}$,
\[
| \Phi_j (\bX ; \bA) - \Psi_j ( \mu_* ; \bA) | \leq \cE_{j,n} \cdot \Psi_j ( \mu_* ; \bA).
\]

\noindent
\textbf{Step 1: Stieltjes transform.}

We can rewrite the Stieljes transform as
\[
\frac{1}{n} \Tr( \bG) = \frac{1}{\lambda} \Big( 1 - \frac{1}{n} \Tr \big(\bX^\sT \bX (\bX^\sT \bX + \lambda )^{-1} \big) \Big).
\]
The right-hand side corresponds to the functional $\Phi_2(\bX)$. Under Assumption \ref{ass:concentrated} and the conditions \eqref{eq:condition1_test_error_app_concentrated}, we can apply Theorem \ref{thm:main_det_equiv_summary} and get with probability at least $1 - n^{-D}$ that
\[
\left| \frac{1}{n} \Tr \big(\bX^\sT \bX (\bX^\sT \bX + \lambda )^{-1} \big) - \left(1 - \frac{\lambda}{n \lambda_*} \right)\right| \leq \cE_{2,n} \cdot \left(1 - \frac{\lambda}{n \lambda_*} \right).
\]
We deduce that with probability at least $1 - n^{-D}$
\[
\left| \frac{1}{n} \Tr( \bG) - \frac{1}{n\lambda_*} \right| \leq \cE_{2,n} \cdot \frac{1}{\lambda}\left(1 - \frac{\lambda}{n \lambda_*} \right) \leq 2 \cE_{2,n} \nu_\lambda (n) \cdot \frac{1}{n\lambda_*},
\]
where we used that $\lambda^{-1} ( 1 - \lambda/(n\lambda_*)) = \Tr(\obM) / (n \lambda_*)$ and $\Tr(\obM) \leq 2 \nu_\lambda (n)$ from Lemma \ref{lem:properties_mu_obM}. Replacing $\cE_{2,n}$ by the expression in Theorem \ref{thm:main_det_equiv_summary} yields Eq.~\eqref{eq_thm:stielt_det_equiv_well_concentrated}.

\noindent
\textbf{Step 2: Training error.}

Using $\by = \boldf + \beps$, we decompose the training error into three contributions
\begin{equation}\label{eq:decomposition_train_well_concentrated}
\frac{1}{n}\| \by - \bX \hbtheta_\lambda \|_2^2 =  \frac{1}{n}\| \by - \bX \bX^\sT \bG \by \|_2^2 =  \frac{1}{n} \lambda^2 \by^\sT \bG^2 \by = \lambda^2 \cdot \left(T_1 + 2T_2 + T_3 \right),
\end{equation}
where we denoted
\[
T_1 =  \frac{1}{n}\boldf^\sT \bG^2 \boldf , \qquad T_2 = \frac{1}{n} \beps^\sT \bG^2 \boldf, \qquad T_3 =  \frac{1}{n}\beps^\sT \bG^2 \beps.
\]

\paragraph*{Term $T_1$.} Introduce $\bA_* = \bSigma^{-1} \bbeta_* \bbeta_*^\sT \bSigma^{-1} $. Observe that we can rewrite this term as
\[
T_1 = \frac1n\boldf^\sT \bG^2 \boldf = \frac1n \btheta_* \bR \bX^\sT \bX \bR \btheta_* = \frac{1}{n} \Tr\big( \bSigma^{1/2} \bA_* \bSigma^{1/2} \bR \bX^\sT \bX \bR \big),
\]
which corresponds to the functional $\Phi_4(\bX;\bA_*)$. Under the assumptions of Theorem \ref{thm:well_concentrated_test_train}, we can apply Theorem \ref{thm:main_det_equiv_summary} and get with probability at least $1 - n^{-D}$ that
\begin{equation}\label{eq:T1_train_well}
\left|T_1 - \Psi_4 ( \mu_* ; \bA_* ) \right| \leq \cE_{4,n}  \cdot \Psi_4 ( \mu_* ; \bA_* ).
\end{equation}

\paragraph*{Term $T_3$.} Using Hanson-Wright inequality \eqref{eq:Hanson_Wright_noise} on $\beps$ conditional on $\bX$, there exists a constant $C_{\eps,D}$ such that with probability at least $1 - n^{-D}$
\begin{align}\label{eq:HW_T3_well)}
\left|  T_3 - \frac{\sigma_\eps^2}{n} \Tr( \bG^2)\right| \leq C_{\eps,D} \frac{\log(n)}{n} \cdot \sigma_\eps^2 \| \bG^2 \|_F.
\end{align}
We control the mean and variance terms separately in the inequality above.

For the mean, we can rewrite 
\[
\frac{1}{n}\Tr(\bG^2) = \frac{1}{\lambda} \left( \frac{1}{n}\Tr(\bG) - \frac{1}{n}\Tr( \bX \bX^\sT \bG^2 ) \right).
\]
The term linear in $\bG$ was bounded in the first step above, while the second term corresponds to the functional $\Phi_4(\bX ; \bSigma^{-1})$. 
Hence, by Theorem \ref{thm:main_det_equiv_summary}, we conclude that with probability at least $1- n^{-D}$
\[
\left| \frac{1}{n}  \Tr( \bX \bX^\sT \bG^2 ) - \Psi_4 ( \mu_*; \bSigma^{-1} ) \right| \leq \cE_{4,n}  \cdot \Psi_4 ( \mu_*; \bSigma^{-1} ).
\]
Note that 
\[
\begin{aligned}
\frac{1}{\lambda} \left( \sfs_n (\lambda) -  \Psi_4 ( \mu_*; \bSigma^{-1} ) \right) = \frac{1}{n\lambda\lambda_*} \left( 1 - \frac{\frac{\lambda_*}{n}\Tr( \bSigma ( \bSigma +\lambda_*)^{-2} )}{1 - \Upsilon_2}\right)= \frac{1}{(n\lambda_*)^2} \cdot \frac{1}{1 - \Upsilon_2} =: \oT_3,
\end{aligned}
\]
where we used that $\frac{\lambda_*}{n}\Tr( \bSigma ( \bSigma +\lambda_*)^{-2} ) = \Upsilon_1 - \Upsilon_2$ and $1 - \Upsilon_1 = \frac{\lambda}{n\lambda_*}$. Thus, the mean is bounded by
\begin{equation}\label{eq:mean_T3_train_well}
\begin{aligned}
    \left| \frac{1}{n} \Tr(\bG^2)  - \oT_3 \right| \leq&~ \frac{1}{\lambda} \left| s_n (\bX , \lambda) - \sfs_n (\lambda) \right| + \frac{1}{\lambda}|\Phi_4(\bX ; \bSigma^{-1}) - \Psi_4 ( \mu_*; \bSigma^{-1} ) | \\
    \leq &~ \max \left\{ 2 \cE_{2,n} \nu_\lambda (n) , \cE_{4,n}  \right\} \cdot 2 \frac{n\lambda_*}{\lambda (n \lambda_*)^2}\\
    \leq &~ C \nu_\lambda (n) \cE_{4,n} \cdot \oT_3,
\end{aligned}
\end{equation}
where we used $\Psi_4 ( \mu_*; \bSigma^{-1} ) \leq \sfs_n$, the identity $n\lambda_*/\lambda = 1 + \Tr(\obM) \leq 2 \nu_\lambda (n)$, and $(1- \Upsilon_2)^{-1} \geq 1$, and that $\cE_{2,n} \leq \cE_{4,n}$ from their expressions in Theorem \ref{thm:main_det_equiv_summary}.

Let us now bound the variance term in the right-hand side of Eq.~\eqref{eq:HW_T3_well)}.Using $\| \bG\|_\op \leq 1/\lambda$, we first note that
\[
\frac{\Tr(\bG^4)}{\Tr(\bG^2) \oT_3} \leq \frac{1}{\lambda^2\oT_3} \leq  (1 + \Tr(\obM))^2 \leq (2 \nu_\lambda(n))^2.
\]
Therefore, we deduce that the right-hand side is bounded by
\begin{equation*}
\frac{\log(n)}{\sqrt{n}} \cdot \sigma_\eps^2 \| \bG^2 \|_F \leq  C \nu_{\lambda} (n) \log(n) \sigma_\eps^2 \sqrt{\frac{\Tr (\bG^2)}{n} \oT_3} \leq C_{x,K,D} \cdot \nu_{\lambda} (n) \log(n) \sigma_\eps^2 \oT_3,
\end{equation*}
where we used that $\Tr(\bG^2)/n \leq C_{x,K,D} \oT_3$ by Eq.~\eqref{eq:mean_T3_train_well} and conditions \eqref{eq:condition1_test_error_app_concentrated}. 
Combining this bound and Eq.~\eqref{eq:mean_T3_train_well} into Eq.~\eqref{eq:HW_T3_well)} yields with probability at least $1 - n^{-D}$
\begin{equation}\label{eq:T3_train_well}
| T_3 - \sigma_\eps^2 \oT_3 | \leq C_{x,\eps,D,K} \cdot \nu_\lambda (n)  \cE_{4,n} \cdot \sigma_\eps^2 \oT_3 .
\end{equation}

\paragraph*{Term $T_2$.} Again, using Hanson-Wright inequality \eqref{eq:Hanson_Wright_noise}, we have with probability at least $1 - n^{-D}$
\[
| T_2 | \leq C_{\eps,D} \frac{\log(n)}{n} \sigma_\eps \sqrt{ \boldf^\sT \bG^4 \boldf}.
\]
We observe that, similarly as above,
\[
\frac{ \boldf^\sT \bG^4 \boldf}{\oT_3 \cdot \boldf^\sT \bG^2 \boldf} \leq \frac{1}{\lambda^2 \oT_3} \leq (2 \nu_\lambda (n))^2 ,
\]
and therefore, we get using Eq.~\eqref{eq:T1_train_well} that 
\begin{equation}\label{eq:T2_train_well}
| T_2 | \leq C_{\eps,D} \frac{\nu_\lambda (n)\log(n)}{\sqrt{n}} \sigma_\eps \sqrt{ \oT_3 \cdot \frac{\boldf^\sT \bG^2 \boldf}{n}} \leq C_{x,\eps,K,D} \frac{\nu_\lambda (n)\log(n)}{\sqrt{n}} \left\{ \sigma_\eps^2 \oT_3 + \Psi_4 (\mu_* ; \bA_* ) \right\}.
\end{equation}

\paragraph*{Combining the terms.} Combining the bounds \eqref{eq:T1_train_well}, \eqref{eq:T3_train_well} and \eqref{eq:T2_train_well} into Eq.~\eqref{eq:decomposition_train_well_concentrated} yields that with probability at least $1 - n^{-D}$, we have
\[
\begin{aligned}
    \left| \cL_{\train} (\bbeta_* ; \bX,\beps,\lambda) - \sL_{n} (\bbeta_* , \lambda) \right| \leq&~ \lambda^2 |  T_1 - \Psi_4 (\mu_*; \bA_*) | + 2 \lambda^2 |T_2| + \lambda^2 |T_3 -\sigma_\eps^2 \oT_3 | \\
    \leq&~ C_{x,\eps,D,K} \cdot  \nu_\lambda (n) \max \left\{  \cE_{4,n} , n^{-1/2} \log(n)\right\} \cdot \sL_{n} (\bbeta_* , \lambda),
\end{aligned}
\]
where we used $\sL_{n} (\bbeta_* , \lambda) = \lambda^2 \Psi_4 (\mu_*; \bA_*) + \sigma_\eps^2 \lambda^2\oT_3 $. Replacing $ \cE_{4,n}$ by its expressions from Theorem \ref{thm:main_det_equiv_summary} concludes the proof of Eq.~\eqref{eq_thm:train_det_equiv_well_concentrated}.

\noindent
\textbf{Step 3: Test error.}

We proceed similarly to Step 2 and decompose the test error into three contributions
\begin{equation}\label{eq:decomposition_test_well_concentrated}
\| \btheta_* - \bR \bX^\sT \by \|_{\bSigma}^{2} = Q_1 - 2 Q_2 + Q_3,
\end{equation}
where we denoted
\[
Q_1 = \| \btheta_* - \bR \bX^\sT\bX \btheta_* \|_{\bSigma}^{2}, \qquad Q_2 = \< \beps, \bX \bR \bSigma ( \btheta_* - \bR \bX^\sT\bX \btheta_*)\>, \qquad Q_3 = \| \bR \bX^\sT \beps \|_{\bSigma}^2.
\]

\paragraph*{Term $Q_1$.} Recall that we denoted $\bA_* = \bSigma^{-1} \bbeta_* \bbeta_*^\sT \bSigma^{-1}$. This  term can be rewritten as
\[
Q_1 = \lambda^2 \| \bR \btheta_* \|_{\bSigma}^2 = \lambda^2 \Tr( \bSigma^{1/2} \bA_*  \bSigma^{1/2} \bR \bSigma \bR ),
\]
which corresponds to the functional $\Phi_3 (\bX ; \bA_*)$.  Under the assumptions of Theorem \ref{thm:well_concentrated_test_train}, we can apply again Theorem \ref{thm:main_det_equiv_summary} and get with probability at least $1 - n^{-D}$ that
\begin{equation}\label{eq:Q1_test_well}
\left|Q_1 - \lambda^2 \Psi_3 ( \mu_* ; \bA_* ) \right| \leq \cE_{3,n} \cdot \lambda^2\Psi_3 ( \mu_* ; \bA_* ).
\end{equation}

\paragraph*{Term $Q_3$.} Using Hanson-Wright inequality \eqref{eq:Hanson_Wright_noise} on $\beps$, there exists a constant $C_{\eps,D}$ such that with probability at least $1 - n^{-D}$
\begin{align}\label{eq:HW_Q3_well)}
\left|  Q_3 - \sigma_\eps^2 \Tr(\bSigma \bR \bX^\sT \bX \bR) \right| \leq C_{\eps,D} \log(n) \cdot \sigma_\eps^2 \| \bX \bR \bSigma \bR \bX^\sT \|_F.
\end{align}
The mean corresponds to the functional $\Phi_4 (\mu_*; \id)$. From Theorem \ref{thm:main_det_equiv_summary}, we get that with probability at least $1 - n^{-D}$
\begin{equation}\label{eq:mean_Q3_test_well}
\begin{aligned}
    \left|  \Tr(\bSigma \bR \bX^\sT \bX \bR)  - n \Psi_4 (\mu_* ; \id) \right| \leq&~ \cE_{4,n} \cdot n \Psi_4 (\mu_* ; \id).
\end{aligned}
\end{equation}
For the variance term, we use that $\| \bM \|_\op \leq C_{x,D}\nu_\lambda (n)/ n$ with probability at least $1 - n^{-D}$ by Lemma \ref{lem:tech_bounds_norm_M}, and obtain
\[
\frac{n \| \bZ \bM^2 \bZ^\sT \|_F^2 }{(\Tr(\bZ \bM^2 \bZ^\sT) +1)^2} \leq n \| \bM\|_\op \leq C_{x,D} \nu_\lambda (n).
\]
Therefore, the right-hand side of Eq.~\eqref{eq:HW_Q3_well)} is bounded by
\[
\log(n) \cdot \sigma_\eps^2 \| \bX \bR \bSigma \bR \bX^\sT \|_F \leq C_{x,D,K} \frac{\nu_\lambda (n) \log (n)}{\sqrt{n}} \left\{ \sigma_\eps^2 \cdot n \Psi_4 (\mu_* ; \id)  + \sigma_\eps^2 \right\}.
\]
Combining this bound and Eq.~\eqref{eq:mean_Q3_test_well} yields with probability at least $1 - n^{-D}$ that
\begin{equation}\label{eq:Q3_test_well}
| Q_3 - \sigma_\eps^2 \cdot n \Psi_4 (\mu_* ; \id) | \leq C_{\eps,x,D} \max \left\{ \cE_{4,n}  ,   n^{-1/2} \nu_\lambda(n) \log(n) \right\} \cdot \left\{ \sigma_\eps^2 \cdot n \Psi_4 (\mu_* ; \id)  + \sigma_\eps^2 \right\} .
\end{equation}

\paragraph*{Term $Q_2$.} Using Hanson-Wright inequality \eqref{eq:Hanson_Wright_noise}, we have with probability at least $1 - n^{-D}$
\begin{equation}\label{eq:Q2_test_well}
 \begin{aligned}
     | Q_2 | \leq&~ C_{\eps,D} \log(n) \sigma_\eps \lambda \sqrt{ \< \btheta_*, \bR \bSigma \bR \bX^\sT \bX \bR \bSigma \bR \btheta_* \>}\\
     \leq&~ C_{\eps,D} \log(n) \| \bSigma^{1/2} \bR \bX^\sT \|_\op  \sqrt{\sigma_\eps^2 \cdot T_1 }\\
     \leq&~ C_{\eps,x,D} \frac{\nu_\lambda (n)^{1/2} \log(n)}{\sqrt{n}}  \left\{\sigma_\eps^2 +  \lambda^2 \Psi_3 (\mu_*; \bA_* ) \right\},
 \end{aligned}   
\end{equation}
where we used that $\| \bSigma^{1/2} \bR \bX^\sT \|_\op \leq \| \bM \|_\op^{1/2} \leq C_{x,D} (\nu_\lambda(n) / n)^{1/2}$ with probability at least $1 - n^{-D}$ by Lemma \ref{lem:tech_bounds_norm_M}.

\paragraph*{Combining the terms.} Combining the bounds \eqref{eq:Q1_test_well}, \eqref{eq:Q3_test_well} and \eqref{eq:Q2_test_well} into Eq.~\eqref{eq:decomposition_test_well_concentrated} yields with probability at least $1 - n^{-D}$
\[
\begin{aligned}
    \left| \cR_{\test} (\bbeta_* ; \bX,\beps,\lambda) - \sR_{n} (\bbeta_* , \lambda) \right| \leq&~  |  Q_1 - \lambda^2 \Psi_3 (\mu_*; \bA_*) | + 2 |Q_2| +  |Q_3 -\sigma_\eps^2 \cdot n \Psi_4 (\mu_*; \id) | \\
    \leq&~ C_{\eps,x, D,K}  \max \left\{ \cE_{3,n}  , \cE_{4,n} , n^{-1/2} \nu_\lambda (n) \log(n)\right\} \cdot \sR_{n} (\bbeta_* , \lambda),
\end{aligned}
\]
where we used $\sR_{n} (\bbeta_* , \lambda) =  \lambda^2 \Psi_3 (\mu_*; \bA_*) + \sigma_\eps^2 (n \Psi_4 (\mu_*; \id) +1) $. Replacing $\cE_{3,n}, \cE_{4,n}$ by their expressions from Theorem \ref{thm:main_det_equiv_summary} concludes the proof of Eq.~\eqref{eq_thm:test_det_equiv_well_concentrated}.
\end{proof}

\addcontentsline{toc}{section}{References}
\bibliographystyle{amsalpha}
\bibliography{bibliography.bbl}

\clearpage

\appendix

\section{Proof of deterministic equivalents in the dimension free regime}\label{app:det_equiv}

In this appendix, we prove deterministic equivalents for the different functionals  of the random feature matrix $\Phi (\bX)$ that appear in the proofs of the main results. These results might be of independent interest and for the sake of readability, this appendix is self-contained and can be read independently from the rest of the paper. 

In Section \ref{app_det_equiv:assumptions}, we restate our assumptions, the functionals of interest and their associated deterministic equivalents. We start with some preliminaries in Section \ref{app_det_equiv:preliminaries} where we outline the general proof strategy. Sections \ref{app_det_equiv:proof_M}, \ref{app_det_equiv:proof_ZZM}, \ref{app_det_equiv:proof_MM} and \ref{app_det_equiv:proof_MZZM} provide proofs for these different deterministic equivalents. Finally, we defer the proof of some technical bounds to Section \ref{app_det_equiv:technical}.

\subsection{Definitions and assumptions}
\label{app_det_equiv:assumptions}

We consider a feature vector $\bx \in \R^p$ with covariance matrix $\bSigma = \E[\bx \bx^\sT]$. 
Without loss of generality, we assume $\| \bSigma \|_\op = 1$ and $\bSigma$ to be diagonal
\[
\bSigma = \diag ( \xi_1 , \xi_2 , \xi_3 , \ldots )  ,
\]
where $1 = \xi_1 \geq \xi_2 \geq \xi_3 \geq \cdots > 0$ are the positive eigenvalues in non-increasing order. We allow $p = \infty$ by further assuming that $\Tr( \bSigma) < \infty$. In that case, $\bx$ is a random element in the Hilbert space $\ell_2 := \{ \bx = (x_1 , x_2 , x_3,\ldots): \sum_{j = 1}^\infty x_j^2  < \infty\} $ with inner product $\< \bu, \bv \> := \< \bu, \bv \>_{\ell_2} = \sum_{j =1}^\infty u_j v_j$, and $\bSigma$ is understood to be a trace class self-adjoint operator. We will treat both cases $p<\infty$ and $p = \infty$ in a unified manner by using, with a slight abuse of notations, the matrix calculus notations for (infinite-dimensional) linear operators. In particular, $\| \cdot \|_2$ and $\< \cdot , \cdot \>$ will denote the norm and inner product in both euclidean space and $\ell_2$ space depending on the context.

In this appendix, we assume that $\bx$ satisfy Assumption \ref{ass:main_assumptions}.(a) with $\evn = p$ (no high-degree part). Recall that the effective rank of $\bSigma$, which we denote here $r_\bSigma$ to emphasize the dependency on the covariance matrix, is given as the smallest scalar such that $r_{\bSigma} (n ) \geq n$ and
\[
r_{\bSigma} (n) \geq \frac{\sum_{j = k+1}^p}{\xi_{k+1}}, \qquad \quad \text{for all } 0\leq k \leq \min(n,p)-  1.
\]
For convenience, we gather below the assumptions that the feature vector $\bx$ satisfies throughout Appendix \ref{app:det_equiv}.

\begin{assumption}\label{ass_app:deterministic_equivalent}
    We assume that the covariance $\bSigma = \E [ \bx \bx^\sT ]$ is a trace class operator, i.e., $\Tr(\bSigma) = \E [ \| \bx \|_2^2 ] <\infty$. Without loss of generality, we further assume that $\| \bSigma \|_\op = 1$ and $\bSigma = \diag (\xi_1 , \xi_2 , \xi_3 , \ldots)$, where $1 = \xi_1 \geq \xi_2 \geq \xi_3 \geq \cdots > 0$ are the positive eigenvalues in non-increasing order. 
    
    We further assume that there exist $\sfc_x, \sfC_x , \beta >0$ and $\varphi_1 (p)>0$ such that for any deterministic p.s.d.~matrix $\bA \in \R^{p\times p}$ and vector $\bv \in \R^{p}$,
        \begin{align}
                    \P \left( \big\vert \< \bv , \bx \>^2 - \bv^\sT \bSigma \bv) \big\vert \geq t \cdot  \bv^\sT \bSigma \bv \right) \leq&~ \sfC_x \exp \left\{ - \sfc_x t^{1/\beta } \right\},  \tag{a1} \label{eq_app:DE_concentration_b2}\\
            \P \left( \big\vert \bx^\sT \bA \bx - \Tr( \bSigma \bA) \big\vert \geq t \cdot \varphi_1 (p)  \cdot \big\| \bSigma^{1/2} \bA \bSigma^{1/2} \big\|_F \right) \leq&~ \sfC_x \exp \left\{ - \sfc_x t^{1/\beta } \right\}  .\label{eq_app:DE_concentration_b1} \tag{a2}
        \end{align}


\end{assumption}


We are given $n$ i.i.d.~features $(\bx_i)_{i \in [n]}$, and we denote $\bX = [ \bx_1 , \ldots , \bx_n ]^\sT \in \R^{n \times p}$ the feature matrix. We introduce $\bz_i = \bSigma^{-1/2} \bx_i$ the whitened features, and $\bZ = \bX \bSigma^{-1/2}$ the whitened feature matrix. The test error and GCV estimator are functionals of $\bX$. In particular, they depend on the following resolvent matrix
\[
\bR = ( \bX^\sT \bX + \lambda \id_p)^{-1} .
\]
In this section we consider functionals that depend on products of $\bX$, $\bR$ and deterministic matrices. 

Throughout this appendix, we assume that $\lambda >0$. Recall that we defined the effective regularization $\lambda_*$ to be the unique non-negative solution of 
\begin{align}\label{eq:det_equiv_fixed_point_lambda_star}
n - \frac{\lambda}{\lambda_*} = \Tr (\bSigma (\bSigma + \lambda_* \id )^{-1} )  .
\end{align}
 We consider the change of variable $\mu_*:= \mu_* (\lambda) = \lambda / \lambda_*$, such that $\mu_*$ is the unique non-negative solution of 
\begin{align}\label{eq:det_equiv_fixed_point_mu_star}
\mu_* = \frac{n}{1 + \Tr(\bSigma (\mu_* \bSigma + \lambda)^{-1} )} .
\end{align}
Both $\mu_*$ and $\lambda_*$ are increasing functions with $\lambda$.
We introduce the following deterministic matrix
\[
\obR = ( \mu_* \bSigma + \lambda \id)^{-1} .
\]
We further consider the rescaled resolvents 
\[
 \bM = \bSigma^{1/2} \bR \bSigma^{1/2} , \qquad  \obM = \bSigma^{1/2} \obR \bSigma^{1/2} .
\]

Our goal is to prove explicit non-asymptotic bounds between functionals of $\bX$ and their deterministic equivalents. More precisely, denote $\Phi (\bX)$ the feature matrix functional and $\Psi (\mu_*)$ its associated deterministic equivalent, where $\mu_*$ is the solution of the fixed point equation \eqref{eq:det_equiv_fixed_point_mu_star}. We will prove that for any constant $D>0$, we have with probability at least $1 - n^{-D}$ that
\[
| \Phi (\bX) - \Psi (\mu_* ) | \leq \cE(n) \cdot \Psi (\mu_*),
\]
where $\cE(n)$ is a function explicit in $\varphi_1(p)$, $r_\bSigma$, $\lambda$, and $n$, and implicit in $D$ and the other constants appearing in Assumption \ref{ass_app:deterministic_equivalent}. 

In this appendix, we consider four functionals: for a general p.s.d.~matrix $\bA \in \R^{p\times p}$, define
\[
\begin{aligned}
    \Phi_1(\bX;\bA) = &~ \Tr \left( \bSigma^{1/2} \bA \bSigma^{1/2} (\bX^\sT \bX + \lambda )^{-1}  \right), \\
    \Phi_2 (\bX) = &~  \Tr \left( \frac{\bX^\sT \bX}{n} ( \bX^\sT \bX + \lambda )^{-1} \right), \\
    \Phi_3(\bX ; \bA) =&~  \Tr \left( \bSigma^{1/2} \bA \bSigma^{1/2} ( \bX^\sT \bX + \lambda)^{-1} \bSigma ( \bX^\sT \bX + \lambda)^{-1} \right) ,\\
    \Phi_4 ( \bX ; \bA) = &~ \Tr\left(\bSigma^{1/2} \bA \bSigma^{1/2} ( \bX^\sT \bX + \lambda)^{-1}\frac{\bX^\sT \bX}{n} ( \bX^\sT \bX + \lambda)^{-1} \right)   . 
\end{aligned}
\]
We will show that these functionals are well approximated by the following functions 
\[
\begin{aligned}
    \Psi_1 ( \mu_*; \bA) :=&~  \Tr \left( \bA \bSigma (\mu_* \bSigma + \lambda )^{-1}  \right) ,\\
     \Psi_2 ( \mu_*) :=&~\frac{1}{n} \Tr\left( \bSigma ( \bSigma + \lambda_* )^{-1} \right) ,\\
      \Psi_3 ( \mu_*; \bA) :=&~ \frac{\Tr( \bA \bSigma^2 ( \mu_* \bSigma + \lambda)^{-2})}{1 - \frac{1}{n} \Tr(\bSigma^2 (\bSigma + \lambda_*)^{-2})} , \\
       \Psi_4 ( \mu_*; \bA) :=&~ \frac{1}{n} \cdot \frac{\Tr(\bA \bSigma^2 ( \bSigma + \lambda_*)^{-2})}{n -  \Tr(\bSigma^2 ( \bSigma + \lambda_* )^{-2} ) }.
\end{aligned}
\]
Intuitively, the empirical covariance $\widehat{\bSigma} = \bX^\sT \bX / n$ can be effectively approximated by $\mu_* \bSigma / n = \bSigma / (1 + \Tr( \obM))$. In the case of $\Phi_3$ and $\Phi_4$, because the two resolvents are not independent, the deterministic equivalents are rescaled by a factor $(1 - \frac{1}{n} \Tr( \bSigma^2 ( \bSigma + \lambda_* )^{-2} ) )^{-1}$.

The next section introduce some notations and describe the general proof strategy. The proof of the deterministic equivalents for $\Phi_1 (\bX;\bA) , \Phi_2 (\bX), \Phi_3(\bX;\bA), \Phi_4(\bX;\bA)$ can be found in Sections \ref{app_det_equiv:proof_M}, \ref{app_det_equiv:proof_ZZM}, \ref{app_det_equiv:proof_MM} and \ref{app_det_equiv:proof_MZZM} respectively.

\subsection{Preliminaries}
\label{app_det_equiv:preliminaries}

Recall that we track the dependency in $\{n,\lambda,r_\bSigma,\varphi_1 (p)\}$. For the other constants, we will denote $C_{a_1,\ldots ,a_k}$ constants that only depend on the values of $\{a_i \}_{i \in [k]}$.  We use $a_i = `x'$ to denote the dependency on the constants $\sfc_x,\sfC_x, \beta$ of Assumption \ref{ass_app:deterministic_equivalent}. In particular, the value of these constants is allowed to change from line to line.  For convenience, we introduce for a p.s.d.~matrix $\bA$ the notation
\[
\varphi_1^{\bA} (p) = \begin{cases}
    1 \qquad &\text{if $\bA$ is rank $1$},\\
    \varphi_1 (p) &\text{otherwise},
\end{cases}
\]
so that we can compactly write Assumption \ref{ass_app:deterministic_equivalent} in a single equation
\begin{equation}\label{eq:single_equ_assumption_FirstOrder}
 \P \left( \big\vert \bx^\sT \bA \bx - \Tr( \bSigma \bA) \big\vert \geq t \cdot \varphi_1^\bA (p)  \cdot \big\| \bSigma^{1/2} \bA \bSigma^{1/2} \big\|_F \right) \leq \sfC_x \exp \left\{ - \sfc_x t^{1/\beta } \right\}  .
\end{equation}

The proofs will follow a general blueprint to prove deterministic equivalents that has been successfully applied in various settings (see for example \cite[Theorem 2.6]{couillet2022random}).
The goal is to create an approximate martingale interpolation between the deterministic equivalent $\Psi (\mu_*)$ and the random functional $\Phi(\bX) $ that depends on the $n$ features $\{ \bx_1 , \bx_2 , \ldots , \bx_n \}$ by removing one feature $\bx_i$ at a time. A natural approach is to first construct an exact martingale interpolation between $\Phi (\bX)$  and $\E [ \Phi(\bX)]$ by taking the partial expectation of $\Phi(\bX)$ over $\{ \bx_{i+1} , \ldots , \bx_n\}$ for $i=0,\ldots,n$, before bounding the difference between $\E[\Phi(\bX)]$ and $\Psi (\mu_*)$. Thus, the bound is split between a martingale and a deterministic part
\[
\big\vert \Phi(\bX) - \Psi (\mu_*) \big\vert \leq \underbrace{\big\vert \Phi(\bX) - \E[\Phi(\bX)] \big\vert}_{(\sfM)} + \underbrace{\big\vert\E[\Phi(\bX)] - \Psi (\mu_*)\big\vert}_{(\sfD)}.
\]

For the martingale part $(\sfM)$, we denote $\E_{i}$ the expectation over features $\{ \bx_{i+1}, \ldots, \bx_n\}$, for $i = 0, \ldots , n$, and create the following martingale difference sequence
\[
\Phi(\bX) - \E[\Phi(\bX)]  = \sum_{i = 1}^n \Big( \E_i - \E_{i-1} \Big)\Phi(\bX)  .
\]
We bound each of these terms as follows. We introduce $\bX_{i}$ the feature matrix where we removed feature $\bx_i$ and use that $( \E_i - \E_{i-1}) \Phi (\bX_{i}) =0$ to replace each difference by $\big( \E_i - \E_{i-1} \big)(\Phi(\bX) - \Phi(\bX_{i}) )$. We can then use a leave-one-out argument to bound each terms with high probability and apply Azuma-Hoeffding inequality with a standard truncation argument.

For the deterministic part $(\sfD)$, we follow a leave-one-out type computation. Introduce $\bX_- \in \R^{(n-1) \times p}$ the feature matrix where we removed one feature and define the associated resolvent and scaled resolvent:
\[
\bR_- := (\bX_-^\sT \bX_- + \lambda )^{-1}  , \qquad \bM_- := \bSigma^{1/2} \bR_- \bSigma^{1/2}  .
\]
We further introduce the following deterministic matrices:
\[
\obR_- := \left( \frac{n}{1 + \kappa} \bSigma + \lambda \id \right)^{-1} , \qquad \obM_- := \bSigma^{1/2} \obR_- \bSigma^{1/2}  ,
\]
where $\kappa = \E[ \Tr(\bM_-) ]$. This matrix intuitively corresponds to an approximate deterministic equivalent with one feature removed. Note that, by the fixed point equation \eqref{eq:det_equiv_fixed_point_mu_star}, the matrices $\obR$ and $\obM$ correspond to the case where $\kappa$ is replaced by $\Tr(\obM)$. Because we have $\E[ \Tr(\bM_-) ] \approx \obM$, then the two deterministic matrices are approximately equal $\obM_- \approx \obM$. The deterministic part can therefore be decomposed into
\[
\big\vert \E [\Phi (\bX)] - \Psi (\mu_*)\big\vert \leq \big\vert \E [\Phi (\bX)] - \Psi (n/(1+\kappa)) \big\vert + \big\vert  \Psi (n/(1+\kappa)) - \Psi (\mu_*) \big\vert.
\]
The first term can be bounded by expanding the difference between $\bM$ and $\obM_-$, and using a leave-one-out argument to make the denominator and numerator independent. In particular, we will abundantly use the following Sherman-Morrison identities
\begin{equation}\label{eq:standard_identities_SMW}
    \bM = \bM_- - \frac{\bM_- \bz \bz^\sT \bM_-}{1 + \bz^\sT \bM_- \bz}, \qquad\textrm{and}\qquad \bM \bz = \frac{\bM_- \bz}{1 + \bz^\sT \bM_- \bz}  .
\end{equation}

\paragraph*{Technical bounds.}

Our proofs will crucially rely on the following high probability bound on the operator norm of $\bM$. This lemma is a modification of \cite[Lemma 7.2]{cheng2022dimension} and the proof proceeds in a similar manner.

\begin{lemma}\label{lem:tech_bounds_norm_M} 
    Under Assumption \ref{ass_app:deterministic_equivalent} and for any constant $D >0$, there exist constant $\eta \in (0,1/2)$ that only depends on $\sfc_x,\sfC_x,\beta$, constant $C_D >0$ that only depends on $D$, and constant $C_{x,D}>0$ that only depends on $\sfc_x,\sfC_x,\beta,D$, such that for the following holds. For all $n \geq C_D$ and $\lambda >0$, and defining 
    \begin{equation}\label{eq:def_chi}
    \onu_\lambda (n) :=  1 + \frac{ \xi_{\lfloor \eta n \rfloor} \cdot r_{\bSigma}\sqrt{\log(r_\bSigma)} \cdot \{ 1 + n^{-1/2} \varphi_1 (p) \log^\beta (n)   \} }{\lambda}   ,
    \end{equation}
    where we used the convention $\xi_{\lfloor \eta n \rfloor} = 0$ for $\lfloor \eta n \rfloor >n$, we have with probability at least $1 - n^{-D}$ that
    \[
    \| \bM \|_\op \leq C_{x,D} \frac{\onu_\lambda (n)}{n} , \qquad \Tr (\bM) \leq C_{x,D}\cdot \onu_\lambda (n)  , \qquad \| \bM \|_F \leq C_{x,D}\frac{\onu_\lambda (n)}{\sqrt{n}}  .
    \]
\end{lemma}

Lemma \ref{lem:tech_bounds_norm_M} and other similar technical bounds are proved in Section \ref{app_det_equiv:technical}.

\begin{remark}[Alternative interpolation path]\label{rmk:interpolation_path} The paper \cite{cheng2022dimension} considers $\Phi_1 (\bX ; \bA)$ and proves the deterministic equivalent by constructing a different interpolation path. They introduce $\bX_i = [\bx_1, \ldots , \bx_i]^\sT$ and a sequence $\mu_* =: \mu_0 > \mu_1 > \mu_2 > \ldots > \mu_n := 0$, and define the partial deterministic matrices
\[
\bM_i = \bSigma^{1/2} ( \bX_i^\sT \bX_i + \mu_i \bSigma + \lambda \id )^{-1} \bSigma^{1/2},
\]
so that $\bM_0 = \obM$ and $\bM_n = \bM$. They consider the following interpolation
\[
\Tr(\bA \bM) - \Tr( \bA \obM) = \sum_{i =1}^n \Tr( \bA \bM_i ) - \Tr( \bA \bM_{i-1}) ,
\]
and choose $\mu_i \in \sigma ( \bx_1 , \ldots , \bx_{i-1})$ so that the sum is approximately a martingale difference sequence. Using this interpolation path, they achieve a better rate of approximation for some matrices $\bA$. In particular, $\cE_n (n) = O(n^{-1})$ for $\bA = \id$ which recovers the average law fluctuation of the resolvent. In our case, we only obtain fluctuations $\cE_n (n) \approx O(n^{-1/2})$ for all matrix $\bA$ (which matches the local law fluctuations, see Remark \ref{rmk:local_average_law} in the main text). This is due to bounding the expectation in the deterministic part $(\sfD)$ which sum linearly the fluctuations instead of the square root in the case of $\bA = \id$ (see Remark \ref{rmk:fluctuations_det_equiv_TrAM}). However, \cite{cheng2022dimension} assumes $\| \bA \|_\op < \infty$ and their analysis require a much more involved argument to study this approximate martingale difference, by carefully separating an exact martingale part that can be controlled using Azuma-Hoeffding inequality and a lower-order correction term. On the other hand, our proof only uses elementary steps and applies to any p.s.d.~matrix $\bA$. Furthermore, our proof strategy is easily applied to the more complex functionals $\Phi_2 (\bX),\Phi_3 (\bX;\bA)$, and $\Phi_4 (\bX;\bA)$, while \cite{cheng2022dimension} would require to adapt the interpolation path. See Remark \ref{rmk:comparison_Cheng_theorems} for a comparison between Theorem \ref{thm_app:det_equiv_TrAM} and \cite[Corollary 6.5]{cheng2022dimension}.
\end{remark}

\subsection[Deterministic equivalent for $\Tr ( AM )$]{Deterministic equivalent for \boldmath{$\Tr ( AM )$}}
\label{app_det_equiv:proof_M}

In this section, we consider our first functional: for a given p.s.d.~matrix $\bA \in \R^{p \times p}$,
\[
\Phi_1 (\bX;\bA) = \Tr( \bA \bSigma^{1/2} (\bX^\sT \bX + \lambda )^{-1} \bSigma^{1/2}) = \Tr( \bA \bM) .
\]
We show that $\Phi_1 (\bX; \bA)$ is well approximated by the following deterministic equivalent
\[
\Psi_1 ( \mu_*; \bA) =  \Tr( \bA \bSigma^{1/2} (\mu_* \bSigma + \lambda )^{-1} \bSigma^{1/2}) = \Tr( \bA \obM) ,
\]
where we recall that $\mu_*$ is the solution of the fixed point equation~\eqref{eq:det_equiv_fixed_point_mu_star}.

Without loss of generality, we assume that $\Tr( \bA \bSigma) < \infty$ (otherwise $\Tr(\bA \bM) = \infty$ almost surely and our bounds apply trivially). However, we allow for $\| \bA \|_\op = \infty $. In particular, our theorem applies to $\bA = \bSigma^{-1/2} \bu \bu^\sT \bSigma^{-1/2}$ which has $\Tr(\bA \bSigma) = \| \bu \|_2^2 < \infty$ but not necessarily $\| \bA \|_\op  = \| \bSigma^{-1/2} \bu \|_2^2 < \infty$.

Our bounds will depend on the following quantity
\begin{equation}\label{eq:definition_nu_lambda}
\nu_\lambda (n) = 1 + \frac{\xi_{\lfloor \eta n \rfloor}\cdot  r_\bSigma \sqrt{\log (r_{\bSigma})}  }{\lambda},
\end{equation}
which appears in the upper bound of $\| \bM \|_\op$ (Lemma \ref{lem:tech_bounds_norm_M}). Here $\eta \in (0,1/2)$ is a constant that will only depend on the constants $\sfc_x,\sfC_x,\beta$ in Assumption \ref{ass_app:deterministic_equivalent} and we use the convention $\xi_{\lfloor \eta n \rfloor} = 0$ if $\lfloor \eta n \rfloor >p$.

We are now ready to state the first theorem of Appendix \ref{app:det_equiv}.

\begin{theorem}[Deterministic equivalent for $\Tr(\bA \bM)$]\label{thm_app:det_equiv_TrAM}
    Assume the features $(\bx_i)_{i\in[n]}$ satisfy Assumption \ref{ass_app:deterministic_equivalent}. For any constants $D,K>0$, there exist constants $\eta := \eta_x \in (0,1/2)$ (only depending on $\sfc_x,\sfC_x,\beta$), $C_{D,K} >0$ (only depending on $K,D$), and $C_{x,D,K}>0$ (only depending on $\sfc_x,\sfC_x,\beta,D,K$), such that the following holds. Define $\nu_\lambda (n)$ as per Eq.~\eqref{eq:definition_nu_lambda}. For all $n \geq C_{D,K}$ and  $\lambda >0$ satisfying
    \begin{equation}\label{eq:conditions_nu_phi}
    \lambda \cdot \nu_\lambda (n) \geq n^{-K}, \qquad \varphi_1(p) \nu_\lambda (n)^2 \log^{\beta + \frac{1}{2}} (n) \leq K \sqrt{n} ,
    \end{equation}
    and deterministic p.s.d.~matrix $\bA$, we have with probability at least $1 - n^{-D}$ that
\begin{equation}\label{eq:thm_det_equiv_AM}
         \big\vert \Phi_1(\bX;\bA) - \Psi_1 (\mu_* ; \bA) \big\vert \leq C_{x,D,K} \frac{\varphi_1(p) \nu_\lambda (n)^{5/2} \log^{\beta +\frac12} (n) }{\sqrt{n}}     \Psi_1 (\mu_* ; \bA) .
    \end{equation} 
\end{theorem}

Let us provide some intuition about this bound. For generic kernels, we will have typically $\xi_{\lfloor \eta n \rfloor} \cdot r_\bSigma = \widetilde{O}(1)$ so that $\nu_\lambda (n) = \widetilde{O} (1/\min(\lambda , 1))$ (here $\widetilde{O}$ hides constant and $\log (n)$ factors). Thus, the relative approximation error has rate
\[
\big\vert \Phi_1(\bX;\bA) - \Psi_1 (\mu_* ; \bA) \big\vert = \widetilde{O} \left( \frac{\varphi_1(p) }{\min(1,\lambda^{5/2}) \sqrt{n}} \right)     \Psi_1 (\mu_* ; \bA) ,
\]
which will be small as long as $\varphi_1 (p) \ll \sqrt{n}$ and $\lambda$ does not decrease too fast to $0$ with $n$.

To prove Theorem \ref{thm_app:det_equiv_TrAM}, we will follow the proof strategy described in the preliminary section. We decompose the bound into a deterministic part and a martingale part as follows
\[
\begin{aligned}
       \left\vert  \Phi_1(\bX;\bA) - \Psi_1 (\mu_* ; \bA) \right\vert  \leq&~ \left\vert  \E [\Phi_1(\bX;\bA)] - \Psi_1 (\mu_* ; \bA) \right\vert + \left\vert  \Phi_1(\bX;\bA) - \E [\Phi_1(\bX;\bA)] \right\vert \\
       \leq &~ \underbrace{\big\vert \Tr \left\{ \bA (\E[\bM] - \obM ) \right\}\big\vert}_{(\sfD)}  + \underbrace{\big\vert \Tr (\bA \bM ) - \E [ \Tr ( \bA \bM ) ]\big\vert}_{(\sfM)} .
\end{aligned}
\]

The deterministic part $(\sfD)$ is bounded in Proposition \ref{prop:TrAM_LOO}, while the martingale part $(\sfM )$ is bounded in Proposition \ref{prop:TrAM_martingale}.

\begin{proposition}[Deterministic part of $\Tr(\bA\bM)$]\label{prop:TrAM_LOO}
    Assume the same setting as Theorem \ref{thm_app:det_equiv_TrAM}. Then there exist constants $C_K$ and $C_{x,K}$, such that the following holds. For all $n \geq C_K$ and  $\lambda >0$ satisfying Eq.~\eqref{eq:conditions_nu_phi}, and for all p.s.d.~matrix $\bA$, we have
    \begin{equation}\label{eq:det_part_TrAM}
         \big\vert\E [ \Phi_1(\bX;\bA) ] - \Psi_1 (\mu_* ; \bA) \big\vert \leq C_{x,K} \frac{\varphi_1 (p) \nu_\lambda (n)^{5/2} }{\sqrt{n}}    \Psi_1 (\mu_* ; \bA) .
    \end{equation}
\end{proposition}

For the martingale part, we further provide an alternative bound for matrices with $\| \bA \|_\op < \infty$. In that case, define the ratio
\begin{equation}\label{eq:definition_rho_A}
\rho_1 (\bA) := \frac{\Psi_1 (\mu_*; \bA/\| \bA \|_\op)}{\Psi_1 (\mu_* ; \id )} ,
\end{equation}
and
\begin{equation}\label{eq:definition_zeta}
\zeta_1 := \Psi_1 ( \mu_*; \id).
\end{equation}
In particular, note that $\rho_1 (\bA) \leq 1$ for all matrices $\bA$, and $\zeta_1 = \frac{n\lambda_*}{\lambda} - 1 \geq 1 - \lambda/(n\lambda_*) $ from the fixed point equation \eqref{eq:det_equiv_fixed_point_lambda_star}.

\begin{proposition}[Martingale part of $\Tr(\bA\bM)$] \label{prop:TrAM_martingale}
Assume the same setting as Theorem \ref{thm_app:det_equiv_TrAM}. Then there exist constants $C_{K,D}$ and $C_{x,K,D}$, such that the following holds. For all $n \geq C_{K,D}$ and  $\lambda >0$ satisfying Eq.~\eqref{eq:conditions_nu_phi}, and for all p.s.d.~matrix $\bA$, we have with probability at least $1 -n^{-D}$ that
    \begin{equation}\label{eq:mart_part_TrAM}
         \big\vert \Phi_1(\bX;\bA)  - \E [ \Phi_1(\bX;\bA) ]  \big\vert \leq C_{x,K,D} \frac{\varphi^\bA_1 (p) \nu_\lambda (n)^{2} \log^{\beta +\frac12} (n)}{\sqrt{n}}   \Psi_1 (\mu_* ; \bA) .
    \end{equation}
   If we further assume that $\| \bA \|_\op < \infty$, then we have with the same probability 
    \begin{equation}\label{eq:mart_part_TrAM_2}
    \big\vert \Phi_1(\bX;\bA)  - \E [ \Phi_1(\bX;\bA) ]  \big\vert \leq C_{x,K,D} \frac{\varphi_1 (p) \nu_\lambda (n)^{3} \log^{\beta+\frac12} (n)}{n \sqrt{\zeta_1 \rho_1 (\bA)}}  \Psi_1 (\mu_* ; \bA) .
    \end{equation}
\end{proposition}

The proof of these two propositions can be found in the next two sections. Theorem \ref{thm_app:det_equiv_TrAM} is obtained by combining the bounds \eqref{eq:det_part_TrAM} and \eqref{eq:mart_part_TrAM}.

\begin{remark}[Fluctuations]\label{rmk:fluctuations_det_equiv_TrAM}
    In the case of $\| \bA \|_\op < \infty$, the relative approximation error of $| \Tr(\bA \bM) - \E [ \Tr(\bA \bM)]|$ is of order $\widetilde{O} ( \varphi_1 (p)/ (n \sqrt{\rho_1 (\bA)}))$. In particular, when $\varphi_1 (p ) = 1$ and $\bA = \id$, we recover the rate $\widetilde{O}(n^{-1})$ obtained in \cite{cheng2022dimension}, which matches the average law fluctuations. However, the bound on the expectation in the deterministic part $(\sfD)$ prevents the self-averaging of the fluctuations for $\bA = \id$. The relative approximation rate is worst-case over all matrices $\bA$ and we only recover $\widetilde{O}(n^{-1/2})$, the scale of the local law fluctuations.
\end{remark}

\subsubsection[Proof of Proposition \ref{prop:TrAM_LOO}: deterministic part of $\Tr(AM)$]{Proof of Proposition \ref{prop:TrAM_LOO}: deterministic part of \boldmath{$\Tr(AM)$}}\label{app:proof_TrAM_LOO}

\noindent
{\bf Step 0: Reduction to the rank $1$ case.}

Assume that we proved Eq.~\eqref{eq:det_part_TrAM} for all rank $1$ matrices $\bA = \bv \bv^\sT$. Let us show that it implies Proposition \ref{prop:TrAM_LOO} for all p.s.d.~matrices $\bA \in \R^{p \times p}$. We introduce a random vector $\bu \in \R^p$ independent of $\bX$ such that $\E[ \bu\bu^\sT ] = \bA$. Then we note that we can write
\begin{equation}\label{eq:reduction_to_rank_1}
\begin{aligned}
    \left| \E [ \Phi_1 ( \bX;\bA)] - \Psi_1 ( \mu_* ; \bA) \right| = &~ \left| \E_{\bu} \left[ \E_{\bX} [ \Phi_1 (\bX;\bu\bu^\sT) ] - \Psi_1 (\mu_* ; \bu \bu^\sT) \right]\right|\\
    \leq&~\E_{\bu} \left[ \left| \E_{\bX} [ \Phi_1 (\bX;\bu\bu^\sT) ] - \Psi_1 (\mu_* ; \bu \bu^\sT) \right| \right]\\
    \leq&~ C_{x,K} \frac{\varphi_1 (p) \nu_\lambda (n)^{5/2} }{\sqrt{n}}  \E_{\bu} \left[  \Psi_1 (\mu_* ; \bu \bu^\sT)\right] \\
    =&~ C_{x,K} \frac{\varphi_1 (p) \nu_\lambda (n)^{5/2} }{\sqrt{n}} \Psi_1 (\mu_* ; \bA).
\end{aligned}
\end{equation}
Therefore we assume in the rest of the proof that $\bA$ is a rank $1$ matrix and denote $\bA = \bv \bv^\sT$.

\noindent
{\bf Step 1: Decomposing the bound.}

Recall that we denote $\bX_- \in \R^{(n-1) \times p}$ the feature matrix where we removed one feature and its associated resolvent $\bR_- := (\bX_-^\sT \bX_- + \lambda )^{-1}$ and scaled resolvent $\bM_- := \bSigma^{1/2} \bR_- \bSigma^{1/2}$. We define the following deterministic matrix
\[
 \obM_- = \bSigma^{1/2} \left( \frac{n}{1 + \kappa} \bSigma + \lambda \right)^{-1} \bSigma^{1/2}  , \qquad \kappa = \E [ \Tr(\bM_- )] .
\]
The proof proceeds by decomposing the bound into two parts using $\obM_-$ as an intermediate matrix:
\[
\begin{aligned}
\big\vert \Phi_1 (\bX ; \bA) - \Psi_1 (\mu_*; \bA) \big\vert =&~ \big\vert \Tr \left\{ \bA (\E[\bM] - \obM ) \right\}\big\vert \\
\leq &~  \big\vert \Tr \left\{ \bA (\E[\bM] - \obM_- ) \right\}\big\vert + \big\vert \Tr \left\{ \bA (\obM_- - \obM ) \right\}\big\vert  .
\end{aligned}
\]

\noindent
{\bf Step 2: Bounding the term $\big\vert \E [\Tr \{ \bA (\bM - \obM_-) \}]\big\vert$.}

First, notice that we can write by exchangeability
\[
\E [\bM - \obM_- ] = \E \left[ \bM \left( \frac{n\id}{1+ \kappa} - \bZ^\sT \bZ \right) \obM_-  \right] = n \E \left[ \bM \left( \frac{\id}{1+ \kappa} - \bz \bz^\sT  \right) \obM_-  \right] .
\]
Using the Sherman-Morrison identities \eqref{eq:standard_identities_SMW}, we decompose the integrand into
\[
\bM \left( \frac{\id}{1+ \kappa} - \bz \bz^\sT  \right) \obM_-  = \bDelta_1 + \bDelta_2 + \bDelta_3, 
\]
where we introduced
\[
\begin{aligned}
    \bDelta_1 := &~ \frac{\bM_- (\id_p - \bz \bz^\sT ) \obM_-}{1 + \kappa} , \\
    \bDelta_2 := &~ \frac{\bM_- \bz \bz^\sT \obM_-}{(1+ \kappa) (1 +\bz^\sT \bM_- \bz)} (\bz^\sT \bM_- \bz - \kappa ) , \\
    \bDelta_3 :=&~ - \frac{\bM_- \bz \bz^\sT \bM_- \obM_-}{1 + \bz^\sT \bM_- \bz}.
\end{aligned}
\]
Note that $\bM_-$ is independent of $\bz$ which only appears in the numerator of $\bDelta_1$. Thus  $\E_\bz[\bDelta_1] = \bzero$ and it only remains to bound
\[
\big\vert \E [\Tr \{ \bA (\bM - \obM_-) \}]\big\vert \leq n \big\vert \E [ \Tr(\bA\bDelta_2) ] \big\vert + n \big\vert \E [ \Tr(\bA\bDelta_3) ] \big\vert .
\]

\paragraph*{Bounding $n \big\vert \E [ \Tr(\bA\bDelta_2) ] \big\vert$:}
 Using H\"older's inequality twice, first with respect to $\bz$ and then $\bM_-$, we obtain
\[
\begin{aligned}
   &~ \E \left[ \big\vert \bz^\sT \obM_- \bA \bM_- \bz  (\bz^\sT \bM_- \bz - \kappa )  \big\vert 
\right] \\
\leq&~  \E_{\bz} \left[( \bv^\sT \obM_- \bz )^{3} \right]^{1/3} \E_{\bM_-}\left[  \E_\bz \left[ (\bv^\sT \bM_- \bz )^{3} \right]^{1/3} \E_{\bz} \left[ ( \bz^\sT \bM_- \bz - \kappa)^3\right]^{1/3} \right] \\
\leq&~\E_{\bz} \left[( \bv^\sT \obM_- \bz )^{3} \right]^{1/3} \E_{\bM_-}\left[  \E_\bz \left[ (\bv^\sT \bM_- \bz )^{3} \right]^{2/3} \right]^{1/2}  \E_{\bM_-} \left[ \E_{\bz} \left[ ( \bz^\sT \bM_- \bz - \kappa)^3\right]^{2/3} \right]^{1/2} .
\end{aligned}
\] 
Each of these terms can be bounded using Lemma
\ref{lem:tech_bound_zAz} in Section \ref{app_det_equiv:technical}. For the first term, recall that $\| \obM_- \|_\op \leq (1+\kappa)/n$ by definition and $\obM_- \preceq C_{x,K} \cdot \nu_\lambda (n) \obM$ by Lemma \ref{lem:tech_upper_bound_AM}.(a). Thus, 
\[
\begin{aligned}
    \E_{\bz} \left[ (\bv^\sT \obM_- \bz )^3\right]^{1/3} \leq &~C_x \sqrt{ \bv^\sT \obM_-^2 \bv } \leq C_{x,K} \sqrt{\frac{(1+\kappa)\nu_\lambda(n)}{n} \Tr(\bA \obM)}.
\end{aligned}
\]
Similarly for the second term, using Lemma \ref{lem:tech_upper_bound_AM}.(b) applied to $\bM_-$, we obtain
\[
\begin{aligned}
\E_{\bM_-}\left[  \E_\bz \left[ (\bv^\sT  \bM_-  \bz )^{3} \right]^{2/3} \right]^{1/2} 
\leq &~ C_{x} \E_{\bM_-} \left[ \bv^\sT \bM_-^2 \bv \right]^{1/2} \leq C_{x,K} \frac{\nu_\lambda (n)}{\sqrt{n}}  \sqrt{ \Tr(\bA \obM)}.
\end{aligned}
\]
Finally, for the last term, using Lemma \ref{lem:tech_upper_bound_AM}.(b) and Lemma \ref{lem:tech_bound_M_EM} applied to $\bM_-$, and recalling that we assume that $\varphi_1(p) \nu_\lambda (n)^2 \log^{\beta + \frac{1}{2}} (n) \leq K \sqrt{n}$, we get
\[
 \E_{\bM_-} \left[ \E_{\bz} \left[ ( \bz^\sT \bM_- \bz - \kappa)^3\right]^{2/3} \right]^{1/2 } \leq  C_{x,K} \frac{\varphi_1(p) \nu_\lambda (n)}{\sqrt{n}} .
\]
Combining these bounds,  we achieve
\begin{equation}\label{eq:AM_bound_Delta_1}
    n \left\vert  \E [ \Tr(\bA \bDelta_2) ]\right\vert \leq C_{x,K} \frac{\varphi_1 (p) \nu_\lambda (n)^{5/2} }{\sqrt{n}}  \Tr(\bA \obM) .
\end{equation}

\paragraph*{Bounding $n \big\vert \E [ \Tr(\bA\bDelta_3) ] \big\vert$:}

Proceeding similarly as for $\bDelta_2$, we obtain
\begin{equation}\label{eq:AM_bound_Delta_2}
\begin{aligned}
n \big\vert \E [ \Tr(\bA\bDelta_3) ] \big\vert \leq&~ n \E \left[ (\bv^\sT \bM_- \bz)^2 \right]^{1/2} \E \left[(\bv^\sT \obM_- \bM_- \bz)^2 \right]^{1/2} \leq C_{x,K} \frac{\nu_\lambda (n)^3}{n}   \Tr(\bA \obM) ,
\end{aligned}
\end{equation}
where we used Lemma \ref{lem:tech_upper_bound_AM}.(a) and \ref{lem:tech_upper_bound_AM}.(b) to get
\[
\begin{aligned}
 \E \left[ \bz^\sT \bM_- \obM_- \bA \obM_- \bM_- \bz \right]
    \leq&~ C_x  \E \left[ \bv^\sT \obM_- \bM_-^2 \obM_- \bv \right] \leq C_{x,K} \frac{\nu_\lambda (n)^4}{n^3}  \Tr(\bA \obM)  .
\end{aligned}
\]

\paragraph*{Combining the bounds:} By combining the bounds \eqref{eq:AM_bound_Delta_1} and \eqref{eq:AM_bound_Delta_2}, and recalling that we assume $\nu_\lambda (n)^2 \leq K\sqrt{n}$, we deduce that
\begin{equation}\label{eq:TrAM_det_part1}
\begin{aligned}
    \big\vert \E [\Tr \{ \bA (\bM - \obM_-) \}]\big\vert
    \leq C_{x,K} \frac{\varphi_1 (p) \nu_\lambda (n)^{5/2} }{\sqrt{n}}  \Tr(\bA \obM).
    \end{aligned}
\end{equation}
Further note that by Lemma \ref{lem:tech_upper_bound_AM}.(b)
\[
\E[ \bz^\sT \bM_-^2 \bz] \leq C_{x} \left\{ \E[\Tr(\bM_-^2)] + \varphi_1 (p) \E[\| \bM_-^2 \|_F] \right\} \leq C_{x,K} \frac{\nu_{\lambda}(n)^2}{n } \left( 1 +  \frac{\varphi_1(p) }{n^{1/2}} \right) \leq C_{x,K} \frac{\nu_{\lambda}(n)^2}{n },
\]
where we used the assumption $\varphi_1(p) \leq K \sqrt{n}$. Thus, we can go through Step 2 with $\bA $ replaced by $ \id$, and obtain
\begin{equation}\label{eq:TrAM_det_part1_A=Id}
    \left\vert \E[\Tr (\bM - \obM_-)] \right\vert  \leq C_{x,K} \frac{\varphi_1(p) \nu_\lambda (n)^{5/2}}{\sqrt{n}}.
\end{equation}

\noindent
{\bf Step 3: Bounding the term $\big\vert \Tr \left\{ \bA (\obM_- - \obM ) \right\}\big\vert$.}

First, note that we can rewrite this term as 
\begin{equation}\label{eq:algebra_AoM-_AoM}
\begin{aligned}
   \left\vert\Tr ( \bA ( \obM_- - \obM ) ) \right\vert =&~ n \left\vert\Tr ( \bA  \obM_- \bSigma \obM )  \right\vert \cdot \frac{| \kappa - \Tr (\obM) |}{(1 + \kappa)(1+ \Tr ( \obM) )} \\
   \leq &~ \left\vert \Tr ( \bA \obM ) \right\vert \cdot \frac{| \kappa - \Tr (\obM) |}{1+ \Tr ( \obM) },
\end{aligned}
\end{equation}
where we used that $\| \obM_-\|_\op \leq (1 + \kappa)/n$ by definition and $\| \bSigma \|_\op =1$ by assumption.
We decompose the numerator into three terms
\begin{equation}\label{eq:decompo_kappa_oM}
 \kappa - \Tr (\obM) = \E[ \Tr ( \bM_- - \bM)] + \E[\Tr (\bM - \obM_-)] + \Tr (\obM_- - \obM ) .
\end{equation}
For the third term, applying inequality \eqref{eq:algebra_AoM-_AoM} to $\bA$ replaced by $\id$ produces
\[
\left\vert \Tr (  \obM_- - \obM  ) \right\vert \leq \frac{\Tr ( \obM ) }{1+ \Tr ( \obM)} \cdot \left\vert \kappa - \Tr (\obM) \right\vert , 
\]
which, when injected in Eq.~\eqref{eq:decompo_kappa_oM} and rearranging the terms, yields
\[
\left\vert \kappa  - \Tr (\obM) \right\vert \leq (1+ \Tr ( \obM)) \left\{ \left\vert \E[ \Tr ( \bM_- - \bM)] \right\vert + \left\vert \E[\Tr (\bM - \obM_-)] \right\vert \right\}  .
\]
The first inequality \eqref{eq:algebra_AoM-_AoM} becomes
\[
 \left\vert\Tr ( \bA ( \obM_- - \obM ) ) \right\vert\leq  \left\vert \Tr ( \bA \obM ) \right\vert\left\{ \left\vert \E[ \Tr ( \bM_- - \bM)] \right\vert + \left\vert \E[\Tr (\bM - \obM_-)] \right\vert \right\} .
\]
The first term is bounded in Lemma \ref{lem:tech_upper_bound_AM}.(c) 
\[
\left| \E[ \Tr ( \bM_- - \bM)] \right| \leq C_{x,K} \frac{\nu_\lambda (n)^2}{n}  .
\]
The second term was bounded by Eq.~\eqref{eq:TrAM_det_part1_A=Id} in Step 2. We conclude the bound of the first part by combining the above bounds
\begin{equation}\label{eq:TrAM_det_part2}
\begin{aligned}
    \left\vert\Tr ( \bA ( \obM_- - \obM ) ) \right\vert \leq C_{x,K} \frac{\varphi_1(p) \nu_\lambda (n)^{5/2}}{\sqrt{n}} \Tr(\bA \obM) .
    \end{aligned}
\end{equation}

Combining the bounds \eqref{eq:TrAM_det_part1} and \eqref{eq:TrAM_det_part2} concludes the proof of this proposition.

\subsubsection[Proof of Proposition \ref{prop:TrAM_martingale}: martingale part of $\Tr(AM)$]{Proof of Proposition \ref{prop:TrAM_martingale}: martingale part of \boldmath{$\Tr(AM)$}}\label{app:proof_TrAM_martingale}

\noindent
{\bf Step 1: Truncation of the martingale difference sequence.}

We rewrite this term as a martingale difference sequence
\[
\begin{aligned}
     S_n := \Tr(\bA\bM) - \E[\Tr(\bA\bM)] = \sum_{i =1}^n  \left( \E_i - \E_{i-1} \right) \Tr( \bA \bM) =:\sum_{i = 1}^n \Delta_i ,
\end{aligned}
\]
where we recall that we denote $\E_i$ the partial expectation over features $\{\bx_{i+1} , \ldots , \bx_n \}$. For any constant $R>0$, we define the truncated martingale sequence and the remainder to be
\[
\begin{aligned}
   \Tilde{S}_n :=&~ \sum_{i = 1}^n \Delta_i \ind_{\Delta_i \in [-R,R]} - \E_{i-1} \left[\Delta_i \ind_{\Delta_i \in [-R,R]}\right]  ,\\
   R_n :=&~ S_n - \Tilde{S}_n = \sum_{i =1}^n \Delta_i \ind_{i \not\in [-R,R]} - \E_{i-1} \left[ \Delta_i \ind_{\Delta_i \not\in [-R,R]}\right] ,
\end{aligned}
\]
so that $\Tilde{S}_n$ is a martingale difference sequence with increments bounded by $2R$. Azuma-Hoeffding inequality implies that
\begin{equation}\label{eq:Azuma-Hoeffding}
\P \left( |\Tilde{S}_n| \geq t \right) \leq 2 \exp \left( -\frac{t^2}{2nR^2}\right) .
\end{equation}
Therefore, we can choose a constant $C_D$ such that $|\Tilde{S}_n| \leq C_D \sqrt{n \log(n)} R $ with probability at least $1 - n^{-D}$. The rest of the proof consists in choosing $R$ such that $|\Delta_i| \leq R$ for all $i \in [n]$ with probability at least $1 - n^{-D}$, and bounding $\E_{i-1} \left[ \Delta_i \ind_{\Delta_i \not\in [-R,R]}\right]$ also with probability $1 -n^{-D}$.

\noindent
{\bf Step 2: Bounding $|\Delta_i|$ with high probability.}

In this step, we consider general p.s.d. matrix $\bA$. We defer the proof for the alternative bound in the case of $\| \bA \|_\op < \infty$ to Step 5.
Denote $\bM_{i}$ the scaled resolvent where we removed feature $\bx_i$. First note that
\[
\Delta_i = \left( \E_i - \E_{i-1} \right) \Tr( \bA \bM) = \left( \E_i - \E_{i-1} \right) \Tr\left\{ \bA (\bM - \bM_{i}) \right\}  ,
\]
where we used that $\E_i [\bM_{i}] = \E_{i-1} [\bM_{i}]$.
Using the Sherman-Morrison identities \eqref{eq:standard_identities_SMW}, we have the following decomposition
\begin{equation}\label{eq:decompo_martingale_term_AM}
\begin{aligned}
\Tr \left\{ \bA ( \bM - \bM_{i} ) \right\}=&~ \frac{   \bz_i^\sT \bM_{i} \bA \bM_{i} \bz_i }{1 + \bz_i^\sT \bM_{i} \bz_i} \\
=&~ \frac{   \bz_i^\sT \bM_{i} \bA \bM_{i} \bz_i }{1+\Tr(\bM_{i})} -  \frac{   \bz_i^\sT \bM_{i} \bA \bM_{i} \bz_i  ( \bz_i^\sT \bM_{i} \bz_i - \Tr(\bM_{i}))}{ (1+\Tr(\bM_{i})) (1 + \bz_i^\sT \bM_{i} \bz_i) }  .
\end{aligned}
\end{equation}
Let's consider the first term:
\[
\left( 1 - \E_{\bz_i} \right)  \bz_i^\sT \bM_{i} \bA \bM_{i} \bz_i = \bz_i^\sT  \bM_{i} \bA \bM_{i}  \bz_i - \Tr( \bM_{i} \bA \bM_{i} ) .
\]
From Lemma \ref{lem:tech_bound_zAz}, conditional on features $\bx_1 , \ldots , \bx_{i-1}$, there exists a constant $C_{x,D}$ such that
\[
\E_i \left[ \left\vert \bz_i^\sT \bM_{i} \bA \bM_{i}  \bz_i - \Tr( \bM_{i} \bA \bM_{i} ) \right\vert \right] \leq C_{x,D} \log^\beta (n) \varphi_1^{\bA} (p)  \E_{i} [\|\bM_{i} \bA \bM_{i} \|_F] 
\]
with probability at least $1 - n^{-D}$. For general p.s.d.~matrix $\bA$, we simply use that
\[
 \E_{i} [\| \bM_{i} \bA \bM_{i} \|_F ] \leq  \E_{i} [\Tr( \bM_{i} \bA \bM_{i} ) ] \leq C_{x,K,D} \frac{\nu_{\lambda} (n)^2}{n } \Tr(\bA \obM) 
\]
with probability at least $1 - n^{-D}$ by Lemma \ref{lem:tech_upper_bound_AM}.(c). 
For the second term and either $j \in \{i-1, i\}$, we use Cauchy-Schwarz inequality to get
\[
\begin{aligned}
&~  \E_j \left[   \left\vert \bz_i^\sT \bM_{i} \bA \bM_{i} \bz_i  ( \bz_i^\sT \bM_{i} \bz_i - \Tr(\bM_{i})) \right\vert \right] \\
\leq&~ \E_j \left[   \left\vert \bz_i^\sT \bM_{i} \bA \bM_{i} \bz_i \right\vert^2 \right]^{1/2}  \E_j \left[   \left\vert   \bz_i^\sT \bM_{i} \bz_i - \Tr(\bM_{i})\right\vert^2 \right]^{1/2}  .
\end{aligned}
\]
Using Lemma \ref{lem:tech_bound_zAz} over $\bz_i$ and Lemma \ref{lem:tech_upper_bound_AM}.(b), we get the following  with probability at least $1 - n^{-D}$. For general p.s.d.~matrix $\bA$,
\[
\begin{aligned}
    \E_i \left[   \left\vert \bz_i^\sT \bM_{i} \bA \bM_{i} \bz_i \right\vert^2 \right]^{1/2} \leq &~ C_{x,D} \varphi_1^{\bA} (p) \log^\beta (n) \E_i [\Tr( \bM_{i} \bA \bM_{i})^2]^{1/2} \\
    \leq&~ C_{x,D,K} \frac{\varphi_1^{\bA} (p)  \nu_\lambda (n)^2 \log^\beta (n) }{n} \Tr(\bA \obM).
\end{aligned}
\]
Finally, we have simply
\[
\begin{aligned} 
\E_i \left[   \left\vert   \bz_i^\sT \bM_{i} \bz_i - \Tr(\bM_{i}) \right\vert^2 \right]^{1/2}\leq C_{x,D,K} \frac{\varphi_1(p) \nu_\lambda (n) \log^\beta (n)}{\sqrt{n}}.
\end{aligned}
\]
The same bounds apply to $\E_{i-1}$ without the factor $\log^\beta (n)$.

Combining these bounds, the following hold with probability at least $1 - n^{-D}$. For general p.s.d.~matrix $\bA$, we have
\[
\begin{aligned}
    &~\left|  \left( \E_i - \E_{i-1} \right)  \frac{   \bz_i^\sT \bM_{i} \bA \bM_{i} \bz_i  ( \bz_i^\sT \bM_{i} \bz_i - \Tr(\bM_{i}))}{ (1+\Tr(\bM_{i})) (1 + \bz_i^\sT \bM_{i} \bz_i) }  \right| \\
    \leq&~  C_{x,D,K} \frac{\varphi^{\bA}_1(p) \varphi_1(p) \nu_\lambda (n)^3 \log^{2\beta} (n) }{n^{3/2}} \Tr(\bA \obM).
\end{aligned}
\]

We conclude via a union bound, changing the choice of $D$, that with probability at least $1 - n^{-D}$ we have for all $i \in [n]$: for general $\bA$ (recalling conditions~\eqref{eq:conditions_nu_phi})
\[
| \Delta_i | \leq C_{x,D,K} \frac{\varphi_1^\bA (p) \nu_\lambda (n)^2 \log^\beta (n) }{n} \Tr(\bA \obM).
\]
We denote $R_D$ either of this right-hand side.

\noindent
{\bf Step 3: Bounding $\E_{i-1} [ \Delta_i \ind_{\Delta_i \not\in [-R,R]} ]$ with high probability.}

Consider $\Tilde D > 0$ and $R_{\Tilde D}$ as chosen in the previous step, such that with probability at least $1 - n^{-\Tilde D}$, we have $|\Delta_i| \leq R_{\Tilde D}$ for all $i \in [n]$. By Cauchy-Schwarz inequality, 
\[
\E_{i-1} [ \Delta_i \ind_{\Delta_i \not\in [-R_{\Tilde D},R_{\Tilde D}]}] \leq \P_{i-1}(\Delta_i \not\in [-R_{\Tilde D},R_{\Tilde D}])^{1/2} \E_{i-1} [ \Delta_i^2]^{1/2} .
\]
Note that with probability at least $1 -n^{-D}$, by Lemma \ref{lem:tech_upper_bound_AM}.(b),
\[
\begin{aligned}
    \E_{i-1} [ \Delta_i^2]^{1/2} = &~ \E_{i-1}\left[ \left( \left( \E_i - \E_{i-1} \right) \Tr(\bA \bM) \right)^2 \right]^{1/2} \\
    \leq &~ 2 \E_{i-1} \left[ \Tr (\bA \bM)^2 \right]^{1/2} \leq C_{x,D,K} \cdot \nu_{\lambda} (n) \Tr (\bA \obM).
\end{aligned}
\]
Furthermore, by Markov's inequality,
\[
\begin{aligned}
    \P \left( \P_{i-1}(\Delta_i \not\in [-R_{\Tilde D},R_{\Tilde D}]) \leq n^{D} \P (\Delta_i \not\in [-R_{\Tilde D},R_{\Tilde D}] ) \right) \leq n^{-D}.
\end{aligned}
\]
Combining these bounds, we deduce that with probability at least $1 - n^{-D}$, 
\[
\left| \E_{i-1} [ \Delta_i \ind_{\Delta_i \not\in [-R_{\Tilde D},R_{\Tilde D}]}] \right| \leq  C_{x,D,K}  \cdot n^{(D-\Tilde D)/2} \nu_{\lambda} (n) \Tr (\bA \obM).
\]

\noindent
{\bf Step 4: Concluding the proof.}

Via union bound and changing the choice of $D$, and choosing $\Tilde D = D+6$, we deduce that with probability at least $1 - n^{-D}$,
\[
| R_n| \leq  \sum_{i \in [n]}\left| \E_{i-1} [ \Delta_i \ind_{\Delta_i \not\in [-R_{\Tilde D},R_{\Tilde D}]}] \right| \leq C_{x,D,K} \frac{\nu_\lambda (n)}{n^2} \Tr(\bA \obM).
\]
Combined with Eq.~\eqref{eq:Azuma-Hoeffding}, we conclude that
\[
| S_n| \leq C_D \sqrt{n \log(n)} R_{\Tilde D},
\]
 which finishes the proof for general $\bA$ by replacing $R_{\Tilde D}$ by its expression obtained in step 2.

\noindent
\textbf{Step 5: The case $\| \bA \|_\op < \infty$.} 

In the case of $\| \bA \|_\op < \infty$, we can tighten the bounds on the two terms in Eq.~\eqref{eq:decompo_martingale_term_AM}. For the first term, we get by Lemma \ref{lem:tech_upper_bound_AM}.(c) that with probability at least $1 - n^{-D}$,
\[
 \E_{i} [\| \bM_{i} \bA \bM_{i} \|_F ]  \leq C_{x,K,D} \frac{\nu_{\lambda} (n)^2}{n^{3/2} }  \| \bA \|_\op^{1/2} \Tr(\bA \obM)^{1/2} . 
\]
We therefore have with probability at least $1 - n^{-D}$
\[
\begin{aligned}
   \E_{i-1} \left[ \left| \left( 1 - \E_{\bz_i} \right)  \bz_i^\sT \bM_{i} \bA \bM_{i} \bz_i \right| \right] \leq&~ C_{x,K,D} \frac{\varphi^{\bA}_1 (p) \nu_{\lambda} (n)^2 \log^\beta (n)}{n^{3/2} }  \| \bA \|_\op^{1/2} \Tr(\bA \obM)^{1/2} \\
   =&~ C_{x,K,D} \frac{\varphi^{\bA}_1 (p) \nu_{\lambda} (n)^2 \log^\beta (n)}{n^{3/2}  \sqrt{ \zeta_1 \rho_1(\bA)}}   \Tr(\bA \obM).
\end{aligned}
\]

For the second term, 
\[
\begin{aligned}
   &~ \E_i \left[   \left\vert \bz_i^\sT \bM_{i} \bA \bM_{i} \bz_i \right\vert^2 \right]^{1/2} \\
    \leq &~ C_{x,D} \left\{ \E_i [\Tr( \bM_{i} \bA \bM_{i})^2]^{1/2} + \varphi_1^{\bA} (p) \log^\beta (n) \E_i [ \| \bM_{i} \bA \bM_{i}\|_F^2 ]^{1/2} \right\}  , \\
    \leq&~ C_{x,D,K} \frac{\nu_\lambda (n)^2}{n} \left\{ \Tr(\bA \obM ) + \delta_{\varphi_1^{\bA} \neq 1} \frac{\varphi_1 (p) \log^\beta (n) }{\sqrt{n} }   \sqrt{\| \bA \|_\op \Tr(\bA \obM)} \right\}.
\end{aligned}
\]
Hence, we obtain
\[
\begin{aligned}
    &~\left|  \left( \E_i - \E_{i-1} \right)  \frac{   \bz_i^\sT \bM_{i} \bA \bM_{i} \bz_i  ( \bz_i^\sT \bM_{i} \bz_i - \Tr(\bM_{i}))}{ (1+\Tr(\bM_{i})) (1 + \bz_i^\sT \bM_{i} \bz_i) }  \right| 
    \\
    \leq&~  C_{x,D,K} \frac{\varphi_1(p) \nu_\lambda (n)^3 \log^{\beta} (n) }{n^{3/2}} \left\{ \Tr(\bA \obM) + \delta_{\varphi_1^{\bA} \neq 1} \frac{\varphi_1 (p) \log^\beta (n) }{\sqrt{n} }   \sqrt{\| \bA \|_\op \Tr(\bA \obM)} \right\}.
\end{aligned}
\]
We conclude via an union bound: for all $i \in [n]$,
\[
| \Delta_i | \leq C_{x,D,K}  \frac{\varphi_1 (p) \nu_\lambda (n)^3 \log^\beta (n) }{n^{3/2}} \left\{ 1 + \frac{1}{\sqrt{\zeta_1 \rho_1(\bA)}}  \right\}\Tr(\bA \obM).
\]

Repeating Steps 3 and 4, we obtain the second bound \eqref{eq:mart_part_TrAM_2}.

\subsection[Deterministic equivalent for $\Tr (Z^\sT Z M)$]{Deterministic equivalent for \boldmath{$\Tr (Z^\sT Z M)$}}
\label{app_det_equiv:proof_ZZM}

We consider our second functional
\[
\Phi_2 (\bX) = \frac{1}{n} \Tr( \bX^\sT \bX ( \bX^\sT \bX + \lambda )^{-1} ) = \frac{1}{n}  \Tr(\bZ^\sT \bZ \bM) .
\]
We show in this section that $\Phi_2(\bX)$ is well approximated by the following deterministic equivalent:
\[
\Psi_2(\mu_*) = \frac{1}{n}  \Tr( \mu_* \bSigma ( \mu_* \bSigma + \lambda)^{-1} ) = \frac{1}{n} \Tr( \bSigma ( \bSigma + \lambda_* )^{-1} )  .
\]
Note that we have simply $\Psi_2 (\mu_*) = 1 - \lambda / (n \lambda_* )$ by definition of $\lambda_*$.

\begin{theorem}[Deterministic equivalent for $\Tr(\bZ^\sT \bZ \bM)$]\label{thm_app:det_equiv_TrZZM}
    Assume the features $(\bx_i)_{i\in[n]}$ satisfy Assumption \ref{ass_app:deterministic_equivalent}. For any constants $D,K>0$, there exist constants $\eta := \eta_x \in (0,1/2)$ (only depending on $\sfc_x,\sfC_x,\beta$), $C_{D,K} >0$ (only depending on $K,D$), and $C_{x,D,K}>0$ (only depending on $\sfc_x,\sfC_x,\beta,D,K$), such that the following holds. Define $\nu_\lambda (n)$ as per Eq.~\eqref{eq:definition_nu_lambda}. For all $n \geq C_{D,K}$ and  $\lambda >0$ satisfying
    \begin{equation}\label{eq:conditions_Phi_2}
\lambda \cdot \nu_\lambda (n) \geq n^{-K}, \qquad \varphi_1(p) \nu_\lambda (n)^2 \log^{\beta + \frac{1}{2}} (n) \leq K \sqrt{n} ,
    \end{equation}
     we have with probability at least $1 - n^{-D}$ that
    \begin{equation*}
         \big\vert \Phi_2(\bX) - \Psi_2 (\mu_*) \big\vert \leq C_{x,D,K} \frac{\varphi_1(p) \nu_\lambda (n)^{5/2} \log^{\beta + \frac{1}{2}} (n) }{\sqrt{ n  }}   \Psi_2 (\mu_* ) .
    \end{equation*}
\end{theorem}

To prove this second deterministic equivalent, we use a leave-one-out argument and rewrite $\Phi_2(\bX)$ as a sum of $\Phi_1 (\bX_i ; \bz_i \bz_i^\sT)$. For each $i \in [n]$, we can apply Theorem \ref{thm_app:det_equiv_TrAM} over the randomness in $\bX_i$ (conditional on $\bz_i$) and show that it concentrates on $\Psi_1 (\mu_* ; \bz_i \bz_i^\sT)$. We then conclude by using that the average of $\Psi_1 (\mu_* ; \bz_i \bz_i^\sT)$ concentrates on $\E_{\bz_i} [\Psi_1 (\mu_* ; \bz_i \bz_i^\sT)] = \Psi_1 (\mu_*;\id) $.


\begin{proof}[Proof of Theorem \ref{thm_app:det_equiv_TrZZM}]
First, observe that we can rewrite the functional as
\[
\Phi_2(\bX) = \frac{1}{n} \Tr(\bZ^\sT \bZ \bM) = \frac{1}{n} \sum_{i=1}^n \frac{\bz_i^\sT \bM_i \bz_i}{1 + \bz_i^\sT \bM_i \bz_i}.
\]
Introduce $\tmu_*$ the solution to the fixed point equation \eqref{eq:det_equiv_fixed_point_mu_star} with $n-1$ instead of $n$, and denote $\wbM = \bSigma ( \widetilde  \mu_* \bSigma + \lambda)^{-1}$. Applying Theorem \ref{thm_app:det_equiv_TrAM} with $\bA = \bz_i\bz_i^\sT$ over the randomness in $\bX_i$ and doing a union bound over $i \in [n]$, we have with probability at least $1 - n^{-D}$
\[
\begin{aligned}
    \left| \frac{1}{n} \sum_{i=1}^n \frac{\bz_i^\sT \bM_i \bz_i}{1 + \bz_i^\sT \bM_i \bz_i} - \frac{1}{n} \sum_{i=1}^n \frac{\bz_i^\sT \wbM \bz_i}{1 + \bz_i^\sT \wbM \bz_i} \right| \leq&~ \frac{1}{n} \sum_{i = 1}^n \frac{| \Phi_1 (\bX_i;\bz_i \bz_i^\sT) - \Psi_1 (\tmu_*;\bz_i \bz_i^\sT) |}{1 + \Psi_1 (\tmu_*;\bz_i \bz_i^\sT)}\\
    \leq&~ \cE_{1,n} \cdot \frac{1}{n} \sum_{i = 1}^n \frac{\Psi_1 (\tmu_*;\bz_i \bz_i^\sT)}{1+\Psi_1 (\tmu_*;\bz_i \bz_i^\sT)},
\end{aligned}
\]
where we used that $\cE_{1,n-1} \leq  C \cE_{1,n}$ for $n$ bigger than a universal constant.

Denote 
\[
S := \frac{1}{n}\sum_{i = 1}^n S_i, \qquad S_i := \left| \bz_i^\sT \wbM \bz_i - \Tr( \wbM) \right|, 
\]
and for a constant $R>0$,
\[
S_R := \frac{1}{n} \sum_{i = 1}^n S_i \ind_{| S_i | \leq R}.
\]
Recall that with probability at least $1 - n^{-D-1}$
\[
|S_i| \leq | \bz_i^\sT \wbM \bz_i - \Tr(\wbM) | \leq C_{x,D}\cdot  \varphi_1( p) \Tr( \wbM).
\]
Thus fixing $R$ to be the right-hand side of the previous display and using a union bound, we have $S = S_R$ with probability at least $1 - n^{-D}$. We can apply Hoeffding's inequality to $S_R$ and obtain with probability at least $1 - n^{-D}$,
\[
|S_R| \leq C_{D} \frac{\sqrt{n\log(n)} R}{n} \leq C_{x,D} \frac{\varphi_1 (p) \log^{\frac12}(n)}{\sqrt{n}} \Tr(\wbM).
\]
Combining these bounds, we get
\[
\begin{aligned}
    \left|\frac{1}{n} \sum_{i = 1}^n \frac{\Psi_1 (\tmu_*;\bz_i \bz_i^\sT)}{1+\Psi_1 (\tmu_*;\bz_i \bz_i^\sT)} - \frac{\Tr(\wbM)}{1+\Tr(\wbM)} \right| \leq &~ \frac{1}{1+ \Tr(\wbM)} \cdot \frac{1}{n} \sum_{i = 1}^n  \left| \Psi_1 (\tmu_*;\bz_i \bz_i^\sT) - \Tr(\wbM) \right| \\
    \leq&~ C_{x,D} \frac{\varphi_1 (p) \log^{\frac12}(n)}{\sqrt{n}} \frac{\Tr(\wbM)}{1 + \Tr(\wbM)}
\end{aligned}
\]
with probability at least $1 - n^{-D}$. Finally note that by Lemma \ref{lem:properties_mu_obM}, we have 
\[
\left| \frac{\Tr(\wbM)}{1 + \Tr(\wbM)} - \Psi_2(\mu_*) \right| = \left| \Psi_2(\tmu_*) - \Psi_2(\mu_*) \right| \leq C \frac{\nu_\lambda (n)}{n}\Psi_2(\mu_*).
\]
Combining the above bounds, and recalling the conditions \eqref{eq:conditions_Phi_2}, we conclude that
\[
\begin{aligned}
    \left| \Phi_2 (\bX) - \Psi_2 (\mu_*) \right| \leq&~ \left|\Phi_2 (\bX) - \frac{1}{n} \sum_{i = 1}^n \frac{\Psi_1 (\tmu_*;\bz_i \bz_i^\sT)}{1+\Psi_1 (\tmu_*;\bz_i \bz_i^\sT)} \right| + \left| \frac{1}{n} \sum_{i = 1}^n \frac{\Psi_1 (\tmu_*;\bz_i \bz_i^\sT)}{1+\Psi_1 (\tmu_*;\bz_i \bz_i^\sT)} - \Psi_2 (\tmu_* ) \right| \\
    &~ + \left| \Psi_2(\tmu_*) - \Psi_2(\mu_*) \right|\\
    \leq &~ C_{x,D,K} \frac{\varphi_1 (p) \nu_\lambda(n)^{5/2} \log^{\beta + \frac{1}{2}} (n)}{\sqrt{n}} \Psi_2 (\mu_*) ,
\end{aligned}
\]
with probability at least  $1 - n^{-D}$.
\end{proof}

\subsection[Deterministic equivalent for $\Tr (AM^2)$]{Deterministic equivalent for \boldmath{$\Tr (AM^2)$}}
\label{app_det_equiv:proof_MM}

In this section, we consider the third functional 
\[
\Phi_3 (\bX;\bA) = \Tr( \bA \bSigma^{1/2} ( \bX^\sT \bX + \lambda)^{-1} \bSigma ( \bX^\sT \bX + \lambda)^{-1}\bSigma^{1/2}) = \Tr(\bA \bM^2)  .
\]
We show that $\Phi_3 (\bX;\bA)$ is well approximated by the following deterministic equivalent:
\[
\Psi_3 (\mu_* ; \bA) = \frac{\Tr( \bA \bSigma^2 ( \mu_* \bSigma + \lambda)^{-2})}{1 - \frac{1}{n} \Tr(\bSigma^2 (\bSigma + \lambda_*)^{-2})} .
\]
 We assume without loss of generality that $\Tr(\bA \bSigma^2) < \infty$. In particular, this allows to consider matrices $\bA = \bSigma^{-1} \bv \bv^\sT \bSigma^{-1}$ with $\| \bv \|_2 < \infty$. 

\begin{theorem}[Deterministic equivalent for $\Tr(\bA \bM^2)$]\label{thm_app:det_equiv_TrAMM}
    Assume the features $(\bx_i)_{i\in[n]}$ satisfy Assumption \ref{ass_app:deterministic_equivalent}. For any constants $D,K>0$, there exist constants $\eta := \eta_x \in (0,1/2)$ (only depending on $\sfc_x,\sfC_x,\beta$), $C_{D,K} >0$ (only depending on $K,D$), and $C_{x,D,K}>0$ (only depending on $\sfc_x,\sfC_x,\beta,D,K$), such that the following holds. Define $\nu_\lambda (n)$ as per Eq.~\eqref{eq:definition_nu_lambda}. For all $n \geq C_{D,K}$ and  $\lambda >0$ satisfying
    \begin{equation}\label{eq:conditions_thm3}
\lambda \cdot \nu_\lambda (n) \geq n^{-K}, \qquad \varphi_1(p) \nu_\lambda (n)^{5/2} \log^{\beta + \frac{1}{2}} (n) \leq K \sqrt{n} ,
    \end{equation}
    and deterministic p.s.d.~matrix $\bA$, we have with probability at least $1 - n^{-D}$ that
    \begin{equation*}
         \big\vert \Phi_3(\bX;\bA) - \Psi_3 (\mu_* ; \bA) \big\vert \leq C_{x,D,K} \frac{\varphi_1(p) \nu_\lambda (n)^{6} \log^{2\beta + \frac12} (n) }{\sqrt{n}}  \Psi_3 (\mu_* ; \bA) .
    \end{equation*}
\end{theorem}

To prove this theorem, we follow the same strategy as in the proof of Theorem \ref{thm_app:det_equiv_TrAM}. We decompose the bound into 
\[
\begin{aligned}
       \left\vert  \Phi_3(\bX;\bA) - \Psi_3 (\mu_* ; \bA) \right\vert  \leq&~ \underbrace{ \left\vert  \E [\Phi_3(\bX;\bA)] - \Psi_3 (\mu_* ; \bA) \right\vert}_{(\sfD)}  + \underbrace{\left\vert  \Phi_3(\bX;\bA) - \E [\Phi_3(\bX;\bA)] \right\vert}_{(\sfM)} ,
\end{aligned}
\]
where the deterministic $(\sfD)$ and martingale $(\sfM )$ parts are bounded in the two following propositions.

\begin{proposition}[Deterministic part of $\Tr(\bA\bM^2)$]\label{prop:TrAMM_LOO}
    Assume the same setting as Theorem \ref{thm_app:det_equiv_TrAMM}. Then there exist constants $C_K$ and $C_{x,K}$, such that the following holds. For all $n \geq C_K$ and  $\lambda >0$ satisfying Eq.~\eqref{eq:conditions_thm3}, and for all p.s.d.~matrix $\bA$, we have
    \begin{equation}\label{eq:det_part_TrAMM}
         \big\vert\E [ \Phi_3(\bX;\bA) ] - \Psi_3 (\mu_* ; \bA) \big\vert \leq C_{x,K} \frac{\varphi_1 (p) \nu_\lambda (n)^{6} }{\sqrt{n}}   \Psi_3 (\mu_* ; \bA) .
    \end{equation}
\end{proposition}

\begin{proposition}[Martingale part of $\Tr(\bA\bM^2)$] \label{prop:TrAMM_martingale}
Assume the same setting as Theorem \ref{thm_app:det_equiv_TrAMM}. Then there exist constants $C_{K,D}$ and $C_{x,K,D}$, such that the following holds. For all $n \geq C_{K,D}$ and  $\lambda >0$ satisfying Eq.~\eqref{eq:conditions_thm3}, and for all p.s.d.~matrix $\bA$, we have with probability at least $1 -n^{-D}$ that
    \begin{equation}\label{eq:mart_part_TrAMM}
         \big\vert \Phi_3(\bX;\bA)  - \E [ \Phi_3(\bX;\bA) ]  \big\vert \leq C_{x,K,D} \frac{\varphi^\bA_1 (p) \nu_\lambda (n)^{4} \log^{2\beta +\frac12} (n)}{\sqrt{n}}   \Psi_3 (\mu_* ; \bA) .
    \end{equation}
\end{proposition}

Similarly to Proposition \ref{prop:TrAM_martingale}, we can strengthen the bound on the martingale part by further assuming $\| \bSigma^{1/2} \bA \bSigma^{1/2} \|_\op < \infty$ and get
\[
 \big\vert \Phi_3(\bX;\bA)  - \E [ \Phi_3(\bX;\bA) ]  \big\vert \leq  C_{x,K,D} \frac{\varphi_1 (p) \nu_\lambda (n)^{4} \log^{2\beta +\frac12} (n)}{n \sqrt{\zeta_1 \rho_1 \big(\obM^{1/2} \bA \obM^{1/2}\big)} }  \Psi_3 (\mu_* ; \bA) ,
\]
by following a similar proof. 

The proof of these two propositions can be found in the next two sections. Theorem \ref{thm_app:det_equiv_TrAMM} is obtained by combining the bounds \eqref{eq:det_part_TrAMM} and \eqref{eq:mart_part_TrAMM}.

\subsubsection[Proof of Proposition \ref{prop:TrAMM_LOO}: deterministic part of $\Tr(AM^2)$]{Proof of Proposition \ref{prop:TrAMM_LOO}: deterministic part of \boldmath{$\Tr(AM^2)$}}\label{app:proof_TrAMM_LOO}

The proof will proceeds similarly to the proof of Proposition \ref{prop:TrAM_LOO}. For the sake of brevity, we will omit some repetitive details. Recall that, following the same argument as in Eq.~\eqref{eq:reduction_to_rank_1}, we can reduce ourselves to the rank-$1$ case $\bA = \bv \bv^\sT$. 

We start by considering the difference
\[
\E[\Tr(\bA \bM^2) - \Tr(\bA \obM_-^2) ]=\underbrace{2\E\left[\Tr \left\{ \bA \left( \bM - \obM_- \right) \obM_-\right\} \right]}_{=: B_1 (\bA)} +  \underbrace{\E \left[ \Tr\left\{ \bA \left( \bM - \obM_- \right)^2 \right\} \right]}_{=: B_2(\bA)} . 
\]
For convenience, we introduce the following notations
\[
\begin{aligned}
\bQ_- :=&~ \obM_- \bA\obM_-, \qquad\qquad \bQ = \obM \bA\obM,\\
\tcE_n (\bA) :=&~ \frac{\varphi_1 (p)}{\sqrt{n}} \Big\{ 1 + \delta_{\varphi^{\bA} \neq 1} \cdot   \frac{\varphi_1(p)}{\sqrt{n\zeta_1 } } \rho_1 \Big(\obM^{1/2} \bA \obM^{1/2}\Big)^{-1/2}\Big\}.
\end{aligned}
\]
We will further drop the dependency in $\bA$ and simply denote $\tcE_n$ in the case $\varphi^{\bA} = 1$, i.e., $\tcE_n = \varphi_1(p) / \sqrt{n}$.
Under the setting of Proposition \ref{prop:TrAMM_LOO}, we will show the following three claims:

\begin{claim}\label{claim:B1}
    There exists a constant $C_{x,K}$ such that
    \begin{align}\label{eq:claim1_B1}
    | B_1 (\bA) | \leq C_{x,K}  \cdot  \nu_\lambda (n)^{6} \tcE_n \cdot \Tr( \bA \obM^2 ).
    \end{align}
\end{claim}

\begin{claim}\label{claim:B2}
    There exists a constant $C_{x,K}$ such that
    \begin{align}\label{eq:claim2_B2}
    \left| B_2(\bA) - \frac{n \E [ \Tr (\bM^2) ]}{(1 + \kappa)^2}\Tr( \obQ_-) \right| \leq C_{x,K} \cdot \nu_\lambda (n)^{6} \tcE_n (\bA) \cdot \Tr( \bA \obM^2 ).
    \end{align}
\end{claim}

\begin{claim}\label{claim:TrM2-}
     There exists a constant $C_{x,K}$ such that
     \begin{align}
     \left| \Tr( \obQ_- ) - \Tr (\bA \obM^2) \right| \leq&~ C_{x,K} \cdot \nu_\lambda (n)^{5/2}\tcE_n \cdot \Tr( \bA \obM^2 ), \label{eq:claim3_eq1}\\
     \left| \frac{n \Tr( \obQ_- )}{(1 + \kappa)^2} - \frac{n\Tr(\bA \obM^2)}{(1 + \Tr(\obM))^2}\right| \leq &~ C_{x,K} \cdot \nu_\lambda (n)^{5/2}\tcE_n \cdot n\Tr(\bA \obM^2). \label{eq:claim3_eq2}
     \end{align}
\end{claim}

Combining Claims \ref{claim:B1} and \ref{claim:B2} implies the bound
\[
\left| \E[ \Tr(\bA \bM^2) ] - \Tr( \obQ_-) - \frac{n \E [ \Tr(\bM^2)]}{(1+\kappa)^2} \Tr(\obQ_-)\right| \leq  C_{x,K}  \cdot  \nu_\lambda (n)^{6} \tcE_n (\bA) \cdot \Tr( \bA \obM^2 ).
\]
First consider the case of $\bA = \id$. Combining the above display with Claim \ref{claim:TrM2-} yields
\begin{equation}\label{eq:combine_MM}
\begin{aligned}
&~\left| \E[ \Tr( \bM^2) ] - \Tr( \obM^2) - \frac{n \Tr(\obM^2)}{(1+\Tr(\obM))^2} \E [ \Tr(\bM^2)]  \right| \\
\leq&~ C_{x,K} \nu_\lambda (n)^{6} \tcE_n (\id) \Tr( \obM^2 ) + \left| \Tr( \obM_-^2 ) - \Tr (\obM^2) \right|  \\
&~+ \left|  \frac{n \Tr( \obM_-^2 )}{(1 + \kappa)^2} - \frac{n\Tr( \obM^2)}{(1 + \Tr(\obM))^2}\right| \E[\Tr(\bM^2)] \\
\leq&~ C_{x,K}  \cdot  \nu_\lambda (n)^{6} \tcE_n (\id) \cdot \Tr(  \obM^2 ),
\end{aligned}
\end{equation}
where we used that $\E[ \Tr(\bM^2)] \leq C_{x,K} \nu_\lambda (n)^2 \Tr(\obM^2)$ by Lemma \ref{lem:tech_upper_bound_AM}.(b). 
Observe that 
\[
1 - \frac{n \Tr(\obM^2)}{(1+\Tr(\obM))^2} = 1 - \frac{1}{n} \Tr( \bSigma^2 (\bSigma +\lambda_*)^{-2} ) \geq 1 - \frac{1}{n} \Tr( \bSigma (\bSigma +\lambda_*)^{-1} ) = \frac{\lambda}{n\lambda_*}>0.
\]
Thus, dividing both side of the inequality~\eqref{eq:combine_MM} by $1 - n \Tr(\obM^2)/(1+\Tr(\obM))^2$ results in
\[
\left| \E [ \Tr( \bM^2) ] - \Psi_3 (\mu_* ; \id )\right| \leq C_{x,K} \nu_\lambda(n)^6 \tcE_n (\id) \cdot \Psi_3 (\mu_* ; \id ).
\]
In fact, going through the bounds in the proof of Claims \ref{claim:B1} and \ref{claim:B2} with $\bA = \id$ without keeping track of the relative bounds, we have the following upper bound
\[
\left| \E [ \Tr( \bM^2) ] - \Psi_3 (\mu_* ; \id )\right| \leq C_{x,K} \frac{\nu_\lambda(n)^6\varphi_1(p)}{n^{3/2}}\cdot \frac{1}{1 - \frac{1}{n}\Tr(\bSigma^2 (\bSigma +\lambda_*)^{-2})}.
\]
Finally, note that
\[
\Tr( \bA \obM^2) + \frac{n\Psi_3(\mu_*;\id)}{(1+ \Tr(\obM))^2}\Tr( \bA \obM^2)  = \Psi_3 (\mu_* ; \bA).
\]
By combining the above displays, we conclude that
\[
\begin{aligned}
    \left| \E[ \Tr(\bA \bM^2) ] - \Psi_3 (\mu_* ; \bA) \right| \leq&~C_{x,K} \nu_\lambda (n)^{6} \tcE_n (\bA) \Tr( \bA\obM^2 ) + \left| \Tr( \bA\obM_-^2 ) - \Tr (\bA\obM^2) \right|  \\
&~+ \left|  \frac{n \Tr( \bA \obM_-^2 )}{(1 + \kappa)^2} - \frac{n\Tr( \bA \obM^2)}{(1 + \Tr(\obM))^2}\right| \E[\Tr(\bM^2)] \\
&~+ \frac{n\Tr( \bA \obM^2)}{(1 + \Tr(\obM))^2} \left| \E [ \Tr( \bM^2) ] - \Psi_3 (\mu_* ; \id )\right| \\
\leq&~ C_{x,K}  \cdot  \nu_\lambda (n)^{6} \tcE_n (\bA) \cdot \Psi_3 (\mu_* ; \bA),
\end{aligned}
\]
where we used that $\E[ \Tr(\obM^2)] \leq C_{x,K} \nu_\lambda(n)^2/ n$ and $\Tr(\bA \obM^2) \leq \Psi_3 (\mu_* ;\bA)$. Recalling that we can reduce ourselves to the case $\bA$ rank $1$ (see Eq.~\eqref{eq:reduction_to_rank_1}), this finishes the proof of Proposition \ref{prop:TrAMM_LOO} by replacing $\tcE_n$ by its expression.

The rest of this section is devoted to proving Claims \ref{claim:B1}-\ref{claim:TrM2-}.

\paragraph*{Proof of Claim \ref{claim:B1}.}
To establish this claim, we simply observe that we can rewrite $B_1$ as
\[
\Tr \left\{ \bA \left( \bM - \obM_- \right) \obM_-\right\} = \Tr\left\{ \obM_- \bA \obM_- \left( \frac{n\id}{1 + \kappa}  - \bZ^\sT \bZ\right) \bM \right\} . 
\]
Thus, we can follow the same proof as in Proposition \ref{prop:TrAM_LOO} and immediately obtain Eq.~\eqref{eq:claim1_B1}.

\paragraph*{Proof of Claim \ref{claim:B2}.}
For convenience, let us introduce the following notations: we denote $\bM_i$ the scaled resolvent where we removed feature $\bx_i$ for $i\in \{1,2\}$, and $\bM_{12}$ the scaled resolvent where both features $\bx_1, \bx_2$ are removed. We will further denote for $i\in \{1,2\}$
\[
S_i = \bz_i^\sT \bM_i \bz_i  , \qquad \tS_i = \bz_i^\sT \bM_{12} \bz_i  .
\]
Following the same decomposition as in proof of Proposition \ref{prop:TrAM_LOO}, we introduce for $i \in \{1,2\}$
\[
\begin{aligned}
    \bDelta_{i1} = &~\frac{\id - \bz_i \bz_i^\sT}{1 + \kappa} , \qquad \bDelta_{i2} = \frac{\bz_i \bz_i^\sT }{(1 + \kappa) (1 + S_i)} (S_i - \kappa)  ,\\
    \bDelta_{i3} =&~ - \frac{ \bz_i \bz_i^\sT\bM_i}{1 + S_i} ,
\end{aligned}
\]
so that the following identity holds
\begin{equation}\label{eq:decompo_Delta_i_M}
\left( \frac{\id}{1 + \kappa}  - \bz_i^\sT \bz_i\right) \bM =  \left( \bDelta_{i1} +\bDelta_{i2} + \bDelta_{i3} \right)^\sT \bM_i  .
\end{equation}

By exchangeability, we can decompose $B_2 (\bA)$ as
\[
\begin{aligned}
 B_2(\bA) =&~ \E \left[ \Tr \left\{ \obM_- \bA \obM_- \left( \frac{n\id}{1 + \kappa}  - \bZ^\sT \bZ\right) \bM^2 \left( \frac{n\id}{1 + \kappa}  - \bZ^\sT \bZ\right) \right\} \right]  = (\rmI) + (\rmII), 
\end{aligned}
\]
where
\[
\begin{aligned}
    (\rmI) &= n \E \left[ \Tr \left\{ \obM_- \bA \obM_- \left( \frac{\id}{1 + \kappa}  - \bz_1^\sT \bz_1\right) \bM^2 \left( \frac{\id}{1 + \kappa}  - \bz_1^\sT \bz_1\right) \right\} \right]\\
    &= n\E\Tr \left( \obQ_- \left\{ \bDelta_{11} +\bDelta_{12} + \bDelta_{13} \right\}^\sT  \bM_1^2 \left\{ \bDelta_{11} +\bDelta_{12} + \bDelta_{13} \right\} \right) \\
    &= n\E\sum_{r,s \in [3]}  \Tr\left( \obQ_- \bDelta_{1r}^\sT   \bM_1^2 \bDelta_{1s}  \right)  .
    \\
 (\rmII) &=  n (n-1) \E \left[ \Tr \left\{ \obM_- \bA \obM_- \left( \frac{\id}{1 + \kappa}  - \bz_1^\sT \bz_1\right) \bM^2 \left( \frac{\id}{1 + \kappa}  - \bz_2^\sT \bz_2\right) \right\} \right] \\
 &=
n (n-1)\E \Tr \left( \obQ_- \left\{ \bDelta_{11} +\bDelta_{12} + \bDelta_{13} \right\}^\sT  \bM_1 \bM_2 \left\{ \bDelta_{21} +\bDelta_{22} + \bDelta_{23} \right\} \right) \\
 &=n (n-1)\E \sum_{r,s \in [3]}  \Tr\left( \obQ_- \bDelta_{1r}^\sT   \bM_1 \bM_2 \bDelta_{2s}  \right) .
\end{aligned}
\]

Each of these twelve terms can be bounded in an analogous manner. Here, we will demonstrate the approach for the six leading terms, and omit the details for the other terms for the sake of brevity. We will show that the only non-vanishing term is
$\obQ_- \bDelta_{11}^\sT   \bM_1^2 \bDelta_{11}$.
For the remaining error terms, we will liberally use the bounds
\begin{equation*}
    \E[((S_i - \kappa))^q]^{1/q} \le C_{x,K,q} \frac{\varphi_1(p) \nu_\lambda(n)}{\sqrt{n}},
\end{equation*}
whose proof follows similarly to Step 2 in the proof of Proposition~\ref{prop:TrAM_LOO}, and
\begin{equation*}
    \E_{\bz_1,\bz_2}[(\bz_1^\sT \bB \bz_2)^q]^{1/q} \le C_{x,K} \left(\| \bB\|_F + \varphi_1(p)^{1/2} \|\bB^2\|_F^{1/2}\right), 
\end{equation*}
where $\bB$ is independent of $\bz_1,\bz_2$, which follows from Lemma~\eqref{lem:tech_bound_zAz}. Furthermore, for the terms of $(\rmII)$ we will heavily use the expansion
\begin{equation}\label{eq:expansion_M1M2}
\begin{aligned}
    \bM_1\bM_2 = &~\left( \bM_{12} - \frac{\bM_{12} \bz_2 \bz_2^\sT \bM_{12}}{1 + \tS_2}\right) \left( \bM_{12} - \frac{\bM_{12} \bz_1 \bz_1^\sT \bM_{12}}{1 + \tS_1}\right) \\
    =&~ \bM_{12}^2 - \frac{\bM_{12}\bz_2\bz_2^\sT \bM_{12}^2}{1 + \tS_2} 
    - \frac{\bM_{12}^2 \bz_1\bz_1^\sT \bM_{12}}{1+\tS_1} + \frac{\bM_{12}\bz_2\bz_2^\sT \bM_{12}^2 \bz_1\bz_1^\sT \bM_{12}}{(1+\tS_1)(1+\tS_2)},
\end{aligned}
\end{equation}
which follows from the Sherman-Morrison identities \eqref{eq:standard_identities_SMW}.

\paragraph*{Bounding $\E [ \Tr\left( \obQ_- \bDelta_{11}^\sT   \bM_1^2 \bDelta_{11}  \right)]$ of $(\rmI)$:}
We expand $\bDelta_{11}$ and obtain
\[
\begin{aligned}
    (1 +\kappa)^2 \E \left[ \Tr\left( \obQ_- \bDelta_{11}^\sT   \bM_1^2 \bDelta_{11}  \right) \right] =&~ -  \E [ \Tr\left( \obQ_-  \bM_1^2   \right)] +  \E[\bz_1^\sT \obQ_- \bz_1 \bz_1^\sT \bM_1^2 \bz_1]  .
\end{aligned}
\]
Using Lemma \ref{lem:tech_upper_bound_AM}.(b), the first term is bounded by
\[
\begin{aligned}
    \left| \E [ \Tr\left( \obQ_-  \bM_1^2   \right)] \right| \leq&~ C_{x,K} \frac{\nu_\lambda(n)^2}{n^2} \Tr( \obQ_-)  \leq C_{x,K} \frac{\nu_\lambda(n)^4}{n^2} \Tr( \bA \obM^2).
\end{aligned}
\]
Meanwhile, the second term satisfy
\[
\begin{aligned}
    \left| \E[\bz_1^\sT \obQ_- \bz_1 \bz_1^\sT \bM_1^2 \bz_1] - \Tr(\obQ_-)\E [\Tr(\bM_1^2)] \right| 
    \leq&~ \Tr(\obQ_-) \E_{\bz_1}\left[ \left(\bz_1^\sT \E[\bM_1^2] \bz_1 - \E[\Tr( \bM_1)^2]\right)^2 \right]^{1/2} \\
    \leq&~ C_{x,K} \frac{\varphi_1(p) \nu_\lambda (n)^4}{n^{3/2}}\Tr(  \bA \obM^2 ) .
\end{aligned}
\]

\noindent Further using that by Lemma \ref{lem:tech_upper_bound_AM}.(c)
\[
\Tr(\obQ_-) \left| \E [\Tr(\bM_1^2)] - \E [\Tr(\bM^2)] \right|\leq C_{x,K} \frac{\nu_\lambda (n)^4}{n^2} \Tr(\obQ_-) \leq C_{x,K} \frac{\nu_\lambda (n)^5}{n^{3/2}}\Tr(  \bA \obM^2 ) ,
\]
we deduce
\begin{align*}
n   \left| \E [ \Tr\left( \obQ_- \bDelta_{11}^\sT   \bM_1^2 \bDelta_{11}  \right)] - \frac{\Tr(\obQ_-)\E [\Tr(\bM^2)]}{(1+\kappa)^2} \right| &\leq C_{x,K}  \nu_\lambda (n)^5 \tcE_n \Tr( \bA \obM^2 ) .
\end{align*}

\paragraph*{Bounding $\E [ \Tr\left( \obQ_- \bDelta_{11}^\sT   \bM_1^2 \bDelta_{12}  \right)]$ of $(\rmI)$:}

Expanding the matrix $\bDelta_{11}$ once again we obtain
\begin{align*}
    &~(1+\kappa)^2\E[\Tr( \obQ_- \bDelta_{11}^\sT   \bM_1^2 \bDelta_{12} )]
    \\
    =&~\E\left[\frac{\Tr(\obQ_- \bM_1^2 \bz_1\bz_1^\sT) (S_1 -\kappa) }{(1 + S_1)}\right] 
    - \E\left[ \frac{\Tr(\obQ_- \bz_1\bz_1^\sT \bM_1^2 \bz_1\bz_1^\sT) (S_1 - \kappa)}{(1+S_1)}\right].
\end{align*}
The first term is bounded similarly as in the previous paragraph: 
\begin{align*}
    &~ \left| \E\left[\frac{\Tr(\obQ_- \bM_1^2 \bz_1\bz_1^\sT) (S_1 -\kappa) }{(1 + S_1)}\right] \right|\\
    \leq &~  \E \left[ (S_1 -\kappa)^3\right]^{1/3}\E \left[ (\bz_1^\sT \bM_1^2 \obQ_- \bM_1^2 \bz_1)^{3/2}\right]^{1/3}  \E\left[(\bz_1^\sT \obQ_- \bz_1 )^{3/2} \right]^{1/3} \\
    \le&~ C_{x,K} \frac{\varphi_1 (p) \nu_\lambda (n)}{\sqrt{n}}\left\{ \E\left[\Tr\left(\bM_1^2\obQ_- \bM_1^2\right)^{3/2}\right]^{2/3} + \delta_{\varphi_1^\bA \neq 1} \cdot \varphi_1 (p) \E \left[ \|\bM_1^2 \obQ_- \bM_1^2\|_F^{3/2} \right]^{2/3} \right\}^{1/2}\\
    &~ \times \left\{ \Tr(\obQ_-) + \delta_{\varphi_1^\bA \neq 1} \cdot \varphi_1 (p)  \| \bQ_- \|_F \right\}^{1/2} \\
    \le&~  C_{x,K} \frac{\varphi_1(p) \nu_\lambda(n)^5}{n^{5/2}} \left\{ \Tr(\bA \obM^2) + \delta_{\varphi_1^\bA \neq 1} \cdot \frac{\varphi_1(p) }{\sqrt{n}} \| \obM^{1/2} \bA \obM^{1/2} \|_\op^{1/2} \Tr(\bQ)^{1/2} \right\}\\
\le&~ C_{x,K} \frac{\nu_\lambda(n)^{5}}{n^2} \tcE_{n} (\bA) \Tr(\bA \obM^2) \\
\leq&~ C_{x,K} \cdot \frac{1}{n}\tcE_n (\bA) \Tr(\bA \obM^2), 
\end{align*}
where we used condition Eq.~\eqref{eq:conditions_thm3} in the last inequality.

Meanwhile, the second term can be bounded via H\"older's as 
\begin{equation}\label{eq:D11M11D22_T2} 
\begin{aligned}
    &~\E\left[ \frac{\Tr(\obQ_- \bz_1\bz_1^\sT \bM_1^2 \bz_1\bz_1^\sT) (S_1 - \kappa)}{(1+S_1)}\right]\\
    \le&~ 
    \E \left[(\bz_1^\sT\obQ_- \bz_1)^3\right]^{1/3}\E\left[(\bz_1^\sT \bM_1^2 \bz_1)^3\right]^{1/3} \E\left[(S_1 - \kappa)^3\right]^{1/3}\\
    \le &~
    C_{x,K}\left(\Tr(\obQ_-) + \delta_{\varphi_1^\bA \neq 1} \cdot \frac{\varphi_1(p)\nu_\lambda(n)^{1/2}}{\sqrt{n}}\sqrt{\| \obM_-^{1/2} \bA \obM_-^{1/2}\|_\op \Tr(\obQ_-)}  \right)
   \cdot \frac{\varphi_1(p) \nu_\lambda(n)^3}{n^{3/2}}  \\
   \leq&~C_{x,K} \frac{\nu_\lambda(n)^{5}}{n} \tcE_n (\bA) \cdot \Tr(\bA \obM^2),
    \end{aligned}
\end{equation}
which gives the desired error bound when scaled by the factor $n$.


\paragraph*{Bounding $\E [ \Tr\left( \obQ_- \bDelta_{12}^\sT   \bM_1^2 \bDelta_{13}  \right)]$ of $(\rmI)$:}
Replacing $\bDelta_{1s}$ by their expression, we get
\begin{align}
    |(1+\kappa)\E [ \Tr( \obQ_- \bDelta_{12}^\sT   \bM_1^2 \bDelta_{13}  )]|
    &\le \E[|\Tr(\obQ_- \bz_1\bz_1^\sT\bM_1^2 \bz_1\bz_1^\sT \bM_1 )(S_1 -\kappa)   |].
\end{align}
By comparing to Eq.~\eqref{eq:D11M11D22_T2}, one directly sees that this term is of lower order than the previous term (by noting the additional factor of $\bM_1$).
Note that a similar argument bounds the remaining terms of $(\rmI)$ which we omit.

\paragraph*{Bounding $\Tr\left( \obQ_- \bDelta_{11}^\sT   \bM_1 \bM_2 \bDelta_{21}  \right)$ of $(\rmII)$:}
Expanding $\bM_1\bM_2$ as described in Eq.~\eqref{eq:expansion_M1M2}, we have
\[
\begin{aligned}
&~ \E \left[\Tr\left( \obQ_- \bDelta_{11}^\sT   \bM_1 \bM_2 \bDelta_{21}  \right) \right]\\
=&~ \E \left[ \Tr\left( \obQ_- \bDelta_{11}^\sT   \left( \bM_{12} - \frac{\bM_{12} \bz_2 \bz_2^\sT \bM_{12}}{1 + \tS_2}\right) \left( \bM_{12} - \frac{\bM_{12} \bz_1 \bz_1^\sT \bM_{12}}{1 + \tS_1}\right) \bDelta_{21} \right) \right]\\
=&~\E \left[ \frac{1}{(1+ \tS_1)(1+\tS_2)}\Tr\left( \obQ_- \bDelta_{11}^\sT \bM_{12} \bz_2 \bz_2^\sT \bM_{12}^2\bz_1 \bz_1^\sT \bM_{12} \bDelta_{21} \right) \right] ,
\end{aligned}
\]
where we used that $\E_{\bz_i} [ \bDelta_{i1} ] =0$. Decomposing $\bDelta_{i1}$, we obtain
\[
\begin{aligned}
  &~(1+\kappa)^2  \Tr\left( \obQ_- \bDelta_{11}^\sT \bM_{12} \bz_2 \bz_2^\sT \bM_{12}^2\bz_1 \bz_1^\sT \bM_{12} \bDelta_{21} \right) \\
  =&~\bz_2^\sT \bM_{12}^2 \bz_1\bz_1^\sT\bM_{12}\bz_2 \left(\bz_2^\sT \bM_{12}\bz_1\bz_1^\sT \obQ_- \bz_2 - \Tr(\bM_{12}\bz_1\bz_1^\sT \obQ_-)\right)\\
  &~+ \bz_2^\sT \bM_{12}^2 \bz_1 \bz_2^\sT \obQ_- \bM_{12} \bz_2\bz_1^\sT\bM_{12}\bz_2 -
 \bz_2^\sT \bM_{12}^2 \bz_1  \bz_1^\sT\bM_{12}\obQ_- \bM_{12} \bz_2 .
\end{aligned}
\]
Applying H\"older's inequality to the first term and using that $\|\bB\|_F \le \Tr(\bB)$ for p.s.d.~matrices, we obtain
\begin{align*}
&~ \left| \E\left[\bz_2^\sT \bM_{12}^2 \bz_1\bz_1^\sT\bM_{12}\bz_2 \left(\bz_2^\sT \bM_{12}\bz_1\bz_1^\sT \obQ_- \bz_2 - \Tr(\bM_{12}\bz_1\bz_1^\sT \obQ_-)\right) \right] \right|\\
\le&~ 
\E\left[(\bz_2^\sT \bM_{12}^2 \bz_1)^3\right]^{1/3}\E\left[(\bz_1^\sT\bM_{12}\bz_2)^3\right]^{1/3}
\E\left[ \left(\bz_2^\sT \bM_{12}\bz_1\bz_1^\sT \obQ_- \bz_2 - \Tr(\bM_{12}\bz_1\bz_1^\sT \obQ_-)\right)^3\right]^{1/3}\\
\le&~ C_{x,K}\left(\E[\|\bM_{12}^2\|_F^3]^{1/3} + \varphi_1(p)^{1/2} \E[\|\bM_{12}^4\|_F^{3/2}]^{1/3} \right)
\left(\E[\|\bM_{12}\|_F^3]^{1/3} + \varphi_1(p)^{1/2} \E [\|\bM_{12}^2\|_F^{3/2}]^{1/3} \right)\\
&~ \times \left\{ \E \left[(\bz_1\bM_{12}\bz_2 )^{6}\right]^{1/6} \E \left[ (\bz_1^\sT \obQ_- \bz_2)^{6}\right]^{1/6}  + \E \left[(\bz_1\bM_{12} \obQ_- \bM_{12} \bz_1 )^{3}\right]^{1/6} \E \left[ (\bz_1^\sT \obQ_- \bz_1)^{3}\right]^{1/6} \right\} \\
\le&~  C_{x,K}\frac{\nu_\lambda (n)^4}{n^2} \Big\{  \Big(\E \left[ \| \bM_{12} \|_F^6 \right]^{1/6} + \varphi_1 (p)^{1/2} \E \left[ \| \bM_{12}^2 \|_F^3\right]^{1/6}  \Big)  \left( \| \obQ_- \|_F + \varphi^\bA_1(p)^{1/2} \| \obQ_-^2 \|_F^{1/2} \right) \\
&~ + \Big( \E \left[\Tr( \obQ_- \bM_{12}^2)^3)\right] +  \varphi_1^\bA (p)^3 \E \left[ \| \bM_{12}\obQ_- \bM_{12} \|_F^3\right]  \Big)^{1/6}   \left( \Tr( \obQ_- ) + \varphi^{\bA}_1(p) \| \obQ_-\|_F \right)^{1/2} \Big\}\\
\le&~ C_{x,K}  \frac{\nu_\lambda (n)^4}{n^{5/2}}  \cdot \left( \Tr (\obQ_-) + \delta_{\varphi_1^\bA \neq 1} \varphi_1(p)^{1/2} \sqrt{\| \obQ_- \|_\op \| \Tr (\obQ_-)} \right)  \\
\leq&~ C_{x,K} \frac{\nu_\lambda (n)^6}{n^2} \tcE_n (\bA) \Tr(\bA \obM^2),
\end{align*}
where we recall that $\varphi_1 (p) \leq C_{x,K} \sqrt{n}$ by condition \eqref{eq:conditions_thm3}.
Following a similar procedure for the second term gives 
\begin{align*}
    &~\left| \E\left[\bz_2^\sT \bM_{12}^2 \bz_1 \bz_2^\sT \obQ_- \bM_{12} \bz_2\bz_1^\sT\bM_{12}\bz_2\right] \right|\\
    \leq&~ \E\left[ (\bz_2^\sT \bM_{12}^2 \bz_1)^4 \right]^{1/4}\E\left[ (\bz_2^\sT \obQ_-  \bz_2)^2 \right]^{1/4} \E\left[ (\bz_2^\sT \bM_{12} \obQ_- \bM_{12} \bz_2)^2 \right]^{1/4} \E\left[ (\bz_2^\sT \bM_{12} \bz_1)^4 \right]^{1/4}\\
    \le&~ C_{x,K} \nu_\lambda (n)^4 \frac{1}{n^{3/2}} \cdot  \frac{1}{n} \cdot \frac{1}{n^{1/2}} \cdot \left(\Tr(\obQ_-) + \delta_{\varphi_1^\bA \neq 1} \varphi_1(p) \|\obQ_-\|_F\right)\\
    \le&~ C_{x,K} \frac{\nu_\lambda (n)^6}{n^{5/2}}\tcE_n (\bA) \Tr(\bA \obM^2),
\end{align*}
which is of lower order.

Finally, for the last term we proceed similarly by writing
\begin{align*}
    &~ \left| \E\left[\bz_1^\sT \bM_{12} \obQ_- \bM_{12} \bz_2 \cdot \bz_2^\sT \bM_{12}^2 \bz_1 \right] \right|\\
  \leq  &~
    \E\left[(\bz_1^\sT \bM_{12} \obQ_- \bM_{12} \bz_2)^2\right]^{1/2}\E\left[(\bz_2^\sT \bM_{12}^2 \bz_1)^2\right]^{1/2}\\
   \leq&~
    C_{x,K}\left(\E [\|\bM_{12} \obQ_- \bM_{12}\|_F^2]^{1/2} + \delta_{\varphi_1^\bA \neq 1} \varphi_1(p) \E [\| (\bM_{12} \obQ_- \bM_{12})^2\|_F ]^{1/2} \right)\\
    &~ \times
    \left(\E[\|\bM_{12}^2 \|_F^2]^{1/2} + \varphi_1(p) \E[\| \bM_{12}^4 \|_F]^{1/2}\right)\\
   \leq&~ C_{x,K} \frac{\nu_\lambda (n)^4 \varphi_1 (p)^{1/2}}{n^{7/2}} \left\{ \Tr(\obQ_- ) + \delta_{\varphi_1^\bA \neq 1} \varphi_1(p) \|\obQ_-\|_F \right\} \\
    \leq&~C_{x,K} \frac{\nu_\lambda (n)^6}{n^{11/4}}\tcE_n (\bA) \Tr(\bA \obM^2),
\end{align*}
which is also of lower order after scaling by $n^2$.

\paragraph*{Bounding $\Tr\left( \obQ_- \bDelta_{12}^\sT   \bM_1 \bM_2 \bDelta_{22}  \right)$ of $(\rmII)$:}
We expand $\bM_1\bM_2$ via the resolvent identities \eqref{eq:expansion_M1M2} which gives four terms. We bound each of the four terms separately. For the first term (involving $\bM_{12}^2$), we introduce 
\[
\bD_{i} = \frac{S_i - \kappa}{1+S_i} \bz_i \bz_i^\sT , \qquad \quad \tbD_{i} = \frac{\tS_i - \kappa}{1+\tS_i} \bz_i \bz_i^\sT.
\]
Using Sherman-Morrison identity \eqref{eq:standard_identities_SMW}, note that we have
\[
S_1 = \bz_1^\sT \bM_1 \bz_1 = \bz_1^\sT \bM_{12} \bz_1 - \frac{(\bz_1^\sT \bM_{12} \bz_2)^2}{1 + \bz_2^\sT \bM_{12} \bz_2},
\]
and we can decompose the difference as
\[
\bD_{i} - \tbD_{i} = - \frac{(\kappa +1) (\bz_1^\sT \bM_{12} \bz_2)^2 }{(1+ \tS_2)(1 +\tS_1) (1 + S_i)}\bz_i \bz_i^\sT=: - \delta_i \cdot (\bz_1^\sT \bM_{12} \bz_2)^2\bz_i \bz_i^\sT.
\]
We expand this first term as
\begin{equation}\label{eq:decompo_tDelta12}
\begin{aligned}
&~ (1+ \kappa)^2 \E \left[\Tr\left( \obQ_- \bDelta_{12}   \bM_{12}^2 \bDelta_{22}  \right)\right] \\
= &~ \E \left[\Tr\left( \obQ_- \tbD_{1}   \bM_{12}^2 \tbD_{2}  \right)\right] - 2 \E \left[  \delta_1 (\bz_1^\sT \bM_{12} \bz_2)^2 \bz_1^\sT \bM_{12}^2 \tbD_{22} \obQ_- \bz_1 \right]\\
&~ + \E \left[  \delta_1 \delta_2 (\bz_1^\sT \bM_{12} \bz_2)^4 \bz_1^\sT \bM_{12}^2 \bz_2 \bz_2^\sT \obQ_- \bz_1 \right].
\end{aligned}
\end{equation}
Note that in the first term, the only dependency on $\bz_i$ appears in $\tbD_{i}$, and we denote $\bD := \E_i [ \tbDelta_{i} ]$. We can bound its operator norm as follows 
\[
\begin{aligned}
    \| \bD \|_\op = \sup_{\| \br \|_2 = 1} | \br^\sT \bD \br| \leq &~ \E_{\bz_1}\left[ \< \bz_1 , \br \>^2 |\bz_1^\sT \bM_{12} \bz_1 - \kappa| \right] \\
    \leq&~ C_x \left\{ \E_{\bz_1} \left[ ( \bz_1^\sT \bM_{12} \bz_1 - \Tr( \bM_{12} ) )^2 \right]^{1/2} + \left|\Tr( \bM_{12} ) - \kappa \right|  \right\}\\
    \leq&~ C_x \left\{ \varphi_1 (p) \| \bM_{12} \|_F + \left|\Tr( \bM_{12} ) - \kappa \right|  \right\}.
\end{aligned}
\]
Hence, we can bound the first term in Eq.~\eqref{eq:decompo_tDelta12} by
\[
\begin{aligned}
   \left| \E \left[\Tr\left( \obQ_- \tbD_{1}   \bM_{12}^2 \tbD_{2}  \right)\right] \right|=&~ \left|\E  \left[\Tr\left( \obQ_- \tbD   \bM_{12}^2 \tbD  \right)\right]\right|\\
   \leq&~ \Tr(\obQ_-) \E \left[ \| \bM_{12}\|_\op^4 \right]^{1/2} \E \left[  \| \bD \|_\op^4\right]^{1/2}\\
   \leq&~ C_{x,K} \Tr(\obQ_-) \cdot \frac{\nu_\lambda (n)^2}{n^2} \cdot \frac{\varphi_1 (p)^2 \nu_\lambda (n)^2}{n}  \\
   \leq&~ C_{x,K} \frac{\varphi_1(p) \nu_\lambda (n)^6}{n^{5/2}} \Tr(\bA \obM^2).
\end{aligned}
\]
The second term in the decomposition \eqref{eq:decompo_tDelta12} can be bounded less carefully:
\[
\begin{aligned}
&~\frac{1}{(1+\kappa)}\left| \E \left[  \delta_1 (\bz_1^\sT \bM_{12} \bz_2)^2 \bz_1^\sT \bM_{12}^2 \tbD_{22} \obQ_- \bz_1 \right] \right| \\
\leq&~  \E\left[ (\bz_1^\sT \bM_{12} \bz_2)^8 \right]^{1/4} \E\left[ (\bz_1^\sT \bM_{12}^2 \bz_2)^4 \right]^{1/4} \E\left[ (\bz_2^\sT \obQ_- \bz_1 )^4 \right]^{1/4} \E\left[ (\tS_i - \kappa)^4 \right]^{1/4} \\
\leq&~ C_{x,K} \cdot \frac{\nu_\lambda (n)^2}{n} \cdot \frac{\nu_\lambda (n)^2}{n^{3/2}} \cdot \frac{\varphi_1 (p) \nu_\lambda (n)}{\sqrt{n}} \cdot \left\{ \Tr( \obQ_- ) + \varphi^{\bA}_1(p) \cdot \| \obQ_- \|_F \right\}\\
\leq&~ C_{x,K} \cdot \frac{\nu_\lambda(n)^6}{n^{5/2}} \tcE_n (\bA) \cdot \Tr(\bA \obM^2).
\end{aligned}
\]
The bound on the third term in the decomposition \eqref{eq:decompo_tDelta12} follows from a similar computation.

To analyze the second and third terms in the expansion of $\bM_1\bM_2$, we proceed similarly to the first term and replace $\bD_i$ by $\tbD_i$. We simply sketch the first term, and the other terms follow similarly as above. By taking the expectation $\E_{\bz_i} [ \tbD_i]$ and making $\bD $ appear, we get
\begin{align*}
   &~ \left| \E\left[\Tr\left(\obQ_- \tbD_1 \frac{\bM_{12}\bz_2\bz_2^\sT \bM_{12}^2 }{1+\tS_2} \tbD_{2}\right)\right]\right| 
   \\
   =&~ \left| \E\left[\Tr\left(\obQ_- \bD \frac{\bM_{12}\bz_2\bz_2^\sT \bM_{12}^2 \bz_2 \bz_2^\sT }{(1+\tS_2)^2} \right) (\tS_2 - \kappa)\right]\right|\\
   \leq&~ \E\left[ (\bz_2^\sT \bM_{12}^2 \bz_2)^4 \right]^{1/4} \E\left[ (\bz_2^\sT \obQ_- \bz_2)^2 \right]^{1/4} \E\left[ (\bz_2^\sT \bM_{12} \bD \obQ_- \bD \bM_{12} \bz_2)^2 \right]^{1/4} \E\left[ (\tS_2 - \kappa)^4 \right]^{1/4} \\
   \le&~  C_{x,K} \frac{\varphi_1 (p)^2 \nu_\lambda (n)^5}{n^3} \left\{ \Tr(\obQ_-) + \varphi_1^\bA (p) \| \obQ_- \|_F \right\}\\
\le&~ C_{x,K}\frac{\nu_\lambda (n)^6}{n^2} \tcE_n (\bA) \Tr(\bA\obM^2).
\end{align*}

Finally, for the final term in the expansion of $\bM_1\bM_2$, we do not need to be careful and we simply using H\"older's inequality to get 
\begin{align*}
&\left| \E\left[\Tr\left(\obQ_- \bDelta_{12}^\sT 
\frac{\bM_{12}\bz_2\bz_2^\sT \bM_{12}^2 \bz_1\bz_1^\sT \bM_{12}}{(1+\tS_1)(1+\tS_2)}
\bDelta_{22}\right)\right]\right| \\
\le&~
\E\left[(S_1 - \kappa)^8\right]^{1/4} \E \left[ (\bz_1^\sT \bM_{12} \bz_2)^8 \right]^{1/4} \E \left[ (\bz_1^\sT \bM_{12}^2 \bz_2)^4 \right]^{1/4} \E \left[ (\bz_1^\sT \obQ_- \bz_2)^4 \right]^{1/4} \\
\leq&~  C_{x,K} \frac{\varphi_1 (p)^2 \nu_\lambda (n)^6}{n^{7/2}} \left\{ \Tr(\obQ_-) + \varphi_1^\bA (p) \| \obQ_- \|_F \right\}\\
\le&~ C_{x,K}\frac{\nu_\lambda (n)^6}{n^{5/2}} \tcE_n (\bA) \Tr(\bA\obM^2).
\end{align*}

\paragraph*{Bounding $\Tr\left( \obQ_- \bDelta_{12}^\sT   \bM_1 \bM_2 \bDelta_{23}  \right)$ of $({\rm II})$:}

Once again we use the resolvent identities \eqref{eq:expansion_M1M2} to expand $\bM_1\bM_2$ and bound each of the resulting four terms separately.
 Note that $\bDelta_{i3} = -(1+\kappa) \bDelta_{i2} \bM_i/(S_i - \kappa)$. Thus, by comparison to the previous term, the error term obtained is the same as the previous part up to an additional $\bM_2$ matrix and one less $(S_i - \kappa)$, and is of lower order. The remaining terms in $({\rm II})$ similarly depend on $\bDelta_{i3}$ and are bounded via a similar argument, which we omit for the sake of brevity.

\paragraph*{Proof of Claim \ref{claim:TrM2-}.} For convenience, we will denote $\tmu_- = n / (1+\kappa)$. First, note that in the proof of Proposition \ref{prop:TrAM_LOO}, we showed that
\[
\begin{aligned}
    \frac{|\mu_* - \tmu_-|}{\tmu_-} = \frac{| \kappa - \Tr(\obM)}{1 + \Tr( \obM)} \leq C_{x,K} \frac{\varphi_1(p) \nu_\lambda (n)^{5/2}}{\sqrt{n}} = C_{x,K} \cdot \nu_\lambda (n)^{5/2}\tcE_n.
\end{aligned}
\]
In particular, we have $|\mu_* - \tmu_-| \leq C_{x,K} \tmu_-$ by condition \eqref{eq:conditions_thm3}.
Thus, we obtain
\[
\begin{aligned}
    \left| \Tr( \obQ_- ) - \Tr( \bA \obM^2) \right| = &~ \left| \Tr\left[ \bA \bSigma^2 ( \mu_* \bSigma + \lambda)^{-2} ( \tmu_- \bSigma + \lambda)^{-2}  \left\{ ( \tmu_- \bSigma + \lambda)^{2} - ( \mu_* \bSigma + \lambda)^{2}\right\}  \right]\right|\\
    \leq&~ \left\{ \frac{| \mu_*^2 - \tmu_-^2|}{\tmu_-^2}  + 2 \frac{| \mu_* - \tmu_- |}{\tmu_-} \right\} \Tr(\bA \bSigma^2 ( \mu_* \bSigma + \lambda)^{-2}) \\
    \leq&~ C_{x,K} \cdot \nu_\lambda (n)^{5/2}\tcE_n \Tr(\bA \obM^2).
\end{aligned}
\]
Similarly,
\[
\begin{aligned}
    \left| \frac{n \Tr( \obQ_- )}{(1+\kappa)^2} - \frac{n \Tr(\bA \obM^2)}{(1 + \Tr(\obM))^2}\right| =&~ \left| \frac{\tmu_-^2 \Tr( \obQ_- )}{n} - \frac{\mu_*^2 \Tr( \bA \obM^2)}{n} \right|  \\
    \leq&~ \frac{\mu_*^2}{n} \left| \Tr( \obQ_- ) - \Tr( \bA \obM^2) \right| + \frac{\left| \mu_*^2 - \tmu_-^2 \right|}{\tmu_-^2} \cdot \frac{\tmu_-^2}{n}  \Tr( \obQ_- ) \\
    \leq&~  C_{x,K} \cdot  \nu_\lambda (n)^{5/2}\tcE_n n \Tr(\bA \obM^2),
\end{aligned}
\]
which finishes the proof of this claim.

\subsubsection[Proof of Proposition \ref{prop:TrAMM_martingale}: martingale part of $\Tr(AM^2)$]{Proof of Proposition \ref{prop:TrAMM_martingale}: martingale part of \boldmath{$\Tr(AM^2)$}}\label{app:proof_TrAMM_martingale}
The proof for the martingale part proceeds in a manner analogous to Proposition \ref{prop:TrAMM_martingale}. 
The difference is in Step 2 and obtaining with high probability the bound $|\Delta_i| \le R$ where
\begin{equation*}
    \Delta_i := (\E_i - \E_{i-1})\Tr(\bA \bM^2)  =(\E_i - \E_{i-1})\Tr(\bA (\bM^2 - \bM_i^2)).
\end{equation*}
We show below that we can choose
\begin{equation}\label{eq:Rchoice_AMM}
    R =  C_{x,D} \frac{\nu_\lambda(n)^4}{n}    \varphi_1^\bA (p) \log^\beta(n) \Tr(\bA \obM^2).
\end{equation}
For bounding $\E_{i-1} [\Delta_i \ind_{\Delta_i \not\in[-R,R]}]$ in Step 3, note that
\[
\E_{i-1} [ \Delta_i^2 ]^{1/2} \leq 2 \E_{i-1} \left[ \Tr(\bA \bM^2)^2\right]^{1/2} \leq C_{x,D,K} \cdot \nu_\lambda(n)^2 \Tr(\bA \obM^2),
\]
by Lemma \ref{lem:tech_upper_bound_AM}.(b), and the rest of the proof follows similarly. 

Let us show Eq.~\eqref{eq:Rchoice_AMM}. We omit some details that are similar to the proof of Proposition \ref{prop:TrAM_martingale} for the sake of brevity. First, we use the Sherman-Morrison identities \eqref{eq:standard_identities_SMW} and decompose 
\begin{align*}
    \Tr\bA( \bM^2 - \bM_i^2) =&~ \Tr \bA\Big( -2 \frac{\bM_i^2 \bz_i \bz_i^\sT \bM_i }{1+S_i} 
    + \frac{\bM_i \bz_i\bz_i^\sT \bM_i^2 \bz_i \bz_i^\sT \bM_i}{(1 + S_i)^2}  \Big).
\end{align*}

Let's consider $\E_j(\Tr\bA(\cdot))$ with $j \in \{i,i-1\}$ applied to each of these two terms. 
For the first term, this is bounded with probability at least $1 - n^{-D}$ by
\begin{align*}
   \left| \E_i \left[ \frac{\Tr(\bA \bM_i^2 \bz_i \bz_i^\sT \bM_i)}{1 + S_i}\right] \right| \leq&~  \E_i \left[ |\bz_i^\sT \bM_i \bA \bM_i \bz_i|   \right]^{1/2} \E_i \left[ |\bz_i^\sT \bM_i^2 \bA \bM_i^2 \bz_i|   \right]^{1/2}\\
   \leq&~ C_{x,D} \log^\beta (n) \varphi_1^\bA(p) \E_i \left[ \|\bM_i \bA \bM_i\|_F \right]^{1/2} \E_i \left[ \|\bM_i^2 \bA \bM_i^2\|_F \right]^{1/2} \\
    \le &~
     C_{x,D,K} \log^\beta(n) \frac{\nu_\lambda(n)^3}{n}\varphi_1^\bA(p) \Tr(\bA \obM^2),
\end{align*}
where we simply used $\|\bM_i \bA \bM_i \|_F \leq \Tr(\bM_i \bA \bM_i)$ for a p.s.d.~matrix and Lemma \ref{lem:tech_upper_bound_AM}.(b).

Similarly, the second term us bounded with probability at least $1 - n^{-D}$ by
\[
\begin{aligned}
       \left| \E_i \left[ \frac{\Tr(\bA \bM_i \bz_i \bz_i^\sT \bM_i^2 \bz_i \bz_i^\sT \bM_i)}{(1 + S_i)^2}\right] \right| \leq&~  \E_i \left[ |\bz_i^\sT \bM_i \bA \bM_i \bz_i|^2   \right]^{1/2}  \E_i \left[ |\bz_i^\sT \bM_i^2 \bz_i|^2   \right]^{1/2}\\
   \leq&~ C_{x,D} \log^{2\beta} (n) \varphi_1^\bA(p) \E_i \left[ \|\bM_i \bA \bM_i\|_F^2 \right]^{1/2} \cdot \frac{\nu_\lambda (n)^2}{n} \\
    \le &~
     C_{x,D,K} \log^{2\beta} (n) \frac{\nu_\lambda(n)^4}{n}\varphi_1^\bA(p) \Tr(\bA \obM^2).
\end{aligned}
\]
Similar bounds hold with $\E_i$ replaced by $\E_{i-1}$ (without the factors $\log^\beta (n)$). Finally, to obtain a relative bound, note that $\Tr(\bA \obM^2) \leq \Psi_3 (\bA;\mu_*)$.

\subsection[Deterministic equivalent for $\Tr (AMZ^\sT Z M)$]{Deterministic equivalent for \boldmath{$\Tr (AMZ^\sT Z M)$}}
\label{app_det_equiv:proof_MZZM}

We consider the fourth and last functional
\[
\Phi_4 (\bX ; \bA) = \frac{1}{n}\Tr(\bA \bSigma^{1/2} ( \bX^\sT \bX + \lambda)^{-1} \bX^\sT \bX ( \bX^\sT \bX + \lambda)^{-1}\bSigma^{1/2} ) = \frac{1}{n} \Tr( \bA \bM \bZ^\sT \bZ \bM)  .
\]
We show that $\Phi_4 (\bX;\bA)$ is well approximated by the following deterministic equivalent:
\[
\Psi_4 (\mu_* ; \bA) =  \frac{1}{n^2} \cdot\frac{\Tr(\bA \bSigma^2 ( \bSigma + \lambda_*)^{-2})}{1 - \frac{1}{n} \Tr(\bSigma^2 ( \bSigma + \lambda_* )^{-2} ) }. 
\]
Note that we have the identity $\Psi_4 (\mu_* ; \bA) = \frac{\mu_*^2}{n^2} \Psi_3 (\mu_*;\bA)$. Again, we can assume without loss of generality that $\Tr( \bA \bSigma^2) < \infty$.

\begin{theorem}[Deterministic equivalent for $\Tr(\bA \bM \bZ^\sT \bZ \bM)$]\label{thm_app:det_equiv_TrAMZZM}
    Assume the features $(\bx_i)_{i\in[n]}$ satisfy Assumption \ref{ass_app:deterministic_equivalent}. For any constants $D,K>0$, there exist constants $\eta \in (0,1/2)$ (only depending on $\sfc_x,\sfC_x,\beta$), $C_{D,K} >0$ (only depending on $K,D$), and $C_{x,D,K}>0$ (only depending on $\sfc_x,\sfC_x,\beta,D,K$), such that the following holds. Define $\nu_\lambda (n)$ as per Eq.~\eqref{eq:definition_nu_lambda}. For all $n \geq C_{D,K}$ and  $\lambda >0$ satisfying
    \begin{equation}\label{eq:conditions_thm4}
\lambda \cdot \nu_\lambda (n) \geq n^{-K}, \qquad \varphi_1(p) \nu_\lambda (n)^{5/2} \log^{\beta + \frac{1}{2}} (n) \leq K \sqrt{n} ,
    \end{equation}
    and deterministic p.s.d.~matrix $\bA$, we have with probability at least $1 - n^{-D}$ that
    \begin{equation}
         \big\vert \Phi_4(\bX;\bA) - \Psi_4 (\mu_* ; \bA) \big\vert \leq C_{x,D,K} \frac{\varphi_1(p) \nu_\lambda (n)^{6} \log^{\beta + \frac12} (n ) }{\sqrt{n}}   \Psi_4 (\mu_* ; \bA) .
    \end{equation}
\end{theorem}

Again, the proof proceeds by bounding separately 
the deterministic $(\sfD)$ and martingale $(\sfM )$ parts. This is done in the following two propositions.

\begin{proposition}[Deterministic part of $\Tr(\bA\bM \bZ^\sT \bZ \bM)$]\label{prop:TrAMZZM_LOO}
    Assume the same setting as Theorem \ref{thm_app:det_equiv_TrAMZZM}. Then there exist constants $C_K$ and $C_{x,K}$, such that the following holds. For all $n \geq C_K$ and  $\lambda >0$ satisfying Eq.~\eqref{eq:conditions_thm4}, and for all p.s.d.~matrix $\bA$, we have
    \begin{equation}\label{eq:det_part_TrAMZZM}
         \big\vert\E [ \Phi_4(\bX;\bA) ] - \Psi_4 (\mu_* ; \bA) \big\vert \leq C_{x,K} \frac{\varphi_1 (p) \nu_\lambda (n)^{6} }{\sqrt{n}}   \Psi_4 (\mu_* ; \bA) .
    \end{equation}
\end{proposition}

\begin{proposition}[Martingale part of $\Tr(\bA\bM \bZ^\sT \bZ \bM)$] \label{prop:TrAMZZM_martingale}
Assume the same setting as Theorem \ref{thm_app:det_equiv_TrAMZZM}. Then there exist constants $C_{K,D}$ and $C_{x,K,D}$, such that the following holds. For all $n \geq C_{K,D}$ and  $\lambda >0$ satisfying Eq.~\eqref{eq:conditions_thm4}, and for all p.s.d.~matrix $\bA$, we have with probability at least $1 -n^{-D}$ that
    \begin{equation}\label{eq:mart_part_TrAMZZM}
         \big\vert \Phi_4(\bX;\bA)  - \E [ \Phi_4(\bX;\bA) ]  \big\vert \leq C_{x,K,D} \frac{\varphi^\bA_1 (p) \nu_\lambda (n)^{4} \log^{\beta +\frac12} (n)}{\sqrt{n}}   \Psi_4 (\mu_* ; \bA) .
    \end{equation}
\end{proposition}

The proof of these two propositions can be found in the next two sections. Theorem \ref{thm_app:det_equiv_TrAMZZM} is obtained by combining the bounds \eqref{eq:det_part_TrAMZZM} and \eqref{eq:mart_part_TrAMZZM}.

\subsubsection[Proof of Proposition \ref{prop:TrAMZZM_LOO}: deterministic part of $\Tr(AMZ^\sT Z M)$]{Proof of Proposition \ref{prop:TrAMZZM_LOO}: deterministic part of \boldmath{$\Tr(AMZ^\sT Z M)$}}\label{app:proof_TrAMZZM_LOO}

We will reduce the proof of this proposition to Proposition \ref{prop:TrAMM_LOO}. By exchangeability and using Sherman-Morrison identity \eqref{eq:standard_identities_SMW}, we decompose $\E[\Phi_4(\bX;\bA)]$ as 
\[
\begin{aligned}
\frac{1}{n}\E \left[ \Tr( \bA \bM \bZ^\sT \bZ \bM) \right] =&~ \E \left[  \frac{\bz^\sT \bM_- \bA \bM_- \bz}{(1 + S)^2}\right] \\
= &~ \frac{\E[\Tr ( \bM_- \bA \bM_- )]}{(1 + \kappa)^2} + \E \left[ \frac{(\kappa - S) ( 2 + \kappa+ S)}{(1+\kappa)^2 (1 + S)^2} \bz^\sT \bM_- \bA \bM_- \bz\right].
\end{aligned}
\]
where we denoted $S = \bz^\sT \bM_- \bz$. Thus it suffices to bound the following two terms
\begin{equation}\label{eq:decompo_deterministic_4}
\begin{aligned}
    \left|\E[\Phi_4(\bX;\bA)] -\Psi_4(\mu_*;\bA)\right| \le&~
    \left|\frac{\E[\Tr (  \bA \bM_-^2 )]}{(1 + \kappa)^2} - \Psi_4( \mu_*;\bA)\right|\\
    &+ 
    \left| \E \left[ \frac{(\kappa - S) ( 2 + \kappa+ S)}{(1+\kappa)^2 (1 + S)^2} \bz^\sT \bM_- \bA \bM_- \bz\right] \right| .
\end{aligned}
\end{equation}

For the first term,  recall that we introduced the notations: $\tmu_*$ is the solution of the fixed point equation \eqref{eq:det_equiv_fixed_point_mu_star} where we replaced $n$ by $n-1$, and $\tmu_- = n/(1+ \kappa)$. By Proposition \ref{prop:TrAMM_LOO}, we have 
\[
\left| \E[\Tr (  \bA \bM_-^2 )] - \Psi_3 (\tmu_* ; \bA ) \right| \leq \cE^{(\sfD)}_{3,n-1}  \cdot \Psi_3 (\tmu_* ; \bA ),
\]
where we denoted $\cE^{(\sfD)}_{3,n-1}$ the rate in Eq.~\eqref{eq:det_part_TrAMM}.
For $n \geq C$, we have $\cE_{3,n-1}^{(\sfD)} \leq C \cE_{3,n}^{(\sfD)}$ and by Lemma \ref{lem:properties_mu_obM},
\[
|\Psi_3 (\tmu_* ; \bA ) - \Psi_3 (\mu_* ; \bA )| \leq C \frac{\nu_\lambda(n)^2}{n} \Psi_3 (\mu_* ; \bA ).
\]
Furthermore, from the proof of Claim \ref{claim:TrM2-} in Proposition \ref{prop:TrAMM_LOO}, we get
\[
\frac{| \mu_* - \tmu_- |}{\mu_*} \leq C_{x,K} \frac{\varphi_1(p) \nu_\lambda (n)^{5/2}}{\sqrt{n}}.
\]
Combining this inequality with the previous bounds, we obtain
\[
\begin{aligned}
 \left|\frac{\E[\Tr (  \bA \bM_-^2 )]}{(1 + \kappa)^2} - \Psi_4( \mu_* ; \bA)\right| =&~ \left|\frac{\tmu_-^2}{n^2}\E[\Tr (  \bA \bM_-^2 )] - \frac{\mu_*^2}{n^2}\Psi_3( \mu_* ;\bA)\right| \\
 \leq&~ \frac{\tmu_*^2}{\mu_*^2} \cdot \frac{\mu_*^2}{n^2} \left| \E[\Tr (  \bA \bM_-^2 )] - \Psi_3 ( \mu_* ; \bA) \right| + \frac{| \tmu_-^2 - \mu_*^2|}{\mu_*^2} \cdot \frac{\mu_*^2}{n^2} \Psi_3 ( \mu_* ; \bA)\\
 \leq&~ C \cE_{3,n}^{(\sfD)} \cdot \Psi_4 (\mu_* ; \bA).
\end{aligned}
\]

It remains to bound the second term in Eq.~\eqref{eq:decompo_deterministic_4}. Recall that we can reduce ourselves to $\bA$ rank $1$ following Eq.~\eqref{eq:reduction_to_rank_1}. We simply apply H\"older's inequality and obtain 
\begin{align*}
    &~\left| \E\left[\frac{(\kappa- S) ( 2 + \kappa+S)}{(1+\kappa)^2}\bz^\sT \bM_- \bA \bM_- \bz\right]\right| \\
    \le&~ \E\left[(\kappa - S)^3 \right]^{1/3} \E\left[(2 + \kappa + S)^3\right]^{1/3} \E\left[(\bz^\sT \bM_- \bA \bM_- \bz)^3\right]^{1/3}\\
    \leq&~ C_{x,K} \frac{\varphi_1(p) \nu_\lambda(n)^2}{\sqrt{n}}   \E[\Tr( \bA \bM_-^2)^3]^{1/3} \\
  \leq&~C_{x,K} \frac{\varphi_1(p) \nu_\lambda(n)^4}{\sqrt{n}} \Tr(\bA \obM^2) \\
  \leq&~ \cE_{3,n}^{(\sfD)} \cdot \Psi_{4}(\bA; \mu_*),
\end{align*}
which concludes the proof.

\subsubsection[Proof of Proposition \ref{prop:TrAMZZM_martingale}: martingale part of $\Tr(AMZ^\sT Z M)$]{Proof of Proposition \ref{prop:TrAMZZM_martingale}: martingale part of \boldmath{$\Tr(AMZ^\sT Z M)$}}\label{app:proof_TrAMZZM_martingale}

The martingale argument proceeds in a manner similar to the proofs of Propositions \ref{prop:TrAM_martingale} and \ref{prop:TrAMM_martingale}. It remains to modify Step 2 and bound with high probability each term in the martingale difference sequence
\begin{equation*}
    \Delta_i := \frac{1}{n}(\E_i - \E_{i-1}) \Tr(\bZ^\sT \bZ \bM \bA \bM).
\end{equation*}
We show below that $|\Delta_i| \leq R$ with probability at least $1 - n^{-D}$ with
\begin{equation}\label{eq:Rchoice_AMZZM}
    R =  C_{x,D} \frac{\nu_\lambda(n)^4 \log^\beta(n)}{n}    \varphi_1^\bA (p)  \Psi_4( \mu_* ; \bA).
\end{equation}
For Step 3 and bounding $\E_{i-1}[\Delta_i \ind_{\Delta_i \not\in[-R,R]}] $,observe that
\[
\E_{i-1} [ \Delta_i^2]^{1/2} \leq 2 \E_{i-1} \left[ \frac{\Tr(\bz^\sT \bM_- \bA \bM_- \bz)^2}{(1+S)^4}\right]^{1/2} \leq C_{x,D,K} \cdot \nu_\lambda (n)^2 \Tr(\bA \obM^2),
\]
so that we can follow the same proof as in Proposition \ref{prop:TrAM_martingale}.

Let us bound $\Delta_i$ with high probability. We start by decomposing the term $\Delta_i$ and adding and subtracting carefully chosen terms. 
Noting that 
\begin{equation*}
        \Delta_i = \frac{1}{n}(\E_i - \E_{i-1}) \left(\Tr(\bZ^\sT \bZ \bM \bA \bM) -\Tr(\bZ_i^\sT \bZ_i \bM_i \bA \bM_i) \right),
\end{equation*}
we'll write (recall that $S_i = \bz_i^\sT \bM_i \bz_i$)
\begin{align*}
    &~\Tr(\bZ^\sT \bZ \bM \bA \bM) -\Tr(\bZ_i^\sT \bZ_i \bM_i \bA \bM_i)\\
    = &~ 
    \bz_i \bM \bA\bM\bz_i +   \Tr(\bM \bA\bM \bZ_i^\sT \bZ_i)
    -\Tr(\bM_i \bA \bM_i \bZ_i^\sT\bZ_i)
    \\
    =&~ 
    \frac1{(1+S_i)} \left\{ \bz_i \bM \bA\bM_i\bz_i -   \Tr(\bM \bA\bM_i \bz_i\bz_i^\sT \bM_i \bZ_i^\sT \bZ_i)
    -\Tr(\bM_i\bz_i\bz_i^\sT \bM_i \bA \bM_i \bZ_i^\sT\bZ_i) \right\}
    \\
    =&~\frac1{(1+S_i)} \Tr(\bM \bA \bM_i \bz_i\bz_i^\sT (\id - \bM_i\bZ_i^\sT\bZ_i )) 
    +\frac1{(1+S_i)} \Tr(\bM_i \bA \bM_i \bz_i\bz_i^\sT (\id - \bM_i\bZ_i^\sT\bZ_i ))\\
    &- \frac1{(1+S_i)} \Tr(\bz_i^\sT \bM_i \bA\bM_i \bz_i).
\end{align*}

Observing that 
\begin{equation*}
  \id - \bM_i \bZ_i^\sT\bZ_i = \lambda  \bM_i \bSigma^{-1},
\end{equation*}
we can write for $j\in\{i-1,i\}$, with probability at least $1 - n^{-D}$,
\begin{align*}
&~\left|\frac1n\E_{j} \left[
\frac1{(1+S_i)} \Tr(\bM \bA \bM_i \bz_i\bz_i^\sT (\id - \bM_i\bZ_i^\sT\bZ_i ))
\right]\right| \\
\le&~
\lambda \frac1n\E_{j} \left[ | \bz_i^\sT \bM_i \bSigma^{-1} \bM \bA \bM_i \bz_i
| \right]\\
\le&~ \frac{\lambda}{n} \E_j[\bz_i^\sT \bM_i \bSigma^{-1}\bM \bA \bM\bSigma^{-1}\bM_i \bz_i ]^{1/2} \E_j[\bz_i\bM_i \bA\bM_i \bz_i]^{1/2}\\
\le&~ C_{x,D}\frac\lambda{n} \varphi_1^\bA  (p) \log^\beta(n) \E_j \left[ \Tr(\bM_i \bSigma^{-1}\bM \bA \bM \bSigma^{-1} \bM_i) \right]^{1/2} \E_j \left[ \Tr(\bM_i \bA \bM_i) \right]^{1/2}\\
\le &~ C_{x,D} \frac{\varphi_1^\bA (p) \log^\beta(n)\nu_\lambda(n)^2}{n}  \Tr(\obM\bA \obM) \leq  C_{x,D} \frac{\varphi_1^\bA (p)  \log^\beta(n)\nu_\lambda(n)^4}{n}  \Psi_4 (\mu_* ; \bA), 
\end{align*}
where we used that $\bM_i \preceq \bSigma / \lambda$ by definition and Lemma \ref{lem:tech_upper_bound_AM}.(b). A similar argument then shows that
\begin{align*}
&\left|\frac1n\E_{j} \left[
\frac1{(1+S_i)} \Tr(\bM_i \bA \bM_i \bz_i\bz_i^\sT (\id - \bM_i\bZ_i^\sT\bZ_i ))
\right]\right|
\le C_{x,D} \frac{\varphi_1^\bA (p) \log^\beta(n)\nu_\lambda(n)^4}{n}  \Psi_4 (\mu_* ; \bA),
\end{align*}
with the same probability. Meanwhile, the bound
\begin{align*}
&\left|\frac1n\E_{j} \left[
\frac1{(1+S_i)}  \bz_i^\sT \bM_i \bA \bM_i \bz_i
\right]\right|
\le C_{x,D} \frac{\varphi_1^\bA (p)  \log^\beta(n)\nu_\lambda(n)^4}{n}  \Psi_4 (\mu_* ; \bA)
\end{align*}
follows directly from Lemma~\ref{lem:tech_bound_zAz} and Lemma \ref{lem:tech_upper_bound_AM}.(b).

\subsection{Technical bounds}
\label{app_det_equiv:technical}

In this section, we prove Lemma \ref{lem:tech_bounds_norm_M} and other technical bounds that are repeatedly used in the proofs of the deterministic equivalents. We first state the different lemmas and defer their proofs to separate sections.

The following lemma is a direct consequence of Assumption \ref{ass_app:deterministic_equivalent}.

\begin{lemma}\label{lem:tech_bound_zAz}
    Assume the features $\bx_1,\bx_2$ are independent and satisfy Assumption \ref{ass_app:deterministic_equivalent}. Then, for any constant $D >0$, there exists a constant $C_{x,D}$ such that for all p.s.d.~matrix $\bB$ independent of $\bx_1 , \bx_2$, it holds with probability at least $1 - n^{-D}$ over the randomness in $\bx_1, \bx_2$, that
    \[
    \begin{aligned}
    \big\vert \bx_1^\sT \bB \bx_1  - \Tr(\bSigma \bB) \big\vert \leq&~ C_{x,D} \cdot \varphi_1^{\bB} (p) \log^\beta (n) \| \bSigma^{1/2} \bB \bSigma^{1/2} \|_F, \\
    \big\vert \bx_1^\sT \bB \bx_2 \big\vert \leq&~ C_{x,D} \cdot  \log^\beta (n) \left( \| \bSigma^{1/2} \bB \bSigma^{1/2}\|_F + \varphi_1^{\bB} (p)^{1/2} \| \bSigma^{1/2} \bB \bSigma \bB \bSigma^{1/2} \|_F^{1/2} \right).
    \end{aligned}
    \]
    Moreover, for all integer $q$, there exists a constant $C_{x,q}$ such that for all p.s.d.~matrix $\bB$ independent of $\bx_1,\bx_2$, we have 
        \begin{align*}
    \E_{\bx_1} \left[ \big\vert \bx_1^\sT \bB \bx_1  - \Tr(\bSigma \bB) \big\vert^q \right]^{1/q} \leq&~ C_{x,q} \cdot \varphi_1^{\bB} (p)  \| \bSigma^{1/2} \bB \bSigma^{1/2} \|_F  , \\
    \E_{\bx_1,\bx_2} \left[ \big\vert \bx_1^\sT \bB \bx_2  \big\vert^q \right]^{1/q} \leq&~ C_{x,q}  \left( \| \bSigma^{1/2} \bB \bSigma^{1/2}\|_F + \varphi_1^{\bB} (p)^{1/2} \| \bSigma^{1/2} \bB \bSigma \bB \bSigma^{1/2} \|_F^{1/2} \right).
    \end{align*}
\end{lemma}

The proof of Lemma \ref{lem:tech_bound_zAz} can be found in Section \ref{app_det_equiv:proof_lem_zAz}. The next lemma collects some simple properties on $\mu_*$ and $\obM$.

\begin{lemma}\label{lem:properties_mu_obM}
    Recall that $\mu_*$ is the solution of the fixed point equation \eqref{eq:det_equiv_fixed_point_mu_star}, and $\obM= \bSigma ( \mu_* \bSigma + \lambda)^{-1}$. Let $\eta := \eta (\sfc_x,\sfC_x,\beta) \in (0,1/2)$ be defined as in Lemma \ref{lem:tech_bounds_norm_M} and $\nu_\lambda (n)$ as per Eq.~\eqref{eq:definition_nu_lambda}. Then
    \begin{align*}
   1+ \Tr ( \obM ) \leq 2 \nu_\lambda ( n) , \qquad\quad \| \obM \|_\op \leq \frac{2 \nu_\lambda (n)}{n}.
    \end{align*}
     Consider $\widetilde  \mu_*$ the solution to the fixed point equation \eqref{eq:det_equiv_fixed_point_mu_star} with $n-1$ instead of $n$, and denote $\wbM = \bSigma ( \widetilde  \mu_* \bSigma + \lambda)^{-1}$. Then, we have $| \mu_* - \widetilde  \mu_* | \leq 1$. If we further assume that $\nu_\lambda(n) \leq \sqrt{n}$ and $n \geq 5$, then  there exists a constant $C>0$ such that for any p.s.d.~matrix $\bA$,
    \begin{align*}
    \left| \Psi_1 ( \tmu_* ; \bA ) - \Psi_1 ( \mu_* ; \bA ) \right| \leq&~ C \frac{\nu_\lambda (n)}{n} \Psi_1 ( \mu_* ; \bA ), \\
    \left| \Psi_2 ( \tmu_* ) - \Psi_2 ( \mu_* ) \right| \leq&~ C \frac{\nu_\lambda (n)}{n} \Psi_2 ( \mu_* ), \\
    \left| \Psi_3 (\tmu_*;\bA) - \Psi_3 ( \mu_* ; \bA) \right| \leq&~ C \frac{\nu_\lambda (n)}{n} \Psi_3 ( \mu_* ; \bA ), \\
    \left| \Psi_4 (\tmu_*;\bA) - \Psi_4 ( \mu_* ; \bA) \right| \leq&~ C \frac{\nu_\lambda (n)}{n} \Psi_4 ( \mu_* ; \bA ).
    \end{align*}
\end{lemma}

The proof of Lemma \ref{lem:properties_mu_obM} is deferred to Section \ref{app_det_equiv:proof_lemma_mu_obM}.

   Recall that $\E_i$ denote the partial expectation over features $\{\bx_{i+1} , \ldots , \bx_n \}$. Furthermore, $\bM_{-}$ denotes the scaled resolvent where we removed one feature $\bx_i$, and 
   \[
   \obM_- = \bSigma \left( \frac{n}{1+\kappa} \bSigma  + \lambda \right)^{-1} , \qquad  \kappa = \E[ \Tr(\bM_{-}) ] .
   \]
   The following lemma gathers the bounds on the trace and Frobenius norm of the matrices $\bM$ and $\obM_-$ in terms of $\obM$. We defer its proof to Section \ref{app:proof_lem_tech_upper_bound_AM}.

\begin{lemma}\label{lem:tech_upper_bound_AM}
    Assume the features $(\bx_i)_{i \in [n]}$ satisfy Assumption \ref{ass_app:deterministic_equivalent} and let $\eta := \eta (\sfc_x,\sfC_x,\beta) \in (0,1/2)$ be chosen as in Lemma \ref{lem:tech_bounds_norm_M}. Let $\nu_\lambda (n)$ be defined as per Eq.~\eqref{eq:definition_nu_lambda}. Further assume that there exists a constant $K$ such that 
    \begin{equation}\label{eq:conditions_tech_upper_bound_AM}
    \lambda \cdot \nu_\lambda (n) \geq n^{-K}, \qquad\quad \varphi_1(p) \log^\beta (n)  \leq K \sqrt{n}.
    \end{equation}
    Then, for any integers $D,q,\ell >0$, the following hold.
    \begin{itemize}
    \item[(a)] There exist constants $C_K$ and $C_{x,K}$ such that for all $n \geq C_K$ and $\lambda >0$ satisfying conditions \eqref{eq:conditions_tech_upper_bound_AM}, it holds that 
    \[
    \obM_- \preceq C_{x,K} \cdot \nu_\lambda(n) \obM, \qquad\quad \| \obM_- \|_\op \leq C_{x,K} \frac{\nu_\lambda(n)}{n}.
    \]
    As a consequence, for all p.s.d.~matrix $\bA$ and $r \in \{0,\ldots, \ell\}$ 
    \begin{align*}
        \Tr( \bA \obM_-^\ell ) \leq&~ C_{x,K}^\ell \cdot \frac{\nu_\lambda (n)^\ell}{n^r} \Tr( \bA \obM^{\ell -r} )  ,\\
            \| \obM_- \bA \obM_- \|_F \leq&~ C_{x,K}^2 \cdot \frac{\nu_\lambda (n)^2}{n^{3/2}} \| \bA \|_\op^{1/2} \Tr (\bA \obM)^{1/2}  .
    \end{align*}
    
        \item[(b)] There exist constants $C_{K,D}$, $C_{x,K,D,q}$, and $C_{x,K,D,q,\ell}$, such that for all $n \geq C_{K,D}$ and $\lambda >0$ satisfying conditions \eqref{eq:conditions_tech_upper_bound_AM}, it holds with probability at least $1 - n^{-D}$ that for all p.s.d.~matrix $\bA$ and integers $j \in \{0, \ldots , n\}$ and $r \in \{ 0, \ldots , \ell\}$,
            \begin{align*}
        \E_j \left[ \Tr( \bA \bM^\ell )^q \right]^{1/q} \leq&~ C_{x,D,K,q,\ell} \cdot \frac{\nu_\lambda (n)^\ell}{n^r} \Tr( \bA \obM^{\ell-r} )  ,\\
            \E_j \left[ \| \bM \bA \bM \|_F^q \right]^{1/q} \leq&~ C_{x,D,K,q} \cdot \frac{\nu_\lambda (n)^2}{n^{3/2}} \| \bA \|_\op^{1/2} \Tr (\bA \obM)^{1/2}  .
    \end{align*}
    In particular, for $j=0$, $C_{K,D} := C_{K}$, and $C_{x,K,D,q,\ell} := C_{x,K,q,\ell}$, we have 
    \begin{equation*}
    \begin{aligned}
        \E \left[ \| \bM^\ell \|_\op^q\right]^{1/q} \leq&~  C_{x,K,q,\ell} \frac{\nu_\lambda (n)^\ell}{n^\ell}  , \qquad \E \left[ \Tr(\bM^\ell)^q \right]^{1/q} \leq C_{x,K,q,\ell}  \frac{\nu_\lambda (n)^\ell}{n^{\ell - 1}}  , \\
        \E \left[ \| \bM^\ell \|_F^q \right]^{1/q} \leq&~ C_{x,K,q,\ell} \frac{\nu_\lambda (n)^\ell}{n^{\ell - 1/2}}  .
        \end{aligned}
    \end{equation*}

            \item[(c)] There exist constants $C_{K}$ and $C_{x,K}$, such that for all $n \geq C_{K}$ and $\lambda >0$ satisfying conditions \eqref{eq:conditions_tech_upper_bound_AM}, it holds that
            \[
            \begin{aligned}
                \left| \E[ \bM_{-} - \bM ] \right| \leq&~ C_{x,K} \frac{\nu_\lambda (n)^2}{n}, \\
                \left| \E[ \bM_{-}^2 - \bM^2 ] \right| \leq&~ C_{x,K} \frac{\nu_\lambda (n)^4}{n^2}.
            \end{aligned}
            \]


    \end{itemize}
\end{lemma}

We further prove separately the following bound on $\Tr(\bM) - \E[\Tr(\bM)]$ using a bounded difference martingale sequence, which corresponds to Proposition \ref{prop:TrAM_martingale} with $\bA = \id$.

\begin{lemma}\label{lem:tech_bound_M_EM}
    Under Assumption \ref{ass_app:deterministic_equivalent} and for all integers $q,K >0$, there exist constants $C_{K}$ and $C_{x,K,q}$ such that the following hold. Let $\eta := \eta (\sfc_x,\sfC_x,\beta) \in (0,1/2)$ and $\nu_\lambda (n)$ be defined as per Eq.~\eqref{eq:definition_nu_lambda}. For all $n \geq C_{K}$ and $\lambda >0$ satisfying
    \begin{equation*}
    \lambda \cdot \nu_\lambda (n) \geq n^{-K}, \qquad\quad \varphi_1(p) \log^{\beta} (n)  \leq K \sqrt{n},
    \end{equation*}
    we have 
    \[
    \begin{aligned}
    \E\left[ (\Tr(\bM) - \E[\Tr(\bM)])^q\right]^{1/q} \leq&~ C_{x,K,q}  \frac{\varphi_1(p) \nu_\lambda (n)^3 \log^{\beta +\frac 12} (n)}{n}  .
    \end{aligned}
    \]
\end{lemma}

The proof of Lemma \ref{lem:tech_bound_M_EM} can be found in Section \ref{app:proof_lem_tech_bound_M_EM}.

\subsubsection{Proof of Lemma \ref{lem:tech_bounds_norm_M}}\label{app:proof_lem_tech_bounds_norm_M}

The bound on the operator norm $\| \bM \|_\op$ follows from the same argument as in \cite[Lemma 7.2]{cheng2022dimension}. We reproduce the main steps of the proof for the reader's convenience. The other bounds will follow from this first bound. Without loss of generality, we set $\xi_{p+1} = \xi_{p+2} = \ldots = \xi_n = 0$ if $p < n$. In that case, $\bSigma^{-1}$ corresponds to the pseudo-inverse and $\lambda_{\min} (\bA)$ denotes the smallest positive eigenvalue.

\noindent
\textbf{Step 1. Decomposing the bound on $\| \bM \|_\op$.}

For any $k \leq n$, we consider the decomposition of the feature vector into $\bx = (\bx_0, \bx_+)$ where $\bx_0$ corresponds to the first $k$ coordinates of the feature, and $\bx_+$ the last $p-k$ coordinates. We denote $\bSigma_0 = \E [ \bx_0 \bx_0^\sT]$ and $\bSigma_+ = \E[ \bx_+ \bx_+^\sT ]$ their associated covariance matrices such that $\bSigma = \text{diag} ( \bSigma_0, \bSigma_+)$. We further introduce the feature matrices $\bX_0 = [\bx_{0,1} , \ldots \bx_{0,n}]^\sT \in \R^{n \times k}$ and $\bX_+ = [\bx_{+,1} , \ldots \bx_{+,n}]^\sT \in \R^{n \times (p-k)}$, so that $\bX = [\bX_0,\bX_+]$. Their whitened counterparts are denoted by $\bZ_0 = \bX_0 \bSigma_0^{-1/2}$ and $\bZ_+ = \bX_+ \bSigma_+^{-1/2}$. 

Using \cite[Lemma C.1]{cheng2022dimension}, we have
\[
\bX^\sT \bX + \lambda \id \succeq \left( 1 + \frac{2 \lambda_{\max} ( \bX_+^\sT \bX_+)}{\lambda}\right)^{-1} \begin{pmatrix}
    \bX_0^\sT \bX_0  & \bzero \\
    \bzero & \bzero 
\end{pmatrix}+ \frac{\lambda}{2} \id  ,
\]
such that the matrix $\bM$ satisfies
\begin{equation}\label{eq:block_bound_op_M}
\bM = \bSigma^{1/2} \bR \bSigma^{1/2} \preceq \begin{pmatrix}
   \bSigma_0^{1/2} \left( \left\{ 1 + 2 \lambda_{\max} ( \bX_+^\sT \bX_+)/\lambda\right\}^{-1}  \bX_0^\sT \bX_0 + \frac{\lambda}{2} \id \right)^{-1} \bSigma_0^{1/2} & \bzero \\
    \bzero &  \frac{2}{\lambda} \bSigma_+ 
\end{pmatrix}  ,
\end{equation}
Hence, $\| \bM \|_\op$ admits the following upper bound
\begin{equation}\label{eq:Mop_upper_decompo}
    \begin{aligned}
        \| \bM \|_{\op} \leq \frac{1}{\lambda_{\min} (\bZ_0^\sT \bZ_0)} \left( 1 + \frac{2 \lambda_{\max} ( \bX_+^\sT \bX_+)}{\lambda}\right) + \frac{2 \xi_{k+1}}{\lambda}  .
    \end{aligned}
\end{equation}

\noindent
{\bf Step 2: Bounding $\lambda_{\max} (\bX_+^\sT \bX_+)$.}

If $\xi_{k+1} = 0$, then this term is simply  zero. We assume below that $\xi_{k+1} >0$.
Let us introduce the matrices $\bS_i = \bx_{+,i} \bx_{+,i}^\sT \in \R^{(p-k) \times (p-k)}$ for $i \in [n]$, such that we can write
\[
\bS := \bX_+^\sT \bX_+ = \sum_{i \in [n]} \bx_{+,i} \bx_{+,i}^\sT = \sum_{i \in [n]} \bS_i  .
\]
Note that $\| \bS_i \|_\op = \| \bx_{+,i} \|_2^2$. By Assumption \ref{ass_app:deterministic_equivalent}, we have
\[
\P \left( \big\vert \|  \bx_{+} \|_2^2 - \Tr( \bSigma_+ ) \big\vert \geq t \cdot \varphi_1 (p) \cdot \big\| \bSigma_+ \big\|_F \right) \leq \sfC_x \exp \left\{ - \sfc_x t^{1/\beta } \right\}  .
\]
Hence, there exists a constant $C_{x,D}$ that only depends on $\sfc_x, \sfC_x,\beta,D$, such that with probability at least $1 - n^{-D}$, we have for all $i \in [n]$,
\begin{equation}\label{eq:tech_def_Ln}
\begin{aligned}
\| \bS_i \|_\op \leq&~ \Tr (\bSigma_+) + C_{x,D} \cdot \varphi_1(p) \log^\beta (n) \sqrt{\xi_{k+1}  \Tr (\bSigma_+)} \\
\leq&~ C_{x,D} \cdot r_\bSigma \cdot \xi_{k+1} \left( 1+  n^{-1/2}\varphi_1(p) \log^\beta (n)\right) =: L_n  ,
\end{aligned}
\end{equation}
where we used $\Tr (\bSigma_+) \leq r_\bSigma \cdot \xi_{k+1}$ and $r_\bSigma  \geq n$ by definition of the effective rank.
We now apply a matrix concentration inequality with a standard truncation argument. Let $\Tilde{\bS}_i := \bS_i \ind_{\| \bS_i \|_\op \leq L_n}$ and consider 
\[
\Tilde{\bS} := \sum_{i \in [n]} \Tilde{\bS}_i = \sum_{i\in[n]} \bS_i \ind_{\| \bS \|_\op \leq L_n} .
\]
From our choice of $L_n$ in Eq.~\eqref{eq:tech_def_Ln}, we have $\Tilde{\bS} = \bS$ with probability at least $1 - n^{-D}$. Using that $\Tilde{\bS}_i$ are independent symmetric matrices, we can upper bound the matrix variance by
\[
\begin{aligned}
 \Var( \Tilde{\bS} ) \preceq \sum_{i\in[n]}  \Var( \Tilde{\bS}_i )  \preceq \sum_{i \in [n]} \E \big[\Tilde{\bS}_i^2 \big]\preceq n L_n   \E \big[\bS_i \big] \preceq  n L_n  \bSigma_+ =: \bV_+  ,
\end{aligned}
\]
where we used that $\| \Tilde{\bS}_i \|_\op \leq L_n$, $\Tilde{\bS}_i \preceq \bS_i$ and $\E[ \bS_i ]= \bSigma_+ $. Denoting $v_n := \|  \Var( \Tilde{\bS} )\|_\op$, we have $v_n \leq nL_n \xi_{k+1}$. Furthermore, by definition of the effective rank, we have the following bound on the intrinsic dimension
\[
\text{intDim} (\bSigma_+) = \frac{\Tr(\bV_+)}{\| \bV_+\|_\op} \leq \frac{\Tr(\bSigma_+)}{\xi_{k+1}} \leq r_{\bSigma} .
\]
Hence, we can apply the matrix Bernstein inequality with intrinsic dimension \cite[Theorem 7.3.1]{tropp2015introduction}, and get for all $t \geq \sqrt{v_n} + 2L_n/3$,
\[
\P \left( \big\| \Tilde{\bS}  - \E[ \Tilde{\bS}]\big\|_\op \geq t \right) \leq 4 r_{\bSigma} \cdot \exp \left( \frac{-t^2/2}{v_n + 2L_n/3} \right)  .
\]
Finally, we can bound the mean using
\[
\| \E[ \Tilde{\bS}] \|_\op = n \| \E [ \Tilde{\bS}_i] \|_\op \leq n \| \E [ \bS_i ]\|_\op \leq n \xi_{k+1} .
\]
Combining the above bounds, we deduce that there exist constants $C_{x,D},C_{x,D}'$ that only depend on $\sfc_x,\sfC_x,\beta, D$, such that with probability at least $1 - n^{-D}$,
\[
\begin{aligned}
    \| \Tilde{\bS} \|_\op \leq&~ \| \E[ \Tilde{\bS}] \|_\op + C_{x,D} ( \sqrt{v_n} + L_n) \sqrt{\log(r_\bSigma n)} \\
    \leq&~ C_{x,D}' \cdot r_\bSigma \cdot \xi_{k+1} \sqrt{\log (r_\bSigma n)} \left(1 + n^{-1/2} \varphi_1(p) \log^\beta (n) \right)  .
\end{aligned}
\]
Recalling that $\bS = \Tilde{\bS}$ with probability at least $1 - n^{-D}$, we deduce that there exists a constant $C_{x,D}$ such that with probability at least $1 - 2n^{-D}$, 
\begin{equation}\label{eq:bound_X+X+}
\lambda_{\max} ( \bX_+^\sT \bX_+ ) \leq C_{x,D} \cdot r_\bSigma \cdot \xi_{k+1} \sqrt{\log(r_\bSigma)} \left(1 + n^{-1/2} \varphi_1(p) \log(n)^\beta \right) .
\end{equation}

\noindent
{\bf Step 3: Bounding $\lambda_{\min} (\bZ_0^\sT \bZ_0)$.}

From Assumption \ref{ass_app:deterministic_equivalent}, we have for all vectors $\bu \in \R^k$, $\| \bu \|_2 = 1$,
\[
\begin{aligned}
 \E [  \< \bu , \bz_0 \>^4] =&~ \int_0^\infty 2t \P ( \< \bu ,  \bz_0 \>^2 \geq t ) \de t \\
 \leq&~ 1 + \int_0^\infty 2 (t+1) \P ( \< \bu ,  \bz_0 \>^2 -  1 \geq t ) \de t\\
 \leq&~ 1 + 2\sfC_x  \int_0^\infty (t + 1) \exp \left\{ - \sfc_x t^{1/\beta}\right\} \de t =: C_x  .
 \end{aligned}
\]
Therefore, we can directly apply \cite[Theorem 2.2]{yaskov2014lower}: there exists a constant $C_D$ that only depends on $D$ such that with probability at least $1 - n^{-D}$,
\[
\lambda_{\min} \left( \frac{1}{n} \bZ_0^\sT \bZ_0 \right) \geq 1 - 4 C_x^{1/2} \sqrt{\frac{k}{n}} - C_D \frac{\log(n)}{n}  .
\]
Hence, there exist constants $\eta \in (0,1/2)$ that only depends on $\sfc_x,\sfC_x,\beta$, and $C_D$ that only depends on $D$ such that for all $n \geq C_D$ and taking $k = \lfloor \eta n \rfloor - 1$, it holds with probability at least $1 - n^{-D}$ that
\begin{equation}\label{eq:bound_Z0Z0}
\lambda_{\min} \left(  \bZ_0^\sT \bZ_0 \right) \geq \frac{n}{2} .
\end{equation}

\noindent
{\bf Step 4: Concluding the bounds on $\| \bM \|_\op$, $\Tr( \bM )$ and $\| \bM \|_F$.}

Combining the bounds \eqref{eq:bound_X+X+} and \eqref{eq:bound_Z0Z0} into Eq.~\eqref{eq:Mop_upper_decompo}, we get that with probability at least $1-n^{-D}$,
\[
\begin{aligned}
\| \bM \|_{\op}\leq&~ \frac{2}{n} \left\{ 1 + \frac{C_{x,D} \cdot r_\bSigma \cdot \xi_{\lfloor \eta n \rfloor} \sqrt{\log(r_\bSigma)} \left(1 + n^{-1/2} \varphi_1(p) \log(n)^\beta \right)}{\lambda}\right\} + \frac{2\xi_{\lfloor \eta n \rfloor}r_\bSigma}{n\lambda} \\
\leq&~ C_{x,D} \frac{\onu_\lambda (n)}{n } .
\end{aligned}
\]

Using the block matrix upper bound \eqref{eq:block_bound_op_M}, we have with probability at least $1 - n^{-D}$,
\[
\Tr (\bM) \leq \lfloor \eta n \rfloor \cdot \| \bM\|_\op + \frac{2\Tr(\bSigma_+)}{\lambda} \leq C_{x,D} \cdot \onu_\lambda (n)  ,
\]
where we used that $\Tr(\bSigma_+) \leq r_\bSigma \cdot \xi_{\lfloor \eta n \rfloor}$. Finally, combining the bounds on $\| \bM \|_\op$ and $\Tr(\bM)$, we get
\[
\| \bM \|_F \leq \sqrt{\| \bM \|_\op \Tr(\bM)} \leq C_{x,D} \frac{\onu_\lambda (n)}{\sqrt{n} } ,
\]
with probability at least $1 - n^{-D}$, which concludes the proof.

\subsubsection{Proof of Lemma \ref{lem:tech_bound_zAz}} 
\label{app_det_equiv:proof_lem_zAz}

    The first inequality directly follows from Assumption \ref{ass_app:deterministic_equivalent} by choosing $t = C_{x,D} \log^{\beta} (n)$. For the second inequality, we apply twice the first bound, first over $\bx_1$ and second over $\bx_2$:
    \[
    \begin{aligned}
        \left| \bx_1^\sT \bB \bx_2 \right| = &~ \sqrt{\bx_1^\sT \bB \bx_2 \bx_2^\sT \bB \bx_1} \\
        \leq&~ C_{x,D} \log^{\beta/2} (n) \sqrt{\bx_2^\sT \bB \bSigma \bB \bx_2} \\
        \leq&~ C_{x,D} \log^{\beta} (n) \left\{ \Tr (\bSigma^{1/2} \bB \bSigma \bB \bSigma^{1/2})^{1/2} + \varphi_1^{\bB} (p)^{1/2} \| \bSigma^{1/2} \bB \bSigma \bB \bSigma^{1/2}\|_F \right\},
    \end{aligned}
    \]
    where in the first inequality we used that $\bB \bx_2 \bx_2^\sT \bB $ is rank $1$ and we can apply the first inequality \eqref{eq_app:DE_concentration_b2} in Assumption \ref{ass_app:deterministic_equivalent} without $\varphi_1(p)$. In the second inequality, note that if $\bB$ is rank $1$, so is $\bB \bSigma \bB$.
    
    For the bound in expectation, we integrate the tail bound and get
    \[
    \begin{aligned}
     \E_{\bx} \left[ \big\vert \bx^\sT \bB \bx  - \Tr(\bSigma \bB) \big\vert^q \right] =&~ q \int_{0}^\infty t^{q-1} \P \left( \left\vert \bx^\sT \bB \bx  - \Tr(\bSigma \bB) \right\vert \geq t \right) \de t \\
     \leq&~ q \sfC_x  ( \varphi^\bB_1(p) \| \bSigma^{1/2} \bB \bSigma^{1/2} \|_F )^q \int t^{q-1} \exp \left\{ - \sfc_x t^{1/\beta} \right\} \de t \\
     \leq&~ C_{x,q} ( \varphi^\bB_1(p) \| \bSigma^{1/2} \bB \bSigma^{1/2} \|_F )^q  .
     \end{aligned}
    \]
    The second bound in expectation is obtained by applying twice the first bound, first over $\bx_1$ and second over $\bx_2$:
    \[
    \begin{aligned}
        \E_{\bx_1,\bx_2} \left[ | \bx_1^\sT \bB \bx_2|^q \right]^{1/q} =&~ \E_{\bx_2 }\left[  \E_{\bx_1}  \left[ |\bx_1^\sT \bB \bx_2 \bx_2^\sT \bB \bx_1 |^{q/2}\right]\right]^{1/q}\\
        \leq&~ C_{x,q} \E_{\bx_2}\left[ | \bx_2^\sT \bB \bSigma \bB \bx_2 |^{q/2} \right]^{1/q}\\
        \leq&~ C_{x,q} \left\{ \Tr (\bSigma^{1/2} \bB \bSigma \bB \bSigma^{1/2})^{1/2} + \varphi_1^{\bB} (p)^{1/2} \| \bSigma^{1/2} \bB \bSigma \bB \bSigma^{1/2}\|_F \right\}.
    \end{aligned}
    \]

\subsubsection{Proof of Lemma \ref{lem:properties_mu_obM}}
\label{app_det_equiv:proof_lemma_mu_obM}

Using the same notations as in the proof of Lemma \ref{lem:tech_bounds_norm_M}, we have
\[
\begin{aligned}
\Tr (\obM) =&~ \Tr \left( \bSigma_0 (\mu_* \bSigma_0 + \lambda)^{-1} \right) +  \Tr \left( \bSigma_+ (\mu_* \bSigma_+ + \lambda)^{-1} \right) \\
\leq&~\frac{\lfloor \eta n \rfloor}{\mu_*} + \frac{\Tr(\bSigma_+)}{\lambda} \\
\leq&~ \frac{1}{2}\left(1 + \Tr(\obM) \right) + \frac{\xi_{\lfloor \eta n \rfloor} \cdot r_{\bSigma}}{\lambda},
\end{aligned}
\]
where we used the fixed point equation \eqref{eq:det_equiv_fixed_point_mu_star}, $\eta \leq 1/2$, and the definition of $r_\bSigma$. Rearranging this inequality yields 
\begin{equation}\label{eq:bound_Tr(obM)}
1 + \Tr (\obM)  \leq 2 + 2 \frac{\xi_{\lfloor \eta n \rfloor} \cdot r_{\bSigma}}{\lambda} \leq 2 \nu_\lambda (n).
\end{equation}
Then, we simply use that
\[
\| \obM \|_\op \leq \frac{1}{\mu_*} = \frac{1 + \Tr (\obM)}{n}\leq \frac{2\nu_\lambda(n)}{n}.
\]

By definition of $\wbM$ and $\obM$, we have
\begin{align}\label{eq:diff_wbM_obM}
    \wbM - \obM = (\mu_* -\widetilde  \mu_*) \wbM \obM.
\end{align}
Applying the fixed point equation $\widetilde  \mu_* = (n-1)/(1 + \Tr(\wbM))$, we get by simple algebraic manipulation
\[
\mu_* -\widetilde  \mu_* = \frac{n \Tr \left( \wbM - \obM \right)}{(1 + \Tr(\wbM) )(1 + \Tr(\obM))} + \frac{1}{1 + \Tr(\wbM)}.
\]
Using Eq.~\eqref{eq:diff_wbM_obM} into this equation along with the bound $\| \obM \|_\op \leq 1 / \mu_* = (1 + \Tr(\obM))/n$, we obtain
\[
| \mu_* -\widetilde  \mu_* | \leq |\mu_* -\widetilde  \mu_*| \frac{\Tr(\wbM)}{1 + \Tr(\wbM)} + \frac{1}{1 + \Tr(\wbM)} ,
\]
which immediately implies $| \mu_* -\widetilde  \mu_* | \leq 1$. 

Combining this bound with Eq.~\eqref{eq:diff_wbM_obM} and applying Eq.~\eqref{eq:bound_Tr(obM)} to $\wbM$, we obtain for the first functional
\[
\left| \Psi_1 (\tmu_* ; \bA) - \Psi_1 (\mu_* ; \bA) \right| = \left| \Tr \left( \bA ( \wbM - \obM) \right) \right| \leq \| \wbM \|_\op \Tr (\bA \obM) \leq 2 \frac{\nu_\lambda (n)}{n}\Psi_1 (\mu_* ; \bA) .
\]
The second functional direcly follows from this first inequality by noting that
\[
\begin{aligned}
    \left| \Psi_2 (\tmu_* ) - \Psi_2 (\mu_*) \right| =&~ \left| \frac{\Psi_1 (\tmu_*;\id )}{1 + \Psi_1 (\tmu_*;\id )} - \frac{\Psi_1 (\mu_*;\id )}{1 + \Psi_1 (\mu_*;\id )}\right| \\
    =&~  \frac{\left| \Psi_1 (\tmu_*;\id ) - \Psi_1 (\mu_*;\id ) \right|}{(1 + \Psi_1 (\tmu_*;\id )) (1 + \Psi_1 (\mu_*;\id ))} \\
    \leq&~ 2\frac{\nu_\lambda (n)}{n} \frac{\Psi_1 (\mu_*;\id )}{1 + \Psi_1 (\mu_*;\id )}.
\end{aligned}
\]

For the third functional, we introduce $\sU, \sV$ the numerator and denominator of $\Psi_3 (\mu_*;\bA)$, and $\sU_-,\sV_-$ the numerator and denominator of $\Psi_3 (\tmu_*;\bA)$. First, note that
\[
\frac{n}{\mu_* } = 1 + \Tr(\obM) \leq 2 \nu_\lambda (n), \qquad\qquad \mu_* \geq \frac{\sqrt{5}}{2}
\]
by the assumption that $\nu_\lambda (n) \leq \sqrt{n}$ and $n \geq 5$. In particular, $\tmu_*^{-1} \leq C \nu_\lambda (n)/n $.
Thus by simple algebra, we get for the numerator
\[
\begin{aligned}
    \left| \sU - \sU_- \right| =&~ \left| \Tr( \bA \wbM^2) - \Tr(\bA \obM^2)\right|\\
    =&~ \left| \Tr\left(\bA \bSigma^2 (\tmu_* \bSigma + \lambda)^{-2}  (\mu_* \bSigma + \lambda)^{-2}\left\{(\mu_* \bSigma + \lambda)^{2}- (\tmu_* \bSigma + \lambda)^{2} \right\} \right)\right|\\
    \leq&~\left\{ \frac{| \mu_*^2 - \tmu_*^2|}{\tmu_*^2} + 2 \frac{| \mu_* - \tmu_*|}{\tmu_*}\right\} \Tr(\bA \obM^2)\\
    \leq&~ C \frac{\nu_\lambda (n)}{n}\sU.
\end{aligned}
\]
For the numerator, we proceed similarly
\[
\begin{aligned}
\left| \sV - \sV_- \right| =&~ \left| \frac{\mu_*^2}{n} \Tr( \obM^2) - \frac{\tmu_*^2}{n} \Tr( \wbM^2) \right|\\
\leq&~   \frac{\mu_*^2}{n}  \left|\Tr( \obM^2) -  \Tr( \wbM^2) \right| + \frac{| \mu_*^2 - \tmu_*^2|}{\tmu_*^2} \cdot  \frac{\tmu_*^2}{n } \Tr (\wbM^2) \\
\leq&~ 2 \left\{ \frac{| \mu_*^2 - \tmu_*^2|}{\tmu_*^2} + \frac{\mu_* | \mu_* - \tmu_*|}{\tmu_*^2}\right\} \cdot  \frac{\tmu_*^2}{n } \Tr (\wbM^2) \\
\leq&~ C\frac{\nu_\lambda (n)}{n} \sV_-.
\end{aligned}
\]
Combining the above bounds, we obtain
\[
\begin{aligned}
    \frac{| \Psi_3 (\tmu_*;\bA) - \Psi_3 (\mu_*;\bA)|}{\Psi_3(\mu_*;\bA)} =&~ \frac{\left| \sV \sU_- - \sU \sV_- \right|}{\sU \sV_-} \\
    \leq&~ \frac{| \sU - \sU_-|}{\sU} + \left( 1 + \frac{| \sU - \sU_-|}{\sU} \right) \frac{| \sV - \sV_-|}{\sV_-} \leq C \frac{\nu_\lambda(n)}{n}.
\end{aligned}
\]
The bound on $\Psi_4 (\tmu_* ; \bA)$ follows from a similar argument and we omit it.

\subsubsection{Proof of Lemma \ref{lem:tech_upper_bound_AM}}
\label{app:proof_lem_tech_upper_bound_AM}

We will follow the same notations as in the proof of Lemma \ref{lem:tech_bounds_norm_M}.

\paragraph*{Proof of Lemma \ref{lem:tech_upper_bound_AM}.(a):} This inequality simply follows by noting that
\[
\mu_* \bSigma + \lambda = \frac{n\bSigma}{1 + \Tr(\obM)} + \lambda \preceq \max\left(1 , \frac{1+\kappa}{1+ \Tr(\obM)} \right) \left( \frac{n \bSigma}{1+\kappa} + \lambda \right) ,
\]
so that
\[
\obM_- \preceq \max\left(1 , \frac{1+\kappa}{1+ \Tr(\obM)} \right) \cdot \obM \preceq (1+\kappa) \obM .
\]
Denote $\cA$ the event $\Tr( \bM_- ) \leq C_{x,D} \cdot \nu_\lambda (n)$ such that $\P (\cA) \geq 1 - n^{-D}$ by applying Lemma \ref{lem:tech_bounds_norm_M} to $\bM_-$ with $n \geq C$. Choosing $D = K+1$, we get 
\[
\begin{aligned}
  \kappa =  \E [ \Tr(\bM_-)] \leq&~ C_{x,D} \cdot \nu_\lambda (n) \E [ \ind_{\cA}] + \E[ \ind_{\cA^c} \Tr( \bM_-) ] \\
    \leq &~ C_{x,D} \cdot \nu_\lambda (n) + \frac{\Tr(\bSigma)}{\lambda} n^{-D} \\
    \leq&~ C_{x,D} \cdot \nu_\lambda (n) + \frac{n^{1-D}}{\lambda} + \frac{r_{\bSigma} \cdot \xi_{\lfloor \eta n \rfloor}}{\lambda}n^{-D} \leq C_{x,K} \cdot \nu_\lambda (n)  ,
\end{aligned}
\]
where we used $\Tr(\bSigma) \leq \lfloor \eta n \rfloor \cdot \| \bSigma\|_\op + \Tr(\bSigma_+) \leq n + r_{\bSigma} \cdot \xi_{\lfloor \eta n \rfloor}$ and the assumption $\lambda \cdot \nu_\lambda (n) \geq n^{-K} $. 

The bound on the operator norm follows by writing 
\[
\| \obM_- \|_\op \leq \frac{1 +\kappa}{n} \leq C_{x,K} \frac{\nu_\lambda(n)}{n}.
\]
The other inequalities are direct consequences of these two first inequalities.

\paragraph*{Proof of Lemma \ref{lem:tech_upper_bound_AM}.(b):} 
 Let $k = \lfloor \eta n \rfloor -1$ be chosen as in Lemma \ref{lem:tech_bounds_norm_M}. Note that because of condition \eqref{eq:conditions_tech_upper_bound_AM}, we have $\onu_\lambda(n)$ from Lemma \ref{lem:tech_bounds_norm_M} that satisfies $\onu_\lambda (n) \leq C_K \nu_{\lambda} (n)$. We write in block matrix form
    \[
    \obM = \begin{pmatrix}
        \bSigma_0 ( \mu_*\bSigma_0 + \lambda )^{-1} & \bzero\\
        \bzero & \bSigma_+ ( \mu_* \bSigma_+ + \lambda )^{-1}  
    \end{pmatrix} \succeq \begin{pmatrix}
        \bSigma_0 ( \mu_* \bSigma_0 + \lambda)^{-1} & \bzero\\
        \bzero & \frac{\bSigma_+}{\mu_* \cdot \xi_{\lfloor \eta n \rfloor} + \lambda} 
    \end{pmatrix}.
    \]
    Similarly, recalling Eq.~\eqref{eq:block_bound_op_M}, we have with probability at least $1 - n^{-D}$ that
    \[
    \bM \preceq \begin{pmatrix}
     \bSigma_0 \left( \gamma^{-1} \bSigma_0 + \lambda/2 \right)^{-1}   & \bzero\\
        \bzero & 2\bSigma_+ /\lambda ,  
    \end{pmatrix}.
    \]
    where we denoted $\gamma = C_{x,K,D} \cdot \nu_\lambda (n) / n$.
    Thus, we deduce that 
    \[
    \bM \preceq  \left\{ \max \left( \gamma \mu_* , 2 \right) + 2 \frac{\mu_* \cdot \xi_{\lfloor \eta n \rfloor} + \lambda}{\lambda} \right\}  \obM \preceq C_{x,D,K} \cdot \nu_\lambda (n) \obM,
    \]
    where we used that $\mu_* \leq n$. 

    Denote $\cA$ the event $\bM \preceq C_{x,K,D'} \cdot \nu_\lambda (n)/n$ such that $\P (\cA) \geq 1 - n^{-D'}$ by Lemma \ref{lem:tech_bounds_norm_M}. Note that on $\cA^c$, we have simply $\| \bM \|_\op \leq 1/\lambda$. By triangular inequality,
 \[
    \E_j [ \Tr(\bA \bM^\ell)^q ]^{1/q}\leq C_{x,D',\ell} \frac{ \nu_\lambda (n)^\ell}{n^r} \Tr(\bA \obM^{\ell-r}) + \frac{\Tr(\bA\bSigma^{\ell-r})}{\lambda^\ell} \E_j [ \ind_{\cA^c}]^{1/q}   ,
    \]
    where we used that $\bM \preceq C_{x,D',K} \cdot \nu_\lambda (n) \obM$ on event $\cA$, and $\| \bSigma \|_\op =1$.
    First, note that by definition of $\obM$,
    \[
    \Tr(\bA\bSigma^{\ell - r}) \leq (\mu_* + \lambda)^{\ell - r} \Tr(\bA\obM^{\ell-r})  .
    \]
    Furthermore, by Markov's inequality, we have
    \[
    \P \left( \E_j [ \ind_{\cA^c}] \geq n^D \E [ \ind_{\cA^c}] \right) \leq n^{-D} .
    \]
    Hence, combining these bounds, we obtain
    \[
    \begin{aligned}
    \E_j [ \Tr(\bA \bM^\ell)^q ]^{1/q} \leq&~ C_{x,D',\ell} \cdot \frac{\nu_\lambda (n)^\ell}{n^r} \Tr(\bA \obM^{\ell-r}) + \left( \frac{n}{\lambda} + 1 \right)^{\ell-r} n^{(D - D')/q} \Tr(\bA \obM^{\ell-r})\\ \leq&~ C_{x,K,D,q,\ell} \cdot \nu_\lambda (n)^\ell \Tr(\bA \obM^\ell) ,
    \end{aligned}
    \]
    where we chose $D' = D + q(K+\ell)$ and used the assumption that $\lambda \cdot \nu_\lambda (n) \geq n^{-K}$. The other inequalities follow similarly.

\paragraph*{Proof of Lemma \ref{lem:tech_upper_bound_AM}.(c):} Using Sherman-Morrison identity \eqref{eq:standard_identities_SMW} and Lemma \ref{lem:tech_bound_zAz}, we get
    \[
    \begin{aligned}
     \left\vert \E\left[ \Tr (\bM - \bM_-)\right] \right\vert =&~ \left| \E \left[ \frac{\bz^\sT \bM_-^2  \bz}{1 + \bz^\sT \bM_- \bz}\right] \right|\\
     \leq&~ C_x \left\{ \E [ \Tr( \bM_-^2) ] + \varphi_1 (p) \E [ \|\bM_-^2 \|_F  ] \right\}  \\
     \leq&~ C_{x,K} \left\{ \frac{\nu_\lambda (n)^2}{n} + \frac{\varphi_1 (p) \nu_\lambda(n)^2}{n^{3/2}} \right\} \leq C_{x,K} \frac{\nu_\lambda(n)^2}{n},
     \end{aligned}
    \]
    where we used Lemma \ref{lem:tech_upper_bound_AM}.(b) in the second inequality and condition \eqref{eq:conditions_tech_upper_bound_AM} in the last inequality.

We proceed similarly for the second inequality. Note that we can decompose
\[
\begin{aligned}
    \E[ \Tr(\bM^2 - \bM_-^2)] = - 2 \E\left[ \frac{\bz^\sT \bM_-^3 \bz}{1 + \bz^\sT \bM_- \bz}\right] + \E\left[\frac{(\bz^\sT \bM_-^2 \bz)^2}{(1 + \bz^\sT \bM_- \bz)^2}\right].
\end{aligned}
\]
Applying Lemma \ref{lem:tech_bound_zAz} over $\bz$ followed by Lemma \ref{lem:tech_upper_bound_AM}.(b) on $\bM_-$, we deduce that
\[
\begin{aligned}
   \left| \E[ \Tr(\bM^2 - \bM_-^2)] \right| \leq C_{x,K} \frac{\nu_\lambda (n)^3}{n^2} +  C_{x,K} \frac{\nu_\lambda (n)^4}{n^2},
\end{aligned}
\]
which concludes the proof of the lemma.



    \subsubsection{Proof of Lemma \ref{lem:tech_bound_M_EM}}\label{app:proof_lem_tech_bound_M_EM}

    From the proof of Proposition \ref{prop:TrAM_martingale} with $\bA = \id$, we see that for any $D >0$ there exists $C_{x,K,D}$ such that with probability at least $1 -n^{-D}$,
    \[
    \left\vert \Tr(\bM) - \E[\Tr(\bM)] \right\vert \leq R_{D}:= C_{x,K,D} \frac{\varphi_1(p) \nu_\lambda (n)^3 \log^{\beta +\frac12 (n)}}{n}.
    \]
 Denote $\cA_{D}$ this event. By triangular inequality, we have
    \[
    \begin{aligned}
        \E\left[\left( \Tr(\bM) - \E[ \Tr(\bM)]\right)^q\right]^{1/q} \leq&~ R_{D} + 2 \frac{\Tr(\bSigma)}{\lambda} \E [\ind_{\cA_{\tilde{D}}^c}]^{1/q} \leq  R_{D} + 2 \frac{\Tr(\bSigma)}{\lambda} n^{-D/q}  .
    \end{aligned}
    \]
   We conclude by setting $D = q(K+1)$ and using the condition $\lambda \cdot \nu_\lambda(n) \geq n^{-K}$.


\clearpage

\section{Proofs for the test error and GCV estimator}
\label{app:main_proofs}

In this appendix, we prove the main theorems of this paper. We begin in Section \ref{app_test:preliminaries} by introducing some notations that will be useful in the proofs. For the readers' convenience, we then restate our assumptions and the different deterministic equivalents in Section \ref{app_test:assumptions}. The proofs of the non-asymptotic bounds for the Stieltjes transform, training error, and test error can be found in Sections \ref{app_test_error:Stieltjes_general}, \ref{app_test_error:training_error_general}, and \ref{app_test_error:test_error_general} respectively. We defer the proof of some technical results to Sections \ref{app_test_error:reduced_fixed_points} and \ref{app_test_error:tech_high-degree_target}. Finally, we use these non-asymptotic bounds to prove the uniform convergence of the GCV estimator (Theorem \ref{thm:abstract_GCV}) in Section \ref{app_test:GCV}.

\subsection{Preliminaries}
\label{app_test:preliminaries}

Let us begin by introducing some convenient notations that we will use throughout Appendix \ref{app:main_proofs}. Throughout the proofs, we will track the dependency in $\lambda$, $r_{\seff,\evn}, \varphi_1 (\evn), \varphi_{2,n} (\evn)$. For the remaining constants, recall that we denote $C_{a_1,\dots,a_k}$ constants that only depend on the values of $\{a_i\}_{i\in[k]}$. In particular, the value of these constants is allowed to change from line to line. We use $a_i = `x$' to denote the dependency on $\sfc_x,\sfC_x,\beta$ from Assumptions \ref{ass:main_assumptions}, and $a_i = `\eps$' to denote the dependency on $\tau^2_\eps$ from Assumption \ref{ass:noise_subGaussian}. 

For readability, we will simplify the notations from the main text in the proofs. We will use the subscript `$0$' instead of `$\leq \evn$' to refer to the low-degree part and `$+$' instead of `$>\evn$' to refer to the high-degree part. Thus we will denote the low-degree feature $\bx_0 = \bx_{\leq \evn}$ and high-degree feature $\bx_+ = \bx_{>\evn}$. Similarly their associated covariance matrices will be denoted $\bSigma_0 = \bSigma_{\leq \evn} = \E[ \bx_0 \bx_0^\sT]$ and $\bSigma_+ = \bSigma_{>\evn} = \E[\bx_+ \bx_+^\sT]$. We introduce the feature matrices $\bX_0 = [\bx_{0,1} , \ldots \bx_{0,n}]^\sT \in \R^{n \times \evn}$ and $\bX_+ = [\bx_{+,1} , \ldots \bx_{+,n}]^\sT \in \R^{n \times (p-\evn)}$, and their whitened counterparts $\bZ_0 = \bX_0 \bSigma_0^{-1/2}$ and $\bZ_+ = \bX_+ \bSigma_+^{-1/2}$. We can therefore write the feature matrices and covariance in block matrix form
\[
\bX = [\bX_0,\bX_+], \qquad\quad \bZ = [\bZ_0,\bZ_+], \qquad\quad \bSigma = \diag (\bSigma_0 , \bSigma_+).
\]
We further denote $\gamma_+ = \Tr(\bSigma_+)$ the self-induced regularization from the high-frequency part of the kernel, and $\xi_+ = \lambda_{\max} (\bSigma_+) = \xi_{\evn+1}$. We introduce $\lambda_+ = \lambda+\gamma_+$ to denote the effective ridge regularization of the model associated to $\bX_+$.

We will denote
\begin{equation}\label{eq:def_nu_0_lambda_+}
\nu_{\lambda_+} (n) = 1 + \frac{\xi_{\lfloor \eta n \rfloor, \evn} r_{\seff, \evn} (n) \sqrt{\log(r_{\seff,\evn} (n))}}{\lambda_+},
\end{equation}
where $r_{\seff,\evn}$ is the effective rank of $\bSigma_0$ defined in Definition \ref{def:effective_tail_rank}.(b) and we recall that we denote $\xi_{\lfloor \eta n \rfloor, \evn} = \xi_{\lfloor \eta n \rfloor}$ if $\lfloor \eta n \rfloor \leq \evn$ and $\xi_{\lfloor \eta n \rfloor, \evn} = 0$ otherwise.

Recall that we consider a target function $f_* (\bx) = \< \bbeta_*, \bz\> = \<\btheta_* , \bx\>$, where $\btheta_* = \bSigma^{-1/2} \bbeta_*$. We denote $\bbeta_0 = \bbeta_{*,\leq \evn}$ and $\bbeta_+ = \bbeta_{*,>\evn}$. Similarly, $\btheta_0 = \bSigma_0^{-1/2} \bbeta_0$ and $\btheta_+ = \bSigma_+^{-1/2} \bbeta_+$, and we introduce the low-degree and high-degree parts of the target function $f_0 (\bx) = \< \btheta_0, \bx_0\>$ and $f_+ (\bx) = \< \btheta_+ , \bx_+\>$. We will denote the $n$-dimensional vectors
\[
\begin{aligned}
    \by =&~ (y_1 , \ldots , y_n),\\
    \boldf =&~ (f_*(\bx_1), \ldots , f_*(\bx_n)),\\
    \boldf_0 =&~ ( f_0 (\bx_{1}), \ldots, f_0 (\bx_{n})), \\
    \boldf_+ =&~ ( f_+ (\bx_{1}), \ldots, f_+ (\bx_{n})), \\
    \beps =&~ (\eps_1, \ldots, \eps_n).
\end{aligned}
\]
In particular, $\by = \boldf + \beps$ and $\boldf = \boldf_0 + \boldf_+$.

In this section, we will consider two effective regularizations: $\lambda_*$ associated to the original model $(n,\bSigma,\lambda)$ and $\lambda_{*,0}$ associated to the truncated model $(n,\bSigma_0,\lambda_+)$. Recall that $\lambda_*$ and $\lambda_{*,0}$ are defined as the unique non-negative solutions to the two fixed point equations 
\[
n - \frac{\lambda}{\lambda_*} = \Tr(\bSigma (\bSigma+ \lambda_*)^{-1}), \qquad n - \frac{\lambda_+}{\lambda_{*,0} } = \Tr( \bSigma_0 ( \bSigma_0 + \lambda_{*,0} )^{-1} )  .
\]
It will be convenient to introduce the following notations:
\[
\begin{aligned}
&\Upsilon_{1} = \frac{1}{n}\Tr( \bSigma (\bSigma  + \lambda_{*} )^{-1} ), \qquad &&\Upsilon_{2} =\frac{1}{n} \Tr( \bSigma^2 (\bSigma + \lambda_{*} )^{-2} ), \\
&\Upsilon_{1,0} = \frac{1}{n}\Tr( \bSigma_0 (\bSigma_0 + \lambda_{*,0} )^{-1} ), \qquad &&\Upsilon_{2,0} =\frac{1}{n} \Tr( \bSigma_0^2 (\bSigma_0 + \lambda_{*,0} )^{-2} ).
\end{aligned}
\]

We further introduce the following matrices that will appear in the proofs
\[
\begin{aligned}
&\bG = (\bX \bX^\sT + \lambda)^{-1}, \qquad \bG_0 = (\bX_0\bX_0^\sT + \lambda_+ )^{-1}, \qquad \bDelta_+ = \bX_+ \bX_+^\sT - \gamma_+ \id,\\
& \bR = (\bX^\sT \bX + \lambda)^{-1} , \qquad \bR_0 = (\bX_0^\sT \bX_0 + \lambda_+ )^{-1} , \qquad \bM_0 = \bSigma_0^{1/2} \bR_0 \bSigma_0^{1/2}.
\end{aligned}
\]

Finally, our proofs will crucially rely on applying the deterministic equivalent bounds proved in Appendix \ref{app:det_equiv}. For convenience, we will denote $\cE_{j,n} (\evn)$ the decay rates for the functionals $\Phi_j(\bX_0) , j \in \{1 , \ldots , 4\}$, with regularization parameter $\lambda_+$, such that for any p.s.d.~matrix $\bA \in \R^{\evn \times \evn}$, we have
\begin{equation}\label{eq:decay_rates}
\left| \Phi_j (\bX_0 ; \bA) - \Psi_j (\mu_{*,0} ; \bA) \right| \leq \cE_{j,n} (\evn) \cdot \Psi_j (\mu_{*,0} ; \bA),
\end{equation}
with probability at least $1 - n^{-D}$, where $\mu_{*,0} = \lambda_+ / \lambda_{*,0}$. From Theorems \ref{thm_app:det_equiv_TrAM}--\ref{thm_app:det_equiv_TrAMZZM}, these decay rates are given by
\[
\begin{aligned}
    \cE_{1,n} (\evn) = \cE_{2,n} (\evn) =&~  C_{x,K,D}\frac{\varphi_1 (\evn) \nu_{\lambda_+} (n)^{5/2} \log^{\beta + \frac12} (n) }{\sqrt{n}},\\
    \cE_{3,n} (\evn) =&~   C_{x,K,D}\frac{\varphi_1 (\evn) \nu_{\lambda_+} (n)^{6} \log^{2\beta + \frac12} (n) }{\sqrt{n}}, \\
    \cE_{4,n} (\evn) =&~  C_{x,K,D}\frac{\varphi_1 (\evn) \nu_{\lambda_+} (n)^{6} \log^{\beta + \frac12} (n) }{\sqrt{n}}.
\end{aligned}
\]

\subsection{Assumptions and deterministic equivalents}
\label{app_test:assumptions}

For the reader's convenience, we restate our assumptions here.

\begin{assumption}[Concentration at $n \in \naturals$]\label{ass_app:main_assumptions}
There exist $\sfc_x,\sfC_x,\beta>0$ and $\evn \in \naturals \cup \{ \infty\}$ such that
\[
\frac{\lambda_+}{\xi_+} \geq 2n,
\]
and the following hold.
\begin{itemize}
    \item[\emph{(a)}] \emph{(Low-degree features.)} There exists $\varphi_1 (\evn)>0$ such that for any vector $\bv \in \R^{\evn}$ with $\| \bSigma_{0}^{1/2} \bv \|_2 < \infty$ and any p.s.d.~matrix $\bA \in \R^{\evn \times \evn}$ with $\Tr(\bSigma_{0} \bA) < \infty$, we have
    \begin{align}
        \P \left( \big\vert \< \bv, \bx_{0} \>  \big\vert \geq t \cdot   \bv^\sT \bSigma_{0} \bv \right) \leq&~ \sfC_x \exp \left\{ - \sfc_x t^{2/\beta } \right\}, \label{eq_app:ass_quad_concentration_1} \tag{a1}\\
    \P \left( \big\vert \bx_{0}^\sT \bA \bx_{0} - \Tr( \bSigma_{0} \bA) \big\vert \geq t \cdot \varphi_{1} (\evn) \cdot \big\| \bSigma_{0}^{1/2} \bA \bSigma_{0}^{1/2} \big\|_F \right) \leq&~ \sfC_x \exp \left\{ - \sfc_x t^{1/\beta } \right\}  . \label{eq_app:ass_quad_concentration_2} \tag{a2} 
    \end{align}

    \item[\emph{(b)}] \emph{(High-degree features.)} There exist $p_{2,n} (\evn) \in (0,1)$ and $\varphi_{2,n} (\evn) \geq 1$ such that with probability at least $1 - p_{2,n} (\evn)$, we have
    \begin{equation}
    \| \bDelta_+ \|_\op = \| \bX_{+} \bX_{+}^\sT - \gamma_+ \cdot  \id_n \|_\op \leq \varphi_{2,n} (\evn) \sqrt{\frac{n\xi_+}{\lambda_+}} \cdot \lambda_+. \tag{b1}
    \end{equation}

    \item[\emph{(c)}] \emph{(Target function.)} The high-degree part of the target function satisfies the tail bound
    \begin{equation}
\P \left( | f_{+} (\bx ) | \geq t \cdot \| f_{+} \|_{L^2} \right) \leq \sfC_x \exp \left\{ - \sfc_x t^{2/\beta} \right\}. \tag{c1}
\end{equation}
\end{itemize}
\end{assumption}

\begin{assumption}[Label noise]\label{ass_app:noise_subGaussian}
    The label noise $\eps_i$ is independent, mean-zero, and $\tau_\eps^2$-subgaussian with variance denoted $\sigma_\eps^2 = \E[ \eps_i^2]$.
\end{assumption}

In this appendix, we are interested in proving deterministic approximations to the Stieltjes transform of the empirical kernel matrix
\begin{equation}\label{eq:def_Stieltjes_app}
 s_n (\bX, \lambda) := \frac{1}{n}  \Tr \big[ ( \bX \bX^\sT + \lambda)^{-1} \big],
\end{equation}
and to the training and test errors
\begin{align}
    \cL_{\train} ( \bbeta_*;\bX,\beps, \lambda) =&~ \frac{1}{n} \sum_{i \in [n]} \big(y_i - \hat f_\lambda (\bx_i)\big)^2, \label{eq:def_training_app}\\
    \cR_{\test} ( \bbeta_*;\bX,\beps, \lambda) =&~ \E \big[ \big( y - \hat f_\lambda (\bx) \big)^2\big]. \label{eq:def_test_app}
\end{align}
We show that under Assumptions \ref{ass_app:main_assumptions} and \ref{ass_app:noise_subGaussian}, these quantities are well approximated by the following deterministic equivalents
\begin{align}
  \sfs_n (\lambda) :=&~ \frac{1}{n \lambda_*}, \label{eq:def_equiv_Stieltjes_app}\\
    \sL_n (\bbeta_* , \lambda) := &~   \left( \frac{\lambda}{n \lambda_*}\right)^2 \cdot \frac{\lambda_*^2 \< \bbeta_*, (\bSigma + \lambda_*)^{-2} \bbeta_*\> +\sigma_\eps^2}{1 - \frac{1}{n} \Tr( \bSigma^2 (\bSigma + \lambda_*)^{-2})} , \label{eq:def_equiv_training_app}\\
    \sR_{n} (\bbeta_* , \lambda) := &~ \frac{\lambda_*^2 \< \bbeta_*, (\bSigma + \lambda_*)^{-2} \bbeta_*\> + \sigma_\eps^2 }{1 - \frac{1}{n} \Tr( \bSigma^2 (\bSigma + \lambda_*)^{-2})},\label{eq:def_equiv_test_app}
\end{align}
where $\lambda_*$ is the effective regularization associated to $(n,\bSigma,\lambda)$.

The proof is organized as follows. In Section \ref{app_test_error:reduced_fixed_points}, we introduced reduced deterministic equivalents obtained by treating the high-frequency part of the features as an additive regularization and the high-degree part of the target function as independent label noise. We show in Lemma \ref{lem:reduced_det_equiv} that these reduced deterministic equivalents approximate well the full deterministic equivalents \eqref{eq:def_equiv_Stieltjes_app}--\eqref{eq:def_equiv_test_app} when $\evn$ is chosen sufficiently large. The proofs for the Stieltjes transform, training error and test error can be found in Sections \ref{app_test_error:Stieltjes_general}, \ref{app_test_error:training_error_general} and \ref{app_test_error:test_error_general} respectively. Technical lemmas on the contribution of the high-degree part of the target function are deferred to Section \ref{app_test_error:tech_high-degree_target}.

\subsection{Reduced deterministic equivalents}\label{app_test_error:reduced_fixed_points}

As mentioned in the preliminaries, we consider two effective regularizations: $\lambda_*$ associated to $(n,\bSigma,\lambda)$ and $\lambda_{*,0}$ associated to $(n,\bSigma_0,\lambda_+)$. Recall that they are defined as the unique non-negative solutions to the fixed point equations
\begin{align}
n - \frac{\lambda}{\lambda_*} =&~ \Tr(\bSigma (\bSigma+ \lambda_*)^{-1}),\label{eq:fixed_point_equations_star_app} \\
 n - \frac{\lambda_+}{\lambda_{*,0} } =&~ \Tr( \bSigma_0 ( \bSigma_0 + \lambda_{*,0} )^{-1} )  . \label{eq:fixed_point_equations_0_app}
\end{align}
We will further consider the following change of variables
\[
\mu_* = \lambda/\lambda_*, \qquad \quad \mu_{*,0} = \lambda_+/ \lambda_{*,0}.
\]
Intuitively, the cutoff $\evn$ is chosen such that $\xi_+ \ll \lambda_*$, and thus $\Tr(\bSigma_+ (\bSigma_+ + \lambda_*)^{-1}) \approx \Tr(\bSigma_+)/ \lambda_*$. By rearranging the terms, the fixed point equation~\eqref{eq:fixed_point_equations_star_app} becomes approximately the same as Eq.~\eqref{eq:fixed_point_equations_0_app}, and therefore $\lambda_* \approx \lambda_{*,0}$. The following lemma formalize this intuition.

\begin{lemma}\label{lem:reduced_fixed_point}
    We have
    \[
    0 \leq \frac{\lambda_{*,0} - \lambda_*}{\lambda_{*,0}} \leq \frac{n\xi_+}{\lambda_+}.
    \]
\end{lemma}

\begin{proof}[Proof of Lemma \ref{lem:reduced_fixed_point}]
By removing $\gamma_+/\lambda_*$ on both sides of the fixed point equation for $\lambda_*$ and rearranging the terms, we get
\[
\begin{aligned}
    n - \frac{\lambda_+}{\lambda_*} - \Tr( \bSigma_0 (\bSigma_0 + \lambda_*)^{-1}) = - \frac{\delta}{\lambda_*},
\end{aligned}
\]
where $\delta =  \Tr(\bSigma_+) - \lambda_*\Tr( \bSigma_+ (\bSigma_+ + \lambda_*)^{-1}) = \Tr( \bSigma_+^2  (\bSigma_+ + \lambda_*)^{-1}) >0 $. In particular, using that $\lambda_{*,0} (\lambda_+)$ is increasing in $\lambda_+$, this implies the first inequality $\lambda_{*,0} \geq \lambda_*$. 

Further replacing $n = \lambda_+ / \lambda_{*,0} + \Tr( \bSigma_0 (\bSigma_0 + \lambda_{*,0})^{-1})$, we obtain
\[
\begin{aligned}
- \frac{\delta}{\lambda_*}
    &= \frac{\lambda_+}{\lambda_{*,0}} + \Tr( \bSigma_0 (\bSigma_0 + \lambda_{*,0})^{-1}) - \frac{\lambda_+}{\lambda_*} - \Tr( \bSigma_0 (\bSigma_0 + \lambda_*)^{-1}) \\
    &= \left\{ \frac{\lambda_+}{\lambda_* \lambda_{*,0}}  + \Tr( \bSigma_+ (\bSigma_0 + \lambda_*)^{-1} (\bSigma_0 + \lambda_{*,0})^{-1} \right\}(\lambda_* - \lambda_{*,0})\\
    &\le  \frac{\lambda_+}{\lambda_* \lambda_{*,0}}
    (\lambda_* - \lambda_{*,0})
\end{aligned}
\]
Thus, we deduce that
\[
\frac{\lambda_{*,0} - \lambda_*}{\lambda_{*,0}} \leq \frac{\delta}{\lambda_+} \leq \frac{ n \xi_+ }{ \lambda_+},
\]
where we used that $\delta \leq \xi_+ \Tr( \bSigma_+  (\bSigma_+ + \lambda_*)^{-1}) \leq n \xi_+$.
\end{proof}

In particular, Lemma \ref{lem:reduced_fixed_point} implies that if we choose $\evn \in \naturals$ with $\xi_{\evn +1} \equiv \xi_+\leq \lambda_+/(2n)$, then
\begin{equation}\label{eq:lambda0_lower_bound_lambda_star}
\lambda_* \geq \frac{\lambda_{*,0}}{2} \geq \frac{\lambda_+}{2n}.
\end{equation}

The proof of the deterministic equivalents proceeds by splitting the analysis of the low-frequency and high-frequency parts of the features. For the low-frequency part, we will simply apply the deterministic equivalents studied in Appendix \ref{app:det_equiv}, with covariance $\bSigma_0$ and regularization parameter $\lambda_+$. On the other hand, we will show in Section \ref{app_test_error:tech_high-degree_target} that the high-frequency behaves effectively as an independent additive noise with variance $\| f_+ \|_{L^2}^2 = \| \bbeta_+ \|_2^2$. Hence, we will establish that the Stieltjes transform, training error, and test error are well approximated by the following reduced deterministic equivalents associated to $\lambda_{*,0}$:
\begin{align}
   \sfs_{n,0} (\lambda_+) :=&~ \frac{1}{n \lambda_{*,0}}, \label{eq:def_reduced_equiv_Stieltjes}\\
    \sL_{n,0} (\bbeta_* , \lambda_+) := &~   \left( \frac{\lambda}{n \lambda_{*,0}}\right)^2 \frac{\lambda_{*,0}^2 \< \bbeta_0, (\bSigma_0 + \lambda_{*,0})^{-2} \bbeta_0\> + \| \bbeta_+ \|_2^2 + \sigma_\eps^2 }{1 - \frac{1}{n} \Tr( \bSigma_0^2 (\bSigma_0 + \lambda_{*,0})^{-2})}, \label{eq:def_reduced_equiv_training}\\
    \sR_{n,0} (\bbeta_* , \lambda_+) := &~ \frac{\lambda_{*,0}^2 \< \bbeta_0, (\bSigma_0 + \lambda_{*,0})^{-2} \bbeta_0\> + \| \bbeta_+ \|_2^2 + \sigma_\eps^2 }{1 - \frac{1}{n} \Tr( \bSigma_0^2 (\bSigma_0 + \lambda_{*,0})^{-2})}. \label{eq:def_reduced_equiv_test}
\end{align}
The following lemma shows that the full deterministic equivalents associated to the full model $(n, \bSigma,\lambda)$ are well approximated by these reduced deterministic equivalents associated to the truncated model $(n, \bSigma_0,\lambda_+)$.

\begin{lemma}\label{lem:reduced_det_equiv}
    Let $\sL_n (\bbeta_*,\lambda)$, $\sR_n (\bbeta_*,\lambda)$, and $\sfs_n(\lambda)$ be defined as per Eqs.~\eqref{eq:def_equiv_Stieltjes_app}--\eqref{eq:def_equiv_test_app}, and $\sL_{n,0} (\bbeta_*,\lambda_+)$, $\sR_{n,0} (\bbeta_*,\lambda_+)$, and $ \sfs_{n,0}(\lambda_+)$ be defined as per Eqs.~\eqref{eq:def_reduced_equiv_Stieltjes}--\eqref{eq:def_reduced_equiv_test}, where $\lambda_+ = \lambda + \gamma_+$. Let $\nu_{\lambda_+} (n)$ be defined as per Eq.~\eqref{eq:def_nu_0_lambda_+} and assume that $\xi_+ \leq \lambda_+/(2n)$. Then, there exists a universal constant $C>0$ such that
    \begin{align*}
    \left| \sfs_n(\lambda) - \sfs_{n,0}(\lambda_+) \right|\leq&~ \frac{n\xi_+}{\lambda_+} \cdot \sfs_n(\lambda), \\
        \left| \sL_n (\bbeta_*,\lambda) - \sL_{n,0} (\bbeta_*,\lambda_+)\right|\leq&~C \nu_{\lambda_+} (n) \frac{n\xi_+}{\lambda_+} \cdot \sL_n (\bbeta_*,\lambda),\\
        \left| \sR_n (\bbeta_*,\lambda) - \sR_{n,0} (\bbeta_*,\lambda_+) \right|\leq&~ C \nu_{\lambda_+} (n) \frac{n\xi_+}{\lambda_+}\cdot \sR_n (\bbeta_*,\lambda).
        \end{align*}
\end{lemma}

\begin{proof}[Proof of Lemma \ref{lem:reduced_det_equiv}]
    First, Lemma \ref{lem:reduced_fixed_point} directly implies that 
    \[
    \left| \sfs_n(\lambda) - \sfs_{n,0}(\lambda_+) \right|= \frac{| \lambda_{*,0} - \lambda_* |}{\lambda_{*,0} }\cdot \frac{1}{n\lambda_*} \leq \frac{n\xi_+}{\lambda_+}\sfs_n(\lambda).
    \]
    Denote $\sU$ and $\sV$ the numerator and denominator of $\sR_n (\bbeta_*,\lambda)$, and $\sU_0$ and $\sV_0$ the numerator and denominator of $\sR_{n,0} (\bbeta_*,\lambda_+)$. By simple algebra, we have
    \[
    \begin{aligned}
    \frac{\left| \lambda_*^2 \< \bbeta_0 , (\bSigma_0 + \lambda_{*})^{-2} \bbeta_0 \> -\lambda_{*,0}^2 \< \bbeta_0 , (\bSigma_0 + \lambda_{*,0})^{-2} \bbeta_0 \>  \right|}{\lambda_*^2 \< \bbeta_0 , (\bSigma_0 + \lambda_{*})^{-2} \bbeta_0 \>}
    \leq&~  \frac{\lambda_{*,0}^2 - \lambda_*^2}{\lambda_*^2} +2 \frac{\lambda_{*,0} - \lambda_*}{\lambda_*}  \leq 8 \frac{n\xi_+}{\lambda_+} ,
    \end{aligned}
    \]
    where we used that $\lambda_* \geq \lambda_{*,0} /2$ by Eq.~\eqref{eq:lambda0_lower_bound_lambda_star}, and
    \[
        \begin{aligned}
    \frac{\left|\| \bbeta_+ \|_2^2 -\lambda_{*}^2 \< \bbeta_+ , (\bSigma_+ + \lambda_{*})^{-2} \bbeta_+ \>  \right|}{\lambda_{*}^2 \< \bbeta_+ , (\bSigma_+ + \lambda_{*})^{-2} \bbeta_+ \>}
    \leq&~  \frac{n^2\xi_+^2}{\lambda_+^2} + 2\frac{n\xi_+}{\lambda_+}  \leq 3 \frac{n\xi_+}{\lambda_+} .
    \end{aligned}
    \]
    Thus, we obtain for the numerator
    \begin{equation}\label{eq:numerator_sU}
    | \sU - \sU_0 | \leq 8 \frac{\xi_+}{\lambda_*} \sU.
    \end{equation}
    For the denominator, we have similarly
    \[
    \begin{aligned}
        &~\left| \Tr( \bSigma^2 (\bSigma + \lambda_* )^{-2} ) - \Tr( \bSigma_0^2 (\bSigma_0 + \lambda_{*,0})^{-2})\right| \\
        \leq&~ \left( \frac{\lambda_{*,0}^2 - \lambda_*^2}{\lambda_*^2} + \frac{\lambda_{*,0} - \lambda_*}{\lambda_*} \right) \Tr( \bSigma_0^2 ( \bSigma_0 + \lambda_{*,0})^{-2} ) + \Tr( \bSigma_+^2 (\bSigma_+ + \lambda_* )^2 ) \leq 10n \frac{n\xi_+}{\lambda_*},
    \end{aligned}
    \]
    where in the last inequality we used $\Tr(\bSigma_0^2 ( \bSigma_0 + \lambda_{*,0})^{-2}) \leq\Tr(\bSigma_0 ( \bSigma_0 + \lambda_{*,0})^{-1}) \leq n$ and Eq.~\eqref{eq:lambda0_lower_bound_lambda_star}. Further observe that
    \[
    \begin{aligned}
        \left(1 - \frac{1}{n} \Tr(\bSigma_0^2 ( \bSigma_0 + \lambda_{*,0})^{-2})\right)^{-1} \leq&~  \left(1 - \frac{1}{n} \Tr(\bSigma_0 ( \bSigma_0 + \lambda_{*,0})^{-1})\right)^{-1}\\
        =&~\frac{n\lambda_{*,0}}{\lambda_+} = 1 + \Tr(\obM_0) \leq 2 \nu_{\lambda_+} (n),
    \end{aligned} 
    \]
    where we used Lemma \ref{lem:properties_mu_obM} in the last inequality. We obtain for the denominator
    \begin{equation}\label{eq:denominator_sV}
    | \sV - \sV_0 | \leq 20 \cdot \nu_{\lambda_+} (n) \frac{n\xi_+}{\lambda_+} \sV_0.
    \end{equation}

    We can now combine Eqs.~\eqref{eq:numerator_sU} and \eqref{eq:denominator_sV} to get the following bound on the test error:
    \[
    \begin{aligned}
        \frac{\left| \sR_{n} (\bbeta_* , \lambda) - \sR_{n,0} (\bbeta_*,\lambda) \right|}{\sR_{n} (\bbeta_* , \lambda)} =  \frac{\left| \sU \sV_0 - \sU_0 \sV\right|}{\sU \sV_0} \leq&~ \frac{\left| \sU - \sU_0 \right|}{\sU} + \left(1 + \frac{| \sU - \sU_0|}{\sU} \right) \frac{| \sV - \sV_0|}{\sV_0}\\
        \leq 108  \cdot \nu_{\lambda_+} (n) \frac{n\xi_+}{\lambda_+} .
    \end{aligned}
    \]
    
   Finally, the test error follows by observing that
    \[
    \begin{aligned}
        \frac{\left| \sL_{n} (\bbeta_* , \lambda) - \sL_{n,0} (\bbeta_*,\lambda) \right|}{\sL_{n} (\bbeta_* , \lambda)} \leq &~ \left( \frac{\lambda_*}{\lambda_{*,0}} \right)^2  \frac{\left| \sR_{n} (\bbeta_* , \lambda) - \sR_{n,0} (\bbeta_*,\lambda) \right|}{\sR_{n} (\bbeta_* , \lambda)} + \frac{\lambda_{*,0}^2 - \lambda_*^2}{\lambda_{*,0}^2}\\
        \leq&~110  \cdot \nu_{\lambda_+} (n) \frac{n\xi_+}{\lambda_+} .
    \end{aligned}
\]
\end{proof}

\subsection{The Stieltjes transform}
\label{app_test_error:Stieltjes_general}

\begin{theorem}[Deterministic equivalent for the Stieltjes transform]
\label{thm:stieltjes_general}
Consider $D,K>0$, integer $n$, and regularization parameter $\lambda \geq 0$.  Assume that the features $\{\bx_i\}_{i\in[n]}$ satisfy Assumption \ref{ass:main_assumptions}.(a) and (b) with some $\evn := \evn (n) \in \naturals \cup \{\infty\}$.
There exist constants $\eta := \eta_x \in(0,1/2)$, $C_{D,K} >0$, and $C_{x,D,K}>0$ such that, if it holds that $n \geq C_{D,K}$ and
\begin{equation}\label{eq:condition1_stieltjes}
\lambda_{>\evn} \cdot \nu_{\lambda,\evn} (n) \geq n^{-K},\qquad
\varphi_{2,n} (\evn)  \sqrt{\frac{n\xi_+}{\lambda_+}} \leq \frac{1}{2},
\qquad
 \varphi_1(\evn) \nu_{\lambda,\evn} (n)^{7/2} \log^{\beta + \frac{1}{2}} (n) \leq K \sqrt{n},
\end{equation}
then with probability at least $1 - n^{-D} - p_{2,n}(\evn)$, we have
\[
\left| s_{n} (\bX , \lambda) -  \sfs_{n} (\lambda)  \right| \leq  C_{x,D,K} \cdot \cE_{\sfs,n} (\evn) \cdot \sfs_{n} (\lambda),
\]
where the decay error is given by
\begin{equation}
    \cE_{\sfs,n}  (\evn) := 
    \frac{\varphi_1(\evn) \nu_{\lambda,{\evn}}(n)^{7/2} \log^{\beta +1/2}(n)}{\sqrt{n}}    +  \varphi_{2,n} (\evn) \sqrt{\frac{n\xi_{\evn+1}}{\lambda_{>\evn}}}.
\end{equation}
\end{theorem}

\begin{proof}[Proof of Theorem \ref{thm:stieltjes_general}]
We first simplify $s_{n} (\bX,\lambda)$ by using the concentration of the high-frequency features. From Assumption \ref{ass_app:main_assumptions}.(b), we have with probability at least $1 - p_{2,n} (\evn)$,
\begin{equation}\label{eq:Stieltjes_simplification}
\left| \frac{1}{n} \Tr(\bG) - \frac{1}{n} \Tr(\bG_0 ) \right| = \frac{1}{n} \left| \Tr( \bG_0 \bDelta_+ \bG ) \right| \leq 2 \varphi_{2,n} (\evn)\sqrt{\frac{n\xi_+}{\lambda_+}} \cdot \frac{1}{n} \Tr(\bG_0), 
\end{equation}
where we used that $\| \bG \|_\op \leq \| (\bDelta_+ + \lambda_+)^{-1} \|_\op \leq 2  /\lambda_+$ since we assumed $\varphi_{2,n} (\evn) \sqrt{n}\leq 1/2$. We follow the same argument as in Step 1 of the proof of Theorem \ref{thm:well_concentrated_test_train}. Using that $\bX_0$ satisfy Assumption \ref{ass_app:main_assumptions}.(a) to apply Theorem \ref{thm_app:det_equiv_TrZZM}, we obtain with probability at least $1 - n^{-D}$,
\[
\left| \frac{1}{n} \Tr(\bG_0 ) - \sfs_{n,0} (\lambda_+) \right| \leq 2\cE_{2,n}(\evn)  \nu_{\lambda_+} (n) \cdot \sfs_{n,0} (\lambda_+)
\]
where we recall that
\begin{equation}
    \cE_{2,n}(\evn) = C_{x,D,K} 
    \frac{\varphi_1(\evn) \nu_{\lambda_{+}}(n)^{5/2} \log^{\beta +1/2}(n)}{\sqrt{n}},
\end{equation}
By Lemma \ref{lem:reduced_det_equiv} and noting that $\sfs_{n,0} (\lambda_+) \leq 3 \sfs_n (\lambda)/2$ when $\xi_+ \leq \lambda_+/(2n)$, we deduce that
\[
\left| \frac{1}{n} \Tr(\bG_0 ) - \sfs_{n} (\lambda) \right| \leq C \left\{ \cE_{2,n}(\evn)\nu_{\lambda_+} (n) + \frac{ n\xi_+}{\lambda_+}\right\} \cdot  \sfs_{n} (\lambda).
\]
By condition~\eqref{eq:condition1_stieltjes} along with $n\xi_+/\lambda_+ \le 1/2$, we have the right-hand side bounded by $C_{x,D,K}\cdot \sfs_{n} (\lambda)$ so that $\Tr(\bG_0)/ n \leq C_{x,D,K} \cdot \sfs_n (\lambda)$. Combining the above bound with Eq.~\eqref{eq:Stieltjes_simplification}, we obtain our first approximation guarantee
\[
\left| s_{n} (\bX , \lambda) -  \sfs_{n} (\lambda)  \right| \leq C_{x,D,K} \left\{  \cE_{2,n} (\evn) \nu_{\lambda_+} (n)  + \varphi_{2,n} (\evn)\sqrt{\frac{n\xi_+}{\lambda_+}}  \right\} \cdot \sfs_{n} (\lambda),
\]
with probability at least $1 - n^{-D} - p_{2,n} (\evn)$ via union bound.
\end{proof}

\subsection{The training error}
\label{app_test_error:training_error_general}


\begin{theorem}[Deterministic equivalent for the training error]\label{thm:training_general}
Consider $D,K>0$, integer $n$, regularization parameter $\lambda \geq 0$, and target function $f_* \in L^2 (\cU)$ with parameters $\| \bbeta_*\|_2 = \| f_* \|_{L^2} < \infty$. Assume that the features $\{\bx_i\}_{i\in[n]}$ and $f_*$ satisfy Assumption \ref{ass:main_assumptions} with some $\evn := \evn (n) \in \naturals \cup \{\infty\}$, and the $\{\eps_i\}_{i\in [n]}$ satisfy Assumption \ref{ass:noise_subGaussian}. There exist constants $\eta := \eta_x \in(0,1/2)$, $C_{D,K} >0$, and $C_{x,\eps,D,K}>0$ such that, if it holds that $n \geq C_{D,K}$ and
\begin{equation}\label{eq:condition1_training}
\lambda_{>\evn} \cdot \nu_{\lambda,\evn} (n) \geq n^{-K},\qquad
\varphi_{2,n} (\evn)  \sqrt{\frac{n \xi_{ \evn + 1}}{\lambda_{>\evn}}} \leq \frac{1}{2},
\qquad
 \varphi_1(\evn) \nu_{\lambda,\evn} (n)^{8} \log^{3\beta + \frac{1}{2}} (n) \leq K \sqrt{n},
\end{equation}
then with probability at least $1 - n^{-D} - p_{2,n}(\evn)$, we have
\[
\left|\cL_{\train} ( \bbeta_*;\bX,\beps, \lambda) -    \sL_{n} (\bbeta_*, \lambda) \right| \leq C_{x,\eps,D,K} \cdot \cE_{\sL,n} (\evn) \cdot  \sL_{n} (\bbeta_*, \lambda).
\]
where the error rate is given by
\begin{equation}
    \cE_{\sL,n}  (\evn) := \frac{\varphi_1(\evn) \nu_{\lambda,{\evn}}(n)^{8} \log^{3\beta +1/2}(n)}{\sqrt{n}}    + \nu_{\lambda,{\evn}}(n) \varphi_{2,n} (\evn)\sqrt{\frac{n \xi_{ \evn + 1}}{\lambda_{>\evn}}} .
\end{equation}

\end{theorem}

\begin{proof}[Proof of Theorem \ref{thm:training_general}]

\noindent
\textbf{Step 1: Simplifying and decomposing the training error.}

Observe that we can rewrite the training error as
\begin{equation}
\label{eq:training_error_rewrite}
\cL_{\train} ( \bbeta_*;\bX,\beps, \lambda)= \frac{1}{n}\| \by - \bX \hbtheta_\lambda \|_2^2 =  \frac{1}{n}\| \by - \bX \bX^\sT \bG \by \|_2^2 =  \frac{\lambda^2}{n}  \by^\sT \bG^2 \by .
\end{equation}
We will first study $\by^\sT \bG_0^2 \by$ and then use that $\frac{1}{n}|\by^\sT (\bG^2 - \bG_0^2) \by |$ is small on the event described in Assumption~\ref{ass_app:main_assumptions}.(b). We show below that $\frac{1}{n}\by^\sT \bG_0^2 \by$ has deterministic equivalent $L_{n,0}(\bbeta_*, \lambda_+)$ defined in Section \ref{app_test_error:reduced_fixed_points}.
To do so, recall $\by = \boldf + \beps$ and decompose
\begin{equation*}
    \frac1n\by^\sT \bG_0^2 \by  =  T_1 + 2T_2 + T_3,
\end{equation*}
where we denoted 
\[
T_1 =  \frac{1}{n}\boldf^\sT \bG_0^2 \boldf
, \qquad  T_2 = \frac{1}{n}\beps^\sT \bG_0^2 \boldf  , \qquad T_3 =  \frac{1}{n}\beps^\sT \bG_0^2 \beps.
\]
We control each of these terms separately in the following step.

\noindent
\textbf{Step 2: Controlling the terms $T_i$.}

Let us define the deterministic quantity
\begin{equation*}
    \overline T := \frac{1}{(n \lambda_{*,0})^2}\cdot \frac{1}{1 -\Upsilon_{2,0}}.
\end{equation*}

\paragraph*{The Term $T_1$.} Recall that we denote $\boldf =\boldf_0 +\boldf_+$, so that we can decompose $T_1$ into three contributions $T_1 = T_{1,0} + 2 T_{1,0+} + T_{1,+}$, where
\[
    T_{1,1}= \frac{1}{n} \boldf_0^\sT \bG_0^2 \boldf_0,
    \quad\quad
    T_{1,2} = \frac{1}{n} \boldf_0^\sT \bG_0^2 \boldf_+,
    \quad\quad
    T_{1,3} = \frac{1}{n} \boldf_+^\sT \bG_0^2 \boldf_+.
\]

The term $T_{1,1}$ corresponds to the functional $\Phi_4 (\bX_0; \bA_{*,0} )$ studied in Appendix \ref{app_det_equiv:proof_MZZM} with $\bA_{*,0} = \bSigma_0^{-1} \bbeta_0 \bbeta_0^\sT \bSigma_0^{-1}$. Recalling conditions \eqref{eq:condition1_training} and that $\bX_0$ satisfy Assumption \ref{ass_app:main_assumptions}.(a), we can directly apply Theorem \ref{thm_app:det_equiv_TrAMZZM}. Thus, we obtain with probability at least $1 -n^{-D}$ that
\begin{equation}
\label{eq:T11_train}
\left| T_{1,1} - \Psi_4 ( \mu_{*,0} ; \bA_{*,0} ) \right| \leq \cE_{4,n}(\evn) \cdot \Psi_4 ( \mu_{*,0} ; \bA_{*,0} )
\end{equation}
where we recall that
\begin{equation*}
    \cE_{4,n}(\evn) = C_{x,D,K} \frac{\varphi_1(\evn)\nu_{\lambda_+}(n)^6\log^{\beta +1/2}(n)}{\sqrt{n}}
\end{equation*}

Next, to control $T_{1,2}$, we recall that $f_+$ satisfies Assumption \ref{ass_app:main_assumptions}.(c). Therefore, we can use Eq.~\eqref{eq:high_degree_det_4} in Lemma~\ref{lem:high-degree_part} to deduce that with probability at least $1-n^{-D}$,
\begin{align}
\label{eq:T12_train}
|T_{1,2}| &\le  \cE_{+,n} \sqrt{ \|\bbeta_+\|_2^2 \frac{\Psi_3(\mu_{*,0}; \bA_{*,0})}{\lambda_+^2}  }
\le C_{x,K,D} \cdot \cE_{+,n} \nu_{\lambda_+} (n) \left(\|\bbeta_+\|_2^2 \cdot \overline T + \Psi_4(\mu_{*,0}; \bA_{*,0}) \right),
\end{align}
where 
\begin{equation*}
    \cE_{+,n}(\evn) = C_{x,K,D} \frac{ \varphi_1 (\evn) \nu_{\lambda_+} (n)^6 \log^{3\beta +\frac12} (n)}{\sqrt{n}},
\end{equation*}
and we used that $(n/\mu_{*,0})^2 \leq 4 \nu_{\lambda_+} (n)^2$ by Lemma \ref{lem:properties_mu_obM} and
\begin{equation*}
    \Psi_3(\mu_{*,0}; \bA_{*,0}) = \left(\frac{n}{\mu_{*,0}}\right)^2  \Psi_4(\mu_{*,0}; \bA_{*,0}), \qquad
   \frac{1}{(n \lambda_{*,0})^2}  \le  
   \frac{1}{(n \lambda_{*,0})^2} \frac{1}{1 - \Upsilon_{2,0}} = \overline T.
\end{equation*}

Finally, for the term $T_{1,3}$,
using again that $f_+$ satisfies Assumption \ref{ass_app:main_assumptions}, we apply Eq.~\eqref{eq:high_degree_det_3} in Lemma~\ref{lem:high-degree_part} to get
that with probability at least $1-n^{-D}$,
\begin{align}
\label{eq:T13_train}
\left| T_{1,3} - \|\bbeta_+\|^2_2 \cdot \overline T \right| 
\le 
\cE_{+,n}(\evn)
\frac{\|\bbeta_+\|_2^2 }{\lambda_+^2}
\le
C \cE_{+,n}(\evn) \nu_{\lambda_+ } (n)^2 \cdot
\|\bbeta_+\|_2^2 \cdot \overline T,
\end{align}
where we used again that $(n/\mu_{*,0})^2 \leq 4 \nu_{\lambda_+} (n)^2$.
Combining the bounds in Eq.~\eqref{eq:T11_train}, \eqref{eq:T12_train} and \eqref{eq:T13_train}, we obtain with probability at least $1-n^{-D}$,
\begin{equation}
\begin{aligned}
   &~ \left|T_1 - \left(\Psi_4(\mu_{*,0}; \bA_{*;0}) + \|\bbeta_+\|_2^2\cdot \overline T\right)\right|\\
   \le &~ |T_{1,1} - \Psi_4(\mu_{*,0}; \bA_{*;0})| + 
    |T_{1,2}| + \left|T_{1,3} - \|\bbeta_+\|_2^2\cdot \overline T \right|\\
    \le&~ C \max\{\cE_{4,n}(\evn), \cE_{+,n}(\evn) \nu_{\lambda_+ }(n)^2\}
    \left(  \Psi_4(\mu_{*,0}; \bA_{*,0}) + \|\bbeta_+\|_2^2 \cdot \overline T\right).
\label{eq:T1_bound}
\end{aligned}
\end{equation}

\paragraph*{The term $T_3$.}
To control the term $T_3$, we follow the same argument as in the proof of Theorem \ref{thm:well_concentrated_test_train}.  Namely, by Hanson-Wright inequality (Lemma~\ref{lem:Hanson_Wright_inequality}) we have that 
\begin{equation}
\label{eq:hw-Tr-general}
    \left| T_3 - \frac{\sigma_\epsilon^2}{n} \Tr(\bG_0^2)\right| 
    \le C_{\epsilon, D } 
    \frac{\log(n)}{n} \sigma_{\epsilon}^2 \|\bG_0^2\|_F.
\end{equation}
Recalling that we can write the mean
\begin{equation*}
    \frac1n \Tr(\bG_0^2) = \frac1{\lambda_+} \left( \frac1n \Tr(\bG_0) - \frac1n \Tr(\bX_0 \bX_0^\sT \bG_0^2)\right),
\end{equation*}
we obtain
\begin{equation}
\begin{aligned}
    \left|\frac1n \Tr(\bG_0^2) - \overline T\right| 
    &\le \frac1{\lambda_+} \left|\frac1n \Tr(\bG_0) - \sfs_{n,0}(\lambda_+)\right| + \frac1{\lambda_+} |\Phi_4(\bX_0; \bSigma_0^{-1}) - \Psi_4(\mu_*; \bSigma_0^{-1})|\\
    &\le 2 \max \{ \cE_{2,n}(\evn) \nu_{\lambda_+}(n) , \cE_{4,n}(\evn) \} \frac{\sfs_{n,0}(\lambda_+)}{\lambda_+} \\
    &\le C \nu_{\lambda_+}(n) \cE_{4,n}(\evn) \overline T .
    \label{eq:Tr_G02_bound}
\end{aligned}
\end{equation}
For the variance term, using $\|\bG_0\|_\op \le 2/\lambda_+$, we first note that
\begin{equation*}
    \frac{\Tr(\bG_0^4)}{\Tr(\bG_0^2) \oT} \le \frac{4}{\lambda_+^2 \oT} 
    \le 4\frac{(n \lambda_{*,0})^2}{\lambda_+^2} \le 4 (1 + \Tr(\obM_0))^2 \le C \nu_{\lambda_+}(n)^2.
\end{equation*}
Hence, we can bound 
\begin{equation}
     \label{eq:frobenius_bound_G0}
     \frac{1}{\sqrt{n}}\| \bG_0^2\|_F \le C \nu_{\lambda_+}(n) \sqrt{\oT \frac{\Tr(\bG_0^2)}{n}}  \le C_{x,D,K} \cdot \nu_{\lambda_+}(n) \overline T,
\end{equation}
where we used conditions \eqref{eq:condition1_training} so that $n^{-1} \Tr(\bG_0^2) \leq C_{x,D,K}\cdot \oT$ from Eq.~\eqref{eq:Tr_G02_bound}.

Combining the bounds~\eqref{eq:hw-Tr-general}, \eqref{eq:Tr_G02_bound} and \eqref{eq:frobenius_bound_G0}, we conclude that
\begin{equation}
\label{eq:T3_bound}
   |T_3 - \sigma_\epsilon^2 \overline T | \le C_{\epsilon, D} \cdot \nu_{\lambda_+}(n) \max\{\cE_{4,n}, n^{-1/2} \log(n)\} \cdot \sigma_\epsilon^2 \cdot \overline T.
\end{equation}

\paragraph*{The term $T_2$.}
For the term $T_2$, we use again Hanson-Wright inequality to obtain that with probability at least $1 - n^{-D}$,
\begin{equation}
\label{eq:T2_HW_bound}    
|T_2 | \leq C_{\eps,D} \frac{\log (n)}{n} \sigma_\epsilon\sqrt{\boldf^\sT \bG_0^4 \boldf}.
\end{equation}
Then using again the bound $\|\bG_0\|_\op \le \lambda_{+}^{-1}$, we have
\begin{equation*}
\boldf^\sT \bG_0^4 \boldf
\leq \frac{\boldf^\sT \bG_0^2 \boldf}{\lambda_+^2 } 
\leq C \nu_{\lambda_+}(n)^2  \frac{\boldf^\sT \bG_0^2 \boldf}{n^2 \lambda_{*,0}^2}
\leq C \nu_{\lambda_+}(n)^2  \boldf^\sT \bG_0^2 \boldf \cdot \overline T.
\end{equation*}
Combining this bound with the one in Eq.~\eqref{eq:T2_HW_bound}, we obtain
\begin{align*}
|T_2 |\leq&~ C_{\eps,D} \frac{\log (n)}{ \sqrt{n} } 
\left(
\nu_{\lambda_+}(n)\sqrt{\sigma_\epsilon^2 \cdot T_{1} \cdot \overline T} 
\right).
\end{align*}
Finally, using that $\max\{ \cE_{4,n}(\evn),\nu_{\lambda_+} (n)^2 \cE_{+,n}(\evn) \} \le C_{x,K,D}$ from condition~\eqref{eq:condition1_training} and the bound in~\eqref{eq:T1_bound} we obtain that with probability at least $1-n^{-D}$
\begin{equation}
\label{eq:T2_bound}
    |T_2| \le C_{x,\eps,K,D}
\frac{ \nu_{\lambda_+}(n) \log(n)}{\sqrt{n}} 
\left(\sigma_{\eps}^2 \cdot \overline T +  \|\bbeta_+\|_2^2 \cdot \overline T + \Psi_4(\mu_{*,0}; \bA_{*,0})\right).
\end{equation}

\noindent
\textbf{Step 3: Combining the bounds.}

Combining the bounds~\eqref{eq:T1_bound},~\eqref{eq:T3_bound},~\eqref{eq:T2_bound}, we obtain that with probability at least $1-n^{-D}$, we have
\begin{equation}
\begin{aligned}
\label{eq:F_Ln_bound}
    &~\left|\frac{\lambda^2}n \by^\sT\bG_0 \by - \sL_{n,0}(\bbeta_*, \lambda) \right|\\
    \le&~\lambda^2\left(  |T_{1,1} - \Psi_4(\mu_{*,0}; \bA_{*,0})| + |T_{1,2}|+ |T_{1,3} - \|\bbeta_+\|^2 \cdot \overline T|   +|T_2| +|T_3 - \sigma_\epsilon^2 \cdot \overline T|\right)\\
    \le&~  
C_{x,\eps,D,K} \cdot
\nu_{\lambda_+}(n)^2 
\max\left\{
\cE_{4,n}(\evn), \cE_{+,n}(\evn), n^{-1/2} \log(n)
 \right\}\cdot \sL_{n,0}(\bbeta_*, \lambda), 
\end{aligned}
\end{equation}
where we used that
\begin{equation*}
 \sL_{n,0}(\bbeta_*, \lambda) = 
 \lambda^2 (\|\bbeta_+\|_2^2 + \sigma_\epsilon^2)\cdot \overline T+ \lambda^2\Psi_4 ( \mu_{*,0} ; \bA_{*,0} ).
\end{equation*}

Now note that we can bound the difference on the event in Assumption \ref{ass_app:main_assumptions}.(b) by
\begin{equation}
\begin{aligned}
\label{eq:training_error_simplification}
    \frac{\lambda^2}{n}\left| \by^\sT \bG^2 \by - \by^\sT \bG_0^2 \by\right| =&~   \frac{\lambda^2}{n} \left|\by^\sT(\bG_0 (-2 \bDelta_+ \bG + \bDelta_+ \bG^2 \bDelta_+ ) \bG_0) \by\right|\\
    \leq&~ C \varphi_{2,n}(\evn)\sqrt{\frac{n\xi_+}{\lambda_+}}  \cdot  
     \frac{\lambda^2}{n}\by^\sT\bG_0^2 \by \\
     \le&~ C_{x,\eps,D,K}\cdot \varphi_{2,n}(\evn)\sqrt{\frac{n\xi_+}{\lambda_+}}   \cdot  
      \sL_{n,0}(\bbeta_*, \lambda),
\end{aligned}
\end{equation}
where in the last line, we used the condition~\eqref{eq:F_Ln_bound} with $\nu_{\lambda_+}^2 \cE_{+,n}(\evn) \leq C_{x,\eps,D,K} \sqrt{n}$.
Finally, combining the bounds~\eqref{eq:F_Ln_bound} and~\eqref{eq:training_error_simplification} along with Lemma~\ref{lem:reduced_det_equiv} and using the assumption $n\xi_+/\lambda_+ \le 1/2$, we conclude that
\begin{align*}
   &~|\cL_{\train}(\bbeta_*; \bX, \epsilon, \lambda) - \sL_{n}(\bbeta_*, \lambda)|\\
  \leq&~ C_{x,\eps, D,K}
\left\{ \nu_{\lambda_{+}}(n)^2 \cE_{+,n}(\evn)
 +
 \nu_{\lambda_{+}}(n) \varphi_{2,n}(\evn)\sqrt{\frac{n \xi_+}{\lambda_+}} 
 \right\}  \sL_{n}(\bbeta_*, \lambda),
\end{align*}
with probability at least $1- n^{-D} - p_{2,n}(\evn)$ via union bound.
\end{proof}

\subsection{The test error}
\label{app_test_error:test_error_general}

For the readers' convenience, we restate Theorem \ref{thm:abstract_Test_error} below. Recall that we defined the standard bias and variance terms 
\[
\begin{aligned}
     \cB ( \bbeta_* ; \bX, \lambda) =&~ \big\| \btheta_* - \E_{\beps} \big[ \hbtheta_\lambda \big]  \big\|_{\bSigma}^2 , \\
     \cV (\bX, \lambda) =&~ \Tr \big( \bSigma \Cov_\beps (  \hbtheta_\lambda) \big)  .
\end{aligned}
\]
Their associated deterministic equivalents are given by
\[
\begin{aligned}
        \sB_n (\bbeta_*,\lambda) = &~ \frac{\lambda_*^2 \< \bbeta_* ,  ( \bSigma + \lambda_* )^{-2} \bbeta_* \> }{ 1 - n^{-1} \Tr ( \bSigma^2 ( \bSigma + \lambda_* )^{-2} )}  , \\
    \sV_n  (\lambda) =&~ \frac{\sigma_\eps^2 \Tr ( \bSigma^2 ( \bSigma + \lambda_* )^{-2} ) }{n - \Tr ( \bSigma^2 ( \bSigma + \lambda_* )^{-2} )} .
\end{aligned}
\]

\begin{theorem}[Deterministic equivalent of the test error]\label{thm:test_general}
   Consider $D,K>0$, integer $n$, regularization parameter $\lambda \geq 0$, and target function $f_* \in L^2 (\cU)$ with parameters $\| \bbeta_*\|_2 = \| f_* \|_{L^2} < \infty$. Assume that the features $\{\bx_i\}_{i\in[n]}$ and $f_*$ satisfy Assumption \ref{ass:main_assumptions} with some $\evn := \evn (n) \in \naturals \cup \{\infty\}$, and the $\{\eps_i\}_{i\in [n]}$ satisfy Assumption \ref{ass:noise_subGaussian}. There exist constants $\eta := \eta_x \in(0,1/2)$, $C_{D,K} >0$, $C_{x,D,K}$, and $C_{x,\eps,D,K}>0$ such that, if it holds that $n \geq C_{D,K}$ and
\begin{equation}\label{eq:condition1_test}
\lambda_{>\evn} \cdot \nu_{\lambda,\evn} (n) \geq n^{-K},\qquad
\varphi_{2,n} (\evn)  \sqrt{\frac{n\xi_{\evn+1}}{\lambda_{>\evn}}} \leq \frac{1}{2},
\qquad
 \varphi_1(\evn) \nu_{\lambda,\evn} (n)^{8} \log^{3\beta + \frac{1}{2}} (n) \leq K \sqrt{n},
\end{equation}
then with probability at least $1 - n^{-D} - p_{2,n}(\evn)$, we have
\[
\left|\cR_{\test} ( \bbeta_*;\bX,\beps, \lambda) -    \sR_{n} (\bbeta_*, \lambda) \right| \leq C_{x,\eps,D,K} \cdot \cE_{\sR,n} (\evn) \cdot  \sR_{n} (\bbeta_*, \lambda).
\]
where the error rate is given by
\begin{equation}
    \cE_{\sR,n}  (\evn) :=  \frac{\varphi_1(\evn) \nu_{\lambda,{\evn}}(n)^{6} \log^{3\beta +1/2}(n)}{\sqrt{n}}    +  \nu_{\lambda,\evn} (n) \cdot \varphi_{2,n} (\evn)  \sqrt{\frac{n\xi_{\evn+1}}{\lambda_{>\evn}}}  .
\end{equation}
Furthermore, with the same probability, we have
\[
\begin{aligned}
  \left|\cB (\bbeta_*;\bX,\lambda) -    \sB_{n} (\bbeta_*, \lambda) \right| \leq&~ C_{x,D,K} \cdot \cE_{\sR,n} (\evn) \cdot  \sB_{n} (\bbeta_*, \lambda),  \\
  \left|\cV (\bX,\lambda) -    \sV_{n} (\lambda) \right| \leq&~  C_{x,D,K} \cdot \cE_{\sR,n} (\evn) \cdot  \sV_{n} (\lambda).
\end{aligned}
\]
\end{theorem}

\begin{proof}[Proof of Theorem \ref{thm:test_general}]
We consider directly the full test error $\cR_{\test} ( \bbeta_*;\bX,\beps, \lambda)$. The bias and variance terms will correspond to specific terms in this decomposition.

\noindent
\textbf{Step 1: Simplifying and decomposing the test error.}

Recall that the test error is given by
\begin{equation*}
    \cR_{\test}(\bbeta_*; \bX, \beps, \lambda) = \| \btheta_* - \bX^\sT \bG \by \|_{\bSigma}^2 + \sigma_\eps^2.
\end{equation*}
First, we will control the contributions of the low-degree eigenspaces and high-degree eigenspaces separately, i.e., we decompose
\[
\| \btheta_* - \bX^\sT \bG \by \|_{\bSigma}^2 = \| \btheta_0 - \bX_0^\sT \bG \by \|_{\bSigma_0}^2 + \| \btheta_+ - \bX_+^\sT \bG \by \|_{\bSigma_+}^2.
\]
Second, on the event in Assumption \ref{ass:main_assumptions}.(b), which happens with probability at least $1 - p_{2,n} (\evn)$, we replace the resolvent $\bG$ by $\bG_0$ and first study
\[
\| \btheta_0 - \bX_0^\sT \bG_0 \by \|_{\bSigma_0}^2 , \qquad \quad \| \btheta_+ - \bX_+^\sT \bG_0 \by \|_{\bSigma_+}^2.
\]

\noindent
\textbf{Step 2: Contribution of the low-degree eigenspaces $\| \btheta_0 - \bX_0^\sT \bG_0 \by \|_{\bSigma_0}^2$.}

Recalling that $\by = \boldf + \beps$, we decompose this term into
\begin{equation*}
\| \btheta_0 - \bX_0^\sT \bG_0 \by \|_{\bSigma_0}^2 = Q_1 -2Q_2 + Q_3,
\end{equation*}
where we defined
\begin{equation*}
Q_1 = \| \btheta_0 - \bX_0^\sT \bG_0 \boldf \|_{\bSigma_0}^2, \qquad Q_2 = \< \beps , \bG_0 \bX_0 \bSigma_0 (  \btheta_0 - \bX_0^\sT \bG_0 \boldf ) \> , \qquad Q_3 = \| \bX_0^\sT \bG_0 \beps \|_{\bSigma_0}^2.
\end{equation*}
We will show that this term is close in an appropriate sense to $\sR_{n,0}(\bbeta_*, \lambda_+)$. Let us first control each of the terms $Q_i$ separately.

\paragraph*{The term $Q_1$:} We decompose further
\begin{equation*}
    Q_1 = Q_{1,0} - 2 Q_{1,0+} + Q_{1,+} ,
\end{equation*}
where we introduced
\begin{align*}
    Q_{1,0}:=&~ \lambda_+^2 \< \btheta_0, \bR_0 \bSigma_0 \bR_0 \btheta_0 \>,\\
    Q_{1,0+}:=&~ \lambda_+ \< \btheta_0, \bR_0 \bSigma_0 \bR_0 \bX_0^\sT \boldf_+ \> , \\
    Q_{1,+}:=&~ \< \boldf_+,\bX_0 \bR_0 \bSigma_0 \bR_0 \bX_0^\sT \boldf_+ \>.
\end{align*}

For $Q_{1,0}$, we can write using $\bA_{*,0} := \bSigma_0^{-1} \bbeta_0 \bbeta_0^\sT \bSigma_0^{-1}$,
\begin{equation*}
    Q_{1,0} = \lambda_+^2 \Tr(\bSigma_0^{1/2} \bA_{*,0} \bSigma_0^{1/2} \bR_{0}\bSigma_0 \bR_0)
    =\lambda_+^2 \Tr(\bA_{*,0}  \bM_{0}^2) = \lambda_+^2 \Phi_3(\bX_0; \bA_{*,0}).
\end{equation*}
Recalling that $\bX_0$ satisfies Assumption \ref{ass_app:main_assumptions}.(a) and conditions~\eqref{eq:condition1_test}, we can directly apply Theorem~\ref{thm_app:det_equiv_TrAMM} to obtain that with with probability at least $1-n^{-D}$,
\begin{equation}
\label{eq:Q10}
    |Q_{1,0} - \lambda_+^2 \Psi_3(\mu_{*,0}; \bA_{*,0})| \le \cE_{3,n}( \evn)\cdot  \lambda_+^2 \Psi_3(\mu_{*,0}; \bA_{*,0}),
\end{equation}
where we recall that
\[
\cE_{3,n}( \evn) = C_{x,D,K} \frac{\varphi_1(\evn) \nu_{\lambda_+}(n)^6 \log^{2\beta+\frac12}(n)}{\sqrt{n}}.
\]

For $Q_{1,0+}$, we recall that $f_+$ satisfies Assumption \ref{ass_app:main_assumptions}.(c). We can therefore apply Eq.~\eqref{eq:high_degree_det_2} to deduce that with probability at least $1-n^{-D}$,
\begin{equation}
\begin{aligned}
|Q_{1,0+} | &\le \cE_{+,n}(\evn) \sqrt{
    \lambda_{+}^2\|\bbeta_+\|_2^2 \Psi_3(\mu_{*,0}; \bA_{*,0})}\\
    &\le \cE_{+,n}(\evn) \left\{ \|\bbeta_+\|_2^2 + \lambda_{+}^2\Psi_3(\mu_{*,0}; \bA_{*,0})\right\},
\label{eq:Q10+}
\end{aligned}
\end{equation}
where we recall that
\begin{equation*}
    \cE_{+,n}(\evn) := C_{x,K,D} \frac{ \varphi_1 (\evn) \nu_{\lambda_+} (n)^6 \log^{3\beta +\frac12} (n)}{\sqrt{n}}.
\end{equation*}

\noindent
Finally, for $Q_{1,+}$, using Eq.~\eqref{eq:high_degree_det_2} in Lemma~\ref{lem:high-degree_part}, we immediately obtain that with probability at least $1-n^{-D}$,
\begin{align}
\label{eq:Q1+}
    \left|Q_{1,+} - \|\bbeta_+\|_2^2 \cdot n \Psi_4(\mu_{*,0}; \id)\right| 
    &\le \cE_{+,n}(\evn) \cdot \|\bbeta_+\|_2^2.
\end{align}
Combining the bounds~\eqref{eq:Q10},~\eqref{eq:Q10+},~\eqref{eq:Q1+} then gives
\begin{equation}
\begin{aligned}
\label{eq:Q1}
    \left|Q_1 - \left(\lambda_+^2 \Psi_3(\mu_{*,0}; \bA_{*,0}) + 
    \|\bbeta_+\|_2^2 \cdot n \Psi_4(\mu_{*,0}; \id) \right) \right|
    \le C \cE_{+,n}(\evn) \left\{\lambda_+^2 \Psi_3(\mu_{*,0}; \bA_{*,0}) + \|\bbeta_+\|_2^2\right\},
\end{aligned}
\end{equation}
with probability at least $1 - n^{-D}$.

\paragraph*{The term $Q_3$:}
For this term, we follow the same argument as in the proof of Theorem \ref{thm:well_concentrated_test_train}. First, by Hanson-Wright inequality,
\begin{equation}
\label{eq:HW_Q3}    
|Q_3 - \sigma_{\eps}^2 \Tr( \bG_0 \bX_0 \bSigma_0 \bX_0^\sT \bG_0 ) | \leq C_{\eps,D} \log (n) \sigma_{\eps}^2 \| \bG_0 \bX_0 \bSigma_0 \bX_0^\sT \bG_0 \|_F .
\end{equation}
Noting that $\Tr(\bSigma_0 \bR_0 \bX_0^\sT\bX_0 \bR_0) = \Tr(\bZ_0^\sT\bZ_0 \bM_0^2) = n\Phi_4(\bX_0, \id)$ and recalling that $\bX_0$ satisfies Assumption \ref{ass_app:main_assumptions}.(a),
we directly apply Theorem~\ref{thm_app:det_equiv_TrAMZZM} with $\bA =\id$ to conclude that with probability at least $1 - n^{-D}$,
\begin{equation}
\label{eq:Q3_bias}
\left| \Tr( \bSigma_0 \bR_0 \bX_0^\sT\bX_0 \bR_0 )  - 
n\Psi_4(\id; \mu_{*,0})
 \right| \leq \cE_{4,n}(\evn) \cdot
 n\Psi_4(\id; \mu_{*,0}),
\end{equation}
where we recall that
\begin{equation*}
    \cE_{4,n}(\evn) = C_{x,D,K} \frac{\varphi_1(\evn) \nu_{\lambda_+}(n)^6 \log^{\beta+1/2}(n)}{\sqrt{n}}.
\end{equation*}
Meanwhile, for the term on the right-hand side of~\eqref{eq:HW_Q3}, we write
\begin{align}
    \frac{n \| \bZ_0 \bM_0^2 \bZ_0^\sT \|_F^2}{ (\Tr( \bZ_0 \bM_0^2 \bZ_0^\sT ) +1)^2} \leq &~ \frac{n \| \bM_0 \|_\op\Tr( \bZ_0 \bM_0^2 \bZ_0^\sT ) }{(\Tr( \bZ_0 \bM_0^2 \bZ_0^\sT ) +1)^2} \leq C_{x,D} \cdot \nu_{\lambda_+} (n),
\end{align}
where we used Lemma \ref{lem:tech_bounds_norm_M}.
Using this along with the bound~\eqref{eq:Q3_bias} and condition \eqref{eq:condition1_test} so that $\cE_{4,n}(\evn)\le C_{x,D,K}$, we obtain
\begin{equation}
\label{eq:Q3_var}
\begin{aligned}
\|\bG_0 \bX_0 \bSigma_0 \bX_0^\sT \bG_0\|_F \le&~ C_{x,D} \frac{\nu_{\lambda_+}(n)^{1/2}}{\sqrt{n}}  \left\{ n \Phi(\bX_0, \id) +1\right\}\\
\le&~   C_{x,D,K} \frac{\nu_{\lambda_+}(n)^{1/2}}{\sqrt{n}} \left\{ n \Psi_4(\id; \mu_{*,0}) +1 \right\}.
\end{aligned}
\end{equation}
Combining~\eqref{eq:HW_Q3} along with the bounds~\eqref{eq:Q3_bias} and~\eqref{eq:Q3_var} yields
\begin{equation}
\label{eq:Q3}
    |Q_3 -  \sigma_\eps^2 \cdot n\Psi_4(\mu_{*, 0}; \id)| \le C_{x,\eps, D,K} \cdot \cE_{4,n}(\evn) \left\{ \sigma_\eps^2 \cdot  n\Psi_4(\mu_{*,0}; \id) + \sigma_\eps^2\right\},
\end{equation}
with probability at least $1 - n^{-D}$.

\paragraph*{The term $Q_2$:}
For $Q_2$, we have again by Hanson-Wright inequality, 
\begin{align*}
|Q_2| \leq&~ \sigma_\eps \cdot C_{\eps,D} \log(n) \| \bG_0 \bX_0 \bSigma_0 (  \btheta_0 - \bX_0^\sT \bG_0 \boldf ) \|_{2}  \\
\leq&~ C_{\eps,D} \log(n)  \| \bZ_0 \bM_0 \|_\op \sqrt{\sigma_\eps^2 Q_1} \\
\leq&~ C_{\eps,x,D} \frac{\nu_{\lambda_+} (n)^{1/2} \log(n)}{\sqrt{n} } \left\{ \sigma_\eps^2 + Q_1 \right\},
\end{align*}
where we used Lemma \ref{lem:tech_bounds_norm_M} in the last inequality.
Then using Eq.~\eqref{eq:Q1} and recalling condition \eqref{eq:condition1_test}, we obtain
\begin{equation}
  \label{eq:Q2}  
 |Q_2|\le 
 C_{x,\eps,K,D} \frac{\nu_{\lambda_+} (n)^{1/2} \log(n)}{\sqrt{n} }
 \left\{\lambda_+^2 \Psi_3(\mu_{*,0}; \bA_{*,0}) + \|\bbeta_+\|_2^2 \cdot n \Psi_4(\mu_{*,0}; \id) + \sigma_\eps^2\right\},
\end{equation}
with probability at least $1 - n^{-D}$.

\paragraph*{Combining the bounds on $Q_i$:}
Now recalling the decomposition of $\| \btheta_0 - \bX_0^\sT \bG_0 \by \|_{\bSigma_0}^2$ in terms of $Q_i$ and combining the bounds~\eqref{eq:Q1},~\eqref{eq:Q3},~\eqref{eq:Q2}, we obtain
\begin{align}
\label{eq:step_1_res_gen_test}
&\left|
\| \btheta_0 - \bX_0^\sT \bG_0 \by \|_{\bSigma_0}^2
+ \sigma_\eps^2  + \|\bbeta_+\|_2^2- \sR_{n,0}(\bbeta_*, \lambda_+)
\right| \le C_{x,\eps,D,K} \cE_{+,n}(\evn) \sR_{n,0}(\bbeta_*, \lambda_+),
\end{align}
where we used that
\begin{align*}
    \sR_{n,0}(\bbeta_*, \lambda_+)  =&~ \lambda_+^2 \Psi_3(\mu_{*,0}; \bA_{*;0}) + (\sigma_\eps^2 + \|\bbeta_+\|_2^2) \cdot \left\{ n \Psi_4(\mu_{*,0}; \id) + 1\right\}    \\
    =&~ 
\lambda_+^2 \Psi_3(\mu_{*,0}; \bA_{*;0}) + \frac{\|\bbeta_+\|_2^2 + \sigma_\eps^2 }{1 - \Upsilon_{2,0}},
\end{align*}
by recalling the following identities
\begin{equation*}
    n\Psi_4(\id; \mu_{*,0}) = \frac{\Upsilon_{2,0}}{1 - \Upsilon_{2,0}},
    \qquad \quad
    \lambda_+\Psi_3(\mu_{*,0};\bA_{*,0}) = \frac{\lambda_{*,0}^2 \< \bbeta_0 , (\bSigma_0 + \lambda_{*,0})^{-2} \bbeta_0 \> }{1 - \Upsilon_{2,0}}.
\end{equation*}

\noindent
\textbf{Step 3: Contribution of the high-degree eigenspaces $\| \btheta_+ - \bX_+^\sT \bG_0 \by \|_{\bSigma_+}^2$.}

We decompose this term into
\begin{align}
\label{eq:test_error_high_deg_decomp2}
    \| \btheta_+ - \bX_+^\sT \bG_0 \by \|_{\bSigma_+}^2 = &~ \| \bbeta_+ \|_2^2 - 2 \< \bbeta_+, \bSigma_+^{1/2} \bX_+^\sT \bG_0 \by \> + \< \by, \bG_0 \bX_+ \bSigma_+ \bX_+^\sT \bG_0 \by \>.
\end{align}
For the second term, we simply use that on the event of Assumption \ref{ass_app:main_assumptions}.(b), we have with probability at least $1 - n^{-D}$ that
\begin{equation}
\begin{aligned}
\label{eq:first_bound_high_deg_test_error}
   | \< \bbeta_+, \bSigma_+^{1/2} \bX_+^\sT \bG_0 \by \> | 
   &\le  \| \bbeta_+ \|_2 \| \bSigma_+^{1/2}\|_\op \| \bX_+^\sT \|_\op \| \bG_0 \by  \|_2\\
   &\le C (\xi_+ \gamma_+)^{1/2} \|\bbeta_+\|_2 (\by^\sT\bG_0^2 \by)^{1/2} \\
   &\le C_{\eps,D} \sqrt{\frac{n \xi_+}{\lambda_+}} \left\{ \| \bbeta_+ \|_2^2 + \frac{(n \lambda_{*,0})^2 }{n} \left( \boldf^\sT \bG_0^2 \boldf +  \Tr(\bG_0^2) \right)\right\}\\
   &\leq C_{x,\eps,D,K}
   \sqrt{\frac{n \xi_+}{\lambda_+}} \left( \| \bbeta_+ \|_2^2 + \frac{(n \lambda_{*,0})^2}{\lambda^2}  \sL_{n,0}(\bbeta_*, \lambda) 
    \right),\\
\end{aligned}
\end{equation}
where we used that $\| \bX_+ \bX_+^\sT\|_\op \leq 2\gamma_+$ on the event of Assumption \ref{ass_app:main_assumptions}.(b) in the second inequality, that $\gamma_+ \le \lambda_+ \le n \lambda_{*,0}$ by definition in the third inequality and bound \eqref{eq:hw-Tr-general},
and the bounds~\eqref{eq:T1_bound} and \eqref{eq:Tr_G02_bound} in the last inequality.

A similar argument for the third term shows that
\begin{equation}\label{eq:second_bound_high_deg_test_error}
\begin{aligned}
\< \by, \bG_0 \bX_+ \bSigma_+ \bX_+^\sT \bG_0 \by \> \leq C \xi_+ \gamma_+ \by^\sT \bG_0^2 \by \leq&~ \frac{n \xi_+}{\lambda_+} \frac{(n \lambda_{*,0})^2 \by^\sT \bG_0^2 \by}{n}\\ \le &~ C_{x,\eps,D,K} \frac{n \xi_+}{\lambda_+} \frac{(n \lambda_{*,0})^2}{\lambda^2} \sL_{n,0}(\bbeta_*, \lambda)
\end{aligned}
\end{equation}

Now noting that by definition we have
\begin{equation*}
    \frac{(n\lambda_*,0)^2}{\lambda^2 } \sL_{n,0}(\bbeta_*, \lambda) =  \sR_{n,0}(\bbeta_*,\lambda)
    \qquad\textrm{and}\qquad
    \|\bbeta_+\|_2^2 \le \frac{\|\bbeta_+\|_2^2}{1 - \Upsilon_{2,0}} \le \sR_{n,0}(\bbeta_*, \lambda),
\end{equation*}
we obtain by combining the bounds~\eqref{eq:first_bound_high_deg_test_error} and~\eqref{eq:second_bound_high_deg_test_error} in~\eqref{eq:test_error_high_deg_decomp2}, we obtain
\begin{align}
\label{eq:step2_res_gen}
\left|\| \btheta_+ - \bX_+^\sT \bG_0 \by \|_{\bSigma_+}^2 -
\|\bbeta_+\|_2^2 \right|
&\le C_{ x,\eps, D,K}  \sqrt{\frac{n \xi_+}{\lambda_+}}\cdot 
\sR_{n,0} (\bbeta_*, \lambda).
\end{align}

\noindent
\textbf{Step 4: Truncation of the high frequency part of the resolvent.}

Now what remains is to bound the difference
\[
\begin{aligned}
  &~\left| \| \btheta_* - \bX^\sT \bG \by \|_{\bSigma}^2 - \| \btheta_* - \bX^\sT \bG_0 \by \|_{\bSigma}^2 \right|  \\
  \leq &~ \| \bX^\sT (\bG - \bG_0 ) \by \|_{\bSigma}^2 + 2 \left| \< \btheta_* - \bX^\sT \bG_0 \by , \bSigma \bX^\sT (\bG - \bG_0 ) \by \> \right|.
\end{aligned}
\]
Let's again consider each term on the right-hand side separately.

For the first term in the display above, we bound
\begin{align*}
    \| \bX^\sT (\bG - \bG_0 ) \by \|_{\bSigma}^2 = &~ \| \bX^\sT \bG_0 \bDelta_+ \bG \by \|_{\bSigma}^2 \\
    \leq&~ \| \bG_0 \bX \bSigma \bX^\sT \bG_0 \|_\op \|  \bDelta_+  \|_\op^2 \|  \bG \by \|_2^2 \\
    \leq&~ C_{x,D} \left\{ \nu_{\lambda_+} (n)  + \frac{n\xi_+}{\lambda_+} \right\}\cdot \varphi_2 (\evn)^2 \frac{n\xi_+}{\lambda_+} \cdot \lambda_+^2 \frac{\by^\sT \bG^2 \by }{n} \\
    \leq&~C_{x,\eps,D,K} \cdot \nu_{\lambda_+} (n) \varphi_{2,n} (\evn)^2 \frac{n\xi_+}{\lambda_+} \cdot  \sR_{n,0} (\bbeta_*,\lambda),
\end{align*}
where we used that
\[
\begin{aligned}
\| \bG_0 \bX \bSigma \bX^\sT \bG_0 \|_\op \leq&~ \| \bG_0 \bX_0 \bSigma_0 \bX_0^\sT \bG_0 \|_\op + \| \bG_0 \bX_+ \bSigma_+ \bX_+^\sT \bG_0 \|_\op \\
\leq&~ C_{x,D} \frac{\nu_{\lambda_+} (n)}{n} + C\frac{\xi_+}{\lambda_+}.
\end{aligned}
\]

Meanwhile, for the second term, we have
\begin{align*}
    &~\left| \< \btheta_* - \bX^\sT \bG_0 \by , \bSigma \bX^\sT (\bG - \bG_0 ) \by \> \right|\\
  \leq &~\| \bG_0 \bX \bSigma^{1/2} \|_\op \|  \btheta_* - \bX^\sT \bG_0 \by \|_{\bSigma} \| \bDelta_+ \|_\op \| \bG \by \|_2\\ 
  \leq &~ C_{x,D} \left\{ \nu_{\lambda_+} (n)  + \frac{n\xi_+}{\lambda_+} \right\}^{1/2} \cdot \varphi_{2,n} (\evn) \sqrt{\frac{n\xi_+}{\lambda_+}} \cdot \|  \btheta_* - \bX^\sT \bG_0 \by \|_{\bSigma} \cdot \lambda_+  \frac{\| \bG \by \|_2}{\sqrt{n}}\\
  \leq &~ C_{x,D} \cdot  \nu_{\lambda_+} (n)^{1/2} \varphi_{2,n} (\evn) \sqrt{\frac{n\xi_+}{\lambda_+}} \cdot 
  \left\{ \|  \btheta_* - \bX^\sT \bG_0 \by \|_{\bSigma}^2 + \frac{(n \lambda_{*,0})^2 \by^\sT \bG^2 \by}{n}
  \right\}\\
  \leq&~ C_{x,\eps,D,K} \cdot \nu_{\lambda_+} (n)^{1/2} \varphi_{2,n} (\evn) \sqrt{\frac{n\xi_+}{\lambda_+}} \cdot \sR_{n,0}(\bbeta_*, \lambda).
\end{align*}

Combining the above two bounds, we obtain that
\begin{align}
\label{eq:step_3_res_test_gen}
    &\left| \| \btheta_* - \bX^\sT \bG \by \|_{\bSigma}^2 - \| \btheta_* - \bX^\sT \bG_0 \by \|_{\bSigma}^2 \right|
    \le C_{x,\eps,D,K} \cdot \nu_{\lambda_+} (n) \varphi_{2,n} (\evn) \sqrt{\frac{n\xi_+}{\lambda_+}} \cdot
    \sR_{n,0}(\bbeta_*, \lambda).
\end{align}

\noindent
\textbf{Step 5: Concluding.}

Finally, combining the bounds~\eqref{eq:step_1_res_gen_test},~\eqref{eq:step2_res_gen} and~\eqref{eq:step_3_res_test_gen} resulting from steps 2, 3 and 4 respectively, along with Lemma~\ref{lem:reduced_det_equiv}, we conclude
\begin{align*}
    &~\left| \|\btheta_* - \bX^\sT \bG \by \|_\bSigma^2  + \sigma_\eps^2 - \sR_{n}(\bbeta_*, \lambda)\right|\\
    \le &~
\left| \|\btheta_* - \bX^\sT \bG \by \|_\bSigma^2 - \|\btheta_* - \bX^\sT \bG_0 \by \|_\bSigma^2 \right| + 
\left| \|\btheta_* - \bX^\sT \bG_0 \by \|_\bSigma^2 - (\|\btheta_* - \bX^\sT \bG_0 \by \|_{\bSigma_0}^2 +\|\bbeta_+\|_2^2 ) \right| \\
&~ + \left|\sR_{n,0}(\bbeta_*, \lambda) - \sR_n(\bbeta_*, \lambda) \right| +\left| \|\btheta_* - \bX^\sT \bG_0 \by \|_{\bSigma_0}^2 +\|\bbeta_+\|_2^2 + \sigma_\eps^2 - \sR_{n,0}(\bbeta_*, \lambda) \right|\\
\leq&~ C_{x,\eps,K,D} \left\{ \nu_{\lambda_+} (n)^2 \cE_{+,n} (\evn) + \nu_{\lambda_+} (n) \varphi_{2,n} (\evn)  \sqrt{\frac{n\xi_+}{\lambda_+}} \right\}
\cdot
\sR_{n}(\bbeta_*, \lambda).
\end{align*}
with probability at least $1 - n^{-D} - p_{2,n} (\evn)$ via union bound.
\end{proof}

\subsection{Technical lemmas}
\label{app_test_error:tech_high-degree_target}

Assumption~\ref{ass_app:noise_subGaussian} on the label noise $\eps_i$ implies that the random vector $\eps$ satisfy the Hanson-Wright inequality~\cite{rudelson13hanson}. We recall it here for the reader's convenience.

\begin{lemma}[Hanson-Wright inequality \cite{rudelson13hanson}]\label{lem:Hanson_Wright_inequality}
    Let $\beps = (\eps_1,\ldots,\eps_n)$ be a random vector with independent $\tau_\eps^2$-sub-Gaussian entries, with mean $0$ and $\E [ \eps_i^2] = \sigma_\eps^2$. Then there exists a universal constant $c >0$ such that for any matrix $\bA \in \R^{n \times n}$, we have for all $t \geq 0$,
    \begin{equation}\label{eq:Hanson_Wright_noise}
\P \left(\left| \beps^\sT \bA \beps - \sigma_\eps^2 \cdot\Tr(\bA)  \right| \geq  \sigma_\eps^2 \cdot t \right) \leq 2 \exp \left\{ - c \min \left( \frac{t^2}{\| \bA \|_F^2} , \frac{t}{\| \bA \|_\op} \right) \right\} .
\end{equation}
\end{lemma}

The next two lemmas bound the contributions to the train and test errors of the high-degree part of the target function. The first lemma shows that it is sufficient to bound these contributions in expectation.

\begin{lemma}\label{lem:tech_concentration_high-degree_part}
    Under the same setting and assumptions as Theorems \ref{thm:training_general} and \ref{thm:test_general}, and for any $D,K>0$, there exists constants $\eta := \eta_x \in (0,1/2)$, $C_{K,D}>0$ and $C_{x,K,D}>0$ such that the following hold. Let $\nu_{\lambda_+} (n)$ be defined as per Eq.~\eqref{eq:def_nu_0_lambda_+}. Then for all $n \geq C_{K,D}$ and $\lambda_+ >0$ satisfying conditions \eqref{eq:condition1_test}, we have with probability at least $1 -n^{-D}$ that
    \begin{align}
         \left| \boldf_+^\sT \bG_0 \bX_0 \bSigma_0 \bX_0^\sT \bG_0 \boldf_+ - \E [\boldf_+^\sT \bG_0 \bX_0 \bSigma_0 \bX_0^\sT \bG_0 \boldf_+ ] \right| \leq&~ \tcE_{+,n} \cdot \| \bbeta_+ \|_2^2, \label{eq:high_degree_martingale_1}\\
         \left| \boldf_+^\sT \bG_0 \bX_0 \bSigma_0 \bR_0 \btheta_0 - \E [\boldf_+^\sT \bG_0 \bX_0 \bSigma_0 \bR_0 \btheta_0 ] \right| \leq&~ \tcE_{+,n} \cdot \sqrt{\| \bbeta_+ \|_2^2 \Psi_3 (\mu_{*,0} ; \bA_{*,0} )}, \label{eq:high_degree_martingale_2}\\
         \left| \frac{1}{n} \boldf_+^\sT \bG_0^2 \boldf_+ - \frac{1}{n} \E [\boldf_+^\sT \bG_0^2 \boldf_+ ] \right| \leq&~ \frac{\tcE_{+,n}}{\lambda_+^2} \| \bbeta_+ \|_2^2, \label{eq:high_degree_martingale_3}\\
         \left| \frac{1}{n} \boldf_+^\sT \bG_0^2 \bX_0\btheta_0 - \frac{1}{n} \E [\boldf_+^\sT \bG_0^2 \bX_0\btheta_0 ] \right| \leq&~ \frac{\tcE_{+,n}}{\lambda_+} \sqrt{\| \bbeta_+ \|_2^2 \Psi_3 (\mu_{*,0} ; \bA_{*,0} )}, \label{eq:high_degree_martingale_4}
    \end{align}
    where we denoted $\bA_{*,0} = \bSigma_0^{-1/2} \btheta_0 \btheta_0^\sT \bSigma_0^{-1/2}$ and
    \[
    \tcE_{+,n} = C_{x,K,D} \frac{\nu_{\lambda_+,0} (n)^4 \log^{3\beta +\frac12} (n)}{\sqrt{n}}.
    \]
\end{lemma}

\begin{proof}[Proof of Lemma \ref{lem:tech_concentration_high-degree_part}]
The proof will follow from Azuma-Hoeffding inequality with a standard truncation argument, similarly to the proof of Proposition \ref{prop:TrAM_martingale}. For the sake of brevity, we will omit some repetitive computations and refer to Section \ref{app:proof_TrAM_martingale} for further details on the truncation argument.

For convenience, we drop the subscripts in the proof and simply write $\bx_i$, $\lambda$, $\nu_\lambda (n)$ for $\bx_{0,i}$, $\lambda_+$, $\nu_{\lambda_+}(n)$. We will also denote $h_i := f_+ (\bx_{+,i})$ and $\bh =  (h_1 , \ldots , h_n)$.
For $i\in [n]$. we introduce $\bX_i \in \R^{(n-1) \times \evn}$ the feature matrix where we removed feature $\bx_i$, the whitened matrix $\bZ_i = \bX_i \bSigma^{-1/2}$,  and $\bh_{i} = (h_j )_{j \neq i}$. Further denote 
\[
\bG_i = (\bX_i \bX_i^\sT + \lambda)^{-1}, \qquad \bR_i = (\bX_i^\sT \bX_i + \lambda)^{-1}, \qquad \bM_i = \bSigma^{1/2} \bR_i \bSigma^{1/2}.
\]
 Using these notations, the block matrix inversion formula gives for all $i\in [n]$,
\begin{equation}\label{eq:formula_block_inverse_G}
\bG = g_{ii} \begin{pmatrix}
    1 & - \bg_i^\sT \\
    -\bg_i & \bg_i \bg_i^\sT 
\end{pmatrix} + \begin{pmatrix}
    0 & \bzero \\
    \bzero & \bG_i  
\end{pmatrix}, \qquad g_{ii} := \frac{1}{\lambda (\bz_i^\sT \bM_i \bz_i +1)}, \qquad \bg_i := \bG_i \bX_i\bx_i\, .
\end{equation}

Below, we will repeatedly use that from Assumption \ref{ass_app:main_assumptions}.(c), for all integers $D$ and $q$, there exist constants $C_{x,q}$ and $C_{x,D}$ such that $\| h \|_{L^q} \leq C_{x,q} \| \bbeta_+ \|_2$, and $|h_i|^2 \leq C_{x,D} \log^\beta (n) \| \bbeta_+ \|_2^2$ with probability at least $1 - n^{-D}$.

\paragraph*{Proof of Equation \eqref{eq:high_degree_martingale_1}.} We decompose the difference into a martingale difference sequence
\[
\begin{aligned}
S_n :=&~ \bh^\sT \bG \bX \bSigma \bX^\sT \bG \bh - \E [\bh^\sT \bG \bX \bSigma \bX^\sT \bG \bh ] = \sum_{i=1}^n \Delta_i,\\
 \Delta_i :=&~ \left( \E_i - \E_{i-1} \right) \bh^\sT \bG \bX \bSigma \bX^\sT \bG \bh,
\end{aligned}
\]
where the expectation $\E_i $ is over $\{ \bx_{i+1} , \ldots , \bx_n\}$ and $\{h_{i+1} , \ldots , h_n \}$. We proceed similarly to the proof of Proposition \ref{prop:TrAM_martingale} and bound $| \Delta_i | $ with high probability. Step $3$ follows similarly by noting that
\[
\begin{aligned}
    \E_{i-1} [ \Delta_i^2 ]^{1/2} \leq&~ 2 \E_{i-1} \left[ (\bh^\sT \bG \bX \bSigma \bX^\sT \bG \bh)^2\right]^{1/2} \leq 2 \E_{i-1} \left[ \| \bZ \bM^2 \bZ \|_\op^2 \cdot \| \bh \|_2^4  \right]^{1/2} \\
    \leq&~ C_{x,K,D} \cdot \nu_\lambda (n) \log^{\beta} (n) \| \bbeta_+ \|_2^2,
\end{aligned}
\]
where we used Lemma \ref{lem:tech_bounds_norm_M} and Assumption \ref{ass_app:main_assumptions}.(c) with a union bound. 

Let us bound $| \Delta_i | $ with high probability. First we can remove any quantity that is independent of $\bz_i$ and thus
\[
\Delta_i = (\E_i - \E_{i-1}) \left( \bh^\sT \bG \bX \bSigma \bX^\sT \bG \bh - \bh_{i}^\sT \bG_i \bX_i \bSigma \bX_i^\sT \bG_i \bh_{i}\right). 
\]
From the block inverse formula \eqref{eq:formula_block_inverse_G}, we have by simple algebra
\begin{equation}\label{eq:formula_GAG}
    \begin{aligned}
    \bG \bX \bSigma \bX^\sT \bG = &~ g_{ii}^2 q_i \begin{pmatrix}
    1 & - \bg_i^\sT \\
    -\bg_i & \bg_i \bg_i^\sT 
\end{pmatrix} + g_{ii} \begin{pmatrix}
    0 & - \bq_i^\sT  \\
   - \bq_i & \bq_i \bg_i^\sT + \bg_i \bq_i^\sT
\end{pmatrix}
+ \begin{pmatrix}
    0 & \bzero \\
    \bzero & \bG_i \bX_i \bSigma \bX_i^\sT \bG_i  
\end{pmatrix},
\end{aligned}
\end{equation}
where we denoted 
\[
\begin{aligned}
    q_i :=&~ (\bx_i^\sT \bSigma \bx_i - 2 \bx_i^\sT \bSigma \bX_i^\sT \bg_i + \bg_i \bX_i \bSigma \bX_i^\sT \bg_i  ) = \lambda^2 \bz_i^\sT \bM_i^2 \bz_i , \\
    \bq_i :=&~  \bG_i (\bX_i \bSigma \bX_i^\sT \bg_i - \bX_i\bSigma \bx_i)= - \lambda \bZ_i \bM_i^2 \bz_i  .
\end{aligned}
\]
Using this formula, we decompose the difference
\[
\begin{aligned}
    &~ \bh^\sT \bG \bX \bSigma \bX^\sT \bG \bh - \bh_{i}^\sT \bG_i \bX_i \bSigma \bX_i^\sT \bG_i \bh_{i}
    = g_{ii}^2 q_i ( h_i - \bg_i^\sT \bh_i)^2 - 2g_{ii} (h_i - \bg_i^\sT \bh_i ) \bq_i^\sT \bh_i .
\end{aligned}
\]
By Lemma \ref{lem:tech_upper_bound_AM}.(b), we can bound each of these terms with probability at least $1 - n^{-D}$,
\[
\begin{aligned}
    \E_i \left[ q_i^2\right]^{1/2} \leq&~ C_{x,D,K}\lambda^2 \frac{\nu_{\lambda}(n)^2 \log^{\beta} (n) }{n},\\
    \E_i[ h_i^4 ]^{1/2} \leq&~ C_{x,D} \log^\beta (n) \| \bbeta_+ \|_2^2 ,\\
    \E_i [ (\bg_i^\sT \bh_i)^4]^{1/2} \leq&~ C_{x,D} \log^{\beta} (n) \E \left[ (\bh_i^\sT \bZ_i \bM_i^2 \bZ_i^\sT \bh_i)^2 \right]^{1/2} \leq C_{x,K,D}\cdot  \nu_\lambda (n) \log^{2\beta} (n) \| \bbeta_+ \|_2^2,\\
    \E_i [ (\bq_i^\sT \bh_i)^4]^{1/4} \leq&~ \lambda C_{x,D} \log^{\beta/2} (n) \E\left[ ( \bh_i^\sT \bZ_i \bM_i^4 \bZ_i^\sT \bh_i )^2 \right]^{1/4} \leq \lambda C_{x,D,K} \frac{\nu_\lambda(n)^{3/2} \log^{\beta} (n)}{n} \| \bbeta_+ \|_2.
\end{aligned}
\]
Similar bounds hold with $\E_{i}$ replaced by $\E_{i-1}$. We deduce that with probability at least $1 -n^{-D}$,
\[
| \Delta_i| \leq C_{x,D,K}\frac{\nu_\lambda (n)^3 \log^{3\beta} (n) }{n } \| \bbeta_+ \|_2^2,
\]
where we used that $g_{ii} \leq 1/\lambda$.
Following the same argument as in the proof of Proposition \ref{prop:TrAM_martingale}, we deduce that with probability at least $1 - n^{-D}$,
\[
 \left| \bh^\sT \bG \bX \bSigma \bX^\sT \bG \bh - \E [\bh^\sT \bG \bX \bSigma \bX^\sT \bG \bh ] \right| \leq C_{x,D,K}\frac{\nu_\lambda (n)^3 \log^{3\beta + \frac12} (n) }{\sqrt{n} } \| \bbeta_+ \|_2^2.
\]

\paragraph*{Proof of Equation \eqref{eq:high_degree_martingale_2}.} The proof follows in an analogous manner as Eq.~\eqref{eq:high_degree_martingale_1}. From Eq.~\eqref{eq:formula_block_inverse_G}, we decompose the difference
\[
\begin{aligned}
\bh^\sT \bG \bX \bSigma \bR \btheta_0 - \bh_{i}^\sT \bG_i \bX_i \bSigma \bR_i \btheta_0
    =&~ g_{ii} (h_i - \bg_i^\sT \bh_i) (\bx_i^\sT \bSigma \bR \btheta_0 - \bg_i^\sT \bX_i \bSigma \bR \btheta_0 ) \\
    &~+ \bh_i^\sT \bG_i \bX_i \bSigma (\bR - \bR_i) \btheta_0 .
\end{aligned}
\]
Note that we can rewrite
\[
\begin{aligned}
    \bx_i^\sT \bSigma \bR \btheta_0 - \bg_i^\sT \bX_i \bSigma \bR \btheta_0  = &~  \lambda \bx_i^\sT \bR_i \bSigma \bR_i  \btheta_0  - \lambda \frac{ \bx_i^\sT \bR_i \bSigma \bR_i \bx_i}{1 + \bx_i^\sT \bR_i \bx_i} \bx_i^\sT \bR_i \btheta_0, \\
     \bh_i^\sT \bG_i \bX_i \bSigma (\bR - \bR_i) \btheta_0 =&~ -  \frac{\bh_i^\sT \bG_i \bX_i \bSigma  \bR_i \bx_i}{1 + \bx_i^\sT \bR_i \bx_i} \bx_i^\sT \bR_i \btheta_0.
\end{aligned}
\]
Recall that we denoted $\bA_* = \bSigma^{-1/2} \btheta_0 \btheta_0 \bSigma^{-1/2} $. By Lemmas \ref{lem:tech_bound_zAz} and \ref{lem:tech_upper_bound_AM}.(b), we have with probability at least  $1 - n^{-D}$
\[
\begin{aligned}
    \E_i \left[ \left| \bx_i^\sT \bR_i \bSigma \bR_i  \btheta_0 \right|^2 \right]^{1/2} \leq&~ C_{x,D} \cdot \log^{\beta/2} (n)   \E_i \left[ \Tr ( \bSigma^{-1/2} \btheta_0 \btheta_0^\sT  \bSigma^{-1/2} \bM_i^4 ) \right]^{1/2} \\
    \leq&~ C_{x,K,D} \frac{\nu_\lambda (n)^2}{n} \log^{\beta/2} (n) \sqrt{ \Psi_3(\mu_* ; \bA_*)},\\
 \E_i \left[ \left| \bx_i^\sT \bR_i  \btheta_0 \right|^4 \right]^{1/4} \leq &~ C_{x,K,D} \cdot \nu_\lambda (n) \log^{\beta/2} (n) \sqrt{ \Psi_3(\mu_* ; \bA_*)},\\
    \E_i \left[ \left| \bx_i^\sT \bR_i \bSigma \bR_i \bx_i \right|^{4} \right]^{1/4}  \leq&~ C_{x,D} \cdot  \log^{\beta} (n) \E_i \left[ \Tr( \bM_i^2 ) \right] \leq C_{x,K,D}  \frac{\nu_\lambda (n)^2 \log^{\beta} (n)}{n} , \\
     \E_i \left[ \left| \bh_i^\sT \bG_i \bX_i \bSigma  \bR_i \bx_i\right|^4 \right]^{1/4}  \leq&~ C_{x,D} \log^{\beta/2} (n) \E_i\left[  \left(\bh_i^\sT \bX_i \bR_i \bSigma \bR_i \bSigma \bR_i \bSigma \bR_i \bX_i \bh_i \right)^2 \right]^{1/4} \\
     \leq&~ C_{x,D} \log^{\beta/2} (n) \E_i\left[  \| \bZ_i \bM_i^4 \bZ_i^\sT \|_\op^{2} \cdot \| \bh_i \|_2^4 \right]^{1/4} \\
     \leq&~ C_{x,K,D} \frac{\nu_\lambda (n)^{3/2}  \log^{\beta} (n)}{n} \| \bbeta_+ \|_2 .
\end{aligned}
\]
Similar bounds hold with $\E_{i-1}$ and we obtain that with probability at least $1 - n^{-D}$,
\[
| \Delta_i | \leq  C_{x,K,D} \frac{\nu_\lambda (n)^{7/2} \log^{5\beta/2} (n) }{n} \sqrt{\| \bbeta_+ \|_2^2 \cdot \Psi_3 (\mu_* ; \bA_* )}.
\]
We conclude that with probability at least $1 - n^{-D}$,
\[
\left| \bh^\sT \bG \bX \bSigma \bR \btheta_0 - \E [\bh^\sT \bG_0 \bX \bSigma \bR \btheta_0 ] \right| \leq  C_{x,K,D} \frac{\nu_\lambda (n)^{7/2} \log^{5\beta/2 + \frac{1}{2}} (n) }{\sqrt{n}} \sqrt{\| \bbeta_+ \|_2^2 \cdot \Psi_3 (\mu_* ; \bA_* )}.
\]

\paragraph*{Proof of Equation \eqref{eq:high_degree_martingale_3}.} We consider now the matrix
\begin{equation}\label{eq:formula_G^2}
\bG^2 = g_{ii}^2 (1 + \| \bg_i \|_2^2) \begin{pmatrix}
    1 & - \bg_i^\sT \\
    -\bg_i & \bg_i \bg_i^\sT 
\end{pmatrix} + g_{ii} \begin{pmatrix}
    0 & - \bg_i^\sT \bG_i \\
    - \bG_i \bg_i & \bG_i \bg_i \bg_i^\sT +  \bg_i \bg_i^\sT \bG_i
\end{pmatrix} + \begin{pmatrix}
    0 & \bzero \\
    \bzero & \bG_i^2  
\end{pmatrix}.
\end{equation}
We decompose the difference into
\[
\begin{aligned}
    \frac{1}{n} \bh^\sT \bG^2 \bh - \frac{1}{n} \bh_i^\sT \bG_i^2 \bh_i =&~ \frac{g_{ii}^2}{n} (1 + \| \bg_i \|_2^2) ( h_i - \bg_i^\sT \bh_i)^2 - \frac{ 2g_{ii}}{n} ( h_i - \bg_i^\sT \bh_i) \bg_i^\sT \bG_i \bh_i.
\end{aligned}
\]
Again, using Lemmas \ref{lem:tech_bound_zAz} and \ref{lem:tech_upper_bound_AM}.(b), we have the following bounds with probability at least $1 - n^{-D}$:
\[
\begin{aligned}
    \E_i \left[ \| \bg_i \|^4 \right]^{1/2} \leq&~ C_{x,D} \log^{\beta} (n) \E_{i}\left[  \Tr( \bR_i \bSigma \bR_i \bX_i^\sT \bX_i )^2 \right]^{1/2} \leq C_{x,K,D} \nu_\lambda (n) \log^{\beta} (n),\\
    \E_i \left[ |\bg_i^\sT \bG_i \bh_i|^2 \right]^{1/2} \leq&~ C_{x,D} \log^{\beta /2} (n) \E_i\left[ \bh_i^\sT\bG_i  \bX_i \bR_i \bSigma \bR_i \bX_i^\sT \bG_i \bh_i \right]^{1/2} \\
    \leq&~ C_{x,D} \log^{\beta /2} (n) \E_i\left[  \| \bZ_i \bM_i^2 \bZ_i^\sT \|_\op \cdot \| \bG_i\|_\op^2 \|\bh_i \|_2^2 \right]^{1/2}\\
    \leq&~ C_{x,D} \nu_\lambda (n)^{1/2} \log^{\beta} (n) \frac{\| \bbeta_+ \|_2}{\lambda},
\end{aligned}
\]
where we used that $\| \bG_i \|_\op \leq \lambda^{-1}$. We deduce that with probability at least $1 - n^{-D}$,
\[
\left| \Delta_i \right| \leq C_{x,K,D} \frac{\nu_\lambda (n)^3 \log^{3 \beta} (n)}{\lambda^2 n} \| \bbeta_+ \|_2^2
\]
which implies Eq.~\eqref{eq:high_degree_martingale_3}.

\paragraph*{Proof of Equation \eqref{eq:high_degree_martingale_4}.} We decompose the difference into
\[
\begin{aligned}
    \frac{1}{n} \bh^\sT \bG \bX \bR \btheta_0 -  \frac{1}{n} \bh_i^\sT \bG_i \bX_i \bR_i \btheta_0 =&~ \frac{g_{ii}}{n} (h_i - \bg_i^\sT \bh_i) ( \bx_i^\sT \bR \btheta_0 - \bg_i^\sT \bX_i \bR \btheta_0) + \frac{1}{n} \bh_i^\sT \bG_i \bX_i (\bR - \bR_i ) \btheta_0 .
\end{aligned}
\]
We can rewrite these two terms as
\[
\begin{aligned}
    \bx_i^\sT \bR \btheta_0 - \bg_i^\sT \bX_i \bR \btheta_0=&~  \bx_i^\sT \bR_i \btheta_0  - \bx_i^\sT \bR_i \bX_i^\sT \bX_i \bR_i \btheta_0 - \lambda \frac{\bx_i^\sT \bR_i^2 \bx_i}{1 + \bx_i^\sT \bR_i \bx_i} \bx_i^\sT \bR_i \btheta_0, \\
     \bh_i^\sT \bG_i \bX_i (\bR - \bR_i ) \btheta_0 =&~ - \frac{\bh_i^\sT \bG_i \bX_i \bR_i \bx_i}{1 + \bx_i^\sT \bR_i \bx_i} \bx_i^\sT \bR_i \btheta_0.
\end{aligned}
\]
Again, using Lemmas \ref{lem:tech_bound_zAz} and \ref{lem:tech_upper_bound_AM}.(b), we have the following bounds with probability at least $1 - n^{-D}$:
\[
\begin{aligned}
    \E_i\left[ \left(\bx_i^\sT \bR_i \bX_i^\sT \bX_i \bR_i \btheta_0 \right)^2 \right]^{1/2}  \leq&~ C_{x,D} \log^{\beta/2} (n) \E_i \left[\btheta_0^\sT \bSigma^{-1/2}  \bM_i \bZ_i^\sT \bZ_i \bM_i^2 \bZ_i^\sT \bZ_i \bM_i \bSigma^{-1/2} \btheta_0  \right]^{1/2}\\
    \leq&~ C_{x,K,D} \cdot \nu_\lambda (n) \log^{\beta/2} (n) \sqrt{\Psi_3 (\mu_* ; \bA_* )},\\
     \lambda \E_i \left[ \left( \bx_i^\sT \bR_i^2 \bx_i \right)^4 \right]^{1/4} \leq&~ C_{x,D} \log^{\beta} (n) \E_i \left[ \Tr( \bM_i)^4 + \varphi_1 (\evn)^4 \| \bM_i \|_F^4 \right]^{1/4} \\
     \leq&~ C_{x,K,D} \cdot \nu_\lambda (n) \log^{\beta} (n) ,\\
     \E_i \left[ \left(\bh_i^\sT \bG_i \bX_i \bR_i \bx_i \right)^2\right]^{1/2} \leq&~ C_{x,D} \log^{\beta/2} (n) \E_i \left[ \bh_i^\sT \bG_i \bZ_i \bM_i^2 \bZ_i \bG_i \bh_i \right]^{1/2} \\
     \leq&~ C_{x,K,D} \log^{\beta/2} (n)\E_i \left[ \| \bM_i\|_\op \| \bG_i \|_\op^2 \| \bh_i\|_2^2 \right]^{1/2} \\
     \leq&~ C_{x,K,D} \cdot \lambda^{-1} \nu_\lambda (n)^{1/2} \log^{\beta} (n) \| \bbeta_+ \|_2 .
\end{aligned}
\]
Combining these bounds yields
\[
\left| \Delta_i \right| \leq C_{x,K,D} \frac{\nu_\lambda (n)^3 \log^{3\beta} (n)}{\lambda n } \sqrt{\| \bbeta_+ \|_2^2 \Psi_3 (\mu_* ; \bA_* )}
\]
with probability at least $1 - n^{-D}$, which implies Eq.~\eqref{eq:high_degree_martingale_4}.
\end{proof}

Using the previous lemma, we can now bound the contribution of the high-degree part of the target function.

\begin{lemma}\label{lem:high-degree_part}
    Under the same setting and assumptions as Theorems \ref{thm:training_general} and \ref{thm:test_general}, and for any $D,K>0$, there exists constants $\eta := \eta_x \in (0,1/2)$, $C_{K,D}>0$ and $C_{x,K,D}>0$ such that the following hold. Let $\nu_{\lambda_+} (n)$ be defined as per Eq.~\eqref{eq:def_nu_0_lambda_+}. Then for all $n \geq C_{K,D}$ and $\lambda_+ >0$ satisfying conditions \eqref{eq:condition1_test}, we have with probability at least $1 -n^{-D}$ that
    \begin{align}
         \left| \boldf_+^\sT \bG_0 \bX_0 \bSigma_0 \bX_0^\sT \bG_0 \boldf_+ -  \| \bbeta_+ \|_2^2 \cdot n \Psi_4 (\mu_{*,0} ;\id) \right| \leq&~ \cE_{+,n}(\evn) \cdot \| \bbeta_+ \|_2^2, \label{eq:high_degree_det_1}\\
         \left| \boldf_+^\sT \bG_0 \bX_0 \bSigma_0 \bR_0 \btheta_0 \right| \leq&~ \cE_{+,n} (\evn)\cdot \sqrt{\| \bbeta_+ \|_2^2 \Psi_3 (\mu_{*,0} ; \bA_{*,0} )}, \label{eq:high_degree_det_2}\\
         \left| \frac{1}{n} \boldf_+^\sT \bG_0^2 \boldf_+ - \left(\frac{\mu_{*,0}}{n\lambda_+}\right)^2 \frac{\| \bbeta_+ \|_2^2}{1 - \frac{1}{n} \Tr(\bSigma_0^2 ( \bSigma_0 + \lambda_{*,0})^{-2})} \right| \leq&~ \frac{\cE_{+,n}(\evn)}{\lambda_+^2} \| \bbeta_+ \|_2^2, \label{eq:high_degree_det_3}\\
         \left| \frac{1}{n} \boldf_+^\sT \bG_0^2 \bX_0\btheta_0  \right| \leq&~ \frac{\cE_{+,n} (\evn)}{\lambda_+} \sqrt{\| \bbeta_+ \|_2^2 \Psi_3 (\mu_{*,0} ; \bA_{*,0} )}, \label{eq:high_degree_det_4}
    \end{align}
    where we denoted $\bA_{*,0} = \bSigma_0^{-1/2} \btheta_0 \btheta_0^\sT \bSigma_0^{-1/2}$ and
    \[
    \cE_{+,n} (\evn) = C_{x,K,D} \frac{ \varphi_1 (\evn) \nu_{\lambda_+} (n)^6 \log^{3\beta +\frac12} (n)}{\sqrt{n}}.
    \]
\end{lemma}

\begin{proof}[Proof of Lemma \ref{lem:high-degree_part}]
    In Lemma \ref{lem:tech_concentration_high-degree_part}, we saw that these four quantities concentrate and it only remains to bound their expectations. We will follow the same notations as in the proof of Lemma \ref{lem:tech_concentration_high-degree_part} and omit some repetitive computations for the sake of brevity.

    \paragraph*{Proof of Equation \eqref{eq:high_degree_det_1}.} Using the formula \eqref{eq:formula_GAG}, we can rewrite the expectation as
    \[
    \begin{aligned}
        \E\left[ \bh^\sT \bG \bX \bSigma \bX^\sT \bG \bh \right] =&~ n \E \left[ g_{ii}^2 q_i h_i (h_i - \bg_i^\sT \bh_i) - g_{ii} h_i \bq_i^\sT \bh_i\right].
    \end{aligned}
    \]
    We will first consider the expectation over $\bx_i$. 
    Note that $g_{ii}$ and $q_i$ concentrates on quantities that are independent of $\bx_i$, which we denote
    \[
    \og_{ii} := \frac{1}{\lambda(1 + \Tr( \bM_i))} , \qquad \oq_i := \lambda^2 \Tr(\bM_i^2).
    \]
    We replace $g_{ii}$ and $q_i$ by these terms in the expectation and get  
    \[
    \begin{aligned}
        &~n\left|\E \left[ g_{ii}^2 q_i h_i (h_i - \bg_i^\sT \bh_i) - g_{ii} h_i \bq_i^\sT \bh_i\right] -  \E \left[ \og_{ii}^2 \oq_i h_i (h_i - \bg_i^\sT \bh_i) - \og_{ii} h_i \bq_i^\sT \bh_i\right] \right| \\
        \leq&~ n  \left| \E \left[ \left( g_{ii}^2 - \og_{ii}^2 \right)  q_i h_i (h_i - \bg_i^\sT \bh_i ) \right] \right| + n  \left| \E \left[( q_{i} - \oq_i ) \og_{ii}^2 h_i (h_i - \bg_i^\sT \bh_i )  \right] \right| \\
        &~+ n \left| \E \left[ ( g_{ii} - \og_{ii} ) h_i \bq_i^\sT \bh_i \right] \right| .
    \end{aligned}
    \]
    By H\"older's inequality, the first term is bounded by
    \[
    \begin{aligned}
         n \left| \E \left[ \left( g_{ii}^2 - \og_{ii}^2 \right)  q_i h_i (h_i - \bg_i^\sT \bh_i ) \right] \right|  \leq&~ C n\E \left[ \left| g_{ii}^2 - \og_{ii}^2 \right|^3 \right]^{1/3} \E \left[ | q_i |^3 \right]^{1/3} \E \left[ | h_i |^6 + |  \bg_i^\sT \bh_i |^6 \right]^{1/3} \\
         \leq&~ n \cdot C_{x,K}  \frac{\varphi_1 (\evn)  \nu_\lambda (n)^2 }{\lambda^{2} \sqrt{n}} \cdot \lambda^2 \frac{\nu_\lambda (n)^2}{n} \cdot  \nu_\lambda (n) \| \bbeta_+ \|_2^2 \\
         \leq&~ C_{x,K} \frac{\varphi_1 (\evn) \nu_\lambda (n)^5 }{\sqrt{n}} \| \bbeta_+ \|_2^2,
    \end{aligned}
    \]
    where we used that by Lemma \ref{lem:tech_bound_zAz} and Lemma \ref{lem:tech_upper_bound_AM}.(b),
    \[
    \begin{aligned}
      \E \left[ \left| g_{ii}^2 - \og_{ii}^2 \right|^3 \right]^{1/3} \leq&~  \frac{1}{\lambda^2} \E \left[ \left| \bz_i^\sT \bM_i \bz_i - \Tr( \bM_i) \right|^6 \right]^{1/6}   \E \left[ \left|2 +  \bz_i^\sT \bM_i \bz_i + \Tr( \bM_i) \right|^6 \right]^{1/6} \\
      \leq &~  C_{x,K} \frac{\varphi_1 (\evn)  \nu_\lambda (n)^2 }{\lambda^2 \sqrt{n}}.
    \end{aligned}
    \]
    Similarly,
    \[
    \begin{aligned}
     n  \left| \E \left[( q_{i} - \oq_i ) \og_{ii}^2 h_i (h_i - \bg_i^\sT \bh_i )  \right] \right| \leq&~ Cn \E \left[ \left| \bz_i^\sT \bM_i^2 \bz_i - \Tr( \bM_i^2) \right|^2 \right]^{1/2} \E \left[ | h_i |^4 + |  \bg_i^\sT \bh_i |^4 \right]^{1/2}\\
     \leq&~ C_{x,K} \frac{\varphi_1 (\evn) \nu_\lambda (n)^3}{\sqrt{n}} \| \bbeta_+ \|_2^2,
    \end{aligned}
    \]
    and
    \[
    \begin{aligned}
        n \left| \E \left[ ( g_{ii} - \og_{ii} ) h_i \bq_i^\sT \bh_i \right] \right| \leq&~ Cn \E \left[ | g_{ii} - \og_{ii} |^3\right]^{1/3} \E \left[ | h_i |^3 \right]^{1/3} \E \left[ | \bq_i^\sT \bh_i |^3 \right]^{1/3} \\
        \leq&~ n \cdot C_{x,K} \frac{\varphi_1 (\evn) \nu_\lambda (n)}{\sqrt{n}} \cdot \| \bbeta_+ \|_2 \cdot \frac{\nu_\lambda (n)^{3/2}}{n} \| \bbeta_+ \|_2 \\
        \leq&~ C_{x,K} \frac{\varphi_1 (\evn) \nu_\lambda (n)^{5/2}}{\sqrt{n}} \| \bbeta_+ \|_2^2 .
    \end{aligned}
    \]
    Combining these three bounds yields
    \begin{equation}\label{eq:bound_det_high-degree_1}
        \begin{aligned}
            &~n\left|\E \left[ g_{ii}^2 q_i h_i (h_i - \bg_i^\sT \bh_i) - g_{ii} h_i \bq_i^\sT \bh_i\right] -  \E \left[ \og_{ii}^2 \oq_i h_i (h_i - \bg_i^\sT \bh_i) - \og_{ii} h_i \bq_i^\sT \bh_i\right] \right| \\
            \leq&~ C_{x,K} \frac{\varphi_1 (\evn) \nu_\lambda (n)^5 }{\sqrt{n}} \| \bbeta_+ \|_2^2.
        \end{aligned}
    \end{equation}
    The feature $\bx_i$ only appears in $h_i$, $\bg_i$ and $\bq_i$. Recalling that $\E_{\bx_i}[h_i \bx_i] = 0$ and $\E_{\bx_i} [ h_i^2] = \| \bbeta_+ \|_2^2$, the expectation simplifies to 
    \[
    \E \left[ \og_{ii}^2 \oq_i h_i (h_i - \bg_i^\sT \bh_i) - \og_{ii} h_i \bq_i^\sT \bh_i\right] = \| \bbeta_+ \|_2^2 \cdot \E_{\bM_i} \left[ \frac{\Tr( \bM_i^2)}{(1 + \Tr( \bM_i) )^2}\right].
    \]
    The right-hand side corresponds to the quantity that appears in the proof of Proposition \ref{prop:TrAMZZM_LOO} in Section \ref{app:proof_TrAMZZM_LOO} and we directly have
    \[
    \left| \E_{\bM_i} \left[ \frac{\Tr( \bM_i^2)}{(1 + \Tr( \bM_i) )^2}\right] - \Psi_4 (\mu_* ; \id) \right| \leq C_{x,K} \frac{\varphi_1 (\evn) \nu_\lambda (n)^6}{ n^{3/2}}.
    \]
    Combining this last bound with Eq.~\eqref{eq:bound_det_high-degree_1} and Eq.~\eqref{eq:high_degree_martingale_1} in Lemma \ref{lem:tech_concentration_high-degree_part} yields that with probability at least $1 - n^{-D}$,
    \[
    \begin{aligned}
    \left|  \bh^\sT \bG \bX \bSigma \bX^\sT \bG \bh -  \| \bbeta_+ \|_2^2\cdot n \Psi_4 (\mu_*;\id)  \right| \leq C_{x,K,D} \frac{\varphi_1 (\evn) \nu_\lambda (n)^6 \log^{3\beta + \frac{1}{2}}(n)}{\sqrt{n}} \| \bbeta_+ \|_2^2 . 
    \end{aligned}
    \]

    \paragraph*{Proof of Equation \eqref{eq:high_degree_det_2}.} We proceed similarly as in the proof of Eq.~\eqref{eq:high_degree_det_1}. From the formula~\eqref{eq:formula_block_inverse_G}, we decompose the expectation into
    \[
    \begin{aligned}
        \E\left[ \bh^\sT \bG \bX \bSigma \bR \btheta \right] = &~ n \E \left[ g_{ii} h_i (\bx_i^\sT \bSigma \bR \btheta - \bg_i^\sT \bX_i \bSigma \bR \btheta )\right] \\
        = &~ n \E \left[  \lambda g_{ii} h_i \bx_i^\sT \bR_i \bSigma \bR_i  \btheta_0 \right]  - n \E \left[   g_{ii}^2 q_i h_i \bx_i^\sT \bR_i \btheta_0 \right].
    \end{aligned}
    \]
    We replace again $g_{ii}$ and $q_i$ by $\og_{ii}$ and $\oq_i$, such that
    \[
    \begin{aligned}
        &~n \lambda \left|  \E \left[  ( g_{ii} - \og_{ii} ) h_i \bx_i^\sT \bR_i \bSigma \bR_i  \btheta_0 \right] \right| \\
        \leq&~ n \E\left[ \left| \bz_i^\sT \bM_i \bz_i - \Tr(\bM_i) \right|^3 \right]^{1/3} \E\left[ \left| h_i\right|^3 \right]^{1/3} \E\left[ \left| \bx_i^\sT \bR_i \bSigma \bR_i  \btheta_0\right|^3 \right]^{1/3}\\
        \leq&~ nC_{x,K} \frac{\varphi_1 (\evn) \nu_\lambda (n)}{\sqrt{n}} \cdot \| \bbeta_+ \|_2 \cdot \E \left[ \left| \btheta_0^\sT  \bSigma^{-1/2} \bM_i^4\bSigma^{-1/2}  \btheta_0 \right|^{3/2} \right]^{1/3} \\
        \leq &~ C_{x,K} \frac{\varphi_1 (\evn) \nu_\lambda (n)^2}{\sqrt{n}}   \sqrt{\| \bbeta_+ \|_2^2 \Psi_3 (\mu_* ; \bA_* ) },
    \end{aligned}
    \]
    and, following the same argument as for Eq.~\eqref{eq:high_degree_det_1}, 
    \[
    \begin{aligned}
        &~ n \left|  \E \left[   g_{ii}^2 q_i h_i \bx_i^\sT \bR_i \btheta_0 \right] -  \E \left[   \og_{ii}^2 \oq_i h_i \bx_i^\sT \bR_i \btheta_0 \right]\right| \leq C_{x,K} \frac{\varphi_1 (\evn) \nu_\lambda (n)^5}{\sqrt{n}}  \sqrt{\| \bbeta_+ \|_2^2 \Psi_3 (\mu_* ; \bA_* ) }.
    \end{aligned}
    \]
    Further note that 
    \[
    \E_{\bx_i} \left[  \lambda \og_{ii} h_i \bx_i^\sT \bR_i \bSigma \bR_i  \btheta_0\right] =  \E_{\bx_i} \left[ \og_{ii}^2 \oq_i h_i \bx_i^\sT \bR_i \btheta_0 \right] = 0.
    \]
    Thus, we obtain Eq.~\eqref{eq:high_degree_det_2} by combining the above bounds with Eq.~\eqref{eq:high_degree_martingale_2} in Lemma \ref{lem:tech_concentration_high-degree_part}.

    \paragraph*{Proof of Equation \eqref{eq:high_degree_det_3}.} For this third term, we simplify the expectation using formula~\eqref{eq:formula_G^2} and we obtain
    \[
    \begin{aligned}
        \frac{1}{n} \E[ \bh^\sT \bG^2 \bh ] =&~ \E \left[ g_{ii}^2 (1 + \| \bg_i \|_2^2) h_i ( h_i - \bg_i^\sT \bh_i ) \right] - \E \left[ g_{ii} h_i \bg_i^\sT \bG_i \bh_i\right].
    \end{aligned}
    \]
    Again, we note that $\| \bg_i \|_2^2 = \bx_i^\sT \bR_i \bX_i^\sT \bX_i \bR_i \bx_i$ concentrates over the randomness of $\bx_i$ on
    \[
    \og = \Tr( \bM_i^2 \bZ_i^\sT \bZ_i ).
    \]
    Replacing $g_{ii} $ and $\| \bg_i \|_2^2$ by $\og_{ii} $ and $\og$, we get the bounds
    \[
    \begin{aligned}
        &~\left| \E \left[ g_{ii}^2 (1 + \| \bg_i \|_2^2) h_i ( h_i - \bg_i^\sT \bh_i ) \right]  - \E \left[ \og_{ii}^2 (1 + \og) h_i ( h_i - \bg_i^\sT \bh_i ) \right] \right| \\
        \leq&~ \E \left[  \left| g_{ii}^2 - \og_{ii}^2 \right|^2 \right]^{1/2} \E \left[ (1 + \| \bg_i \|_2^2)^2 h_i^2 ( h_i - \bg_i^\sT \bh_i )^2  \right]^{1/2} \\
        &~+ \E \left[ \left| \| \bg_i \|_2^2 - \og \right|^2 \right]^{1/2} \E \left[  \og_{ii}^4 h_i^2 ( h_i - \bg_i^\sT \bh_i )^2 \right]^{1/2} \\
        \leq&~ C_{x,K} \frac{\varphi_1 (\evn) \nu_\lambda (n)^4}{\lambda^2 \sqrt{n} }  \| \bbeta_+ \|_2^2 , 
    \end{aligned}
    \]
    and similarly
    \[
    \begin{aligned}
        \left| \E \left[ g_{ii} h_i \bg_i^\sT \bG_i \bh_i\right] - \E \left[ \og_{ii} h_i \bg_i^\sT \bG_i \bh_i\right] \right| \leq C_{x,K} \frac{\varphi_1 (\evn) \nu_\lambda (n)^{3/2} }{\lambda^2 \sqrt{n}} \| \bbeta_+ \|_2^2.
    \end{aligned}
    \]
    Using that $\E_{\bx_i} [ h_i \bx_i ] = 0$, the remaining term becomes 
    \[
    \begin{aligned}
        \E \left[\og_{ii}^2 (1 + \og) h_i^2 \right] = &~ \frac{\| \bbeta_+ \|_2^2}{\lambda^2} \cdot \E \left[ \frac{1 + \Tr( \bM_i^2 \bZ_i^\sT \bZ_i )}{ (1 + \Tr(\bM_i))^2}\right]. 
    \end{aligned}
    \]
    Denote $\tmu_*$ the solution to the fixed point equation \eqref{eq:det_equiv_fixed_point_mu_star} with $n-1$ instead of $n$. Then, we decompose the expectation into 
    \[
    \begin{aligned}
        &~ \left|\E \left[ \frac{\Tr( \bM_i^2 \bZ_i^\sT \bZ_i )}{ (1 + \Tr(\bM_i))^2}\right] - \frac{1+ n\Psi_4 (\mu_*;\id)}{(1 + \Tr(\obM))^2} \right| \\
        \leq&~ \left|\E \left[! +  \frac{\Tr( \bM_i^2 \bZ_i^\sT \bZ_i )}{ (1 + \Tr(\bM_i))^2}\right] - \frac{\E [ 1+\Tr( \bM_i^2 \bZ_i^\sT \bZ_i)]}{(1 + \Tr( \obM))^2}  \right| + \left| \E [ \Tr( \bM_i^2 \bZ_i^\sT \bZ_i)] - n\Psi_4 (\mu_* ; \id) \right| \\
        \leq&~ C_{x,K} \frac{\varphi_1 (\evn) \nu_\lambda (n)^6}{\sqrt{n}},
    \end{aligned}
    \]
    where the first term is bounded with Lemma \ref{lem:tech_bound_M_EM} and Proposition \ref{prop:TrAM_LOO}, and the second term is bounded with Proposition \ref{prop:TrAMZZM_LOO} and Lemma \ref{lem:properties_mu_obM}. Combining these bounds with Eq.~\eqref{eq:high_degree_martingale_3} yields
    \[
    \left| \frac{1}{n} \bh^\sT \bG^2 \bh - \frac{\| \bbeta_+ \|_2^2}{\lambda^2} \frac{1+n\Psi_4 (\mu_* ; \id)}{(1 + \Tr(\obM))^2} \right| \leq C_{x,K,D} \frac{\varphi_1 (\evn) \nu_\lambda (n)^6 \log^{3\beta + \frac12} (n)}{\sqrt{n}} \frac{\| \bbeta_+ \|_2^2}{\lambda^2}
    \]
    with probability at least $1 - n^{-D}$.
    
    \paragraph*{Proof of Equation \eqref{eq:high_degree_det_4}.} We proceed similarly as for Eq.~\eqref{eq:high_degree_det_3}. We simplify the expectation
    \[
    \begin{aligned}
        \frac{1}{n} \E \left[ \bh^\sT \bG \bX \bR \btheta_0 \right] =&~ \E \left[ g_{ii} h_i ( \bx_i^\sT \bR \btheta_0 - \bg_i^\sT \bX_i \bR \btheta_0) \right]\\
        =&~ \E \left[ g_{ii} h_i \left(  \bx_i^\sT \bR_i \btheta_0  - \bx_i^\sT \bR_i \bX_i^\sT \bX_i \bR_i \btheta_0 - \lambda^2 g_{ii} \bx_i^\sT \bR_i^2 \bx_i \bx_i^\sT \bR_i \btheta_0 \right) \right].
    \end{aligned}
    \]
    Note that $\bx_i^\sT \bR_i^2 \bx_i$ concentrates over the randomness of $\bx_i$ on $\og_2 = \Tr( \bR_i^2 \bSigma)$. Replacing $g_{ii}$ and $\bx_i^\sT \bR_i^2 \bx_i$ by $\og_{ii}$ and $\og_2$ gives
    \[
    \begin{aligned}
        \left|\E \left[ (g_{ii} -\og_{ii} ) h_i \left(  \bx_i^\sT \bR_i \btheta_0  - \bx_i^\sT \bR_i \bX_i^\sT \bX_i \bR_i \btheta_0 \right) \right] \right| \leq&~ C_{x,K} \frac{\varphi_1 (\evn) \nu_\lambda (n)^2}{\lambda \sqrt{n}} \sqrt{\| \bbeta_+ \|_2^2 \Psi_3 (\mu_* ; \bA_* )},
    \end{aligned}
    \]
    and 
    \[
    \begin{aligned}
       \lambda^2 \left| \E \left[ g_{ii}^2 h_i  \bx_i^\sT \bR_i^2 \bx_i \bx_i^\sT \bR_i \btheta_0 \right]  - \E \left[ \og_{ii}^2 \og_2 h_i  \bx_i^\sT \bR_i \btheta_0 \right] \right|  \leq&~ C_{x,K} \frac{\varphi_1 (\evn) \nu_\lambda (n)^4}{\lambda \sqrt{n}} \sqrt{\| \bbeta_+ \|_2^2 \Psi_3 (\mu_* ; \bA_* )}.
    \end{aligned}
    \]
    Noting that
    \[
    \E_{\bx_i} \left[ \og_{ii} h_i \left(  \bx_i^\sT \bR_i \btheta_0  - \bx_i^\sT \bR_i \bX_i^\sT \bX_i \bR_i \btheta_0 - \lambda^2 \og_{ii} \og_2 \bx_i^\sT \bR_i \btheta_0 \right) \right] = 0,
    \]
    and recalling Eq.~\eqref{eq:high_degree_martingale_4} in Lemma \ref{lem:tech_concentration_high-degree_part} concludes the proof of Eq.~\eqref{eq:high_degree_det_4}.
\end{proof}





\subsection{Uniform convergence of the GCV estimator}
\label{app_test:GCV}

We make the dependency on $\lambda$ of the different quantities explicit in this section and drop the dependencies in the other quantities. In particular, we will denote $\mu_* (\lambda)$, $\lambda_* (\lambda)$, $\Upsilon_1 (\lambda)$, $\Upsilon_2(\lambda)$, and
\[
\bG (\lambda) := (\bX \bX^\sT + \lambda)^{-1}  , \qquad\qquad \bR(\lambda) :=  (\bX^\sT \bX + \lambda)^{-1}.
\]
Furthermore, we will simply denote
\[
\cR_n (\lambda) := \cR_{\test} (\bbeta_*;\bX,\beps, \lambda), \qquad s_n(\lambda) := s_n (\bX,\lambda), \qquad \sR_n(\lambda) := \sR_n (\bbeta_*,\lambda).
\]
We introduce the train error divided by $\lambda^{2}$ given by
\[
\cC_n ( \lambda) := \frac{\cL_{\train} (\bbeta_*;\bX,\beps,\lambda)}{\lambda^2}, \qquad \qquad\sfC_n (\lambda) := \frac{\sL_n (\bbeta_*,\lambda)}{\lambda^2},
\]
which are well defined at $\lambda = 0$.

Recall that we denote $\gamma_+ = \Tr(\bSigma_+)$ and that $\lambda_*(\lambda)$ and $\mu_* (\lambda)$ are increasing function in $\lambda$. In particular, under the assumptions of the theorem, Equation \eqref{eq:lambda0_lower_bound_lambda_star} holds and
\begin{equation}\label{eq:lambda_*_lower_bound_GCV}
\lambda_* (\lambda) \geq \lambda_*(0) \geq \frac{\gamma_+}{2n}.
\end{equation}
We will further use the following identities
\[
\frac{\partial}{\partial \lambda}\lambda_* (\lambda)= \frac{1}{n(1 - \Upsilon_2(\lambda))}, \qquad \quad \frac{\partial}{\partial \lambda}\mu_* (\lambda)= \frac{\Tr( \bSigma (\bSigma + \lambda_*)^{-2})}{n(1 - \Upsilon_2(\lambda))}.
\]

We first lemma show homogeneous Lipschitz bounds on the three deterministic equivalents.

\begin{lemma}\label{lem:Lipschitz_bound_det_equiv}
Under the assumptions of Theorem \ref{thm:abstract_GCV}, there exists a universal constant $C>0$ such that the following holds.
    For any $0\leq \lambda_0,\lambda_1 \leq n^K$ with $| \lambda_1 - \lambda_0| \leq 1$, we have
    \[
    \max\left\{   \left| \frac{\sfs_n(\lambda_1)}{\sfs_n(\lambda_0)} - 1\right|,  \left| \frac{\sfC_n(\lambda_1)}{\sfC_n(\lambda_0)} - 1 \right|, \left| \frac{\sR_n (\lambda_1)}{\sR_n (\lambda_0)} - 1\right|\right\} \leq n^{CK} | \lambda_1 - \lambda_0|
    \]
\end{lemma}

\begin{proof}[Proof of Lemma \ref{lem:Lipschitz_bound_det_equiv}]
    While we are quite loose below, it will suffice for the purpose of our proofs.

    \noindent
    {\bf Step 1: The Stieltjes transform.}

    We first write
    \[
   \frac{\sfs_n(\lambda_1)}{\sfs_n(\lambda_0)} - 1 =  \frac{\lambda_*(\lambda_0)}{\lambda_*(\lambda_1)} - 1 =   \frac{1}{\lambda_* (\lambda_1)} \int_{\lambda_1}^{\lambda_0}  \frac{\partial}{\partial \lambda}\lambda_* (\lambda) \de \lambda  = \frac{1}{\lambda_* (\lambda_1)} \int_{\lambda_1}^{\lambda_0}  \frac{1}{n (1 - \Upsilon_2 (\lambda))} \de \lambda.
    \]
    Recall that $\lambda_* (\lambda_1) \geq \lambda_* (0)$ and therefore $\Upsilon_2 (\lambda) \leq \Upsilon_2 (0)$. Furthermore, note that using the definition of the effective regularization, we can lower bound
\begin{equation}\label{eq:Upsilon_2_tech_lower}
    n \lambda_* (0) (1 - \Upsilon_2 (0)) = \lambda_*(0) \Tr(\bSigma ( \bSigma + \lambda_*(0))^{-2} )\geq \left(1 + \frac{1}{\lambda_*} \right)^{-2} \geq \left( 1 + \frac{2n}{\gamma_+} \right)^{-2} \geq n^{-2 (1+K)},
    \end{equation}
    where we used the assumption $\gamma_+ \geq n^{-K}$ in the last inequality. Therefore we directly obtain the first of the three inequalities
    \begin{equation}\label{eq:bound_diff_lambda_*}
    \left|  \frac{\sfs_n(\lambda_1)}{\sfs_n(\lambda_0)} - 1 \right| \leq \frac{| \lambda_1 - \lambda_0|}{n \lambda_* (0) (1 - \Upsilon_2 (0))} \leq n^{2(1+K)} | \lambda_1 - \lambda_0|,
    \end{equation}
    for any $\lambda_1,\lambda_0 \geq 0$.

    \noindent
    {\bf Step 2: The Training error.}

We decompose this term into $\sfC_n (\lambda) = \sfC_{1,n} (\lambda) + \sfC_{2,n} (\lambda)$ where
\[
\sfC_{1,n} (\lambda) := \frac{1}{n^2} \frac{\sM (\lambda)}{1 - \Upsilon_2 (\lambda)}, \qquad \quad \sfC_{2,n} (\lambda) := \frac{1}{(n \lambda_* (\lambda))^2} \frac{\sigma_\eps^2}{1 - \Upsilon_2 (\lambda)},  
\]
and we denoted $\sM ( \lambda) := \< \bbeta_* , (\bSigma  + \lambda_*(\lambda))^{-2} \bbeta_* \>$ for convenience. We can decompose the bound into
\begin{equation}\label{eq:C1n_decomposition}
\begin{aligned}
    \left| \frac{\sfC_{1,n} (\lambda_1) }{\sfC_{1,n} (\lambda_0) } - 1\right| \leq &~ \left| \frac{\sM (\lambda_1) }{\sM (\lambda_0) } - 1\right| + \left| \frac{1 - \Upsilon_2 (\lambda_0) }{1 - \Upsilon_2 (\lambda_1)} - 1\right| + \left| \frac{\sM (\lambda_1) }{\sM (\lambda_0) } - 1\right|\left| \frac{1 - \Upsilon_2 (\lambda_0) }{1 - \Upsilon_2 (\lambda_1)} - 1\right|.
\end{aligned}
\end{equation}
Assume that $| \lambda_1 - \lambda_0| \leq 1$. Then a straightforward computation gives
\[
\left| \frac{\sM (\lambda_1) }{\sM (\lambda_0) } - 1\right| \leq \left| \frac{\lambda_*(\lambda_1) - \lambda_* (\lambda_0) }{\lambda_*(\lambda_1)} \right| \left(4 + \left|\frac{\lambda_*(\lambda_1) - \lambda_* (\lambda_0) }{\lambda_*(\lambda_1)} \right| \right) \leq n^{4(1+K)} | \lambda_1 - \lambda_0|,
\]
where we used Eq.~\eqref{eq:bound_diff_lambda_*} in the last inequality.

Similarly,
\begin{equation}\label{eq:denominator_bound}
\begin{aligned}
\left| \frac{\Upsilon_2 (\lambda_1) - \Upsilon_2 (\lambda_0) }{1 - \Upsilon_2(\lambda_1)}\right| \leq &~ \frac{\Upsilon_2(\lambda_1)}{1 - \Upsilon_2(\lambda_1)} \left| \frac{\lambda_*(\lambda_1) - \lambda_* (\lambda_0) }{\lambda_*(\lambda_0)} \right|\left(4 + \left|\frac{\lambda_*(\lambda_1) - \lambda_* (\lambda_0) }{\lambda_*(\lambda_1)} \right| \right)\\
\leq&~ \frac{n\lambda_*(0)\Upsilon_2 (0)}{n\lambda_*(0)(1-\Upsilon_2(0))} n^{CK} | \lambda_1 - \lambda_0|\\
\leq&~ n^{CK} | \lambda_1 - \lambda_0|,
\end{aligned}
\end{equation}
where in the last inequality we used Eq.~\eqref{eq:Upsilon_2_tech_lower} and $\lambda_* (0) \Tr( \bSigma^2 ( \bSigma +\lambda_* (0))^{-2} ) \leq \Tr( \bSigma ( \bSigma +\lambda_* (0))^{-1}) = n$. Combining these two bounds in Eq.~\eqref{eq:C1n_decomposition}, we obtain
\[
\left| \frac{\sfC_{1,n} (\lambda_1) }{\sfC_{1,n} (\lambda_0) } - 1\right| \leq n^{CK} | \lambda_1 - \lambda_0|.
\]
The bound on the variance part proceeds similarly
\[
 \left| \frac{\sfC_{2,n} (\lambda_1) }{\sfC_{2,n} (\lambda_0) } - 1\right| = \left| \frac{\lambda_* (\lambda_0)^2}{\lambda_* (\lambda_1)^2} \cdot \frac{1 - \Upsilon_2 (\lambda_0)}{1 - \Upsilon_2 (\lambda_1)}  -1\right| \leq n^{CK} | \lambda_1 - \lambda_0|.
\]
Finally, observe that
\begin{equation}\label{eq:sfC_12}
\left| \frac{\sfC_n (\lambda_1)}{\sfC_n (\lambda_0)} - 1\right| = \left| \frac{\sfC_{1,n} (\lambda_1)+\sfC_{2,n} (\lambda_1)}{\sfC_{1,n} (\lambda_0)+\sfC_{2,n} (\lambda_0)} - 1\right| \leq \left| \frac{\sfC_{1,n} (\lambda_1) }{\sfC_{1,n} (\lambda_0) } - 1\right| + \left| \frac{\sfC_{2,n} (\lambda_1) }{\sfC_{2,n} (\lambda_0) } - 1\right| \leq n^{CK} | \lambda_1 - \lambda_0|.
\end{equation}

\noindent
{\bf Step 3: The test error.}

We proceed similarly by considering the bias and variance separately with $\sR_{n} (\lambda) = \sR_{1,n} (\lambda) + \sR_{2,n} (\lambda)$ where
\[
\sR_{1,n} (\lambda) := \lambda_* (\lambda)^2 \frac{\sM (\lambda)}{1 - \Upsilon_2 (\lambda)} , \qquad \quad \sR_{2,n} (\lambda) := \frac{\sigma_\eps^2}{1 - \Upsilon_2 (\lambda)}.
\]
Using the previous bounds and a decomposition similar to Eq.~\eqref{eq:C1n_decomposition}, we directly get
\[
\begin{aligned}
    \left| \frac{\sR_{1,n} (\lambda_1) }{\sR_{1,n} (\lambda_0) } - 1\right| = &~ \left| \frac{\lambda_* (\lambda_1)^2}{\lambda_* (\lambda_0)^2} \cdot \frac{1 - \Upsilon_2 (\lambda_0)}{1 - \Upsilon_2 (\lambda_1)} \cdot \frac{\sM (\lambda_1)}{\sM (\lambda_0)} - 1\right| \leq n^{CK} | \lambda_1 - \lambda_0|,\\
    \left| \frac{\sR_{2,n} (\lambda_1) }{\sR_{2,n} (\lambda_0) } - 1\right| \leq&~ n^{CK} | \lambda_1 - \lambda_0|.
\end{aligned}
\]
The same argument as Eq.~\eqref{eq:sfC_12} concldues the proof.
\end{proof}

We next show similar bounds directly on the Stietljes transform, training error, and test error.
\begin{lemma}\label{lem:Lipschitz_bound_functionals}
Under the assumptions of Theorem \ref{thm:abstract_GCV}, there exists a universal constant $C>0$ such that the following holds. On the event of Assumption \ref{ass:main_assumptions}.(b) with $\lambda =0$, which happens with probability at least $1 - p_{2,n} (\evn)$, we have for any $0\leq \lambda_0,\lambda_1 \leq n^K$ with $| \lambda_1 - \lambda_0| \leq 1$, 
    \[
    \max\left\{   \left| \frac{s_n(\lambda_1)}{s_n(\lambda_0)} - 1\right|,  \left| \frac{\cC_n(\lambda_1)}{\cC_n(\lambda_0)} - 1 \right|  \right\} \leq n^{CK} | \lambda_1 - \lambda_0|
    \]
    \[
    \left| \frac{\cR_n(\lambda_1)}{\cR_n(\lambda_0)} - 1 \right| \leq n^{CK} | \lambda_1 - \lambda_0 | \left\{ \frac{(n\lambda_* (\lambda_0))^2 \cC_n (\lambda_0)}{\cR_n (\lambda_0)} + \sqrt{\frac{(n\lambda_* (\lambda_0))^2 \cC_n (\lambda_0)}{\cR_n (\lambda_0)}}\right\}.
    \]
\end{lemma}

\begin{proof}[Proof of Lemma \ref{lem:Lipschitz_bound_functionals}]
    Again, we will not be careful while showing these bounds. Note that on the event of Assumption \ref{ass:main_assumptions}.(b), we have for any $\lambda\geq 0$ and using the assumption of our theorem
    \begin{equation}\label{eq:G_op_norm}
    \| \bG \|_\op \leq \frac{2}{\gamma_+} \leq 2 n^{K}.
    \end{equation}
    
From Eq.~\eqref{eq:G_op_norm}, we can directly write
\[
\begin{aligned}
    \left| \frac{s_n(\lambda_1)}{s_n(\lambda_0)}  -1 \right| = \frac{| \Tr( \bG(\lambda_1) - \bG(\lambda_0))|}{\Tr(\bG(\lambda_0))} = |\lambda_1-\lambda_0|\frac{ \Tr(\bG(\lambda_1)  \bG(\lambda_0)) }{\Tr(\bG(\lambda_0))} \leq 2n^K | \lambda_1 - \lambda_0|.
\end{aligned}
\]
Similarly for the training error,
\[
\begin{aligned}
    \left| \frac{\cC_n(\lambda_1)}{\cC_n(\lambda_0)}  -1 \right| =&~ \frac{| \by^\sT \left\{ \bG(\lambda_1)^2 - \bG(\lambda_0)^2 \right\}\by |}{\by^\sT\bG(\lambda_0)^2 \by} \\
    =&~ \frac{| \by^\sT \bG(\lambda_0) \left\{ -2 (\lambda_1 - \lambda_0)\bG(\lambda_1) + (\lambda_1 - \lambda_0)^2\bG(\lambda_1)^2   \right\}\bG(\lambda_0)\by |}{\by^\sT\bG(\lambda_0)^2 \by}\\
    \leq&~ n^{CK} | \lambda_1 - \lambda_0|.
\end{aligned}
\]

Finally for the test error, we follow the same decomposition as in Step 4 of the proof of Theorem \ref{thm:test_general} to get
\[
\begin{aligned}
    &~\left| \frac{\cR_n(\lambda_1)}{\cR_n(\lambda_0)}  -1 \right| \\
    =&~ \frac{| \| \btheta_* - \bX^\sT \bG(\lambda_1) \by \|_{\bSigma}^2 -  \| \btheta_* - \bX^\sT \bG(\lambda_0) \by \|_{\bSigma}^2|}{\cR_n(\lambda_0)} \\
    \leq&~ \frac{\| \bX^\sT ( \bG(\lambda_1) - \bG(\lambda_0))\by \|_{\bSigma}^2}{\cR_n(\lambda_0)} + 2 \frac{\| \btheta_* - \bX^\sT \bG(\lambda_0) \by \|_{\bSigma}\| \bX^\sT ( \bG(\lambda_1) - \bG(\lambda_0))\by \|_{\bSigma}}{\cR_n(\lambda_0)} \\
    \leq&~ | \lambda_1 - \lambda_0|^2 \| \bG (\lambda_1) \bX \bSigma \bX^\sT \bG (\lambda_1)\|_\op \frac{n \cC_n (\lambda_0)}{\cR_n (\lambda_0)} +2 | \lambda_1 - \lambda_0| \| \bG (\lambda_1) \bX \bSigma \bX^\sT \bG (\lambda_1)\|_\op^{1/2} \sqrt{\frac{n \cC_n (\lambda_0)}{\cR_n (\lambda_0)}}.
\end{aligned}
\]
Using that
\[
\frac{1}{(n\lambda_* (\lambda_0))^2}\| \bG (\lambda_1) \bX \bSigma \bX^\sT \bG (\lambda_1)\|_\op \leq \frac{4}{\gamma_+^2} \left\{ \nu_{0_{+}} (n) + \frac{n\xi_+}{\gamma_+} \right\} \leq n^{CK},
\]
we obtain
\[
\left| \frac{\cR_n(\lambda_1)}{\cR_n(\lambda_0)}  -1 \right| \leq n^{CK} | \lambda_1 - \lambda_0 | \left\{ \frac{(n\lambda_* (\lambda_0))^2 \cC_n (\lambda_0)}{\cR_n (\lambda_0)} + \sqrt{\frac{(n\lambda_* (\lambda_0))^2 \cC_n (\lambda_0)}{\cR_n (\lambda_0)}}\right\}.
\]
This concludes the proof of this lemma.
\end{proof}

Using the above two lemmas and an union bound, we can obtain the following approximatioin guarantees uniformly over $\lambda \in [0,\lambda_{\max}]$.

\begin{lemma}\label{lem:uniform_cv_functionals}
   for any $K,D >0$, under the setting and the assumptions of Theorem \ref{thm:abstract_GCV}, we have with probability at least $1 - n^{-D} - p_{2,n} (\evn)$,
   \[
   \begin{aligned}
   \sup_{\lambda \in [0,\lambda_{\max}]} \left| \frac{s_n (\lambda)}{\sfs_n(\lambda)} - 1\right| \leq&~ C_{x,\eps,D,K} \cdot \cE_{\sfs,n,0} (\evn), \\
   \sup_{\lambda \in [0,\lambda_{\max}]} \left| \frac{\cC_n (\lambda)}{\sfC_n(\lambda)} - 1\right| \leq&~ C_{x,\eps,D,K} \cdot \cE_{\sfC,n,0} (\evn), \\
   \sup_{\lambda \in [0,\lambda_{\max}]} \left| \frac{\cR_n (\lambda)}{\sR_n(\lambda)} - 1\right| \leq&~ C_{x,\eps,D,K} \cdot \cE_{\sR,n,0} (\evn),
   \end{aligned}
   \]
   where we denoted
   \[
   \begin{aligned}
       \cE_{\sfs,n,0} (\evn) :=&~   \frac{\varphi_1(\evn) \nu_{0,\evn}(n)^{7/2} \log^{\beta +1/2}(n)}{\sqrt{n}}    +  \varphi_{2,n} (\evn) \sqrt{\frac{n\xi_{\evn+1}}{\gamma_+}}
       , \\
       \cE_{\sfC,n,0} (\evn) :=&~  \frac{\varphi_1(\evn) \nu_{0,{\evn}}(n)^{8} \log^{3\beta +1/2}(n)}{\sqrt{n}}    + \nu_{0,{\evn}}(n) \varphi_{2,n} (\evn)\sqrt{\frac{n \xi_{ \evn + 1}}{\Tr(\bSigma_{>\evn})}} , \\
       \cE_{\sR,n,0} (\evn) := &~\frac{\varphi_1(\evn) \nu_{0,\evn}(n)^{6} \log^{3\beta +1/2}(n)}{\sqrt{n}}    + \nu_{0,\evn}(n) \varphi_{2,n} (\evn)\sqrt{\frac{n \xi_{ \evn + 1}}{\gamma_+}} .
   \end{aligned}
   \]
\end{lemma}

\begin{proof}[Proof of Lemma \ref{lem:uniform_cv_functionals}]
    We consider a $n^{-P}$-grid $\cP_n$ of the segment $[0,\lambda_{\max}]$ (i.e., at most $n^{PK}$ points). Noting that we can always work on the event of Assumption \ref{ass:main_assumptions}.(b), we can use a union bound over the $O(n^{PK})$ points in $\cP_n$ by reparametrizing $D' = D + PK$. Applying Theorems \ref{thm:stieltjes_general}, \ref{thm:training_general} and \ref{thm:test_general}, and observing that the decay rate are maximized at $\lambda =0$, we obtain with probability at least $1 - n^{-D}$ that
\[
   \begin{aligned}
   \sup_{\lambda \in \cP_n} \left| \frac{s_n (\lambda)}{\sfs_n(\lambda)} - 1\right| \leq&~ C_{x,\eps,D,K} \cdot \cE_{\sfs,n,0}, \\
   \sup_{\lambda \in \cP_n} \left| \frac{\cC_n (\lambda)}{\sfC_n(\lambda)} - 1\right| \leq&~ C_{x,\eps,D,K} \cdot \cE_{\sfC,n,0}, \\
   \sup_{\lambda \in \cP_n} \left| \frac{\cR_n (\lambda)}{\sR_n(\lambda)} - 1\right| \leq&~ C_{x,\eps,D,K} \cdot \cE_{\sR,n,0}.
   \end{aligned}
   \]
   We can now use Lemmas \ref{lem:Lipschitz_bound_det_equiv} and \ref{lem:Lipschitz_bound_functionals} to show the approximation guarantees on $\lambda \not\in \cP_n$.

   Denote $\lambda_0$ the point in $\cP_n$ closest to $\lambda$. We get
   \[
   \begin{aligned}
       | s_n (\lambda) - \sfs_n (\lambda)| \leq&~ | s_n (\lambda) - s_n (\lambda_0)|  + | s_n (\lambda_0) - \sfs_n (\lambda_0)|  + | \sfs_n (\lambda_0) - \sfs_n (\lambda)|  \\
       \leq&~ n^{CK -P }  s_n (\lambda_0) + C_{x,\eps,K,D} \cdot \cE_{\sfs,n,0} \sfs_n (\lambda_0) +  n^{CK -P }  \sfs_n (\lambda) \\
       \leq&~ \left[ n^{CK-P}(1 + C_{x,\eps,K,D} \cdot \cE_{\sfs,n,0} ) (1 + n^{CK -P } )+C_{x,\eps,K,D} \cdot \cE_{\sfs,n,0}  \right]\sfs_n (\lambda).
   \end{aligned}
   \]
   Similarly,
   \[
   \begin{aligned}
       | \cC_n (\lambda) - \sfC_n (\lambda)| \leq&~ \left[ n^{CK-P}(1 + C_{x,\eps,K,D} \cdot \cE_{\sfC,n,0} ) (1 + n^{CK -P } )+C_{x,\eps,K,D} \cdot \cE_{\sfC,n,0}  \right]\sfC_n (\lambda).
   \end{aligned}
   \]

   Finally, decompose
   \[
   \begin{aligned}
       &~ | \cR_n (\lambda) - \sR_n (\lambda)|\\
       \leq&~ | \cR_n (\lambda) - \cR_n (\lambda_0)|  + | \cR_n (\lambda_0) - \sR_n (\lambda_0)|  + | \sR_n (\lambda_0) - \sR_n (\lambda)| \\
       \leq&~ n^{CK - P} \left\{ (n\lambda_* (\lambda_0))^2 \cC_n (\lambda_0) + \sqrt{(n\lambda_* (\lambda_0))^2 \cC_n (\lambda_0)\cR_n (\lambda_0)}\right\} + (C_{x,\eps,K,D} \cdot \cE_{\sR,n,0} + n^{CK-P}) \sR_n (\lambda).
   \end{aligned}
   \]
   Observe that 
   \[
   \begin{aligned}
   (n\lambda_* (\lambda_0))^2 \leq&~ (1 + n^{CK-P}) (n\lambda_* (\lambda))^2,\\
   \cC_n (\lambda_0) \leq&~ (1 + C_{x,\eps,K,D} \cdot \cE_{\sfC,n,0}) (1+ n^{CK-P}) \sfC_n (\lambda), \\
   \cR_n (\lambda_0) \leq&~(1 + C_{x,\eps,K,D} \cdot \cE_{\sR,n,0}) (1+ n^{CK-P}) \sR_n (\lambda),
   \end{aligned}
   \]
   and $(n\lambda_* (\lambda))^2 \sfC_n (\lambda) = \sR_n (\lambda)$. Recalling the conditions of Theorem \ref{thm:abstract_GCV} and taking $P >CK$ concludes the proof of this lemma.
\end{proof}

We are now ready to prove our theorem.

\begin{proof}[Proof of Theorem \ref{thm:abstract_GCV}]
We will work on the event in Lemma \ref{lem:uniform_cv_functionals} with probability at least $1 - n^{-D} - p_{2,n}(\evn)$.
    First, for any $\lambda \in [0,\lambda_{\max}]$, we have
    \[
    \begin{aligned}
    \left| \frac{\cC_n (\lambda)}{s_n(\lambda)^2} - \frac{\sfC_n(\lambda)}{\sfs_n(\lambda)^2} \right| \leq &~ \frac{|\cC_n (\lambda) -  \sfC_n(\lambda)|}{\sfs_n(\lambda)^2} + \frac{\cC_n(\lambda)}{s_n(\lambda)^2} \frac{| s_n(\lambda)^2 - \sfs_n(\lambda)^2|}{\sfs_n(\lambda)^2} \\
    \leq &~C_{x,\eps,D,K} \cdot \cE_{\sfC,n,0} \frac{\sfC_n}{\sfs_n(\lambda)^2} + C_{x,\eps,D,K} \cdot \cE_{\sfs,n,0} \left[2 +  C_{x,\eps,D,K} \cdot \cE_{\sfs,n,0} \right] \frac{\cC_n(\lambda)}{s_n(\lambda)^2}.
    \end{aligned}
    \]
    Rearranging the terms and using condition \eqref{eq:condition_GCV_abstract1}, we obtain
    \[
    \left| \frac{\cC_n (\lambda)}{s_n(\lambda)^2} - \frac{\sfC_n(\lambda)}{\sfs_n(\lambda)^2} \right| \leq C_{x,\eps,D,K} \cdot \cE_{\sfC,n,0} \frac{\sfC_n}{\sfs_n(\lambda)^2}.
    \]
    Similarly,
    \[
    \begin{aligned}
        \left| \frac{\widehat{GCV}(\lambda)}{\cR_n (\lambda)} - \frac{\sR_n (\lambda)}{\sR_n (\lambda)}\right| \leq &~ \frac{| \widehat{GCV}(\lambda) - \sR_n (\lambda) | }{\sR_n (\lambda)} + \frac{\widehat{GCV}(\lambda)}{\cR_n (\lambda)} \cdot \frac{| \cR_n (\lambda) - \sR_n (\lambda)|}{\sR_n (\lambda)}\\
        \leq&~ C_{x,\eps,D,K} \cdot \cE_{\sfC,n,0} + C_{x,\eps,D,K} \cdot \cE_{\sR,n,0} (2 + C_{x,\eps,D,K} \cE_{\sR,n,0} )  \frac{\widehat{GCV}(\lambda)}{\cR_n (\lambda)}.
    \end{aligned}
    \]
    Rearranging the terms and using condition \eqref{eq:condition_GCV_abstract1} concludes the proof of this theorem.
\end{proof}

\clearpage

\section{Proofs for the sufficient conditions and examples}

In this appendix, we provide the proofs for the sufficient conditions from Section \ref{sec:kernel_eigendecomposition} and the examples in Section \ref{sec_main:examples}.

\subsection{Proof of Proposition \ref{prop:sufficient_high_degree}}
\label{app:sufficient}

\begin{proof}[Proof of Proposition \ref{prop:sufficient_high_degree}]
    The proof follows similarly to the proof of \cite[Proposition 4]{mei2022generalization}. For the sake of concision, we only outline the main steps. Let us introduce the following notations:
    \[
    K_{\evn:\oevn} (\bu_i , \bu_j ) := \sum_{k=\evn+1}^{\oevn} \xi_k \psi_k (\bu_i) \psi_k (\bu_j),\qquad K_{>\oevn} (\bu_i , \bu_j ) := \sum_{k=\oevn+1}^{\infty} \xi_k \psi_k (\bu_i) \psi_k (\bu_j),
    \]
    and $\ddiag ( \bK_{>\evn} ) := \diag ( \{ K_{>\evn} (\bu_i,\bu_i)\}_{i \in [n]})$,
    \[
    \begin{aligned}
    \bDelta_{\evn:\oevn} := ( K_{\evn:\oevn} (\bu_i , \bu_j ) \delta_{i\neq j} )_{ij \in [n]},\qquad\quad
    \bDelta_{>\oevn} := ( K_{>\oevn} (\bu_i , \bu_j ) \delta_{i\neq j} )_{ij \in [n]},
    \end{aligned}
    \]
    so that $\bK_{>\evn} = \ddiag ( \bK_{>\evn} ) + \bDelta_{\evn:\oevn} + \bDelta_{>\oevn}$. Hence, we can decompose our bound into
    \[
    \begin{aligned}
    &~\E \left[ \left\| \bK_{>\evn} - \Tr( \Kop_{>\evn}) \cdot \id_n \right\|_\op \right] \\
    \leq&~ \E \left[ \left\| \ddiag ( \bK_{>\evn} ) - \Tr( \Kop_{>\evn}) \cdot \id_n \right\|_\op \right] + \E \left[ \left\|  \bDelta_{\evn:\oevn}  \right\|_\op \right] + \E \left[ \left\| \bDelta_{>\oevn}\right\|_\op \right].
    \end{aligned}
    \]
    By Assumption \ref{ass:sufficient_high_degree}.(b), the first term is readily bounded by
    \begin{equation}\label{eq:suff_diag}
    \begin{aligned}
    \E \left[ \left\| \ddiag ( \bK_{>\evn} ) - \Tr( \Kop_{>\evn}) \cdot \id_n \right\|_\op \right] \leq&~ \E\left[ \max_{i \in [n]} \left| K_{>\evn} (\bu_i,\bu_i) - \Tr( \Kop_{>\evn}) \right| \right] \\
    \leq&~ \alpha_1 \sqrt{\frac{n\xi_{\evn+1}}{\Tr(\Kop_{>\evn})}}  \cdot \Tr( \Kop_{>\evn}).
    \end{aligned}
    \end{equation}
    For the third term, we get
    \begin{equation}\label{eq:suff_high}
    \begin{aligned}
    \E \left[ \left\| \bDelta_{>\oevn}\right\|_\op \right] \leq \E \left[ \left\| \bDelta_{>\oevn}\right\|_F^2 \right]^{1/2} \leq&~ n \E [ K_{>\oevn} (\bu_i,\bu_j)^2 ]^{1/2} \\
    \leq&~ n\sqrt{\Tr( \Kop_{>\oevn}^2)} \leq \sqrt{\frac{n\xi_{\evn+1}}{\Tr(\Kop_{>\evn})}}  \cdot \Tr( \Kop_{>\evn}),
    \end{aligned}
    \end{equation}
    where in the last inequality, we used that $n \cdot \Tr( \Kop_{>\oevn}^2) \leq n \xi_{\oevn +1} \Tr( \Kop_{>\oevn}) \leq \xi_{\evn +1} \Tr( \Kop_{>\evn})$ by assumption on $\oevn$.

    For convenience, we denote $\bDelta := \bDelta_{\evn:\oevn}$ from now on. Recall the following standard matrix decoupling argument (see for example \cite[Lemma 4]{mei2022generalization}):
    \begin{equation}\label{eq:decoupling}
    \E \left[ \| \bDelta \|_\op \right] \leq 4 \max_{T \subseteq [n]} \E \left[ \| \bDelta_{T,T^c} \|_\op \right],
    \end{equation}
    where we denoted $\bDelta_{T,T^c} = (\Delta_{ij} )_{i \in T , j \in T^c}$. Note that conditional on $\{ \bu_j\}_{j \in T^c}$, the rows of $\bDelta_{T,T^c}$ are i.i.d. with $\bDelta_{i,T^c} = (K_{\evn:\oevn} (\bu_i,\bu_j) )_{j \in T^c}$. Denote $\E_{T}$ the expectation over $\{ \bu_i\}_{i \in T}$ conditional on $\{ \bu_j\}_{j \in T^c}$. Then, by Proposition \ref{prop:Vershynin} stated below, we have
    \begin{equation}\label{eq:T_bound}
    \E_T \left[ \|\bDelta_{T,T^c}  \|_\op \right] \leq \left\{ \| \bSigma [T] \|_\op |T| \right\}^{1/2} + C \cdot \left\{ \Gamma [T] \cdot \log (|T| \wedge |T^c| ) \right\}^{1/2}, 
    \end{equation}
    where we denoted
    \[
    \bSigma [T] := \E_{\bu_i} \left[ \bDelta_{i,T^c} \bDelta_{i,T^c}^\sT \right] , \qquad \quad \Gamma [T] := \E_T \left[ \max_{i \in [n]} \| \bDelta_{T,T^c} \|_2^2\right].
    \]
    Injecting the bound~\eqref{eq:T_bound} into Eq.~\eqref{eq:decoupling}, we obtain by Jensen's inequality
    \begin{equation}\label{eq:delta_decompo_T}
    \begin{aligned}
    \E \left[ \| \bDelta \|_\op \right] \leq&~ 4 \max_{T \subseteq [n]} \E_{T^c} \E_T \left[ \| \bDelta_{T,T^c} \|_\op \right] \\
    \leq&~ 4 \max_{T \subseteq [n]}  \left[ \left\{ \E_{T^c} \left[ \| \bSigma [T] \|_\op \right] n \right\}^{1/2} + C \cdot \left\{\E_{T^c} \left[  \Gamma [T] \right] \cdot \log (n) \right\}^{1/2} \right].
    \end{aligned}
    \end{equation}
    Following the same argument as in the proof of \cite[Proposition 5]{mei2022generalization}, we first get for any $T \subseteq [n]$
    \[
    \E_{T^c} \left[ \| \bSigma [T] \|_\op \right] \leq \xi_{\evn+1} \cdot \left\{ \E \left[ \| \ddiag ( \bK_{\evn:\oevn} ) \|_\op\right] + \E \left[ \| \bDelta \|_\op \right]\right\}.
    \]
    By Assumption \ref{ass:sufficient_high_degree}.(a) and \cite[Lemma 6]{mei2022generalization}, we have for any integer $p \geq 1$,
    \[
    \begin{aligned}
    \E \left[ \| \ddiag ( \bK_{\evn:\oevn} ) \|_\op\right] \leq \E \left[ \left( \sum_{i \in [n]}  K_{\evn:\oevn} (\bu_i,\bu_i) \right)^p\right]^{1/p} \leq&~ n^{1/p} \cdot  \E \left[K_{\evn:\oevn} (\bu_i,\bu_i)^p\right]^{1/p} \\
    \leq&~ n^{1/p} \cdot (2 \oC p)^{\obeta}  \cdot \E \left[K_{\evn:\oevn} (\bu_i,\bu_i)\right] \\
    =&~ n^{1/p} \cdot (2\oC p)^{\obeta} \cdot  \Tr( \Kop_{\evn:\oevn}).
    \end{aligned}
    \]
    We deduce that 
    \begin{equation}\label{eq:sigma_T_bound}
        \E_{T^c} \left[ \| \bSigma [T] \|_\op \right] \leq \xi_{\evn+1} \cdot n^{1/p} \cdot (2\oC p)^{\obeta} \cdot  \Tr( \Kop_{\evn:\oevn}) + \xi_{\evn+1} \cdot  \E \left[ \| \bDelta \|_\op \right].
    \end{equation}

    Similarly, following the proof of \cite[Proposition 5]{mei2022generalization}, for any $T \subseteq [n]$,
    \begin{equation}\label{eq:gamma_T_bound}
    \begin{aligned}
        \E_{T^c} \left[  \Gamma [T] \right]  \leq n \cdot \E\left[ \max_{i \in T, j\in T^c } \Delta_{ij}^2 \right] \leq&~ n^{1 + 2/p} \cdot \E \left[ \Delta_{ij}^{2p} \right]^{1/p} \\
        \leq&~ n^{1 + 2/p} \cdot  (2 \oC p)^{2\obeta} \cdot \Tr( \Kop_{\evn:\oevn}^2) \\
        \leq&~  n^{1 + 2/p} \cdot  (2 \oC p)^{2\obeta} \cdot \xi_{\evn+1} \Tr( \Kop_{\evn:\oevn}).
    \end{aligned}
    \end{equation}

    Combining Eqs.~\eqref{eq:sigma_T_bound} and \eqref{eq:gamma_T_bound} into Eq.~\eqref{eq:delta_decompo_T}, we obtain
    \[
    \E \left[ \| \bDelta \|_\op \right] \leq C \left\{ n \xi_{\evn+1} \cdot  \E \left[ \| \bDelta \|_\op \right] \right\}^{1/2} + C \left\{ n^{1 + 2/p} \cdot  (2 \oC p)^{2\obeta}\cdot \xi_{\evn+1} \Tr( \Kop_{\evn:\oevn}) \right\}^{1/2}.
    \]
    Rearranging these terms yields
    \[
    \E \left[ \| \bDelta \|_\op \right] \leq C n^{1/p} \cdot (2 \oC p)^{\obeta} \sqrt{\frac{n\xi_{\evn+1}}{\Tr( \Kop_{>\evn})}}\Tr( \Kop_{>\evn}). 
    \]
    In particular, setting $p = \log(n)$, we obtain
    \[
     \E \left[ \| \bDelta \|_\op \right] \leq C  \cdot (2 \oC \log(n))^{\obeta} \sqrt{\frac{n\xi_{\evn+1}}{\Tr( \Kop_{>\evn})}}\Tr( \Kop_{>\evn}). 
    \]
    Proposition \ref{prop:sufficient_high_degree} follows by combining this bound with Eqs.~\eqref{eq:suff_diag} and \eqref{eq:suff_high}.
\end{proof}

\begin{proposition}[{\cite[Theorem 5.48]{vershynin2010introduction}}]\label{prop:Vershynin} Let $\bA \in \R^{n \times k}$ with $\bA = [ \ba_1 , \ldots , \ba_n ]^\sT$ where $\ba_i \in \R^k$ are independent random vectors with common second moment matrix $\bSigma = \E[\ba_i \ba_i^\sT ]$. Let $\Gamma = \E[\max_{i\in[n]} \| \ba_i \|_2^2 ]$. Then there exists a universal constant $C$ such that
\[
\E \left[ \| \bA \|_\op^2\right]^{1/2} \leq \left( \| \bSigma \|_\op \cdot n\right)^{1/2} + C \cdot \left( \Gamma \cdot \log (n \wedge k ) \right)^{1/2}.
\]   
\end{proposition}

\subsection{Example 1: concentrated features}
\label{app:proof_concentrated_high_deg}

Assumption \ref{ass:concentrated} implies the abstract Assumptions \ref{ass:main_assumptions}.(a) and \ref{ass:main_assumptions}.(c) for any $\evn \in \naturals \cup \{ \infty\}$. Indeed, Eq.~\eqref{eq:ass_concentrated_features} applied to $\bA = \btheta_+ \btheta_+^\sT$ implies the tail bound of Assumption \ref{ass:main_assumptions}.(c) (with a slight change of constants $\sfc_x,\sfC_x$). We show in this section Proposition \ref{prop:high_freq_concentration_well_concentrated} which applied to $\bx_{>\evn}$ with covariance $\bSigma_{>\evn}$ and any $D>0$ implies Assumption \ref{ass:main_assumptions}.(b) with 
\[
p_{2,n} (\evn)  = n^{-D}, \qquad \quad \varphi_{2,n} (\evn) = C_{x,D} \log^{2\beta +1} (n).
\]

Below, we consider general features $\bx_i$ satisfying the concentration Assumption \ref{ass:concentrated} with some $\sfc_x,\sfC_x,\beta$. The following proposition shows that the kernel matrix $\bX \bX^\sT$ concentrates on $\Tr (\bSigma) \cdot \id_n$ as long as $n\| \bSigma \|_\op \ll \Tr( \bSigma)$. Note that in the case of $\beta = 1$ (the `sub-Gaussian' case), this result can be proven using a union bound over a $\delta$-net of the unit sphere $\S^{n-1}$ by adapting the proof of \cite[Theorem 5.58]{vershynin2010introduction}. However, this argument does not extend to $\beta > 1$. Here, we use instead a moment method to bound the operator norm of the off diagonal elements of $\bX\bX^\sT$, which works for any $\beta>0$.

\begin{proposition}\label{prop:high_freq_concentration_well_concentrated}
    Let $\bx_1,\bx_2, \ldots , \bx_n$ be $n$ feature vectors satisfying Assumption \ref{ass:concentrated}. Let 
    \[
    \bK = ( \< \bx_i,\bx_j \>)_{ij \in [n]} \in \R^{n \times n}
    \]
    denote the Gram matrix of these feature vectors. Denote $\gamma = \Tr(\bSigma)$ and $\xi_1 = \| \bSigma \|_\op$, and assume that $n\xi_1 / \gamma \leq 1/2$. Then for any $D>0$, there exists a constant $C_{x,D}$ that only depends on $\sfc_x,\sfC_x,\beta,D$, such that with probability at least $1 - n^{-D}$
    \[
    \left\| \bK - \gamma \id_n \right\|_\op \leq C_{x,D} \log^{2\beta +1} (n) \gamma  \sqrt{\frac{n\xi_1}{\gamma}}.
    \]
\end{proposition}

\begin{proof}[Proof of Proposition \ref{prop:high_freq_concentration_well_concentrated}]
For convenience, denote $\bK = \bX \bX^\sT$ and $\gamma = \Tr (\bSigma)$. Recall that we denote $\xi_1 = \| \bSigma \|_\op$. Let us define $\ddiag (\bK)\in \R^{n \times n} $ the diagonal entries of $\bK$, and $\bDelta = \bK  - \ddiag (\bK) \in \R^{n \times n}$ the matrix of the off-diagonal entries of $\bK$, i.e.,
\[
\Delta_{ij} = \begin{cases}
    0 & \text{ if $i=j$},\\
    \< \bx_i , \bx_j \> & \text{ if $i\neq j$}.
\end{cases}
\]
We decompose the bound on the operator norm into
\[
\| \bK - \gamma \id \|_\op \leq \| \ddiag (\bK)  - \gamma \id \|_\op + \| \bDelta \|_\op.
\]
First note that via union bound, there exists a constant $C_{x,D}$ such that with probability at least $1 - n^{-D}$,
\begin{equation}\label{eq:diagonal_operator_bound}
\| \ddiag (\bK)  - \gamma \id \|_\op = \max_{i \in [n]} \left| \| \bx_i \|_2^2 - \gamma \right|  \leq C_{x,D} \cdot \log^{\beta} (n)\sqrt{\frac{\xi_1}{\gamma}} \cdot \gamma,
\end{equation}
where we used that $\Tr(\bSigma^2) \leq \xi_1 \gamma$ and that by Assumption \ref{ass:concentrated},
\[
\P \left( \left| \| \bx_i \|_2^2 - \gamma \right| > t \cdot \sqrt{\Tr(\bSigma^2)} \right) \leq \sfC_x \exp \left\{ - \sfc_x t^{1/\beta} \right\}.
\]

For the off-diagonal term $\bDelta$, we bound the operator norm using a moment method. For any integer $p$, we have
\[
\P \left( \| \bDelta \|_\op \geq t \right) \leq t^{-2p} \E \left[ \| \bDelta \|^{2p} \right] \leq t^{-2p} \E \left[ \Tr( \bDelta^{2p} )\right].
\]
The moments are bounded in Lemma \ref{lem:bound_product_high_freq} and \ref{lem:bound_moment_high_freq} stated below. In particular, for $n\xi_1/\gamma \leq 1/2$, we have
\[
\E \left[ \Tr( \bDelta^{2p} )\right] \leq n\left(\frac{n \xi_1}{\gamma}\right)^p (C_x p )^{(4\beta +2)p} \gamma^{2p}.
\]
Hence, for any $D>0$, there exists a constant $C_D := C(D)$ such that by taking $p = \log (n)$ and
$t = C_D (C_x \log(n) )^{2\beta +1} \gamma \sqrt{n\xi_1/\gamma}$, 
the bound in probability becomes
\[
\P \left( \| \bDelta \|_\op \geq C_{x,D} \log^{2\beta +1} (n) \gamma  \sqrt{\frac{n\xi_1}{\gamma}} \right) \leq n^{-D}.
\]
Combining this bound with Eq.~\eqref{eq:diagonal_operator_bound} concludes the proof.
\end{proof}

The following two lemmas bound the moments of the off-diagonal matrix. The proof generalizes the moment method used to bound the operator norm of Gegenbauer matrices developed in \cite[Proposition 3]{ghorbani2021linearized} and \cite[Lemma 16]{lu2022equivalence}.

\begin{lemma}\label{lem:bound_product_high_freq}
    There exists a constant $C_x$ that only depends on $\sfc_x,\sfC_x,\beta$, such that the following hold. Let $\bx_1, \bx_2, \ldots, \bx_n$ be $n$ random feature vectors satisfying Assumption \ref{ass:concentrated} and denote $K_{ij} = \<\bx_i , \bx_j\>$. For any integer $q \geq 1$, consider a sequence of indices $\bi = (i_1 , \ldots , i_q) \in [n]^q$ and let $\chi (\bi)$ be the number of distinct indices in $\bi$. Let $M(\bi)$ be the product expectation defined by
    \begin{equation}\label{eq:product_Mi_def_lem}
    M (\bi) =  \E \left[ K_{i_1 i_2} K_{i_2 i_3} \cdots K_{i_{2p-1} i_{2p}} K_{i_{2p} i_1} \right] . 
    \end{equation}
Then, we have the inequality
\begin{equation}\label{eq:lem_bound_Mi}
| M (\bi) | \leq \left( \frac{\xi_1}{\gamma}\right)^{\chi(\bi) - 1}(C_x q)^{2\beta q} \gamma^q,
\end{equation}
where we recall that we denote $\gamma = \Tr(\bSigma)$ and $\xi_1 = \| \bSigma \|_\op$.
\end{lemma}

\begin{proof}[Proof of Lemma \ref{lem:bound_product_high_freq}]
    Consider $\bi = (i_1,i_2, \ldots , i_q) \in [n]^q$ a general sequence of $q$ indices in $[n]$. We will use the cyclic convention $i_{q+1} = i_1$. As a warm-up, consider the following two extreme cases: (i) the case $\chi (\bi) =1 $, that is $i_1 = i_2 = \ldots = i_{q}$, then 
    \[
    M(\bi) = \E \left[ K_{i_1 i_1}^q \right] \leq (C_x q)^{\beta q} \Tr( \bSigma)^{q} = (C_x q)^{\beta q} \gamma^q,
    \]
    where we used Lemma \ref{lem:tight_bound_on_powers} with $\bA = \id$. And (ii) the case $\chi (\bi) = | \bi| = q$, that is all the indices in $\bi$ are distinct, then
    \[
    M (\bi) = \E \left[ \Tr( \bx_{i_1} \bx_{i_1}^\sT \bx_{i_2} \bx_{i_2}^\sT \cdots \bx_{i_{q}} \bx_{i_{q}}^\sT )\right] = \Tr \left( \E \left[ \bx_{i_1} \bx_{i_1}^\sT \right]^q \right) = \Tr( \bSigma^q) \leq \left( \frac{\xi_1}{\gamma} \right)^{q-1} \gamma^q.
    \]
  More generally, observe that if an index $i_s$ only appears once in the sequence $\bi$, then we can remove this index by substituting it by $\bSigma$, i.e.,
    \begin{equation}\label{eq:ex_procedure_A}
    \E \left[ \Tr( \bx_{i_1} \bx_{i_1}^\sT  \cdots \bx_{i_{q}} \bx_{i_{q}}^\sT )\right] = \E \left[ \Tr( \bx_{i_1} \bx_{i_1}^\sT \cdots \bx_{i_{s-1}} \bx_{i_{s-1}}^\sT \bSigma \bx_{i_{s+1}} \bx_{i_{s+1}}^\sT \cdots \bx_{i_{q}} \bx_{i_{q}}^\sT )\right].
    \end{equation}
    This procedure can be repeated and at each step, we reduce the length of the sequence $\bi$ by $1$ and add a factor $\bSigma$ at that position. On the other hand, if $i_s = i_{s+1}$ (and let's say with repeated the previous procedure to make a factor $\bSigma^{r_s}$ between $\bx_{i_s} \bx_{i_s}^\sT$ and $\bx_{i_{s+1}} \bx_{i_{s+1}}^\sT$), then
    \begin{equation}\label{eq:ex_procedure_B}
    \begin{aligned}
        &~\E \left[ \Tr( \bx_{i_1} \bx_{i_1}^\sT \cdots \bx_{i_{s}} \bx_{i_{s}}^\sT \bSigma^{r_s} \bx_{i_{s+1}} \bx_{i_{s+1}}^\sT \cdots \bx_{i_{q}} \bx_{i_{q}}^\sT )\right]\\
        = &~ \E \left[ \Tr( \bx_{i_1} \bx_{i_1}^\sT \cdots \bx_{i_{s-1}} \bx_{i_{s-1}}^\sT \bD^{(r_s)}_{i_s} \bx_{i_{s+2}} \bx_{i_{s+2}}^\sT \cdots \bx_{i_{q}} \bx_{i_{q}}^\sT )\right] ,
    \end{aligned}
    \end{equation}
    where we denoted $\bD^{(r_s)}_{i_s} = \| \bSigma^{r_s/2} \bx_{i_s} \|_2^2 \bx_{i_s} \bx_{i_s}^\sT$. Hence, this second procedure replaces indices $i_{s} = i_{s+1} $ by a single index $i_s$ by substituting $\bx_{i_{s}} \bx_{i_{s}}^\sT \bSigma^{r_s} \bx_{i_{s+1}} \bx_{i_{s+1}}^\sT = \bD^{(r_s)}_{i_s}$.

    The two procedures Eqs.~\eqref{eq:ex_procedure_A} and \eqref{eq:ex_procedure_B} can be applied iteratively to a general sequence $\bi$ until it cannot be simplified further. At each step, we reduce the length of the sequence $\bi$ by exactly one and therefore this procedure stops after a finite number of steps.

    Let's define formally this reduction procedure. Following \cite{lu2022equivalence}, we call the first reduction type, a \textit{Type-A reduction}, and the second, a \textit{Type-B reduction}. We introduce the following notations. Denote $\cT_3$ the set of all (finite-sized) ternary trees, i.e., trees $\sfT$ where the root has either no child node (the graph is a single node) or three child nodes that are themselves ternary trees. We will use the notation $\sfT = (\sfT_1,\sfT_2,\sfT_3)$ to denote the three ternary tree sub-graphs (their roots are the child nodes of the root of $\sfT$).  For all $i \in [n]$ and $\sfT \in \cT_3$, we will define the matrices $\bD^{(\sfT)}_i$ and $\obD^{(\sfT)}$ recursively as follows. For $\sfT$ a single node (we denote $\sfT  = `\sflf'$ in that case), we set
    \begin{equation}\label{eq:leaf_D_oD}
    \bD_i^{(\sflf)} = \bx_i \bx_i^\sT, \qquad \quad \obD^{(\sflf)} = \id.
    \end{equation}
    Then we define recursively the matrices for $\sfT = (\sfT_1 , \sfT_2 , \sfT_3)$,
    \begin{align}
     \bD_i^{(\sfT)} =&~ \bD_i^{(\sfT_1)} \obD^{(\sfT_2)} \bD_i^{(\sfT_3)},\label{eq:moment_recursion_D_i}\\
     \obD^{(\sfT)} = &~ \obD^{(\sfT_1)} \E\left[ \bD_i^{(\sfT_2)}\right]\obD^{(\sfT_3)}.\label{eq:moment_recursion_oD}
    \end{align}
    For example, we have
    \[
    \begin{aligned}
        \obD^{(\sfT)}=&~ \bSigma &&\text{for }\sfT = ( \sflf,\sflf,\sflf),\\
         \bD^{(\sfT)}_i=&~ \| \bx_i\|_2^2 \cdot \bx_i \bx_i^\sT &&\text{for }\sfT = ( \sflf,\sflf,\sflf), \\
         \obD^{(\sfT)}=&~ \E \left[\| \bx\|_2^2 \cdot \bx \bx^\sT \right] &&\text{for }\sfT = ( \sflf,( \sflf,\sflf,\sflf),\sflf),\\
         \bD^{(\sfT)}_i=&~ \| \bSigma^{1/2} \bx_i\|_2^2 \cdot \bx_i \bx_i^\sT &&\text{for }\sfT = ( \sflf,( \sflf,\sflf,\sflf),\sflf).
    \end{aligned}
    \]
    More generally, observe that by construction, for any $\sfT \in \cT_3$, there exist trees $\sfG_1, \ldots , \sfG_\ell \in \sT_3$ such that
    \[
    \bD_i^{(\sfT)} = \left( \prod_{j \in [\ell] } \bx_i^\sT \obD^{(\sfG_j)} \bx_i \right) \bx_i \bx_i^\sT .
    \]

    A `state' is defined by three sequences $\bs = (\bi,\bsfT,\bsfG)$ of same length $q:= q (\bs) \geq 1$, with $\bi = (i_1, i_2, \ldots , i_q) \in [n]^q$, $\bsfT = (\sfT_1 , \ldots , \sfT_q) \in (\cT_3)^q$, and $\bsfG = (\sfG_1 , \ldots , \sfG_q) \in (\cT_3)^q$. With a slight abuse of notation, we denote the product associated to state $\bs = (\bi,\bsfT,\bsfG)$ by
    \[
    M (\bs) = \E\left[ \Tr \left( \bD_{i_1}^{(\sfT_1)} \obD^{(\sfG_1)} \bD_{i_2}^{(\sfT_2)} \obD^{(\sfG_2)} \cdots \bD_{i_q}^{(\sfT_q)} \obD^{(\sfG_q)} \right)\right].
    \]
    We are now ready to describe our two types of reduction steps.

    \paragraph*{Type-A reduction:} Assume that after $t$ steps, we reach state $\bs^t$. Then Type-A reduction produces a state $\bs^{t+1}$ with length $q (\bs^{t+1}) = q (\bs^t) - 1$ as follows. For readability, we drop below the superscript $t$ from $\bs := \bs^t$ and $\bs' := \bs^{t+1}$, and denote $q := q (\bs)$. Suppose there exists an index $i_\ell$ in $\bi$ that only appears once. Then following the same idea as in Eq.~\eqref{eq:ex_procedure_A}, we can take the expectation over $\bx_{i_\ell}$ inside the trace and replace 
    \[
    \begin{aligned}
    M (\bs) = &~ \E\left[ \Tr \left( \bD_{i_1}^{(\sfT_1)} \obD^{(\sfG_1)} \cdots   \obD^{(\sfG_{\ell -1 })}  \E \left[ \bD_{i_\ell}^{(\sfT_{\ell})}\right] \obD^{(\sfG_{\ell})} \cdots \bD_{i_q}^{(\sfT_q)} \obD^{(\sfG_q)} \right)\right]\\
    = &~ \E\left[ \Tr \left( \bD_{i_1}^{(\sfT_1)} \obD^{(\sfG_1)} \cdots  \bD_{i_{\ell-1}}^{(\sfT_{\ell-1})} \obD^{(\sfG'_{\ell -1})}  \bD_{i_{\ell+1}}^{(\sfT_{\ell+1})} \cdots \bD_{i_q}^{(\sfT_q)} \obD^{(\sfG_q)} \right)\right] = M(\bs'),
    \end{aligned}
    \]
    where we used the recursion relation \eqref{eq:moment_recursion_oD} in the second equality and $\sfG_{\ell-1}' = ( \sfG_{\ell -1}, \sfT_{\ell}, \sfG_{\ell})$. The new state $\bs'$ is given by
    \[
    \begin{aligned}
    \bi' =&~ (i_1 ,\ldots,i_{\ell-1},i_{\ell+1}, \ldots , i_q ), \\
    \bsfT ' =&~ (\sfT_1 , \ldots , \sfT_{\ell-1},\sfT_{\ell +1}, \ldots, \sfT_q), \\
    \bsfG'=&~ (\sfG_1 , \ldots , \sfG_{\ell-1}',\sfG_{\ell +1}, \ldots, \sfG_q).
    \end{aligned}
    \]
    In summary, a type-A reduction step removes an index $i_\ell$ that only appears once in $\bi$, and merges the trees $\sfG_{\ell -1}, \sfT_{\ell}, \sfG_{\ell}$ into a new ternary tree $\sfG_{\ell-1}'$.
    
    \paragraph*{Type-B reduction:} Again, this reductions produces a new state $\bs'$ with $q(\bs') = q(\bs) - 1$ as follows. Suppose there exists $\ell \in [q]$ such that $i_\ell = i_{\ell-1} $ (recall that we are using the cyclic convention $i_{0} = i_q$). Then following the same idea as in Eq.~\eqref{eq:ex_procedure_B}, we can merge the matrices associated to $i_\ell$ and $i_{\ell-1}$ and write
    \[
    \begin{aligned}
    M (\bs) = &~ \E\left[ \Tr \left( \bD_{i_1}^{(\sfT_1)} \obD^{(\sfG_1)} \cdots   \bD_{i_{\ell-1}}^{(\sfT_{\ell-1})}  \obD^{(\sfG_{\ell-1})} \bD_{i_{\ell-1}}^{(\sfT_{\ell})} \cdots \bD_{i_q}^{(\sfT_q)} \obD^{(\sfG_q)} \right)\right]\\
    = &~ \E\left[ \Tr \left( \bD_{i_1}^{(\sfT_1)} \obD^{(\sfG_1)} \cdots  \obD^{(\sfG_{\ell -2})}  \bD_{i_{\ell-1}}^{(\sfT_{\ell-1}')} \obD^{(\sfG_{\ell })}   \cdots \bD_{i_q}^{(\sfT_q)} \obD^{(\sfG_q)} \right)\right] = M(\bs'),
    \end{aligned}
    \]
    where we used the recursion relation \eqref{eq:moment_recursion_D_i} in the second equality and $\sfT_{\ell - 1}' = ( \sfT_{\ell -1}, \sfG_{\ell - 1}, \sfT_{\ell})$. The new state $\bs'$ is given by
    \[
    \begin{aligned}
    \bi' =&~ (i_1 ,\ldots,i_{\ell-1},i_{\ell+1}, \ldots , i_q ), \\
    \bsfT ' =&~ (\sfT_1 , \ldots , \sfT_{\ell-1}',\sfT_{\ell +1}, \ldots, \sfT_q), \\
    \bsfG'=&~ (\sfG_1 , \ldots , \sfG_{\ell-2},\sfG_{\ell}, \ldots, \sfG_q).
    \end{aligned}
    \]
    In summary, a type-B reduction step removes an index $i_\ell$ such that $i_{\ell-1} = i_\ell$, and merges the trees $ \sfT_{\ell -1}, \sfG_{\ell - 1}, \sfT_{\ell}$ into a new ternary tree $\sfT'_{\ell-1}$.
    
\paragraph*{Reduction procedure:} Given a sequence of indices $\bi = (i_1 , i_2 , i_3 , \ldots , i_q) \in [n]^q$, observe that, by definition \eqref{eq:leaf_D_oD}, we can rewrite
\[
M (\bi) = \E \left[ \Tr( \bx_{i_1} \bx_{i_1}^\sT \bx_{i_2} \bx_{i_2}^\sT \cdots \bx_{i_{q}} \bx_{i_{q}}^\sT )\right] = M (\bs),
\]
where we set $\bs = (\bi,\bsfT,\bsfG)$ with $\sfT_s = `\sflf'$ and $\sfG_s = `\sflf'$ for all $s \in [q]$. We can reduce the length of the sequence by iteratively applying type-A and type-B reductions. By construction, the procedure ends when there are no indices that either appear alone in the sequence or are repeated. 

\begin{example}
    Here we illustrate the procedure on the example $\bi = (1,2,1,3)$. For convenience, denote $\sfH = (\sflf,\sflf,\sflf)$. We obtain the following sequence of states
    \[
    \begin{bmatrix}
    \bs \\
    \bsfT \\
    \bsfG 
    \end{bmatrix} = 
    \begin{bmatrix}
    1 &2 & 1 & 3\\
    \sflf &\sflf &\sflf &\sflf   \\
    \sflf &\sflf &\sflf &\sflf
    \end{bmatrix} \to 
    \begin{bmatrix}
    1  & 1 & 3\\
    \sflf &\sflf &\sflf   \\
     \sfH &\sflf &\sflf
    \end{bmatrix} \to
    \begin{bmatrix}
    1  & 1 \\
    \sflf &\sflf   \\
     \sfH &\sfH
    \end{bmatrix} \to 
    \begin{bmatrix}
    1  \\
    (\sflf, \sfH , \sflf )  \\
     \sfH 
    \end{bmatrix},
    \]
    where we use type-A reduction to remove $`2$', type-A reduction to remove $`3$', and finally type-B reduction to merge the two $`1$' indices. The expectation simplifies to
    \[
    M (\bi) = \E \left[ \Tr \left(  \| \bSigma^{1/2} \bx_1 \|_2^2 \cdot \bx_1 \bx_1^\sT \bSigma\right)\right] = \E \left[ \| \bSigma^{1/2} \bx \|_2^4 \right].
    \]
\end{example}

 The above reduction procedure has a simple visual interpretation: we start from $2q$ vertices numbered $\{1,2,3,\ldots,2q\}$, which correspond to the starting roots. The procedure iteratively constructs a family of non-overlapping ternary trees, with leafs among these $2q$ vertices, that at each step connects three neighboring roots to a new root. 

\paragraph*{Bound on the operator norms:} Let us bound the operator norm of the matrices $\bD^{(\sfT)}_i$ and $\obD^{(\sfT)}$ for trees $\sfT$ obtained through this procedure. Denote $\bD^{(\sfT)}_i = \tbD^{(\sfT)}_i \cdot \bx_i\bx_i^\sT$, where we recall that $\tbD^{(\sfT)}_i$ is a positive scalar with $\tbD^{(\sflf)}_i = 1$ and more generally
\[
\tbD^{(\sfT)}_i = \prod_{j \in [\ell] } \bx_i^\sT \obD^{(\sfG_j)} \bx_i,
\]
for some ternary subtrees $\sfG_j$ of $\sfT$. We further introduce the quantity 
\[
\nu (k) = \left\| \E \left[ \| \bx_i \|_2^{2k} \bx_i \bx_i^\sT \right] \right\|_\op.
\]
Consider $\sfT \in \cT_3$ obtained with the above procedure, and denote $N_A(\sfT)$ and $N_B(\sfT)$ the total number of type-A and type-B reduction steps. Denote $N(\sfT)$ the total number of nodes that are not leafs in $\sfT$. It is immediate that $N_A (\sfT) + N_B (\sfT) = N(\sfT)$. 

We prove by recurrence that for any $\sfT$, there exist $N_A(\sfT)$ nonnegative integers $k_{c_1(\sfT)}, \ldots , k_{c_{N_A}(\sfT)}$ such that $k_{c_1(\sfT)} + k_{c_2(\sfT)} + \cdots + k_{c_{N_A}(\sfT)} = N_B (\sfT)$ and 
\begin{equation}\label{eq:decompo_bound_obD_nu}
\| \obD^{(\sfT)} \|_\op \leq \nu ( k_{c_1(\sfT)}) \nu (k_{c_2(\sfT)}) \cdots \nu (k_{c_{N_A}(\sfT)}).
\end{equation}
Similarly, there exist $N_A(\sfT) + 1$ integers $k_{c_1(\sfT)}, \ldots , k_{c_{N_A +1}(\sfT)}$ such that $k_{c_1(\sfT)}  + \cdots + k_{c_{N_A+1} (\sfT)} = N_B (\sfT)$ and
\begin{equation}\label{eq:decompo_bound_bD_i_nu}
\left\| \E [ \bD^{(\sfT)}_i ] \right\|_\op \leq \nu ( k_{c_1(\sfT)}) \nu (k_{c_2(\sfT)}) \cdots \nu (k_{c_{N_A+1}(\sfT)}),
\end{equation}
and in particular,
\begin{equation}\label{eq:decompo_bound_tbD_i_nu}
\tbD_i^{(\sfT)} \leq \nu ( k_{c_1(\sfT)}) \nu (k_{c_2(\sfT)}) \cdots \nu (k_{c_{N_A}(\sfT)})  \| \bx_i \|_2^{2 k_{c_{N_A+1} (\sfT)}}.
\end{equation}
Indeed, we have 
\[
\| \obD^{(\sflf)} \|_\op = 1 , \qquad \quad \left\| \E [ \bD^{(\sflf)}_i ] \right\|_\op = \nu (0), \qquad \quad \tbD_i^{(\sflf)} = 1.
\]
More generally, consider $\sfT = (\sfT_1 , \sfT_2 , \sfT_3)$. For $\obD^{(\sfT)}$, the last step is a type-A reduction and $N_A (\sfT_1) + N_A (\sfT_2)+ N_A(\sfT_3) = N_A(\sfT) - 1$ and $N_B(\sfT_1 ) + N_B(\sfT_2) + N_B(\sfT_3) = N_B(\sfT)$. Assume that the bounds \eqref{eq:decompo_bound_obD_nu}, \eqref{eq:decompo_bound_obD_nu} and \eqref{eq:decompo_bound_tbD_i_nu} hold for the trees $\sfT_1,\sfT_2,\sfT_3$. Then using the recursion \eqref{eq:moment_recursion_oD}, we get
\[
\begin{aligned}
    \left\| \obD^{(\sfT)} \right\|_\op \leq&~ \left\| \obD^{(\sfT_1)} \right\|_\op  \left\| \E [ \bD^{(\sfT_2)}_i ] \right\|_\op  \left\| \obD^{(\sfT_3)} \right\|_\op  \\
    \leq&~ \prod_{i \in [N_A (\sfT_1)]} \nu ( k_{c_i(\sfT_1)}) \cdot \prod_{i \in [N_A (\sfT_2) + 1]} \nu ( k_{c_i(\sfT_2)}) \cdot \prod_{i \in [N_A (\sfT_3)]} \nu ( k_{c_i(\sfT_3)})\\
    =: &~ \prod_{i \in [N_A(\sfT)]} \nu (k_{c_i (\sfT)} ),
\end{aligned}
\]
where $k_{c_1 (\sfT)} + \ldots + k_{c_{N_A} (\sfT)} = N_B(\sfT_1) + N_{B} (\sfT_2) + N_{B} (\sfT_3)$.

Similarly, consider $\bD_i^{(\sfT)}$. The last step is a type-B reduction. Hence, $N_A (\sfT_1) + N_A (\sfT_2)+ N_A(\sfT_3) = N_A(\sfT)$ and $N_B(\sfT_1 ) + N_B(\sfT_2) + N_B(\sfT_3) = N_B(\sfT) - 1$. Then using the recursion \eqref{eq:moment_recursion_D_i}, we get
\[
\begin{aligned}
    \tbD_i^{(\sfT)} = &~ \tbD_i^{(\sfT_1)} \cdot \bx_i^\sT  \obD^{(\sfT_2)} \bx_i \cdot \tbD_i^{(\sfT_3)} \\
    \leq&~ \| \obD^{(\sfT_2)} \|_\op \tbD_i^{(\sfT_1)}  \tbD_i^{(\sfT_3)} \| \bx_i \|_2^2 \\
    \leq&~ \prod_{i \in [N_A(\sfT_1)]} \nu (k_{c_i (\sfT_1)} )  \prod_{i \in [N_A(\sfT_2)]} \nu (k_{c_i (\sfT_2)} )  \prod_{i \in [N_A(\sfT_3)]} \nu (k_{c_i (\sfT_3)} ) \cdot \| \bx \|_2^{2(1 + k_{c_{N_A+1} (\sfT_1)} + k_{c_{N_A+1} (\sfT_3)} )} \\
    =: &~\prod_{i \in [N_A(\sfT)]} \nu (k_{c_i (\sfT)} ) \cdot \| \bx \|_2^{2 k_{c_{N_A+1} (\sfT)} },
\end{aligned}
\]
where we defined $k_{c_{N_A+1} (\sfT)} =1 + k_{c_{N_A+1} (\sfT_1)} + k_{c_{N_A+1} (\sfT_3)}$. It is immediate to check that indeed $k_{c_1(\sfT)}  + \cdots + k_{c_{N_A+1} (\sfT)} = N_B (\sfT)$. Finally, using this previous bound,
\[
\begin{aligned}
    \left\| \E [ \bD^{(\sfT)}_i ] \right\|_\op =&~ \left\| \E [ \tbD^{(\sfT)}_i \cdot \bx_i \bx_i^\sT ] \right\|_\op \\
    \leq&~ \prod_{i \in [N_A(\sfT)]} \nu (k_{c_i (\sfT)} ) \left\| \E \left[ \| \bx_i \|_2^{2k_{c_{N_A+1} (\sfT)}} \bx_i \bx_i^\sT \right] \right\|_\op = \prod_{i \in [N_A(\sfT)+1]} \nu (k_{c_i (\sfT)} ).
\end{aligned}
\]

The quantity $\nu (k)$ can be bounded using Lemma \ref{lem:tight_bound_on_powers} as follows
\begin{equation}\label{eq:bound_nu_k}
\begin{aligned}
    \nu (k) = &~ \sup_{\| \bu \|_2 = 1} \E \left[ \| \bx_i \|_2^{2k} \< \bu, \bx_i \>^2 \right] \\
    \leq&~ \E \left[ \| \bx_i \|_2^{4k} \right]^{1/2} \sup_{\| \bu \|_2 = 1} \E \left[ \< \bu, \bx_i \>^4 \right]^{1/2}  \\
    \leq&~ (2C_x k)^{\beta k} \Tr(\bSigma)^k \sup_{\| \bu \|_2 = 1} (2C_x)^{\beta} \bu^\sT \bSigma \bu \\
    \leq&~ (C_x k)^{\beta k} \cdot \xi_1 \gamma^k ,
\end{aligned}
\end{equation}
where we used Cauchy-Schwarz inequality on the second line and inequality \eqref{eq:tight_xAx} with $q = 2k$ and $\bA = \id$ and with $q = 2$ and $\bA = \bu \bu^\sT$ on the third line.

\paragraph*{Final state of the procedure:}
We are now ready to bound the product \eqref{eq:product_Mi_def_lem}. Given a sequence of indices $\bi \in [n]^q$, we denote $\bs_f (\bi) = (\obi,\obsfT,\obsfG)$ the state obtained at the end of the reduction procedure. With a slight abuse of notations, we denote $N_A(\bi)$ and $N_B (\bi)$ the number of type-A and type-B steps. Recall that $\chi (\bi)$ denotes the number of distinct indices in the sequence $\bi$ and $q(\bi)$ the length of the sequence. By construction, observe that we must have
\begin{equation}\label{eq:relation_final_state}
N_A (\bi) + N_B(\bi) = q (\bi) - q(\obi), \qquad \quad N_A(\bi) = \chi (\bi) - \chi (\obi),
\end{equation}
meaning that the procedure reduces the length of the sequence by $N_A (\bi) + N_B(\bi)$, and each type-A step reduces the number of distinct indices by one.

Examining the reduction procedure, the final state $\obi$ cannot have indices that either appear alone or are repeated. Thus, there can only be two types of end states: either $q(\obi) = 1$, i.e., the sequence is of size $1$, or $\chi (\obi) \geq 2$ and each index appears at least twice in the sequence and there are no repetitions. We call the sequences $\bi$ such that $q (\obi) = 1$, type-$1$ sequences, and others type-$2$ sequences.

\paragraph{Bound for type-$1$ sequences:} In that case, the end state is given by $\bs_f = (1,\osfT_1,\osfG_1)$. In particular, $N_A(\osfT_1) + N_A(\osfG_1) = N_A (\bi)$ and $N_B (\osfT_1) + N_B (\osfG_1) = N_B (\bi)$. The expectation of the product is simply given by
\[
M (\bi) = M (\bs_f) = \E\left[ \tbD_1^{(\osfT_1)} \cdot \bx_1^\sT \obD^{(\osfG_1)} \bx_1 \right].
\]
Using the decompositions \eqref{eq:decompo_bound_obD_nu} and \eqref{eq:decompo_bound_tbD_i_nu}, and the bound on $\nu(k)$ in Eq.~\eqref{eq:bound_nu_k}, we obtain
\[
\begin{aligned}
    | M (\bi) | \leq &~ \| \obD^{(\osfG_1)} \|_\op \E \left[\tbD_1^{(\osfT_1)} \cdot \| \bx_1 \|_2^2 \right ] \\
    \leq&~ \prod_{i \in [N_A(\osfG_1)]} \nu (k_{c_i (\osfG_1)} ) \cdot \prod_{i \in [N_A(\osfT_1)]} \nu (k_{c_i (\osfT_1)} ) \cdot \E \left[ \| \bx_1 \|_2^{2(1 + k_{c_{N_A+1} (\osfT_1)})} \right ] \\
    \leq&~ \xi_1^{N_A(\osfG_1)} \left\{ (C_x q)^{\beta } \gamma \right\}^{ N_B (\osfG_1)} \times \xi_1^{N_A(\osfT_1)} \left\{ (C_x q)^{\beta} \gamma \right\}^{ N_B (\osfT_1) - k_{c_{N_A+1} (\osfT_1)}} \\
    &~ \times \left\{ (C_xq)^\beta \gamma \right\}^{1 + k_{c_{N_A+1} (\osfT_1)}} \\
    =&~ \xi_1^{N_A (\bi)} (C_xq)^{\beta(N_B (\bi) + 1)} \gamma^{N_B(\bi) + 1}.
\end{aligned}
\]
Recalling the relations \eqref{eq:relation_final_state}, this upper bound becomes
\[
| M (\bi) | \leq \xi_1^{ \chi (\bi) - 1} (C_x q)^{\beta(q - \chi (\bi) + 1)} \gamma^{q - \chi (\bi) + 1} \leq \left(\frac{\xi_1}{\gamma}\right)^{\chi (\bi) - 1} (C_x q)^{\beta q } \gamma^q,
\]
and therefore type-$1$ sequences satisfy the desired inequality \eqref{eq:lem_bound_Mi}.

\paragraph{Bound for type-$2$ sequences:} In that case, the end state is given by $\bs_f = (\obi , \obsfT, \obsfG )$ where $\obi = (\oi_1 , \ldots , \oi_{\oq})$ with $\oi_\ell \neq \oi_{\ell+1}$ for all $\ell \in [\oq]$, and we denoted $\oq = q (\obi)$ for convenience. We have again
\[
\begin{aligned}
    N_A (\osfT_1) + \ldots + N_A (\osfT_{\oq}) + N_A (\osfG_1) + \ldots + N_A (\osfG_{\oq}) =&~ N_A (\bi), \\
        N_B (\osfT_1) + \ldots + N_B (\osfT_{\oq}) + N_B (\osfG_1) + \ldots + N_B (\osfG_{\oq}) =&~ N_B (\bi).
\end{aligned}
\]
The expectation of the product can be rewritten as
\begin{equation}\label{eq:product_type_2}
M (\bi) = M (\bs_f ) = \E \left[ \prod_{j \in [\oq]} \tbD_{\oi_j}^{(\osfT_j)} \cdot \prod_{j \in [\oq]}  \bx_{\oi_j}^\sT \obD^{(\osfG_j)} \bx_{\oi_{j+1}} \right].
\end{equation}
From inequality~\eqref{eq:tight_xAz} in Lemma \ref{lem:tight_bound_on_powers}, we have for any integer $ r \geq 1$,
\[
\begin{aligned}
\E \left[ \left| \bx_{\oi_j}^\sT \obD^{(\osfG_j)} \bx_{\oi_{j+1}} \right|^r \right]^{1/r} \leq&~ (C_x r)^\beta \Tr\Big( \bSigma \obD^{(\osfG_j)} \bSigma \obD^{(\osfG_j)} \Big)^{1/2} \\
\leq&~ (C_x r)^\beta \| \obD^{(\osfG_j)} \|_\op  \xi_1^{1/2} \gamma^{1/2} \\
\leq&~ (C_x r)^\beta \xi_1^{N_A (\osfG_j) + 1/2} (C_x q)^{\beta N_B (\osfG_j)} \gamma^{N_B (\osfG_j) + 1/2}.
\end{aligned}
\]
Similarly, from inequality \eqref{eq:tight_xAx} and assuming $r \leq 2 q$,
\[
\begin{aligned}
\E \left[ \left| \tbD_{\oi_j}^{(\osfT_j)}\right|^r \right]^{1/r} \leq&~ \prod_{i \in [N_A(\osfT_j)]} \nu (k_{c_i (\osfT_j)} ) \E \left[ \| \bx \|_2^{2 r  k_{c_{N_A+1} (\osfT_j)}} \right]^{1/r} \\
\leq&~ \xi_1^{N_A (\osfT_j)} \left\{ ( C_x q)^\beta \gamma\right\}^{N_B (\osfT_j ) - k_{c_{N_A+1} (\osfT_j)}}  \left\{ (C_x r k_{c_{N_A+1} (\osfT_j)} )^{\beta} \gamma \right\}^{ k_{c_{N_A+1} (\osfT_j)}}\\
\leq &~ \xi_1^{N_A (\osfT_j)} (C_x q)^{2\beta N_B (\osfT_j)} \gamma^{N_B (\osfT_j)}.
\end{aligned}
\]
Using these two bounds, we can apply H\"older's inequality to the product \eqref{eq:product_type_2} and obtain
\[
\begin{aligned}
    | M (\bi) | \leq&~ \prod_{j \in [\oq]} \E \left[  \left| \tbD_{\oi_j}^{(\osfT_j)} \right|^{2\oq} \right]^{1/(2\oq)} \prod_{j \in [\oq]} \E \left[ \left|  \bx_{\oi_j}^\sT \obD^{(\osfG_j)} \bx_{\oi_{j+1}} \right|^{2\oq} \right]^{1/(2\oq)}\\
    \leq&~ \prod_{j \in [\oq]}  \xi_1^{N_A (\osfT_j)} (C_x q)^{2\beta N_B (\osfT_j)} \gamma^{N_B (\osfT_j)}  \cdot \prod_{j \in [\oq]}  (C_x \oq)^\beta \xi_1^{N_A (\osfG_j) + 1/2} (C_x q)^{\beta N_B (\osfG_j)} \gamma^{N_B (\osfG_j) + 1/2} \\
    \leq&~ \xi_1^{N_A (\bi) + \oq /2} (C_x q)^{2\beta (N_B(\bi) + \oq)} \gamma^{N_B (\bi) + \oq /2}.
\end{aligned}
\]
Note that for type-2 sequences, we must have $\oq \geq 2 \chi (\obi) $. Thus, recalling the relations \eqref{eq:relation_final_state}, this bound becomes
\[
\begin{aligned}
    | M (\bi) | \leq&~ \xi_1^{\chi (\bi) + \oq /2 - \chi (\obi)} (C_x q)^{2\beta (q - N_A (\bi))} \gamma^{q - (N_A (\bi) + \oq /2)} \leq \left( \frac{\xi_1}{\gamma}\right)^{\chi(\bi)}(C_x q)^{2\beta q} \gamma^q,
\end{aligned}
\]
where we used $\oq \geq 2 \chi (\obi) $ and $\xi_1/\gamma \leq 1$. Therefore type-$2$ sequences satisfy the desired inequality \eqref{eq:lem_bound_Mi} too. Recalling that sequences $\bi$ can only be either type-$1$ or type-$2$, this concludes the proof.
\end{proof}

\begin{lemma}\label{lem:bound_moment_high_freq}
There exists a constant $C_x$ that only depends on $\sfc_x,\sfC_x,\beta$, such that the following holds. Let $\bx_1, \bx_2, \ldots, \bx_n$ be $n$ random feature vectors satisfying Assumption \ref{ass:concentrated} and denote $\bDelta = (\< \bx_i , \bx_j \> \delta_{i\neq j})_{ij \in [n]} \in \R^{n\times n}$ the off-diagonal part of the Gram matrix. Then for any integer $p$, we have
\[
\E [ \Tr( \bDelta^{2p})] \leq n\left(\frac{n \xi_1}{\gamma}\right)^p (C_x p )^{(4\beta +2)p} \gamma^{2p} \left\{ 1 + \sum_{i = 0}^{p-1} \left( \frac{n\xi_1}{\gamma} \right)^i \right\},
\]    
where we recall $\gamma = \Tr(\bSigma)$ and $\xi_1 = \| \bSigma \|_\op$.
\end{lemma}

\begin{proof}[Proof of Lemma \ref{lem:bound_moment_high_freq}]
    We follow a similar proof as \cite[Lemma 17]{lu2022equivalence}. We define the subset of sequences with no repetitions
    \[
    \cC = \{ \bi = ( i_1, i_2, \ldots , i_{2p}) \in [n]^{2p}: \forall j \in [2p], i_{j+1} \neq i_j \},
    \]
    where we use the cyclic convention $i_{2p+1} = i_1$. Observe that we can write the trace as
    \[
    \begin{aligned}
        \E [ \Tr( \bDelta^{2p})] = \sum_{\bi \in \cC} \E \left[ \Delta_{i_1 i_2} \Delta_{i_2 i_3} \ldots \Delta_{i_{2p} i_1}\right] = \sum_{\bi \in \cC} M (\bi),
    \end{aligned}
    \]
    where $M(\bi)$ is studied in Lemma \ref{lem:bound_product_high_freq}. Let $\cC_{k}$ be the subset of sequences $\bi$ in $\cC$ with $\chi (\bi) = k$ distinct indices. From Eq.~\eqref{eq:lem_bound_Mi}, we have for all $\bi \in \cC_k$,
    \begin{equation}\label{eq:C_k_bound_large_k}
        | M (\bi) | \leq \left( \frac{\xi_1}{\gamma}\right)^{k - 1}(C_x p)^{4\beta p} \gamma^{2p}.
    \end{equation}
    On the other hand, for any $\bi \in \cC$, we can simply use H\"older's inequality and directly get 
    \begin{equation}\label{eq:C_k_bound_small_k}
        | M(\bi) | \leq \E\left[ | \Delta_{i_1i_2} |^{2p} \right] \leq (C_x p)^{2\beta p} \left( \frac{\xi_1}{\gamma}\right)^p \gamma^{2p},
    \end{equation}
    where we used that $i_j \neq i_{j+1}$ for all $j\in[2p]$ and Eq.~\eqref{eq:tight_xAz} with $q =2p$ and $\bA = \id$. Thus, we will prefer bound \eqref{eq:C_k_bound_small_k} for $\chi (\bi) = 2, \ldots, p$, and bound \eqref{eq:C_k_bound_large_k} for $\chi (\bi) = p+1, \ldots, 2p$. It is easy to see that the cardinality of $\cC_k$ is bounded by
    \[
    | \cC_k | \leq n^{k} (2p)^{2p}.
    \]
    
    Combining the above bounds, we deduce that
    \[
    \begin{aligned}
    \E [ \Tr( \bDelta^{2p})] \leq&~ \sum_{k=2}^{2p} \sum_{\bi \in \cC_{k}} | M(\bi)| \\
    \leq&~ \sum_{k=2}^p |\cC_k | (C_x p)^{2\beta p} \left( \frac{\xi_1}{\gamma}\right)^p \gamma^{2p} + \sum_{k=p+1}^{2p} | \cC_k | \left( \frac{\xi_1}{\gamma}\right)^{k - 1}(C_x p)^{4\beta p} \gamma^{2p} \\
    \leq&~ n\left(\frac{n \xi_1}{\gamma}\right)^p (C_x p )^{(4\beta +2)p} \gamma^{2p} \left\{ 1 + \sum_{i = 0}^{p-1} \left( \frac{n\xi_1}{\gamma} \right)^i \right\},
    \end{aligned}
    \]
    which concludes the proof.
\end{proof}

\begin{lemma}\label{lem:tight_bound_on_powers}
    Let $\bx_1, \bx_2 \in \R^\evn$, $\evn \in \naturals \cup \{ \infty\}$, be two feature vectors satisfying Assumption \ref{ass:concentrated}. Then there exists a constant $C_{x}\geq 1$ that only depends on $\sfc_x,\sfC_x,\beta$, such that for any positive integer $q \in \naturals$ and p.s.d.~matrix $\bA \in \R^{\evn \times \evn}$ independent of $\bx_1,\bx_2$, we have
    \begin{align}
    \E_{\bx_1} \left[ | \bx_1^\sT \bA \bx_1 |^{q} \right]  \leq &~ (C_x q)^{\beta q}  \Tr(\bSigma \bA)^q, \label{eq:tight_xAx}\\
    \E_{\bx_1,\bx_2} \left[ | \bx_1^\sT \bA \bx_2 |^{q} \right] \leq &~(C_x q)^{\beta q} \Tr\left( \bSigma \bA \bSigma \bA \right)^{q/2}. \label{eq:tight_xAz}
    \end{align}
\end{lemma}

\begin{proof}[Proof of Lemma \ref{lem:tight_bound_on_powers}]
    Similarly to the proof of Lemma \ref{lem:tech_bound_zAz}, we integrate the tail bound in Assumption \ref{ass:concentrated} to get for any $q \in \R_{>0}$
    \[
    \begin{aligned}
        \E_{\bx} \left[ \left| \bx^\sT \bA \bx - \Tr(\bSigma \bA)\right|^q\right] =&~ q \| \bSigma^{1/2} \bA \bSigma^{1/2} \|_F^q \int_0^\infty t^{q-1} \P\left( \left| \bx^\sT \bA \bx - \Tr(\bSigma \bA) \right| \geq t \cdot \| \bSigma^{1/2} \bA \bSigma^{1/2} \|_F \right) \\
        \leq&~ q \sfC_x \| \bSigma^{1/2} \bA \bSigma^{1/2} \|_F^q \int_0^\infty t^{q-1} \exp \left\{ - \sfc_x t^{1/\beta} \right\} \de t \\
        =&~ \frac{q \beta \sfC_x \Gamma (\beta q)}{\sfc_x^{\beta q}} \| \bSigma^{1/2} \bA \bSigma^{1/2} \|_F^q. 
    \end{aligned}
    \]
    Thus, we obtain
    \[
    \begin{aligned}
     \E_{\bx} \left[ | \bx^\sT \bA \bx |^q \right] \leq&~ 2^{q-1} \E_\bx \left[ \left| \bx^\sT \bA \bx - \Tr(\bSigma \bA)\right|^q \right]  + 2^{q-1}  \Tr(\bSigma \bA)^q \\
     \leq&~ \frac{2^{q-1}q \beta \sfC_x \Gamma (\beta q)}{\sfc_x^{\beta q}} \| \bSigma^{1/2} \bA \bSigma^{1/2} \|_F^q + 2^{q-1}  \Tr(\bSigma \bA)^q \\
     \leq&~ 2^q \left\{ \frac{q \beta \sfC_x \Gamma (\beta q)}{\sfc_x^{\beta q}} \vee 1 \right\}  \Tr(\bSigma \bA)^q,
     \end{aligned}
    \]
    where in the last inequality we simply use that $\| \bB\|_F \leq \Tr(\bB)$.
    
    For the second expectation, we simply use this first bound twice, first with respect to $\bx_1$ and then $\bx_2$,
    \[
    \begin{aligned}
        \E_{\bx_1,\bx_2} \left[ | \bx_1^\sT \bA \bx_2 |^q \right] = &~  \E_{\bx_1,\bx_2} \left[ | \bx_1^\sT \bA \bx_2 \bx_2^\sT \bA \bx_1 |^{q/2} \right] \\
        \leq&~ 2^{q/2} \left\{ \frac{ q \beta \sfC_x \Gamma (\beta q/2)}{2\sfc_x^{\beta q/2}} \vee 1 \right\} \E_{\bx_2}\left[ \left| \bx_2^\sT \bA \bSigma \bA \bx_2\right|^{q/2} \right]  \\
        \leq &~ 2^q \left\{ \frac{ q \beta \sfC_x \Gamma (\beta q/2)}{2\sfc_x^{\beta q/2}} \vee 1 \right\}^{2} \Tr\left( \bSigma \bA \bSigma \bA \right)^{q/2} .
    \end{aligned}
    \]
    We conclude by recalling the upper bound on the Gamma function
    \[
    \Gamma (x) \leq C x^{x-1/2} e^{-x} e^{1/(12x)},
    \]
    for some universal constant $C>0$.
\end{proof}

\subsection{Example 2: inner-product kernel on the sphere}\label{app:proof_inner_test}

Below we prove Theorem \ref{thm:spherical_test}. We postpone to Section \ref{app_ex:checking_ass_sphere} the most technical lemmas verifying Assumptions \ref{ass:main_assumptions}.(a) and \ref{ass:main_assumptions}.(b).

Recall the eigendecomposition of the inner-product kernel on the sphere
\[
h (\<\bu , \bu'\>/d) = \sum_{k =0}^\infty \oxi_k \sum_{s \in [B_{d,k}]} Y_{ks} (\bu) Y_{ks} (\bu') . 
\]
Hence we can take our feature vector $\bx = \bSigma^{1/2} \bz$ with
\[
\bz:= \left(1 , Y_{11} (\bu), \ldots , Y_{1B_{d,1}} (\bu), Y_{21} (\bu), \ldots , Y_{2B_{d,2}} (\bu), \ldots \right), \qquad \bSigma := \diag (1, \oxi_1 \id_{B_{d,1}}, \oxi_2 \id_{B_{d,2}}, \ldots).
\]
Note that there exists a constant $C_L>1$ such that $d^k /C_L \leq B_{d,k} \leq d^k C_L$ for all $k =1, \ldots ,L+1$. Hence, using Assumption \ref{ass:inner_product_sphere}.(a), we have
\begin{equation}\label{eq:oxi_k_bound}
\frac{1}{C_L \sfC_L} d^{-k} \leq \oxi_k \leq C_L \sfC_L d^{-k}, \qquad \quad \text{for $k = 0, \ldots ,L$},
\end{equation}
and $\sup_{k\geq L+1} \oxi_k \leq C_L \sfC_L d^{-L-1}$.

 Denote for any integer $p$,
\[
\begin{aligned}
h_{\leq p} ( \< \bu,\bu'\>/d) = &~\sum_{k =0}^p \oxi_k \sum_{s \in [B_{d,k}]} Y_{ks} (\bu) Y_{ks} (\bu'), \\
h_{> p} ( \< \bu,\bu'\>/d) =&~ \sum_{k =p+1}^\infty \oxi_k \sum_{s \in [B_{d,k}]} Y_{ks} (\bu) Y_{ks} (\bu'),
\end{aligned}
\]
the low-frequency part and high-frequency part of the kernel function $h$. Note that
\[
h_{\leq p} (1) = \sum_{k =0}^p \oxi_k B_{d,k}, \qquad \quad h_{> p} ( 1) = \sum_{k =p+1}^\infty \oxi_k B_{d,k}.
\]

Recall that we denote $\ell \geq 1$ the integer closest to $\log(n)/\log(d)$. Take $\evn = B_{d,\leq \ell}$ and $\bz_{\leq \evn} $ containing all spherical harmonics up to degree-$\ell$. Let us check that $\bx$ with Assumption \ref{ass:inner_product_sphere} verify Assumption \ref{ass:main_assumptions} with $\evn$. First, note that
\[
r_\lambda (\evn) = \frac{\lambda + h_{>\ell} (1)}{\xi_{\ell+1}}\geq  \frac{1}{C_L \sfC_L^2 } d^{\ell+1} \geq 2 d^{\ell+\frac12} \geq 2n,
\]
where we used that $h_{>\ell}(1) \geq h_{>L} (1) \geq 1/\sfC_L$ by Assumption \ref{ass:inner_product_sphere}.(b) and Eq.~\eqref{eq:oxi_k_bound} in the first inequality and assumed $d \geq 4 C_L^2 \sfC_L^4$ in the second inequality. 
\begin{itemize}
    \item[(a)] (\textit{Low-degree features.}) From Proposition \ref{prop_app:quad_form_quadratic}, there exist constants $c,C>0$, and $C_\ell >0$ such that for any integer $d >2$ the following hold. For any matrix $\bA \in \R^{\evn \times \evn}$ and vector $\bv \in \R^{\evn}$, we have
\begin{align}
\P \left( \big\vert \< \bv , \bz \> \big\vert \geq t \cdot C_\ell \| \bv \|_2 \right) \leq&~ C^\ell \exp \left\{ -c \ell \cdot t^{2/\ell } \right\}  ,\label{eq:hypercont_ass_sphere} \\
\P \left( \big\vert \bz^\sT \bA \bz - \Tr( \bA) \big\vert \geq t \cdot C_\ell \log^\ell (d)  d^{(\ell-1)/2}  \cdot \| \bA  \|_F \right) \leq&~ C^\ell \exp \left\{ -c \ell \cdot t^{1/\ell } \right\}  .
\end{align}
We can therefore set $\sfC_x := C^\ell$, $\sfc_x := \ell/C_\ell^{\ell}$, $\beta := \ell$, and $\varphi_1 (\evn) := \log^\ell (d) \cdot d^{(\ell-1)/2}$.

\item[(b)] (\textit{High-degree features.}) Applying Proposition \ref{prop:concentration_high_freq_sphere} to kernel function $h_{>\ell}$, for any constant $D>0$, there exists a constant $C_{\ell,D}>0$ such that if $d\geq C_{D,\ell}$ and using that we chose $\ell$ with $n \leq d^{\ell +1/2} \leq B_{d,\ell+1}/2$, we have with probability at least $1 - n^{-D}$,
\[
\| \bX_{>\evn} \bX_{>\evn}^\sT - h_{>\ell} (1) \id \|_\op \leq C_{\ell,D} \log(d)^{2\ell +3} \sqrt{\frac{n}{d^{\ell+1}}} h_{>\ell} (1).
\]
Therefore we can choose $p_{2,n} (\evn) := n^{-D}$ and $\varphi_{2,n} (\evn) := C_{\ell,D} \sfC_L \sqrt{C_L} \log(d)^{2\ell +3}$, where we used Eq.~\eqref{eq:oxi_k_bound} and $h_{>\ell}(1) \leq h(1) \leq \sfC_L$.

\item[(c)] (\textit{Target function.}) By Assumption \ref{ass:inner_product_sphere}.(c) and applying Lemma \ref{lem:tail_bound_hypercontractive} to $f_*$, we first get
\[
\P (|f_*(u)| \geq t \cdot \| f_* \|_{L^2} ) \leq \exp \left\{ L - \frac{L}{2e \sfC_L} t^{2/L} \right\}.
\]
Furthermore, from Eq.~\eqref{eq:hypercont_ass_sphere}, we have
\[
\P ( | \proj_{\leq \ell} f_* (\bu) | \geq t \cdot \| \proj_{\leq \ell} f_* \|_{L^2} ) \leq C^{\ell} \exp\left\{ - c \ell \cdot t^{2/\ell} \right\}.
\]
Note that $| \proj_{>\ell} f_* (\bu)| \leq | f_* (\bu)| + | \proj_{\leq \ell} f_* (\bu)|$. Furthermore, from Assumption \ref{ass:inner_product_sphere}.(c), we have $\| \proj_{\leq \ell} f_* \|_{L^2} \leq \| f_* \|_{L^2} \leq \sfC_L \| \proj_{>\ell} f_* \|_{L^2}$. Thus,
\[
\begin{aligned}
   &~ \P ( | \proj_{> \ell} f_* (\bu) | \geq t \cdot \| \proj_{> \ell} f_* \|_{L^2} ) \\
   \leq&~ \P \left(|f_*(u)| \geq \frac{t}{2\sfC_L} \cdot \| f_* \|_{L^2} \right) + \P \left( | \proj_{\leq \ell} f_* (\bu) | \geq \frac{t}{2\sfC_L} \cdot \| \proj_{\leq \ell} f_* \|_{L^2} \right) \\
   \leq &~ \exp \left\{ L - \frac{L}{e (2\sfC_L)^{1+2/L}} t^{2/L} \right\} + C^\ell \exp \left\{ -  \frac{c \ell}{(2\sfC_L)^{2/\ell}} t^{2/\ell} \right\}.
\end{aligned}
\]
We can therefore set $\sfC_x := C^{L}$, $\sfc_x := L/(e (2\sfC_L)^{1+2/L})$, and $\beta :=L$.
\end{itemize}

Let us now apply Theorem \ref{thm:abstract_Test_error}. Let us bound the effective rank first: for any $k = 0, \ldots , \ell$,
\[
\frac{\sum_{j = k}^\ell \oxi_k B_{d,k}}{\oxi_k} \leq C_L \sfC_L d^{k} h(1)\leq C_L \sfC_L^2 d^{\ell}.
\]
Hence, $r_{\seff} \leq \max( C_L \sfC_L^2 d^{\ell} , n)$. Further note that $\eta $ only depends on $L$. We can now bound $\nu_{\lambda,\evn} (n)$. If $n > \evn /\eta_L$ then $\xi_{\lfloor \eta n \rfloor, \evn} = 0$ and $\nu_{\lambda,\evn} (n) = 1$. Otherwise, for $n \geq d^{\ell -1/2}$, we have $ \eta n \geq \eta  d^{\ell -1/2} \geq 2 B_{d,\leq \ell-1}$ for $d \geq C_L$ for $C_L$ chosen sufficiently large. Hence $\xi_{\lfloor \eta n \rfloor, \evn} = \oxi_{\ell} \leq C_{L} \sfC_L d^{-\ell}$ and
\[
\begin{aligned}
\nu_{\lambda,\evn} (n) = &~ 1 + \frac{\xi_{\lfloor \eta n \rfloor,\evn}\cdot  r_{\seff,\evn} \sqrt{\log (r_{\seff,\evn})}  }{\lambda_{>\evn}} \\
\leq&~ 1 + \frac{C_{L} \sfC_L d^{-\ell} \max( C_L \sfC_L^2 d^{\ell} , \evn /\eta_L)  \sqrt{\log(\max( C_L \sfC_L^2 d^{\ell} , \evn /\eta_L) }) }{h_{>L} (1)}  \\
\leq&~ C_L' \sqrt{\log (d)}.
\end{aligned}
\]
We are now ready to verify the conditions \eqref{eq:condition_test_abstract} in Theorem \ref{thm:abstract_Test_error} with constants $D>0$ and $K=1$. The first inequality is immediately verified by using that $\lambda_{>\evn} \geq h_{>L}(1) \geq 1/\sfC_{L}$ for $n \geq \sfC_{L}$ and $K=1$. For the second condition,
\[
\varphi_{2,n} (\evn) \sqrt{\frac{n}{r_\lambda (\evn)}} \leq C_{\ell,D} \sfC_L \sqrt{C_L} \log(d)^{2\ell +3} d^{-1/4} \leq 1/2,
\]
where we took $d \geq C_{L,D}$ and used that $n \leq d^{\ell +1/2}$. For the third condition,
\[
\frac{\varphi_1 (\evn) \nu_{\lambda,\evn} (n)^8 \log^{3L +1/2} (n) }{n}\leq C_L \log^{4L + \frac{9}{2} } (d) \sqrt{\frac{d^{\ell - 1}}{n}} \leq 1,
\]
where we used that $n \geq d^{\ell - 1/2}$ and took $d \geq C_{L,D}$. Then by Eqs.~\eqref{eq:abstract_test_error} and \eqref{eq:asbstract_relative_error}, we directly get
\[
\left|\cR_{\test} ( f_*;\bU,\beps, \lambda) -    \sR_{n} (f_*, \lambda) \right| \leq C_{L,\eps,D} \cdot \cE_{\sR,n} (\evn) \cdot  \sR_{n} (f_*, \lambda),
\]
where
\[
\begin{aligned}
\cE_{\sR,n}  (\evn) =&~  \frac{\varphi_1(\evn) \nu_{\lambda,{\evn}}(n)^{6} \log^{3\beta +1/2}(n)}{\sqrt{n}}    +  \nu_{\lambda,\evn} (n)  \varphi_{2,n} (\evn)  \sqrt{\frac{n}{r_\lambda (\evn)}} \\
\leq&~  C_{L,D} \log^{3L +\frac{7}{2}} (d) \sqrt{\frac{d^{\ell-1}}{n} } + C_{L,D} \log^{2L +\frac{7}{2}} (d) \sqrt{\frac{n}{d^{\ell+1}} }.
\end{aligned}
\]
Note that for $n \leq d$, we can replace $\log(d)$ by $\log (n)$ in the above. This concludes the proof of this theorem.

\subsection{Verifying the assumptions for inner-product kernels on the sphere}
\label{app_ex:checking_ass_sphere}

In this section, we verify Assumptions \ref{ass:main_assumptions}.(a) and \ref{ass:main_assumptions}.(b) when the feature $\bx$ is from an inner-product kernel on the sphere. We start in Section \ref{app_ex:background_spherical} by recalling basic properties about spherical harmonics polynomial. In Section \ref{app_ex:ass_a_spherical}, we verify Assumption \ref{ass:main_assumptions}.(a) with some technical results defered to Section \ref{app_ex:technical_sphere}. Finally, we verify Assumption \ref{ass:main_assumptions}.(b) in Section \ref{app_ex:ass_b_spherical}.

\subsubsection{Background on spherical harmonics}
\label{app_ex:background_spherical}

In this section, we introduce some notations and briefly overview some properties of spherical harmonics and Gegenbauer polynomials. We refer the reader to \cite{chihara2011introduction,dai2013approximation,ghorbani2021linearized} for additional background on these families of polynomials.

\paragraph*{Function space over the sphere.} We consider $\S^{d-1} (r) = \{ \bx \in \R^d : \| \bu \|_2 = r \}$ the sphere of radius $r$ in $\R^d$. Without loss of generality, we will set the radius $r = \sqrt{d}$ below. We denote $\tau_d$ the uniform probability measure over $\S^{d-1}(\sqrt{d})$ and consider $L^2 (\S^{d-1} (\sqrt{d})) : = L^2 (\S^{d-1} (\sqrt{d}),\tau_d)$ the space of square-integrable functions over $\S^{d-1}(\sqrt{d})$ with respect to $\tau_d$. We denote the scalar product $\< \cdot, \cdot \>_{L^2}$ and associated norm $\| \cdot \|_{L^2}$, where
\[
\< f,g\>_{L^2} := \int_{\S^{d-1} (\sqrt{d})} f(\bu) g(\bu) \,\tau_{d} (\de \bu) .
\]
The function space $L^2 (\S^{d-1} (\sqrt{d}))$ admits the orthogonal decomposition 
\[
L^2(\S^{d-1}(\sqrt{d})) = \bigoplus_{k=0}^{\infty} V_{d,k} ,
\]
where $V_{d,k}$ is the subspace of all degree-$k$ polynomials that are orthogonal (with respect to $\<\cdot , \cdot \>_{L^2}$) to all polynomials with degree less than $k$. We denote $B_{d,k} := \dim (V_{d,k})$ the dimension of each subspace, which are given by $B_{d,0} = 1$, $B_{d,1} = d$, and for $k\geq 2$
\[
B_{d,k} = \frac{d+2k-2}{k}{d+k-3 \choose k-1}.
\]
Note that for each $k$, there exists a constant $C_k$ such that $| B_{d,k} / {d \choose k} - 1  | \leq C_k / d$, and $B_{d,k}$ is equal to leading order in $d$ to ${d \choose k} = d^k/k! \cdot (1+ o_d(1))$. We will further denote
\[
B_{d , \leq k} = \dim \left( \bigoplus_{r=0}^{k} V_{d,r}\right) = \sum_{r=0}^k B_{d,k},
\]
the dimension of the subspace of polynomials of degree at most $k$ in $L^2 (\S^{d-1}(\sqrt{d}))$, which is also of leading order $d^k/k! \cdot (1 + o_d(1))$.

\paragraph*{Spherical harmonics.} For each $k \in \Z_{\geq 0}$, $V_{d,k}$ is the subspace of spherical harmonics of degree-$k$, i.e., homogeneous polynomials satisfying $\Delta q (\bu)= 0$ where $\Delta$ is the Laplace operator. Let $\{ Y^{(d)}_{ks} \}_{s \in [B_{d,k}]}$ be an orthonormal basis of $V_{d,k}$ so that
\[
\< Y_{ks} , Y_{k's'} \>_{L^2} = \int_{S^{d-1} (\sqrt{d})} Y_{ks} (\bu) Y_{k' s'} (\bu) \, \tau_d(\de \bu) = \delta_{kk'}\delta_{ss'}.
\]
The family $\{ Y_{ks} \}_{k \geq0, s\in [B_{d,k}]}$ forms a complete orthonormal basis of $L^2 (\S^{d-1} (\sqrt{d}))$.

\paragraph*{Gegenbauer polynomials.} We consider $\{ Q_{k}^{(d)} \}_{k \geq 0}$ the family of Gegenbauer polynomials on $L^2 ([-1,1], \tau_{d,1} )$, where $\tau_{d,1}$ is the marginal distribution of $\<\bu , \be_1\> / \sqrt{d}$ with $\bu \sim \tau_d$. Specifically, $Q_{k}^{(d)}$ is a degree-$k$ polynomial and
\[
\int_{-1}^1 Q_{k}^{(d)} (u) Q_{k'}^{(d)} (u) \, \tau_{d,1} (\de u) = \int_{\S^{d-1} (\sqrt{d}} Q_{k}^{(d)} (\< \bu,\be_1\>/\sqrt{d}) Q_{k'}^{(d)} (\< \bu,\be_1\>/\sqrt{d}) \, \tau_{d} (\de \bu) = \delta_{kk'}.
\]
Gegenbauer and spherical harmonic polynomials are related through the identity 
\[
Q_k^{(d)} (\<\bu,\bu'\>/d) = \frac{1}{\sqrt{B_{d,k}}}\sum_{s \in [B_{d,k}]} Y_{ks} (\bu) Y_{ks} (\bu'),
\]
for all $\bu,\bu' \in \S^{d-1}(\sqrt{d})$. In particular, this implies the following diagonalization 
\[
h (\<\bu , \bu'\>) = \sum_{k=0}^\infty \frac{\xi_k}{ \sqrt{B_{d,k}}} \sum_{s \in [B_{d,k}]} Y_{ks} (\bu) Y_{ks} (\bu'),
\]
where 
\[
\xi_k = \int_{\S^{d-1} (\sqrt{d})} h( \sqrt{d} \< \bu, \be_1\>) Q^{(d)}_k (\<\bu,\be_1\>/\sqrt{d}) \, \tau_d (\de \bu).
\]

\paragraph*{Explicit representation of spherical harmonics.} We will use an explicit representation of spherical harmonics basis given as product over coordinates of Gegenbauer polynomials in the generalized spherical coordinate system. We refer the reader to \cite{avery2012hyperspherical,dai2013spherical,misiakiewicz2022spectrum} for a proof of the following results.

We first recall the definition of the spherical coordinate system. For $\bu \in \S^{d-1} (\sqrt{d})$, we consider the following change of variable
\begin{equation}\label{eq:spherical_coordinates}
\begin{cases}
u_1 = \sqrt{d} \cdot  \sin ( \theta_{d-1} ) \cdots \sin (\theta_{2}) \sin (\theta_{1}) , \\
u_2 = \sqrt{d} \cdot \sin ( \theta_{d-1} ) \cdots \sin (\theta_{2}) \cos(\theta_{1}) , \\
\cdots \\
u_{d-1} = \sqrt{d} \cdot \sin ( \theta_{d-1} )  \cos (\theta_{d-2}) , \\
u_d = \sqrt{d} \cdot \cos ( \theta_{d-1} ) , \\
\end{cases}
\end{equation}
where $0 \leq \theta_{1} \leq 2\pi$ and $0 \leq \theta_i \leq \pi$ for $i = 2, \ldots, d-1$. Denote $\sigma = (\theta_1,\ldots, \theta_{d-1})$ these coordinates. The uniform probability measure is then given by
\[
\mu (\de \sigma) = \frac{\Gamma (d/2)}{2 \pi^{d/2}} \prod_{j=1}^{d-2} \sin ( \theta_{d-j})^{d-j - 1} \de \theta_{d-1} \cdots \de \theta_{2} \de \theta_{1}  .
\]

For each $k \in \Z_{\geq 0}$, we introduce the set of $d$ indices
\[
\cA_{d,k} = \big\{ \balpha = ( \alpha_1 , \ldots , \alpha_d) \in \naturals^d \big\vert \alpha_1 + \cdots + \alpha_{d-1} = k , \text{if $\alpha_{d-1} > 0$, then $\alpha_d \in \{0,1\}$, o.w. $\alpha_d = 0$} \big\} .
\]
It is easy to verify that
\begin{equation*}
| \cA_{d,k} | = 2 {{d+k - 2}\choose{d-2}} - {{d+k-3}\choose{d-3}} = \frac{2k+d-2}{d-2}  {{d+k-3}\choose{k}} = B_{d,k} .
\end{equation*}

The following proposition from \cite{misiakiewicz2022spectrum} introduce a family of polynomials $\{ Y_{\balpha} \}_{\balpha \in \cA_{d,k}}$ indexed by $\balpha \in \cA_{d,k}$, which forms a complete orthonormal basis of spherical harmonics of $V_{d,k}$.

\begin{proposition}[{\cite[Proposition 2]{misiakiewicz2022spectrum}}]\label{prop:spherical_representation}
For integers $d > 2$ and $k \geq 0$,  and $\balpha \in \cA_{d,k}$, define
\begin{equation}\label{eq:spherical_harmonics_explicit}
Y_{\balpha} (\bu) = C_{\balpha}^{1/2} g_{\balpha} (\theta_{1} ) \prod_{j = 1}^{d-2} \Big\{ Q^{(d_j )}_{\alpha_j} ( \cos (\theta_{d-j})) \sin (\theta_{d-j})^{|\alpha^{j+1}|} \Big\}  ,
\end{equation}
where $| \alpha^{j+1} | = \alpha_{j+1} + \ldots + \alpha_{d-1}$, $d_j = 2 | \alpha^{j+1} | +  d-j + 1$, 
\[
g_{\balpha } (\theta_1) = \begin{cases}
\cos (\alpha_{d-1} \theta_1) & \text{if $\alpha_d = 0$ and $\alpha_{d-1} > 0$,} \\
1 / \sqrt{2} & \text{if $\alpha_d = 0$ and $\alpha_{d-1} = 0$,} \\
\sin (\alpha_{d-1} \theta_1) & \text{if $\alpha_d = 1$ and $\alpha_{d-1} > 0$,}
\end{cases}
\]
and
\[
\begin{aligned}
C_{\balpha} =&~ \frac{2 \pi^{d/2}}{\Gamma(d/2)} \cdot \frac{1}{\pi} \cdot \prod_{j = 1}^{d-2} \frac{\Gamma (d_j/2)}{\sqrt{\pi} \Gamma((d_j - 1)/2)}  .
\end{aligned}
\]
Then $Y_{\balpha}$ is an homogeneous polynomial of degree $k$. Furthermore, $\{ Y_{\balpha} \}_{\balpha \in \cA_{d,k}}$ is an orthonormal basis of $V_{d,k}$, the space of degree $k$ spherical harmonics. 
\end{proposition}

\paragraph*{Hypercontractivity of spherical harmonics.} The subspaces of low-degree spherical harmonics satisfy the celebrated hypercontractivity property.

\begin{lemma}[Spherical hypercontractivity \cite{beckner1992sobolev}]\label{lem:hypercontractivity_sphere}
For any $k \in \N$ and $f_* \in L^2(\S^{d-1} (\sqrt{d}))$ to be a degree $k$ polynomial, for any $q \ge 2$, we have 
\[
\| f_* \|_{L^q(\S^{d-1})}^2 \le (q - 1)^k \cdot \| f_* \|_{L^2(\S^{d-1})}^2. 
\]
\end{lemma}

\subsubsection{Verifying the assumption on the low-degree features}
\label{app_ex:ass_a_spherical}


The following proposition is obtained by tightening the proof of Proposition 1 in \cite{misiakiewicz2022spectrum}.

\begin{proposition}\label{prop_app:quad_form_quadratic}
    There exist universal constants $c,C >0$ such that for any integers $d >2$ and $\ell \geq 0$, the following holds. Let $\bz \in \R^{B_{d,\leq \ell}}$ be the random vector containing all spherical harmonics up to degree $\ell$, i.e.,
    \[
    \bz = \{ Y_{ks} (\bu) \}_{0\leq k \leq \ell, s\in [B_{d,k}]}, \qquad \bu \sim \Unif( \S^{d-1} (\sqrt{d})).
    \]
    Then, there exists a constant $C_{\ell}$ such that for any matrix $\bA \in \R^{B_{d,\leq \ell} \times B_{d,\leq \ell}}$ and vector $\bv \in \R^{B_{d,\leq \ell}}$, we have
\begin{align}
\P \left( \big\vert \< \bv , \bz \> \big\vert \geq t \cdot C_\ell \| \bv \|_2 \right) \leq&~ C^\ell \exp \left\{ -c \ell \cdot t^{2/\ell } \right\}  , \label{eq:prop_quad_1}\\
\P \left( \big\vert \bz^\sT \bA \bz - \Tr( \bA) \big\vert \geq t \cdot C_\ell \log^\ell (d)  d^{(\ell-1)/2}  \cdot \| \bA  \|_F \right) \leq&~ C^\ell \exp \left\{ -c \ell \cdot t^{1/\ell } \right\}  .\label{eq:prop_quad_2}
\end{align}
\end{proposition}

\begin{proof}[Proof of Proposition \ref{prop_app:quad_form_quadratic}]
We first show that by hypercontractivity of spherical harmonics, the proof reduces to studying the variance of $\bz^\sT \bA \bz$. We control the variance by carefully bounding the expectation of product of spherical harmonics using their representation in terms of product of Gegenbauer polynomials stated in Proposition \ref{prop:spherical_representation}. 

\noindent
{\bf Step 1: Reducing the proof to bounding the variance.}

Define $f ( \bu) := \bz^\sT \bA \bz - \Tr( \bA)$ which is a degree-$2\ell$ polynomial in $\bu$. By hypercontractivity property of the subspace of degree-$2\ell$ polynomials on the sphere (Lemma \ref{lem:hypercontractivity_sphere}), Lemma \ref{lem:tail_bound_hypercontractive} gives the following bound on the tail 
\begin{equation}\label{eq:tail_bound_quad_sphere}
\P_{\bu} \left( | f( \bu) | \geq t \cdot \| f \|_{L^2 (\S^{d-1})} \right) \leq  \exp \left( 2\ell - \frac{\ell}{e} t^{1/\ell} \right).
\end{equation}
Hence, the proof of Eq.~\eqref{eq:prop_quad_2} reduces to proving that 
\[
\| f \|_{L^2}^2 = \Var_\bu ( \bz^\sT \bA \bz) \leq C_\ell \log^{2\ell} (d) d^{\ell - 1} \| \bA \|_F^2.
\]
Taking $\bA = \bv \bv^\sT$, we directly get Eq.~\eqref{eq:prop_quad_1} by observing that
\[
\|f \|_{L^2}^2 \leq \E [ \< \bv , \bz \>^4 ] \leq 3^{2\ell} \E [ \< \bv , \bz \>^2]^2 =3^{2\ell} \| \bv \|_2^4,
\]
where in the second inequality we applied that $\< \bv , \bz \>$ is a degree-$\ell$ polynomial and the hypercontractivity property.

\noindent
{\bf Step 2: Decomposing the variance.}

Denote $\cA_{d, \leq \ell} = \bigcup_{0 \leq k \leq \ell} \cA_{d,k}$ so that $\bz = (Y_{\balpha} (\bu))_{\balpha \in \cA_{d,\leq \ell}}$. We decompose the variance into two contributions
\[
\begin{aligned}
 \Var_\bu ( \bz^\sT \bA \bz) =  \E_{\bu}  \left[ \left( \bz^\sT \bA \bz    -  \Tr (\bA) \right)^2 \right]  \leq 2 \left\{ ({\rm I}) + ({\rm II}) \right\} ,
\end{aligned}
\]
where we denoted
\[
\begin{aligned}
({\rm I}) = &~  \Big\vert  \sum_{\balpha \neq \bbeta , \bgamma \neq \bdelta } \E_{\bu} \big[ Y_{\balpha} ( \bu) Y_{\bbeta} ( \bu) Y_{\bgamma} ( \bu) Y_{\bdelta} ( \bu) \big] A_{\balpha, \bbeta} A_{\bgamma, \bdelta }\Big\vert  , \\
({\rm II}) = &~ \Big\vert  \sum_{\balpha , \bbeta  } \left( \E_{\bu} \big[ Y_{\balpha} ( \bu)^2 Y_{\bbeta} ( \bu)^2  \big] - 1 \right) A_{\balpha, \balpha} A_{\bbeta, \bbeta } \Big\vert  ,
\end{aligned}
\]
where the sum is over indices in $\cA_{d,\leq \ell}$. We will bound both terms separately. 

\noindent
{\bf Step 3: Bounding term $(\rmI)$.}

We introduce $\bC$ the square matrix of size $B_{d,\leq \ell}(B_{d,\leq \ell} - 1)$, indexed by pairs $(\balpha, \bbeta)$ with $\balpha\neq \bbeta \in \cA_{d,\leq \ell}$, which contains for $\balpha \neq \bbeta$ and $\bgamma \neq \bdelta$,
\[
\bC_{(\balpha, \bbeta) , ( \bgamma , \bdelta)} =  \E_{\bu} \big[ Y_{\balpha} ( \bu) Y_{\bbeta} ( \bu) Y_{\bgamma} ( \bu) Y_{\bdelta} ( \bu) \big] .
\]
Denote $\ba \in \R^{B_{d,\leq \ell} (B_{d,\leq \ell} - 1)}$ the vector that contains all off-diagonal entries of $\bA$. Then the first term can be bounded by
\[
({\rm I}) = |  \ba^\sT \bC \ba |  \leq   \| \bC \|_{\op} \| \ba \|_2^2 \leq \| \bC \|_{\op}  \| \bA \|_{F}^2  .
\]
To bound the operator norm of $\bC$, we recall that
\begin{equation*}
\| \bC \|_{\op} \leq \| \bC \|_{1, \infty} =  \max_{\balpha \neq \bbeta } \sum_{\bgamma \neq \bdelta} \Big\vert \E_{\bx} \big[ Y_{\balpha} ( \bx) Y_{\bbeta} ( \bx) Y_{\bgamma} ( \bx) Y_{\bdelta} ( \bx) \big]  \Big\vert .
\end{equation*}

Denote $S_{\balpha} = \{ j \in [d] : \alpha_j > 0 \}$ (similarly $S_{\bbeta}, S_{\bgamma} , S_{\bdelta}$) and $r ( \balpha, \bbeta , \bgamma , \bdelta ) \subseteq [d-1]$ the subset of indices $j \in [d-1]$ such that $j$ belongs to exactly one of the sets $S_{\balpha}, S_{\bbeta}, S_{\bgamma} , S_{\bdelta}$ (e.g., $\alpha_j >0$ and $\beta_j = \gamma_j = \delta_j =0$). In the rest of this step, we fix $\balpha \neq \bbeta$ arbitrary and consider $v ( \bgamma , \bdelta) = r ( \balpha, \bbeta , \bgamma , \bdelta ) \cap (S_{\bgamma} \cup S_{\bdelta})$, i.e., the subset of indices where only $\gamma_j$ or $\delta_j$ are non-zero. For convenience denote $T(\bgamma, \bdelta) := \Big\vert \E_{\bu} \big[ Y_{\balpha} ( \bu) Y_{\bbeta} ( \bu) Y_{\bgamma} ( \bu) Y_{\bdelta} ( \bu) \big]  \Big\vert$. This quantity is bounded in Lemma \ref{lem:technical_termI_bound} stated below. We obtain
\begin{equation*}
    \begin{aligned}
     \sum_{\bgamma \neq \bdelta}  T(\bgamma, \bdelta) =&~ \sum_{ u = 0}^{2\ell}  \sum_{S \subseteq [d-1], |S| =u } \sum_{\bgamma \neq \bdelta, v ( \bgamma , \bdelta) = S} T(\bgamma, \bdelta)  \\
     \leq &~ C_\ell \sum_{ u = 0}^{2\ell} \sum_{S \subseteq [d-1], |S| =u } \sum_{\bgamma \neq \bdelta, v ( \bgamma , \bdelta) = S} \prod_{j \in S} \frac{1}{d-j} \, .
    \end{aligned}
\end{equation*}
For fixed $S$, let us bound the number of $\bgamma \neq \bdelta$ such that $v(\bgamma, \bdelta) = S$. If $u >0$, it means that there are at most $\ell -u$ other coordinates $j \in [d-1]$ where $\gamma_j >0$, and $\ell-u$ coordinates where $\delta_j >0$, and either both $\delta_j, \gamma_j >0$ or $j \in S_{\balpha} \cup S_{\bbeta}$: we deduce that there is at most $d^{\ell - u}$ ways of choosing coordinates for $S_{\bgamma} \Delta S_{\bdelta} \setminus (S_{\balpha} \cup S_{\bbeta})$, and then at most $(4\ell)^{2\ell}$ ways of choosing $\bgamma,\bdelta$ compatible with this support. For $u = 0$, because $\bgamma \neq \bdelta$, we can't have $S_{\bgamma} \cup S_{\bdelta} \setminus (S_{\balpha} \cup S_{\bbeta}) = \ell$, hence again there is at most $O(d^{\ell-1})$ such $\bgamma, \bdelta$. We deduce that
\begin{equation*}\label{eq:termI_bbd}
    \begin{aligned}
     \sum_{\bgamma \neq \bdelta}  T(\bgamma, \bdelta) \leq &~ C_\ell d^{\ell-1} \sum_{ u = 0}^{2\ell} \sum_{S \subseteq [d-1], |S| =u }  \prod_{j \in S} \frac{1}{d-j} \\ 
     \leq &~ C_\ell d^{\ell-1} \sum_{ u = 0}^{2\ell} \left( \sum_{j \in [d-1]} \frac{1}{d-j} \right)^u \\
     \leq &~ C_{\ell} d^{\ell-1} \log (d)^{2 \ell}  .
    \end{aligned}
\end{equation*}
Note that this inequality holds uniformly over $\balpha\neq\bbeta$. Thus, we obtain 
\begin{equation}\label{eq:bound_I_quad_sphere}
(\rmI) \leq \| \bC\|_{\op} \|\bA\|_F^2 \leq C_{\ell} d^{\ell-1} \log (d)^{2 \ell} \| \bA \|_F^2.
\end{equation}

\noindent
{\bf Step 4: Bounding term $(\rmII)$.}

Denote $\Tilde{\cA}_{d,\ell} \subset \cA_{d,\ell}$ the subset of $\balpha$ with $\ell$ indices $\alpha_j = 1$, and $\cA^c_{d,\leq \ell} = \cA_{d,\leq \ell} \setminus \Tilde{\cA}_{d,\ell}$. It is easy to check that $|\Tilde{\cA}_{d,\ell} | = {{d}\choose{\ell}}$ and that there exists a constant $C_\ell$ such that $|\cA^c_{d,\leq \ell}| \leq C_\ell d^{\ell -1}$. We decompose the term $(\rmII)\leq 2 (\rmII.1) + 2(\rmII.2)$ with
\[
\begin{aligned}
    (\rmII.1) =&~ \left| \sum_{\alpha,\beta \in \cA^c_{d,\leq \ell}} \left(\E_\bu \left[Y_{\balpha} (\bu)^2 Y_{\bbeta} (\bu)^2\right] - 1 \right) A_{\balpha,\balpha} A_{\bbeta ,\bbeta} \right| = \left| \Tilde{\ba}_1^\sT \Tilde{\bC}_1 \Tilde{\ba}_1 \right|,\\
    (\rmII.2) =&~ \left| \sum_{\alpha,\beta \in \Tilde{\cA}_{d,\ell}} \left(\E_\bu \left[Y_{\balpha} (\bu)^2 Y_{\bbeta} (\bu)^2\right] - 1 \right) A_{\balpha,\balpha} A_{\bbeta ,\bbeta} \right| = \left| \Tilde{\ba}_2^\sT \Tilde{\bC}_2 \Tilde{\ba}_2 \right|,
\end{aligned}
\]
where we introduced the vectors $\Tilde{\ba}_1 = (A_{\balpha,\balpha})_{\balpha \in \cA^c_{d,\leq \ell}}$ and $\Tilde{\ba}_2 = (A_{\balpha,\balpha})_{\balpha \in \Tilde{\cA}_{d,\ell}}$ and the matrices
\[
\begin{aligned}
    \Tilde{\bC}_1 =&~ \left( \E_\bu \left[Y_{\balpha} (\bu)^2 Y_{\bbeta} (\bu)^2\right] - 1 \right)_{\balpha,\bbeta \in \cA^c_{d,\leq \ell}},\\
    \Tilde{\bC}_2 =&~ \left( \E_\bu \left[Y_{\balpha} (\bu)^2 Y_{\bbeta} (\bu)^2\right] - 1 \right)_{\balpha,\bbeta \in \Tilde{\cA}_{d, \ell}}.\\
\end{aligned}
\]

For the term $(\rmII.1)$, observe that by H\"older's inequality and hypercontractivity of degree-$\ell$ polynomials on the sphere,
\[
\E_\bu \left[Y_{\balpha} (\bu)^2 Y_{\bbeta} (\bu)^2\right] \leq \| Y_{\balpha} (\bu)\|_{L^4}^2 \| Y_{\bbeta} (\bu)\|_{L^4}^2 \leq 9^\ell. 
\]
Thus, we simply bound this term using
\begin{equation}\label{eq:bound_rmII.1_quad_sphere}
    (\rmII.1) = \left| \Tilde{\ba}_1^\sT \Tilde{\bC}_1 \Tilde{\ba}_1 \right| \leq \|\Tilde{\bC}_1 \|_\op \| \Tilde{\ba}_1 \|_2^2 \leq   |\cA^c_{d,\leq \ell}| \| \Tilde{\bC}_1 \|_{\max} \| \bA \|_F^2 \leq C_\ell d^{\ell - 1} \| \bA \|_F^2.
\end{equation}

For the term $(\rmII.2)$, we note that the set of spherical harmonics with index $\balpha \in \Tilde{\cA}_{d,\ell}$ corresponds to the family of polynomials of the form 
\[
Y_S (\bu ) = C_{d,\ell} \prod_{j \in S} u_j, \qquad S \subset [d], \;\;\;|S| = \ell,
\]
where $C_{d,\ell}$ is a normalization constant such that $\E[Y_S(\bu)^2]= 1$. Denote $\cB_{\ell} = \{ S \subset [d] :|S| = \ell\}$ the set of all subsets of $[d]$ of size $\ell$, and use these indices interchangeably with $\Tilde{\cA}_{d,\ell}$. We fix $S_0 \in \cB_{\ell}$ and bound the sum
\[
\begin{aligned}
    \sum_{S \in \cB_\ell} |(\Tilde{\bC_1})_{S_0,S} |=&~ \sum_{S \in \cB_\ell, S \cap S_0 \neq \emptyset} |(\Tilde{\bC_1})_{S_0,S} | + \sum_{S \in \cB_\ell, S \cap S_0 = \emptyset} |(\Tilde{\bC_1})_{S_0,S} |\\
    \leq &~ (9^\ell +1) \cdot | \{ S \in \cB_\ell : S \cap S_0 \neq \emptyset \} |  + \frac{C_\ell}{d} \cdot | \{ S \in \cB_\ell : S \cap S_0 = \emptyset \} | \\
    \leq&~ C_\ell d^{\ell - 1},
\end{aligned}
\]
where we used Lemma \ref{lem:technical_termII_bound} stated below to bound terms with $S \cap S_0 = \emptyset$. This bound is uniform over $S_0$. We deduce that $\| \Tilde{\bC}_2 \|_\op \leq \| \Tilde{\bC}_2 \|_{1,\infty} \leq C_\ell d^{\ell - 1}$ and
\[
(\rmII.2) = \left| \Tilde{\ba}_2^\sT \Tilde{\bC}_2 \Tilde{\ba}_2 \right| \leq \| \Tilde{\bC}_2 \|_\op \| \Tilde{\ba}_2\|_2^2 \leq C_\ell d^{\ell - 1} \|\bA \|_F^2.
\]
Combining this bound with Eq.~\eqref{eq:bound_rmII.1_quad_sphere}, we obtain
\begin{equation}\label{eq:bound_II_quad_sphere}
(\rmII) \leq C_{\ell} d^{\ell-1}  \| \bA \|_F^2.
\end{equation}

\noindent
{\bf Step 5: Concluding the proof.}

We conclude the proof by combining the bounds \eqref{eq:bound_I_quad_sphere} and \eqref{eq:bound_II_quad_sphere} into the tail bound \eqref{eq:tail_bound_quad_sphere}.
\end{proof}

 In fact, we can show a slightly stronger result than the previous proposition by taking a subset of $\evn$ spherical harmonics of degree $\ell+1$. We illustrate this result below.
 
\begin{proposition}\label{prop:quad_bound_fraction_spherical}
There exist universal constants $c,C >0$ such that for any integers $d >2$ and $\ell \geq 0$, the following holds. For any integer $d \leq \evn \leq B_{d,\ell+1}$ there exist a feature $\bz$ that contain $m$ orthogonal spherical harmonics of degree $\ell+1$, and a constant $C_{\ell}$ such that for any matrix $\bA \in \R^{\evn \times \evn}$ and vector $\bv \in \R^{\evn}$, we have
\begin{align}
\P \left( \big\vert \bz^\sT \bA \bz - \Tr( \bA) \big\vert \geq t \cdot C_\ell \log^{\ell+1} (d)  \sqrt{\frac{\evn}{d}}  \cdot \| \bA  \|_F \right) \leq&~ C^{\ell+1} \exp \left\{ -c (\ell+1) \cdot t^{1/(\ell+1) } \right\}  .\label{eq:general_quad_2}
\end{align}
\end{proposition}

The set $\cS$ can in fact be chosen at random by including each spherical harmonics with independent probability $\evn/B_{d,\ell+1}$. 

\begin{proof}[Proof of Proposition \ref{prop:quad_bound_fraction_spherical}] Note that for $\evn > {{d}\choose{\ell+1}}$, we can simply use Proposition \ref{prop_app:quad_form_quadratic} by noting that $B_{d,\ell+1} - {{d}\choose{\ell+1}} \leq C_\ell d^\ell$. 

Denote $\cB_{\ell+1} = \{ S \subset [d]: |S| = \ell+1\}$. For $\evn \leq {{d}\choose{\ell+1}}$, we consider $\cS \subset \cB_{\ell+1}$ as constructed in Lemma \ref{lem:existence_subset_spherical} and set $\bz = (Y_S(\bu))_{s \in \cS}$ where we recall that we denoted
\[
Y_S (\bu) = C_{d,\ell} \prod_{j \in S} u_j.
\]
    We proceed similarly to the proof of Proposition \ref{prop_app:quad_form_quadratic}. In particular, it is sufficient to bound the variance of the quadratic form:
    \[
\begin{aligned}
 \Var_\bu ( \bz^\sT \bA \bz) =  \E_{\bu}  \left[ \left( \bz^\sT \bA \bz    -  \Tr (\bA) \right)^2 \right]  \leq 2 \left\{ ({\rm I}) + ({\rm II}) \right\} ,
\end{aligned}
\]
where we denoted
\[
\begin{aligned}
({\rm I}) = &~  \Big\vert  \sum_{(S_1,S_2,S_3,S_4) \in \cS, S_1 \neq S_2 , S_3 \neq S_4 } \E_{\bu} \big[ Y_{S_1} ( \bu) Y_{S_2} ( \bu) Y_{S_3} ( \bu) Y_{S_4} ( \bu) \big] A_{S_1, S_2} A_{S_3, S_4}\Big\vert  , \\
({\rm II}) = &~ \Big\vert  \sum_{S_1,S_2 \in \cS } \left( \E_{\bu} \big[ Y_{S_1} ( \bu)^2 Y_{S_2} ( \bu)^2  \big] - 1 \right) A_{S_1, S_1} A_{S_2, S_2 } \Big\vert  .
\end{aligned}
\]

\noindent
{\bf Step 1: Bounding term $(\rmI)$.} 

We introduce again $\bC$ the square matrix of size $\evn(\evn-1)$ indexed by pairs $(S_1,S_2) \in \cS$ with $S_1 \neq S_2$, given by
\[
\bC_{(S_1,S_2),(S_3,S_4)} = \E_\bu [ Y_{S_1} (\bu) Y_{S_2} (\bu) Y_{S_3} (\bu) Y_{S_4} (\bu)].
\]
We bound the operator norm of $\bC$ with
\[
\| \bC \|_\op \leq \| \bC \|_{1,\infty} = \max_{(S_1,S_2) \in \cS^2, S_1 \neq S_2} \sum_{(S_3,S_4) \in \cS^2, S_3 \neq S_4} \left| \E_\bu [ Y_{S_1} (\bu) Y_{S_2} (\bu) Y_{S_3} (\bu) Y_{S_4} (\bu)]\right|.
\]
First, observe that we have the following two simple bounds on the coordinates of $\bC$. Denote $\Delta$ the standard symmetric difference operation between two sets. If $S_1 \Delta S_2 \neq S_3 \Delta S_4$ (i.e., there exists an index $i \in S_1 \cup S_2 \cup S_3 \cup S_4$ that appears exactly in one or three of those sets), then by symmetry of the coordinate distribution on the sphere
\[
C_{(S_1,S_2),(S_3,S_4)} = 0.
\]
Furthermore, for any $(S_1,S_2,S_3,S_4)^4$, we have by H\"older's inequality
\[
| C_{(S_1,S_2),(S_3,S_4)} | \leq \E[ | Y_{S_1} (\bu)|^4] \leq 9^{\ell+1}.
\]
Therefore, we have simply
\[
\begin{aligned}
\| \bC \|_\op \leq&~ 9^{\ell+1} \max_{(S_1,S_2) \in \cS^2, S_1 \neq S_2}  \left| \{ (S_3 , S_4 ) \in \cS^2 : S_3 \neq S_4, S_1 \Delta S_2 = S_3 \Delta S_4 \} \right| \\
\leq&~ C_\ell \log(d) \frac{\evn}{d},
\end{aligned}
\]
where we used Eq.~\eqref{eq:constraint_cS_2} in Lemma \ref{lem:existence_subset_spherical}.

 \noindent
{\bf Step 2: Bounding term $(\rmII)$.} 

Again, we introduce $\Tilde{\bC}$ the square matrix of size $\evn$ with entries
\[
\Tilde{C}_{S_1,S_2} =  \E [ Y_{S_1} (\bu)^2 Y_{S_2} (\bu)^2] -1.
\]
Following the same argument as in Step 4 of the proof of Proposition \ref{prop_app:quad_form_quadratic}, we directly have
\[
\begin{aligned}
    \| \Tilde{C} \|_\op \leq &~ \max_{S_1 \in \cS} \sum_{S_2 \in \cS} \left| \E [ Y_{S_1} (\bu)^2 Y_{S_2} (\bu)^2] -1 \right| \\
    \leq&~ \max_{S_1 \in \cS} \left\{ (9^{\ell+1} +1) \left| \{ S_2 \in \cS: S_1 \cap S_2 \neq \emptyset \} \right| + \frac{C_{\ell+1}}{d} \left| \{ S_2 \in \cS: S_1 \cap S_2 = \emptyset \} \right| \right\} \\
    \leq&~ (9^{\ell+1} +1) \max_{S_1 \in \cS} \left\{  \left| \{ S_2 \in \cS: S_1 \cap S_2 \neq \emptyset \} \right| \right\} + C_{\ell+1} \frac{ \evn }{d} \\
    \leq&~ C_\ell \frac{ \evn }{d}
\end{aligned}
\]
where we used Eq.~\eqref{eq:constraint_cS_1} in Lemma \ref{lem:existence_subset_spherical}.
\end{proof}

\subsubsection{Technical lemmas}
\label{app_ex:technical_sphere}

The following two lemmas from \cite{misiakiewicz2022spectrum} and \cite{xiao2022precise} bound the expectation of products of spherical harmonics which appear in the proof of Proposition \ref{prop_app:quad_form_quadratic}.

\begin{lemma}[{\cite[Lemma 2]{misiakiewicz2022spectrum}}]\label{lem:technical_termI_bound}
For $ \balpha, \bbeta , \bgamma , \bdelta \in \cA_{d,\ell}$, denote $r (\balpha, \bbeta , \bgamma , \bdelta ) \subseteq [d-1]$ the subset of indices $j \in [d-1]$ such that only one of the $\alpha_j, \beta_j, \gamma_j , \delta_j$ is non-zero. There exists a constant $C_{\ell} >0$ such that for any $ \balpha, \bbeta , \bgamma , \bdelta \in \cA_{d,\leq \ell}$, 
\begin{equation}\label{eq:bound_prod_Y}
    \Big\vert \E_{\bu} \big[ Y_{\balpha} ( \bu) Y_{\bbeta} ( \bu) Y_{\bgamma} ( \bu) Y_{\bdelta} ( \bu) \big]  \Big\vert \leq C_{\ell} \prod_{j \in  r (\balpha, \bbeta , \bgamma , \bdelta )} \frac{1}{d-j } .
\end{equation}
\end{lemma}

\begin{lemma}[{\cite[Lemma 2]{xiao2022precise}}]\label{lem:technical_termII_bound}
    Consider $S_1,S_2 \subset [d]$ such that $|S_1| = |S_2| = \ell$ and define for $r \in \{1,2\}$
    \[
    Y_{S_r} (\bu) = C_{d,\ell} \prod_{j \in S_r} u_j,
    \]
    where $C_{d,\ell}$ is a normalization constant such that $\E[Y_{S_r}(\bu)^2] = 1$. There exists a constant $C_{\ell}$ such that if $S_1 \cap S_2 = \emptyset$,
    \begin{equation}\label{eq:bound_diagonal_quad_form_tech}
    \left| \E\left[ Y_{S_1} (\bu)^2 Y_{S_2} (\bu)^2 \right] - 1 \right| \leq \frac{C_\ell}{d}.
    \end{equation}
\end{lemma}

The following bound is a classical tail bound on hypercontractive functions. We include a proof for completeness.

\begin{lemma}\label{lem:tail_bound_hypercontractive} Let $(\cU , \mu)$ be a probability space and consider a function $f : \cX \to \R$. Assume that there exists $\alpha >0$ such that for any $q \geq 2$, we have
\begin{equation}\label{eq:ass_hypercontractivity}
\| f \|_{L^q (\cX , \mu)}^2 \leq (C_\alpha q)^\alpha \| f \|_{L^2 (\cX , \mu)}^2.
\end{equation}
Then for all $t \geq 0$,
\[
\P_{\bu \sim \mu} \left( | f( \bu) | \geq t \cdot \| f \|_{L^2 (\cX , \mu)} \right) \leq  \exp \left( \alpha - \frac{\alpha}{2e} t^{2/\alpha} \right).
\] 
\end{lemma}

\begin{proof}[Proof of Lemma \ref{lem:tail_bound_hypercontractive}]
    By Markov's inequality and using Eq.~\eqref{eq:ass_hypercontractivity},we have for any $q \geq 2$,
    \[
    \P \left( | f( \bu) | \geq t \cdot \| f \|_{L^2 (\cX , \mu)} \right) \leq \frac{\| f \|_{L^q}^q}{t^q \| f \|_{L^2}^q} \leq  \left( (C_\alpha q)^{\alpha / 2} / t\right)^q.
    \]
    For $t \geq ( C_\alpha 2e)^{\alpha/2}$, we set $q = t^{2/\alpha}/(C_\alpha e)$ so that the bounds is given by
    \[
    \P \left( | f( \bu) | \geq t \cdot \| f \|_{L^2 (\cX , \mu)} \right) \leq \exp\left( - \frac{\alpha}{2eC_\alpha} t^{2/\alpha} \right).
    \]
    For $t < (C_\alpha 2e)^{\alpha/2}$, we simply use $1 \leq \exp (\alpha - \alpha t^{2/\alpha} / (2e C_\alpha) )$.
\end{proof}

\begin{lemma}\label{lem:existence_subset_spherical}
Consider positive integers $\ell \geq 1$ and $d >2$. Then there exists a constant $C_{\ell}>0$ such that the following hold. Denote $\cB_{\ell+1} := \{ S \subset [d] : |S| = \ell+1\}$ the set of all subsets of $[d]$ of size $\ell+1$, and $D_{\ell+1} := | \cB_{\ell+1} | = {{d}\choose{\ell+1}}$. Then for any integer $p$ such that $d  \leq p \leq D_{\ell+1}$, there exists a subset $\cS \subseteq \cB_{\ell+1}$ such that $|\cS| = p$ and 
    \begin{align}
        \max_{S_1 \in \cS}   \left| \{ S_2 \in \cS: S_1 \cap S_2 \neq \emptyset \} \right|  \leq&~ C_\ell \log (d) \frac{p}{d}, \label{eq:constraint_cS_1} \\
        \max_{(S_1,S_2) \in \cS^2, S_1 \neq S_2}   \left| \{ (S_3 , S_4 ) \in \cS^2 : S_3 \neq S_4, S_1 \Delta S_2 = S_3 \Delta S_4 \} \right|  \leq&~ C_\ell \log (d) \frac{p}{d}. \label{eq:constraint_cS_2}
    \end{align}
\end{lemma}

\begin{proof}[Proof of Lemma \ref{lem:existence_subset_spherical}]
    We will show the existence of $\cS$ using a probabilistic argument.  We consider $\cS$ a random subset of $\cB_{\ell +1}$ constructed by including each $S \in \cB_{\ell+1}$ with probability $p/D_{\ell+1}$. First, observe that $|\cS|$ is a binomial random variable with parameters $D_{\ell+1}$ and $p/D_{\ell+1}$. In particular,
    \[
    \P ( | \cS| \geq p ) \geq 1/2.
    \]
    Thus, it is sufficient to prove that $\cS$ satisfy Eqs.~\eqref{eq:constraint_cS_1} and \eqref{eq:constraint_cS_2} with probability at least $3/4$, so that we guarantee the existence of $|\cS| \geq p$ verifying these constraints. Then we can choose a subset $\Tilde{\cS} \subseteq \cS$ with $|\Tilde{\cS} | = p$ to conclude the proof of the lemma.

    Let us fix $S_1 \in \cB_{\ell+1}$ and define 
    \[
    \Tilde{\cB}_{\ell+1} (S_1) = \{ S \in \cB_{\ell +1} : S \cap S_1 \neq \emptyset \},
    \]
    and $\Tilde{D}_{\ell +1} = |  \Tilde{\cB}_{\ell+1} (S_1)|$. Then $| \cS \cap \Tilde{\cB}_{\ell+1} (S_1) |$ is a binomial random variable with parameters $\Tilde{D}_{\ell +1}$ and $p/D_{\ell+1}$. In particular, we have the following well-known Chernoff bound
    \[
    \P ( | \cS \cap \Tilde{\cB}_{\ell+1} (S_1) | \geq k ) \leq \exp \left\{ - \Tilde{D}_{\ell +1} \cdot \sKL\left(\frac{k}{\Tilde{D}_{\ell +1}} \Big\Vert \frac{p}{D_{\ell+1}}   \right)\right\} 
    \]
    Taking $k = \gamma p \Tilde{D}_{\ell +1} / D_{\ell+1}$, we obtain
    \[
    \P \left( | \cS \cap \Tilde{\cB}_{\ell+1} (S_1) | \geq \gamma p \frac{\Tilde{D}_{\ell +1}}{ D_{\ell+1} }\right) \leq \exp \left\{ - p \frac{\Tilde{D}_{\ell +1}}{D_{\ell+1}} (\gamma (\log(\gamma) - 1) + 1)\right\} .
    \]
    Note that there exists constants $0<c_\ell <C_\ell$ such that $c_\ell \leq d\Tilde{D}_{\ell +1}/ D_{\ell+1} \leq C_\ell$ and $D_{\ell+1} \leq C_\ell d^{\ell+1}$. We deduce the first bound
    \[
    \begin{aligned}
        &~ \P\left( \max_{S_1 \in \cS}   \left| \{ S_2 \in \cS: S_1 \cap S_2 \neq \emptyset \} \right|  \geq \gamma \log (d) C_\ell  \frac{p}{d}\right) \\
        \leq&~  \P\left( \max_{S_1 \in \cB_{\ell+1}}   \left| \{ S_2 \in \cS: S_1 \cap S_2 \neq \emptyset \} \right| \geq \gamma\log (d) C_\ell  \frac{p}{d}\right) \\
        \leq&~ | D_{\ell +1} | \P\left(  \left| \{ S_2 \in \cS: S_1 \cap S_2 \neq \emptyset \} \right| \geq \gamma \log (d) C_\ell \frac{p}{d}\right) \\
        \leq&~ C_\ell \exp\left\{ (\ell+1) \log(d) - \gamma c_\ell \frac{p}{d}\log (d) \right\} \\
        \leq &~ 1/8,
    \end{aligned}
    \]
    where in the last inequality we chose $\gamma \geq C$ for a constant $C$ sufficiently large but that only depends on $\ell$. 

    For the second constraint, observe that $|S_1 \Delta S_2 | = 2k$ for $k =0, \ldots, \ell+1$ and in order to have $S_1 \Delta S_2 = S_3 \Delta S_4$, the set $S_3$ needs to contains exactly $k$ indices in $S_1 \Delta S_2$ and $S_4$ must be exactly the set $S_3 \Delta ( S_1 \Delta S_2)$. Denote 
    \[
    \cB_{\leq 2 (\ell+1)} := \left\{ S \subset [d]: S \neq \emptyset, |S| \leq 2(\ell+1) \right\}.
    \]
    From the above discussion, we deduce that the upper bound
    \[
    \begin{aligned}
     &~\max_{(S_1,S_2) \in \cS^2, S_1 \neq S_2}   \left| \{ (S_3 , S_4 ) \in \cS^2 : S_3 \neq S_4, S_1 \Delta S_2 = S_3 \Delta S_4 \} \right|  \\
     \leq&~ {{2(\ell+1)}\choose{\ell+1}}\max_{S_1 \in \cS}  \left| \{ S_2 \in \cB_{\leq 2 (\ell+1)}: S_1 \cap S_2 \neq \emptyset \} \right| .
     \end{aligned}
    \]
    Noting that there exists $C_\ell$ such that $|\cB_{\leq 2 (\ell+1)}| \leq C_\ell d^{2(\ell+1)}$, the same argument as above yields
    \[
    \P\left( \max_{(S_1,S_2) \in \cS^2, S_1 \neq S_2}   \left| \{ (S_3 , S_4 ) \in \cS^2 : S_3 \neq S_4, S_1 \Delta S_2 = S_3 \Delta S_4 \} \right|  \geq  C_\ell  \log (d) \frac{p}{d}\right) \leq \frac{1}{8}.
    \]
    We conclude via an union bound.
\end{proof}

\subsubsection{Verifying the assumption on the high-degree features}
\label{app_ex:ass_b_spherical}

\begin{proposition}\label{prop:concentration_high_freq_sphere}
For any integer $\ell$ and constant $D>0$, there exists a constant $C_{\ell,D}$ such that the following hold. For integers $d,n$ such that $d \geq C_{\ell,D}$ and $ n \leq  B_{d,\ell+1}/2$, let $(\bu_i)_{i \in[ n]} \sim_{iid} \Unif(\S^{d-1}(\sqrt{d}))$. Consider an inner-product kernel $h:[-1,1] \to \R$ orthogonal to all polynomials of degree at most $\ell$, and denote $\bH = (h(\<\bu_i , \bu_j\>/d))_{ij \in [n]}$. Then we have with probability at least $1-n^{-D}$ that
\[
\left\| \bH - h(1)\cdot \id \right\|_\op \leq C_{\ell,D} \cdot \log(d)^{2\ell + 3} \sqrt{\frac{n}{d^{\ell+1}}} h(1).
\]
\end{proposition}

\begin{proof}[Proof of Proposition \ref{prop:concentration_high_freq_sphere}] Recall that we can decompose the kernel function in the orthonormal basis of Gegenbauer polynomials
\[
h (u) = \sum_{k = \ell+1}^\infty \oxi_k \sqrt{B_{d,k}} \cdot Q^{(d)}_k (u), \qquad \oxi_k := \frac{1}{\sqrt{B_{d,k}}}\E_{\bu} \left[h(\<\bu,\be\>/\sqrt{d}) Q^{(d)}_k ( \<\bu,\be\>/\sqrt{d})\right].
\]
Thus, we decompose the kernel matrix as
\[
\bH = \sum_{k = \ell+1}^\infty \oxi_k B_{d,k} \cdot \bQ_k , \qquad \bQ_k := \frac{1}{\sqrt{B_{d,k}}} \left(Q_k^{(d)} (\<\bu_i , \bu_j\>/d) \right)_{ij\in [n]},
\]
and our bound
\begin{equation}\label{eq:bound_high_freq_sup}
\begin{aligned}
\left\| \bH - h(1) \cdot \id\right\|_\op =\left\| \sum_{k = \ell+1}^\infty \oxi_k B_{d,k} \left(\bQ_k -  \id\right) \right\|_\op 
\leq &~ \sum_{k = \ell+1}^\infty \oxi_k B_{d,k} \left\| \bQ_k  - \id \right\|_\op\\
\leq&~ h(1) \cdot \sup_{k\geq \ell+1} \left\| \bQ_k  - \id \right\|_\op .
\end{aligned}
\end{equation}

First, we tighten the proof of \cite[Proposition 8]{lu2022equivalence}. Observe that we can apply the results from Lemmas \ref{lem:bound_product_high_freq} and \ref{lem:bound_moment_high_freq}: indeed $B_{d,k}^{-1/2} Q_k^{(d)} ( \< \bu_i , \bu_i \> /d) = B_{d,k}^{-1/2} Q_k^{(d)} (1) = 1$, and $Q_k^{(d)} ( \< \bu_i , \bu_j \> /d)$ has the same marginal distribution as $Q_k^{(d)} (\<\be, \bu_j\> / \sqrt{d})$ which is a degree-$k$ polynomial in $\bu$ and therefore by hypercontractivity, we have
\[
B_{d,k}^{-q/2} \E [ |  Q_k^{(d)} (\<\bu_i,\bu_j \>/d) |^q ] \leq B_{d,k}^{-q/2} (q- 1)^{qk/2} \| Q_k^{(d)}\|_{L^2}^q = B_{d,k}^{-q/2} (q- 1)^{qk/2}.
\]
Thus, from Lemma \ref{lem:bound_moment_high_freq} with $\xi_1 = B_{d,k}^{-1}$ and $\gamma = 1$, and recalling that we assumed that $n/B_{d,k} \leq 1/2$ for all $k \geq \ell+1$,  we have for any integer $p\geq 1$,
\[
\E \left[ \Tr( (\bQ_k - \id_n)^{2p}) \right] \leq n \left( \frac{n}{B_{d,k}}\right)^p (C p )^{(4k+2)p}.
\]
Therefore, by Markov's inequality,
\[
\P \left( \| \bQ_k - \id_n\|_\op \geq t \cdot \sqrt{\frac{n}{B_{d,\ell+1}}}\right) \leq t^{-2p} n \left( \frac{B_{d,\ell+1}}{B_{d,k}}\right)^p (C p )^{(4k+2)p}.
\]
Setting $p = \log (n)$ and $t  = C_D ( C \log(n)^{2 (\ell+1) +1}$, we obtain
\[
\P \left( \| \bQ_k - \id_n\|_\op \geq C_D ( C \log(n)^{2 (\ell+1) +1} \sqrt{\frac{n}{B_{d,\ell+1}}}\right) \leq \frac{n}{n^{2\log(C_D)}}\cdot \left( \frac{B_{d,\ell+1} (C\log(n))^{4(k - \ell - 1)}}{B_{d,k}}\right)^{\log(n)}.
\]
There exists $C_k>0$ such that $B_{d,\ell+1} / B_{d,k} \leq C_k d^{-(k-\ell-1)}$ and therefore taking $d \geq C'_k$, we get $C_k (C (\ell+1)\log(d)/d)^{4(k - \ell-1)} \leq 1$. Taking $C_D$ sufficiently large yields
\begin{equation}\label{eq:low_k_high_freq_sphere}
\P \left( \| \bQ_k - \id_n\|_\op \geq C_D ( C \log(n)^{2 (\ell+1) +1} \sqrt{\frac{n}{B_{d,\ell+1}}}\right) \leq n^{-D}.
\end{equation}
On the other hand, from \cite[Equation (55)]{ghorbani2021linearized}, we can choose an integer $L := L(D,\ell) \geq \ell +1$ such that for any $d \geq C_{D,\ell}$, we have
\[
\E \left[ \sup_{k\geq L} \left\| \bQ_k  -  \id \right\|_\op \right] \leq d^{-(\ell+1) D - (\ell +1)/2},
\]
and by Markov's inequality
\begin{equation}\label{eq:high_k_high_freq_sphere}
\P \left( \sup_{k\geq L} \left\| \bQ_k  -  \id \right\|_\op  \geq d^{-(\ell+1)/2}\right) \leq n^{-D} .
\end{equation}
Hence, via an union bound over the events \eqref{eq:low_k_high_freq_sphere} for $\ell+1 \leq k < L$ and the event \eqref{eq:high_k_high_freq_sphere} and reparametrizing $D$, we deduce that there exists a constant $C_{\ell,D}$ such that 
\[
\P \left( \sup_{k\geq \ell+1} \left\| \bQ_k -  \id \right\|_\op  \geq C_{\ell,D} \cdot  \log(d)^{2\ell +3} \sqrt{\frac{n}{d^{\ell+1}}}\right) \leq n^{-D}.
\]
Injecting this bound in Eq.~\eqref{eq:bound_high_freq_sup} concludes the proof of this proposition.
\end{proof}

\end{document}